\documentclass{article} 
\usepackage{arxiv}
\usepackage[numbers,sort&compress]{natbib}

\usepackage{microtype}
\usepackage{amsmath,amssymb,amsthm,mathtools,bm}
\usepackage{booktabs}
\usepackage{graphicx}
\usepackage{xcolor}
\usepackage{placeins}
\usepackage{tikz}
\usetikzlibrary{arrows.meta,positioning,fit,calc}
\usepackage{hyperref}
\usepackage{url}

\hypersetup{colorlinks=true,linkcolor=black,citecolor=black,urlcolor=blue,
  pdftitle={Classifier-Free Guidance in Flow Matching: Non-Autonomous Potentials, Overshoot, and Posterior-Mean Control},
  pdfauthor={Jishen Peng and Zheng Ma}}
\newtheorem{theorem}{Theorem}[section]
\newtheorem{proposition}[theorem]{Proposition}
\newtheorem{lemma}[theorem]{Lemma}
\newtheorem{corollary}[theorem]{Corollary}
\theoremstyle{definition}

\theoremstyle{remark}

\DeclareMathOperator{\Cov}{Cov}
\DeclareMathOperator{\Var}{Var}
\DeclareMathOperator{\tr}{tr}
\DeclareMathOperator{\softmax}{softmax}
\newcommand{\E}{\mathbb E}
\newcommand{\R}{\mathbb R}
\newcommand{\dd}{\mathrm d}
\newcommand{\norm}[1]{\left\lVert#1\right\rVert}
\newcommand{\inner}[2]{\left\langle#1,#2\right\rangle}
\newcommand{\vu}{v_t^{\mathrm u}}
\newcommand{\vc}{v_t^{\mathrm c}}
\newcommand{\pu}{p_t^{\mathrm u}}
\newcommand{\pc}{p_t^{\mathrm c}}

\title{Classifier-Free Guidance in Flow Matching:\\
Non-Autonomous Potentials, Overshoot, and Posterior-Mean Control}

\author{
  \normalfont
  \textbf{Jishen Peng}$^1$ \qquad
  \textbf{Zheng Ma}$^{1,2,3,4}$\thanks{Corresponding author: \texttt{zhengma@sjtu.edu.cn}}\\[0.75em]
  \small $^1$ School of Mathematical Sciences, Shanghai Jiao Tong University, Shanghai, China\\
  \small $^2$ Institute of Natural Sciences, MOE-LSC, Shanghai Jiao Tong University, Shanghai, China\\
  \small $^3$ Qing Yuan Research Institute, Shanghai Jiao Tong University, Shanghai, China\\
  \small $^4$ CMA-Shanghai, Shanghai Jiao Tong University, Shanghai, China\\[0.5em]
  \small Code: \url{https://github.com/ppposition/pmc.git}
}

\date{}
\renewcommand{\headeright}{arXiv Preprint}
\renewcommand{\undertitle}{arXiv Preprint}
\renewcommand{\shorttitle}{Classifier-Free Guidance in Flow Matching}

\begin{document}

\maketitle

\begin{abstract}
Classifier-free guidance (CFG) improves conditional generation in Flow
Matching, but strong guidance can distort the generated distribution and reduce
diversity. We provide a geometric account of this behavior by viewing Flow
Matching as a time-varying gradient flow and characterizing how CFG reshapes its
underlying potential. This view explains how stronger alignment can be
accompanied by mean displacement and trajectory concentration, and motivates
controlling guidance through the model-implied terminal posterior mean. We
therefore propose Posterior-Mean-Capped CFG (PMC-CFG), a training-free,
per-sample method that adaptively retains the strongest feasible guidance
without additional network evaluations. Experiments on synthetic and
large-scale image-generation benchmarks show that PMC-CFG limits guidance-induced
distortion and concentration while improving the alignment--diversity trade-off,
with particularly strong benefits when nominal guidance is large.
\end{abstract}

\section{Introduction}
\label{sec:intro}

Diffusion and flow-based generative models learn to transform simple noise into
complex data distributions \cite{ho2020ddpm,lipman2023flow,albergo2022building,liu2023flow}.
For conditional generation, classifier guidance steers this transformation with
the gradient of an auxiliary classifier \cite{dhariwal2021diffusion}, whereas
classifier-free guidance (CFG) jointly learns conditional and unconditional
 predictions and combines them during sampling \cite{ho2022classifierfree}. CFG
 is simple and effective: its conditional--unconditional difference acts as an
 implicit classifier signal, and guided flows have shown substantial gains in
 conditional generation with Flow Matching \cite{ho2022classifierfree,zheng2023guidedflows}.
 Yet increasing its scale improves condition alignment at the cost of diversity
 and can cause oversaturation, mean overshoot, and distributional shrinkage
 \cite{ho2022classifierfree,karras2024badversion,pavasovic2026overshoot}.

Several inference-time methods mitigate excessive guidance by changing its
time schedule \cite{wang2024schedulers,gao2026c2fg}, using analytic
path-integral characterizations to design distribution-guided schedules
\cite{jiang2026analytic}, restricting guidance to selected noise intervals
\cite{kynkaanniemi2024interval}, constraining updates toward the model manifold
\cite{chung2025cfgpp}, rectifying the combination coefficients
\cite{xia2025rectified}, projecting and rescaling the guidance update
\cite{sadat2025apg}, or replacing linear extrapolation with nonlinear
reweighting \cite{pavasovic2026overshoot}. Recent control-theoretic
work instead treats the conditional--unconditional discrepancy as a feedback
signal in continuous-time generative flows \cite{wang2026cfgctrl}. These methods improve guided sampling,
but they do not directly control the state-dependent displacement implied by a
Flow-Matching velocity. Moreover, guided dynamics cannot generally be
identified with sampling from a single power-tilted terminal distribution
 \cite{bradley2024classifier}. This leaves a practical question: which quantity
 should be constrained at each ODE evaluation to prevent excessive guidance
 without discarding useful conditional information?

Two recent perspectives sharpen this question. Adaptive diffusion guidance
casts the scale as a control that may depend on time, state, and condition
\cite{azangulov2026adaptive}, while a flow-specific approach interprets CFG as
homotopy optimization and projects the iterate toward a manifold on which the
conditional and unconditional predictions agree \cite{cai2026manifold}. Our
approach instead preserves their discrepancy as the useful guidance direction
and controls how far that direction may extrapolate the model-implied terminal
mean. This distinction yields a closed-form, per-sample rule rather than a
purely time-dependent schedule or an iterative state projection.

We answer this question in posterior-mean space. For the linear Flow-Matching
path, every velocity evaluation implies a posterior estimate of the terminal
sample, and fixed CFG extrapolates this estimate along the
conditional--unconditional direction. This view motivates
Posterior-Mean-Capped CFG (PMC-CFG), which intersects the nominal CFG ray with a
relative norm ball and selects the farthest feasible point. As illustrated in
Figure~\ref{fig:overview}, PMC-CFG limits excessive extrapolation while retaining
nominal guidance whenever it is already feasible. The resulting controller has
a closed-form per-sample scale, requires no training, and adds no network
evaluations.

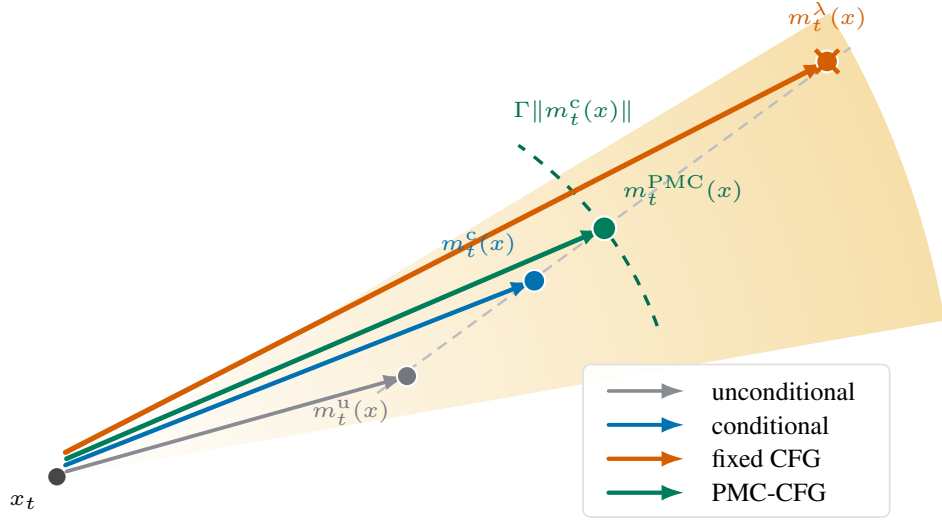
\begin{figure*}[t]
\centering
\resizebox{0.78\textwidth}{!}{%
\begin{tikzpicture}[
  x=1cm,y=1cm,
  arr/.style={-{Latex[length=2.2mm,width=1.5mm]},line cap=round},
  tinyLab/.style={font=\tiny,align=center},
  legendLab/.style={font=\scriptsize,anchor=west}
]
\definecolor{condblue}{HTML}{0072B2}
\definecolor{cfgorange}{HTML}{D55E00}
\definecolor{pmcgreen}{HTML}{009E73}
\definecolor{fanamber}{HTML}{E69F00}
\definecolor{softgray}{HTML}{73777F}
\path[use as bounding box] (0,0) rectangle (9.10,5.00);
\coordinate (O) at (3.00,0.70); 
\coordinate (XT) at (0.70,0.50);
\coordinate (MUNCOND) at (4.00,1.45);
\coordinate (MCOND) at (5.20,2.35);
\coordinate (MPMC) at (5.86,2.845);
\coordinate (MCFG) at (7.96,4.42);

\begin{scope}
  \clip (0,0) rectangle (9.10,5.00);
  \clip (XT) -- ($(XT)+(10:8.50)$)
    arc[start angle=10,end angle=31,radius=8.50] -- cycle;
  \shade[inner color=white,outer color=fanamber!34] (XT) circle (8.50);
\end{scope}

\draw[pmcgreen!70!black,dashed,line width=0.90pt]
  ($(O)+(20:3.575)$) arc[start angle=20,end angle=55,radius=3.575];
\node[tinyLab,pmcgreen!70!black,anchor=south west] at (4.88,3.70)
  {$\Gamma\lVert m_t^{\mathrm c}(x)\rVert$};

\draw[softgray!48,densely dashed,line width=0.65pt]
  ($(MUNCOND)+(-0.28,-0.16)$) -- ($(MCFG)+(0.22,0.13)$);
\filldraw[fill=black!72,draw=white,line width=0.55pt] (XT) circle (2.6pt);
\node[tinyLab,anchor=north east] at ($(XT)+(-0.06,-0.05)$) {$x_t$};

\draw[arr,softgray!85,line width=1.00pt] ($(XT)+(0.09,0.05)$) -- (MUNCOND);
\draw[arr,condblue,line width=1.20pt] ($(XT)+(0.08,0.11)$) -- (MCOND);
\draw[arr,pmcgreen!80!black,line width=1.30pt] ($(XT)+(0.09,0.17)$) -- (MPMC);
\draw[arr,cfgorange,line width=1.40pt] ($(XT)+(0.08,0.23)$) -- (MCFG);

\filldraw[fill=softgray,draw=white,line width=0.55pt] (MUNCOND) circle (2.7pt);
\node[tinyLab,softgray,anchor=north east] at ($(MUNCOND)+(-0.04,-0.10)$)
  {$m_t^{\mathrm u}(x)$};
\filldraw[fill=condblue,draw=white,line width=0.55pt] (MCOND) circle (2.9pt);
\node[tinyLab,condblue,anchor=south east] at ($(MCOND)+(-0.05,0.10)$)
  {$m_t^{\mathrm c}(x)$};
\filldraw[fill=pmcgreen!80!black,draw=white,line width=0.60pt] (MPMC) circle (3.2pt);
\node[tinyLab,pmcgreen!70!black,anchor=south west] at ($(MPMC)+(0.05,0.10)$)
  {$m_t^{\mathrm{PMC}}(x)$};
\filldraw[fill=cfgorange,draw=white,line width=0.55pt] (MCFG) circle (2.9pt);
\draw[cfgorange,very thick] ($(MCFG)+(-0.11,-0.11)$) -- ($(MCFG)+(0.11,0.11)$);
\draw[cfgorange,very thick] ($(MCFG)+(-0.11,0.11)$) -- ($(MCFG)+(0.11,-0.11)$);
\node[tinyLab,cfgorange,anchor=south] at ($(MCFG)+(0,0.16)$)
  {$m_t^{\lambda}(x)$};

\filldraw[fill=white,fill opacity=0.92,draw=black!12,line width=0.45pt,
  rounded corners=2pt] (5.66,0.12) rectangle (8.55,1.56);
\draw[arr,softgray!85,line width=0.95pt] (5.90,1.31) -- (6.62,1.31);
\node[legendLab] at (6.75,1.31) {unconditional};
\draw[arr,condblue,line width=1.05pt] (5.90,0.99) -- (6.62,0.99);
\node[legendLab] at (6.75,0.99) {conditional};
\draw[arr,cfgorange,line width=1.15pt] (5.90,0.67) -- (6.62,0.67);
\node[legendLab] at (6.75,0.67) {fixed CFG};
\draw[arr,pmcgreen!80!black,line width=1.15pt] (5.90,0.35) -- (6.62,0.35);
\node[legendLab] at (6.75,0.35) {PMC-CFG};
\end{tikzpicture}
}
\caption{From posterior guidance to posterior-mean control. The amber fan becomes
more saturated toward the upper-right far field, indicating that
$p(c\mid x_t)$ can remain high far from the data region. The straight arrows
from $x_t$ show the unconditional and conditional posterior means together with
the nominal fixed-CFG and controlled PMC-CFG means. Fixed CFG extrapolates along
the ray from $m_t^{\mathrm u}(x)$ through $m_t^{\mathrm c}(x)$ and can reach the
remote mean $m_t^\lambda(x)$, producing overshoot. The short dashed green arc is
the locally relevant portion of the relative norm boundary
$\lVert m\rVert=\Gamma\lVert m_t^{\mathrm c}(x)\rVert$; PMC-CFG selects its
intersection with the same extrapolation ray, $m_t^{\mathrm{PMC}}(x)$, thereby
suppressing excessive extrapolation. The geometry is schematic.}
\label{fig:overview}
\end{figure*}

Our contributions are:
\begin{itemize}
 \item We characterize linear Flow Matching and CFG as non-autonomous gradient
 flows. We show that the same posterior force that provably increases the local
 conditional-alignment rate can also accumulate mean displacement and contract
 trajectory volume, thereby separating why CFG works from why fixed strong CFG
 fails.
 \item We introduce PMC-CFG, a training-free controller that caps the
 model-implied terminal posterior mean while preserving the CFG direction. We
 derive its closed-form scale and prove pointwise feasibility, maximal retention
 of nominal guidance, and non-interference whenever nominal CFG is feasible.
 \item We evaluate the predicted mechanisms and the resulting
 alignment--diversity trade-off on an analytic eight-point GMM distribution,
 ImageNet-256 class-conditional generation, and SD3.5 text-to-image generation.
\end{itemize}

\section{Preliminaries}
\label{sec:preliminaries}

\subsection{Linear Flow Matching}

Let the source sample be $X_0\sim p_0=\mathcal N(0,I_d)$ and the target sample
be $X_1\sim p_1$, independently. We consider the linear probability path
\begin{equation}
 X_t=(1-t)X_0+tX_1,\qquad t\in[0,1],
 \label{eq:path-main}
\end{equation}
which interpolates from noise at $t=0$ to data at $t=1$. Flow Matching learns
the marginal velocity
$v_t(x)=\E[\dot X_t\mid X_t=x]$ and generates samples by solving
$\dot x_t=v_t(x_t)$. For the linear path, define the model-implied terminal
posterior mean $m_t(x)=\E[X_1\mid X_t=x]$. Then, for $t\in(0,1)$,
\begin{equation}
 v_t(x)=\frac{m_t(x)-x}{1-t}
 =\frac{x}{t}+\frac{1-t}{t}\nabla_x\log p_t(x).
 \label{eq:velocity-main}
\end{equation}
The second equality is the linear-path Tweedie identity
\cite{efron2011tweedie}; Appendix~\ref{app:linear-identities} gives a
self-contained derivation.

\subsection{Classifier-Free Guidance}

For a condition $c$, superscripts $\mathrm u$ and $\mathrm c$ denote
unconditional and conditional quantities, respectively. Classifier-free
guidance (CFG) combines the two learned velocities as
\begin{equation}
 v_t^\gamma=\vu+\gamma(\vc-\vu),
 \label{eq:cfg-main}
\end{equation}
where $\gamma=0$ is unconditional, $\gamma=1$ is purely conditional, and
$\gamma>1$ extrapolates beyond the conditional prediction. Equivalently, the
guided velocity implies the terminal estimate
\begin{equation}
 m_t^\gamma(x)=x+(1-t)v_t^\gamma(x)
 =m_t^{\mathrm c}(x)+(\gamma-1)
   \bigl(m_t^{\mathrm c}(x)-m_t^{\mathrm u}(x)\bigr).
 \label{eq:mean-main}
\end{equation}
Thus standard CFG is linear extrapolation in posterior-mean space. This
identity supplies the geometric quantity that PMC-CFG will constrain in
Section~\ref{sec:method}.

\section{Flow Matching as a Non-Autonomous Gradient Flow}
\label{sec:analysis}

The preliminaries specify the standard Flow-Matching and CFG construction. We
now analyze the resulting dynamics: first as a time-varying potential system,
then through its displacement and local volume change.

\subsection{A time-varying potential}

\begin{proposition}[Non-autonomous gradient representation]
\label{prop:potential-main}
For $t\in(0,1)$ and a positive differentiable intermediate density $p_t$,
\begin{equation}
 v_t(x)=\nabla_x\Phi_t(x),\qquad
 \Phi_t(x)=\frac{\norm{x}^2}{2t}
 +\frac{1-t}{t}\log p_t(x).
 \label{eq:potential-main}
\end{equation}
Equivalently, with $U_t=-\Phi_t$, the sampler obeys
$\dot x_t=-\nabla_xU_t(x_t)$, but its energy evolves as
\begin{equation}
 \frac{\dd}{\dd t}U_t(x_t)
 =\partial_tU_t(x_t)-\norm{\nabla_xU_t(x_t)}^2.
 \label{eq:energy-balance-main}
\end{equation}
\end{proposition}

Because the landscape changes with time, energy need not decrease, and the
endpoint transport is determined by the accumulated dynamics rather than by
any single instantaneous potential. Appendix~\ref{app:linear-identities}
proves Proposition~\ref{prop:potential-main}.

\subsection{CFG constructs a new non-autonomous potential}

Let $\Phi_t^{\mathrm u}$ and $\Phi_t^{\mathrm c}$ be the potentials obtained
from~\eqref{eq:potential-main}, and write
$h_t(x)=\log p(c\mid X_t=x)$. Their difference is determined entirely by the
conditional-to-unconditional density ratio. Indeed, Bayes' rule gives
\begin{align}
 \Phi_t^{\mathrm c}(x)-\Phi_t^{\mathrm u}(x)
 &=\frac{1-t}{t}\left[\log\pc(x)-\log\pu(x)\right]\notag\\
 &=\frac{1-t}{t}h_t(x)-\frac{1-t}{t}\log p(c).
 \label{eq:potential-gap-main}
\end{align}
The common quadratic term in~\eqref{eq:potential-main} cancels in the first
line, and the last term in the second line is independent of $x$. When
$\gamma$ is constant or depends only on time, substituting this potential
difference into the CFG combination~\eqref{eq:cfg-main} yields
\begin{align}
 v_t^\gamma
 &=(1-\gamma)\nabla_x\Phi_t^{\mathrm u}
   +\gamma\nabla_x\Phi_t^{\mathrm c}
  =\nabla_x\Phi_t^\gamma,\label{eq:cfg-gradient-main}\\
 \Phi_t^\gamma
 &=(1-\gamma)\Phi_t^{\mathrm u}+\gamma\Phi_t^{\mathrm c} =\Phi_t^{\mathrm c}
  +(\gamma-1)\frac{1-t}{t}h_t(x)+K_t,
  \label{eq:cfg-potential-main}
\end{align}
where $K_t=-(\gamma-1)(1-t)\log p(c)/t$ is independent of $x$. Thus CFG
superposes a time-weighted condition-posterior potential on the
pure-conditional landscape.

The corresponding force identity is the spatial derivative of the same
potential difference in~\eqref{eq:potential-gap-main}.

\begin{proposition}[Posterior force]
\label{prop:posterior-force}
For $t\in(0,1)$ and positive differentiable intermediate densities,
\begin{equation}
 \vc(x)-\vu(x)
 =\nabla_x\!\left(\Phi_t^{\mathrm c}-\Phi_t^{\mathrm u}\right)
 =\frac{1-t}{t}\nabla_x\log p(c\mid X_t=x).
 \label{eq:force-main}
\end{equation}
\end{proposition}

Consequently,
$v_t^\gamma=\vc+(\gamma-1)(\vc-\vu)$ adds
$(\gamma-1)(1-t)\nabla_xh_t/t$ to the pure-conditional velocity. Equations
\eqref{eq:cfg-potential-main} and~\eqref{eq:force-main} are therefore the
scalar-potential and vector-field forms of the same posterior-guidance term;
its next spatial derivative changes the velocity Jacobian.
Appendix~\ref{app:linear-identities} gives the underlying linear-path
derivation.

For a time-only schedule $\gamma(t)$, the representation
\eqref{eq:cfg-potential-main} still holds, but
$\partial_t\Phi_t^{\gamma(t)}$ acquires an $x$-dependent term proportional to
$\dot\gamma(t)h_t$ (in addition to an $x$-independent gauge term). A schedule
therefore changes both the instantaneous force and the energy-injection term.
The frequently used density
\[
 \widetilde p_{t,\gamma}(x)
 \propto \pu(x)^{1-\gamma}\pc(x)^\gamma
 \propto \pc(x)p(c\mid X_t=x)^{\gamma-1}
\]
is only an instantaneous score interpretation. It does not collapse the
two-parameter transport into sampling from one fixed terminal energy.

\subsection{Why CFG works---and why fixed guidance can fail}

Equation~\eqref{eq:force-main} gives a precise local mechanism for the empirical
effectiveness of CFG. The conditional--unconditional difference is not an
arbitrary correction: at the population optimum it is proportional to the score
of the implicit classifier $p(c\mid X_t=x)$. Thus, relative to the
pure-conditional flow, CFG adds an ascent direction for the current condition
posterior. Indeed, with
$h_t(x)=\log p(c\mid X_t=x)$, the posterior along a guided trajectory satisfies
\begin{align}
 \frac{\dd}{\dd t}h_t(x_t)
 &=\partial_t h_t(x_t)+\nabla h_t(x_t)^\top\vc(x_t)\notag\\
 &\quad+(\gamma-1)\frac{1-t}{t}\norm{\nabla h_t(x_t)}^2.
 \label{eq:posterior-growth-main}
\end{align}
For $\gamma\geq1$, the last term is nonnegative. Thus CFG increases the
instantaneous posterior-growth rate relative to the pure-conditional velocity
at the same time--state pair. It does not imply that the posterior is monotone
along the full trajectory, because the explicit time derivative and the
 pure-conditional transport term can have either sign.

Equation~\eqref{eq:posterior-growth-main} therefore answers ``why CFG works''
without assuming that its endpoint follows a power-tilted distribution: at a
fixed time--state pair, the guidance correction improves the instantaneous
alignment rate by exactly a nonnegative squared-gradient term. This deterministic
Flow-Matching account is complementary to the predictor--corrector interpretation
of diffusion CFG, in which conditional denoising is interleaved with a sharpening
corrector \cite{bradley2024classifier}. The two views agree that guidance
reinforces conditional preference, but the present identity applies directly to
the non-autonomous Flow-Matching ODE. For learned rather than population-optimal
fields, the equality is an ideal-model characterization and model error perturbs
both the direction and its magnitude.

Posterior ascent describes one trajectory but not the diversity of a trajectory
bundle. The relevant velocity Jacobian can be written in terms of posterior
covariance.

\begin{proposition}[Posterior geometry and local volume change]
\label{prop:curvature-main}
Assuming differentiation under the posterior integral is valid,
\begin{align}
 \nabla_xm_t(x)
 &=\frac{t}{(1-t)^2}\Cov(X_1\mid X_t=x),
 \label{eq:mean-jac-main}\\
 H_t^\gamma-H_t^{\mathrm c}
 &=(\gamma-1)\frac{1-t}{t}\nabla_x^2h_t(x)
 \label{eq:hessian-shift-main}\\
 &=(\gamma-1)\frac{t}{(1-t)^3}
 \bigl(\Cov_t^{\mathrm c}-\Cov_t^{\mathrm u}\bigr),
 \label{eq:extra-jac-main}\\
 \nabla\!\cdot(v_t^\gamma-\vc)
 &=(\gamma-1)\frac{t}{(1-t)^3}
 \left[\tr\Cov_t^{\mathrm c}-\tr\Cov_t^{\mathrm u}\right].
 \label{eq:extra-div-main}
\end{align}
where $H_t^\gamma=\nabla_xv_t^\gamma=\nabla_x^2\Phi_t^\gamma$.
For a fixed starting time $s$ and state $x_s$, let
$x_t=X^\gamma(t,s,x_s)$ denote the solution of
$\dot x_t=v_t^\gamma(x_t)$ with $X^\gamma(s,s,x_s)=x_s$. Define the trajectory
Jacobian $J_t=D_{x_s}X^\gamma(t,s,x_s)$. Then $J_s=I_d$ and
\begin{equation}
 \dot J_t=H_t^\gamma(x_t)J_t,
 \qquad
 \frac{\dd}{\dd t}\log|\det J_t|=\nabla\!\cdot v_t^\gamma(x_t).
 \label{eq:liouville-main}
\end{equation}
\label{eq:variational-main}
\end{proposition}

The posterior gradient therefore governs the extra displacement of one
trajectory, while the posterior covariance and Hessian govern the deformation
of neighboring trajectories. In particular, when the conditional posterior
 covariance has smaller trace than the unconditional one, fixed CFG adds
 negative divergence and locally contracts trajectory volume.

The benefit and failure mode thus arise from the same posterior signal. Its first
spatial derivative supplies an alignment-improving direction, but a fixed gain
can accumulate excessive displacement; its second spatial derivative changes
the velocity Jacobian and can concentrate a bundle of trajectories. A useful
controller should therefore retain the posterior direction while adapting its
magnitude to the realized time--state pair.

Both displacement and local volume change accumulate along the realized
trajectory:
\begin{equation}
 x_t-x_s=\int_s^t v_\tau^\gamma(x_\tau)\,\dd\tau,
 \qquad
 \log|\det J_t|=\int_s^t \nabla\!\cdot v_\tau^\gamma(x_\tau)\,\dd\tau.
 \label{eq:accumulated-volume-main}
\end{equation}
\label{eq:accumulated-displacement-main}
Thus, early guidance can leave persistent displacement or contraction at the
endpoint. Its endpoint effect depends on the full time-ordered dynamics, not on
any single instantaneous power-tilted density. Appendix~\ref{app:curvature}
provides the posterior-differentiation and Liouville proofs.

\subsection{Guidance effects in a circular GMM}

We instantiate posterior ascent, mean extrapolation, local contraction, and
branch reweighting in a $K$-component isotropic GMM in $\R^2$,
\begin{equation}
 p_1(x)=\sum_{k=0}^{K-1}w_k\mathcal N(x;\mu_k,\sigma^2I_2),
 \qquad \mu_k=ru_k,
 \label{eq:gmm-main}
\end{equation}
where $u_k=(\cos\theta_k,\sin\theta_k)^\top$. At time $t$,
$X_t\mid k\sim\mathcal N(t\mu_k,C_tI_2)$ with
$C_t=(1-t)^2+t^2\sigma^2$.

\begin{theorem}[Confidence--density decoupling]
\label{thm:far-field-main}
Fix component $0$ and the ray $x=su_0$, $s\geq0$. Let
$P_0(s,t)=p(c=0\mid X_t=su_0)$ and assume $u_0$ uniquely maximizes
$u_0^\top u_k$. For every $t>0$,
\begin{align}
 \partial_s\log P_0(s,t)&\geq0,
 &\partial_s\log p_t(su_0)&<0\quad(s>tr),\notag\\
 P_0(s,t)&\to1,
 &p_t(su_0)&\to0\quad(s\to\infty).
 \label{eq:opposite-main}
\end{align}
\end{theorem}

Thus condition confidence can increase beyond the high-density data region,
and its earlier outward force can leave a persistent displacement.
Appendix~\ref{app:gmm} provides the proof.

For a centered mixture conditioned on component $0$, the observation at $t=0$
is independent of the clean sample, giving
\begin{equation}
 m_0^{\mathrm u}=0,\qquad m_0^{\mathrm c}=\mu_0,\qquad
 \Delta_0=\mu_0,\qquad m_0^\gamma=\gamma\mu_0.
 \label{eq:initial-displacement-main}
\end{equation}
Thus fixed CFG amplifies the full initial gap; for $r=1$ and $\gamma=3$, its
implied endpoint has radius $3$. Appendix~\ref{app:gmm-gap-evolution} provides
a further time-dependent analysis: it derives the exact conditional and
unconditional posterior means, proves that their gap vanishes under a
decreasing envelope, and shows when PMC-CFG recovers its nominal scale.

For an equal-weight rotationally symmetric mixture conditioned on component
$0$, the center field exhibits displacement and contraction simultaneously:
\begin{align}
 (v_t^\gamma-v_t^{\mathrm c})(0)
 &=(\gamma-1)\frac{1-t}{C_t}\mu_0,
 \label{eq:gmm-center-force-main}\\
 \nabla\!\cdot(v_t^\gamma-v_t^{\mathrm c})(0)
 &=-(\gamma-1)\frac{t(1-t)r^2}{C_t^2}<0,
 \qquad \gamma>1.
 \label{eq:gmm-center-divergence-main}
\end{align}
As $t\downarrow0$, $v_t^\gamma(x)\approx\gamma\mu_0-x$: early CFG targets the
extrapolated location $\gamma\mu_0$ while contracting nearby trajectories.

For a concrete multi-component example, take eight equally spaced,
equally weighted components and condition on the adjacent triple
$A=\{0,1,2\}$. Let $\delta=\pi/4$, so that the middle component has angle
$\theta_1=\delta$, and parameterize each nonzero radial shell by
\[
 x_{\rho,\phi}=\rho
 (\cos(\delta+\phi),\sin(\delta+\phi))^\top,
 \qquad \rho>0.
\]
For $t\in(0,1]$, write
$P_t(\rho,\phi)=p(A\mid X_t=x_{\rho,\phi})$. Reflection symmetry and direct
differentiation (Appendix~\ref{app:gmm-branch-reweighting}) give
\begin{equation}
 P_t(\rho,\phi)=P_t(\rho,-\phi),\qquad
 \partial_\phi P_t(\rho,\phi)<0
 \quad (0<\phi\leq\delta).
 \label{eq:gmm-middle-posterior-main}
\end{equation}
Thus, within the conditioned sector on every nonzero radial shell, the
condition posterior is largest on the middle ray and decreases continuously
toward either side branch. Let
$e_\phi=(-\sin(\delta+\phi),\cos(\delta+\phi))^\top$ denote the angular unit
vector. By the posterior-force identity~\eqref{eq:force-main}, for $0<t<1$
and $\gamma>1$,
\begin{equation}
 e_\phi^\top(v_t^\gamma-v_t^{\mathrm c})(x_{\rho,\phi})
 =\frac{(\gamma-1)(1-t)}{t\rho}\,
   \partial_\phi\log P_t(\rho,\phi)<0,
 \qquad 0<\phi\leq\delta.
 \label{eq:gmm-middle-force-main}
\end{equation}
The sign reverses on the other side by symmetry. Hence, at each time--state
pair in the sector, CFG adds an angular velocity toward the middle ray. This is
a local statement about the non-autonomous vector field; the resulting branch
occupancy still depends on the full trajectory transport.

\begin{figure*}[t]
\centering
\begin{minipage}{0.195\textwidth}
\centering
\includegraphics[width=\linewidth]{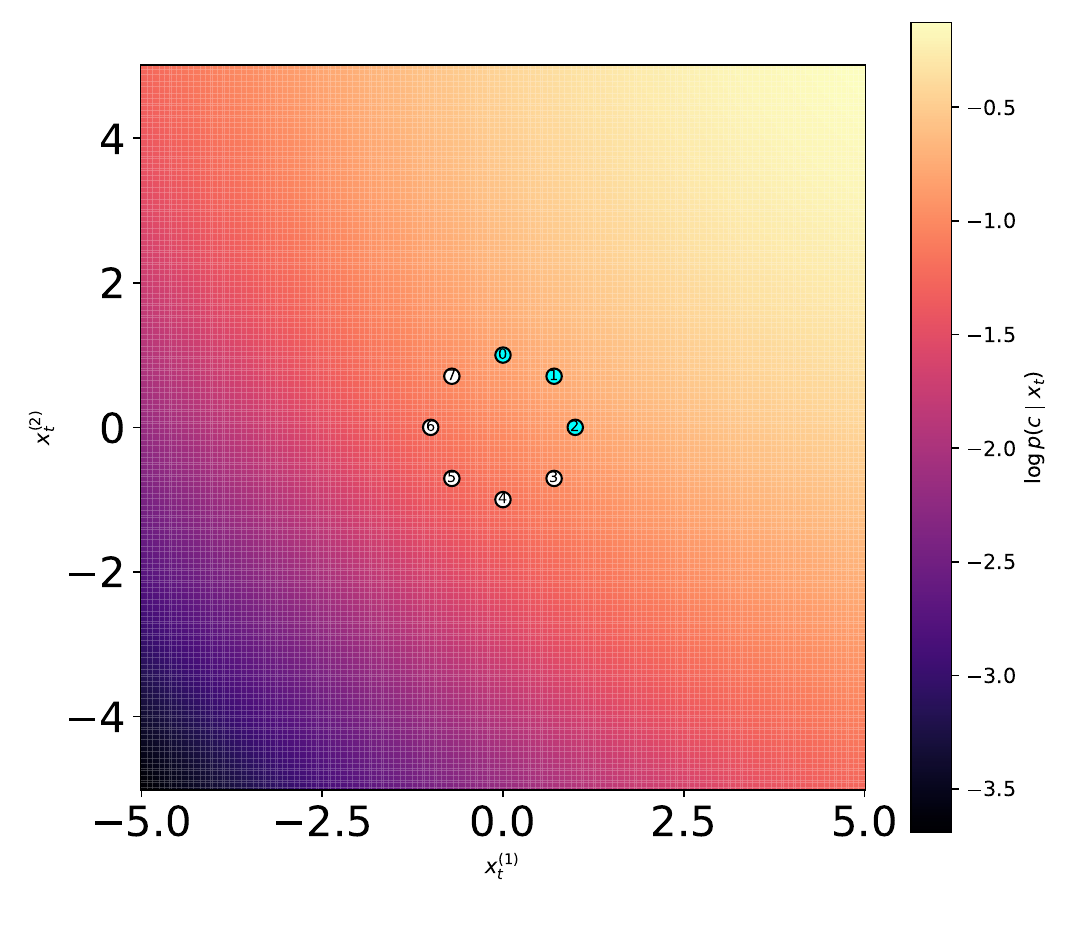}
\end{minipage}\hfill
\begin{minipage}{0.195\textwidth}
\centering
\includegraphics[width=\linewidth]{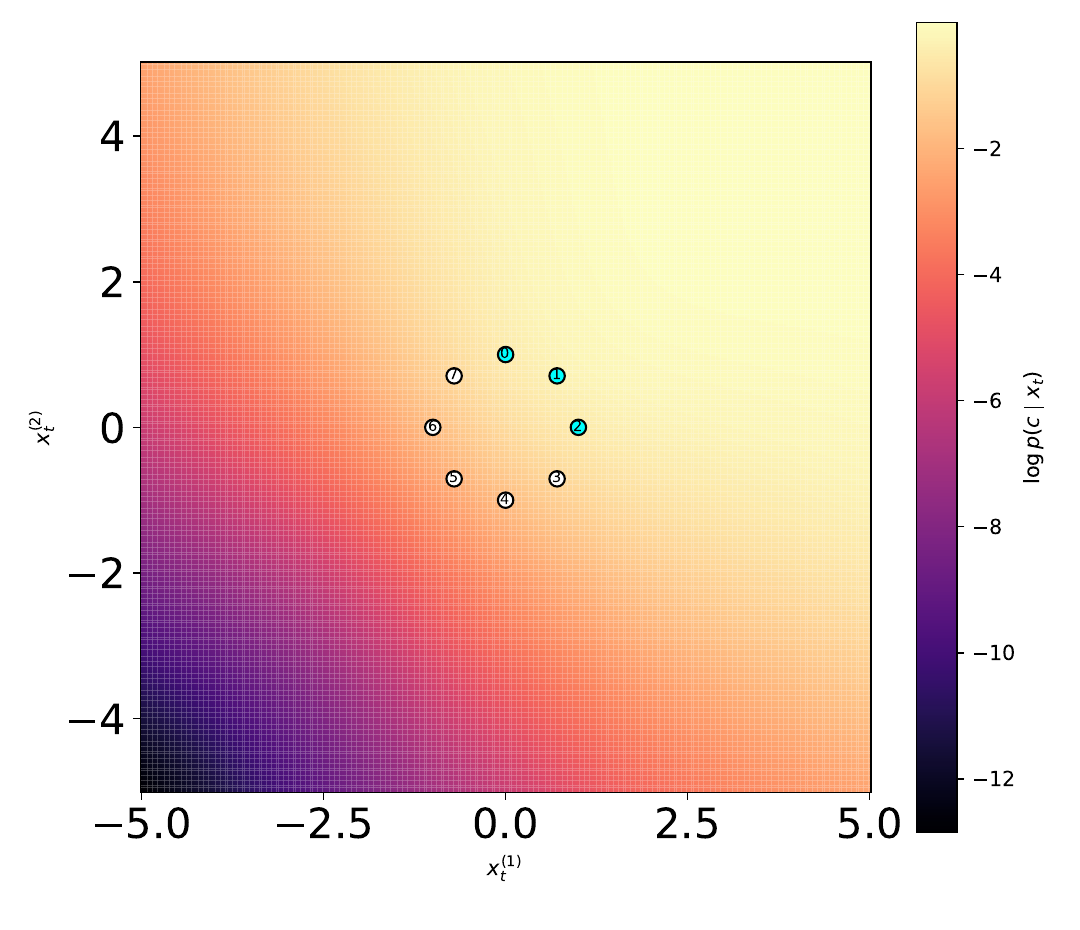}
\end{minipage}\hfill
\begin{minipage}{0.195\textwidth}
\centering
\includegraphics[width=\linewidth]{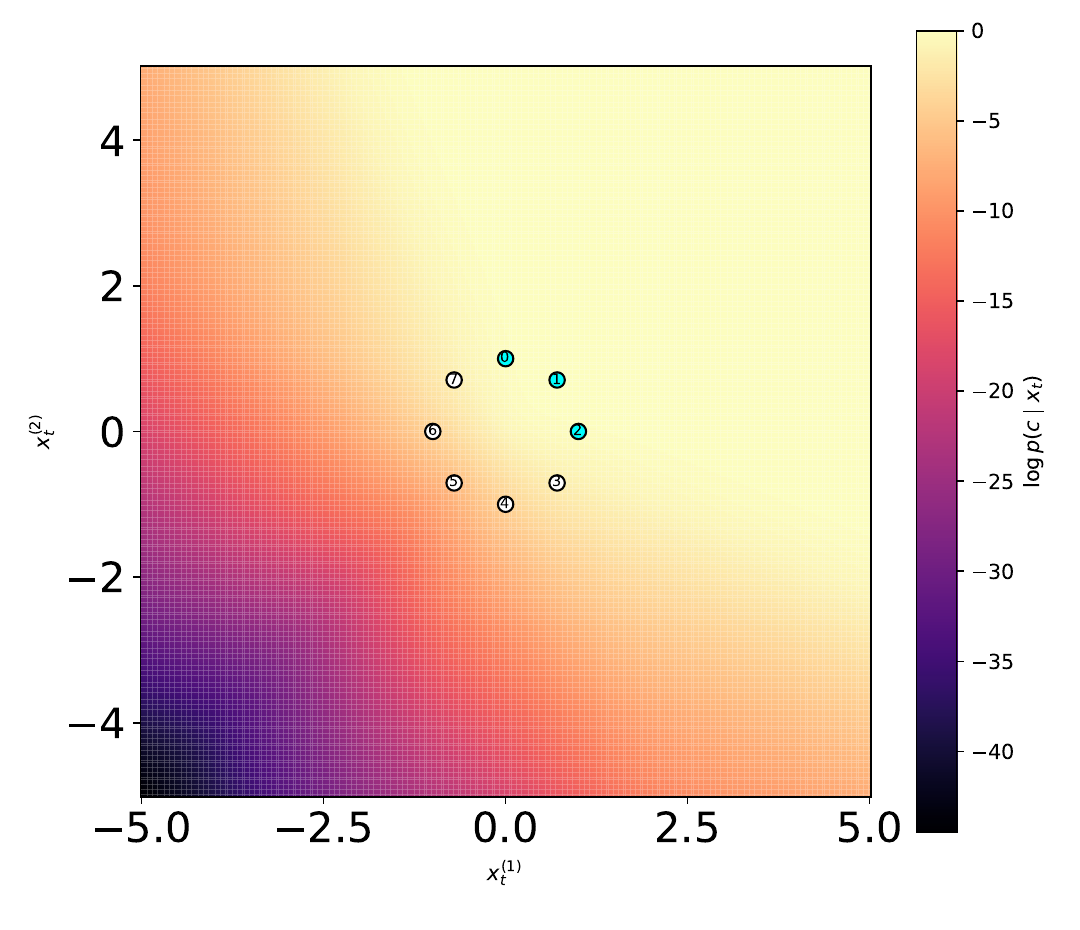}
\end{minipage}\hfill
\begin{minipage}{0.195\textwidth}
\centering
\includegraphics[width=\linewidth]{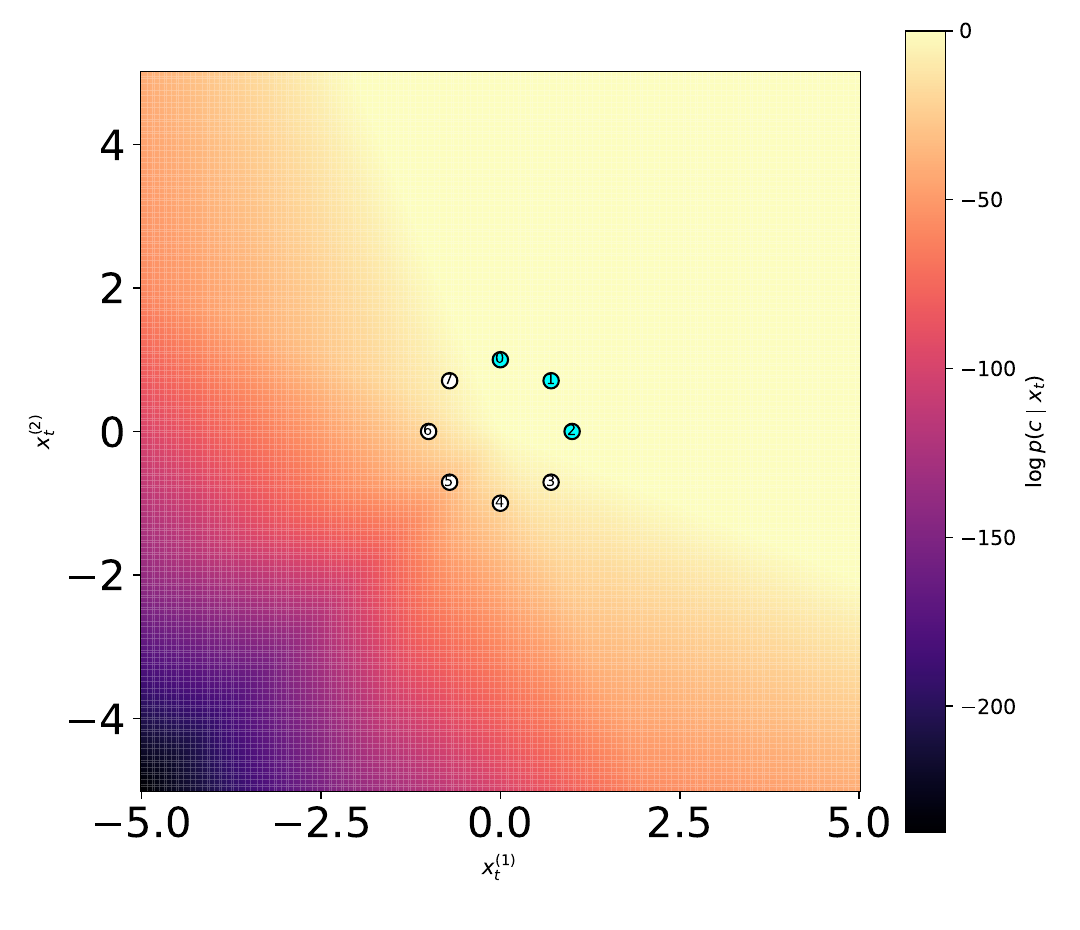}
\end{minipage}\hfill
\begin{minipage}{0.195\textwidth}
\centering
\includegraphics[width=\linewidth]{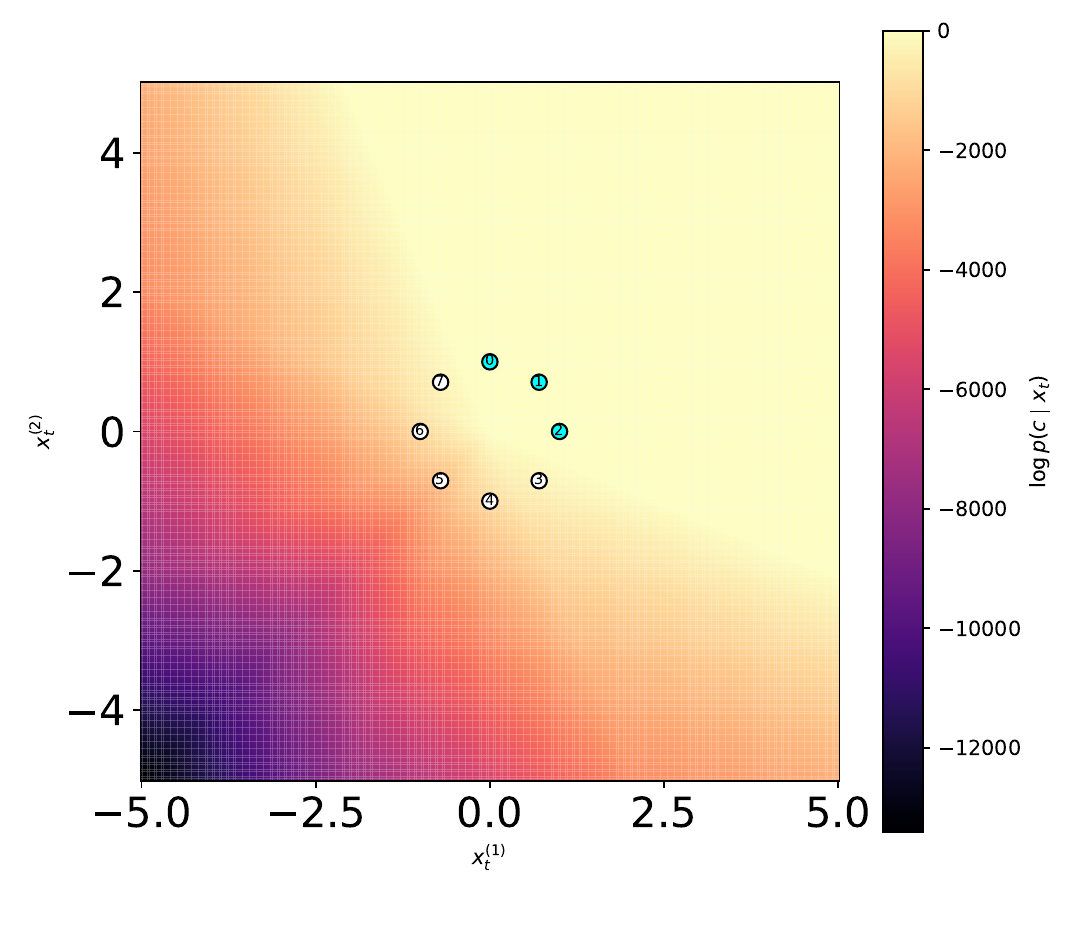}
\end{minipage}
\caption{Condition-posterior landscape $\log p(A\mid x_t)$ for the eight-point
GMM with $A=\{0,1,2\}$ at $t\in\{0.2,0.4,0.6,0.8,1.0\}$ (left to right).
On each nonzero radial shell within the conditioned sector, the posterior peaks
on the middle ray and decreases angularly toward the two side branches.}
\label{fig:gmm-log-posterior-time}
\end{figure*}

Figure~\ref{fig:gmm-log-posterior-time} visualizes this equal-weight,
middle-direction mechanism.

The time-varying posterior potential also reflects a pre-existing imbalance
between conditioned branches, although for a different reason than the
geometric middle-branch preference above.

\begin{proposition}[Posterior-potential bias under branch imbalance]
\label{prop:gmm-weight-amplification-main}
In the eight-component GMM, let $A=\{0,1\}$ consist of adjacent components with
$w_0>w_1>0$, and assume that the complementary weights are reflection
symmetric about the angular bisector of $u_0$ and $u_1$:
$w_2=w_7$, $w_3=w_6$, and $w_4=w_5$. At the interpolated component centers
$x_{k,t}=t\mu_k$, define
$P_{A,k}(t)=p(A\mid X_t=x_{k,t})$. Then, for $t\in(0,1]$,
$P_{A,0}(t)>P_{A,1}(t)$. Consequently, for $0<t<1$ and $\gamma>1$, the CFG
correction to the non-autonomous potential,
$(\gamma-1)(1-t)\log p(A\mid X_t=x)/t$, is larger at $x_{0,t}$ than at
$x_{1,t}$.
\end{proposition}

Thus the extra time-dependent potential favors the already heavier branch at
the two interpolated centers. This potential ordering alone is not a theorem
that the terminal branch occupancy must increase: the velocity is determined
by spatial potential gradients, and occupancy by their accumulation along the
full non-autonomous transport. Table~\ref{tab:gmm} reports the corresponding
endpoint behavior empirically.
Appendix~\ref{app:gmm} contains the derivations and the nonadjacent
configuration.

\section{Posterior-Mean-Capped Guidance}
\label{sec:method}

\subsection{A per-sample geometric constraint}

Equation~\eqref{eq:force-main} shows that CFG moves samples along
$\nabla_x\log p(c\mid X_t=x)$. In the circular GMM,
Theorem~\ref{thm:far-field-main} shows that this posterior can increase beyond
the high-density data region, so its gradient can drive low-density
extrapolation. Because the posterior-mean gap is proportional to the same
gradient, the limitation of fixed CFG is not its direction but its open-loop
magnitude: the same scale is applied without checking the displacement implied
at the current sample. PMC-CFG closes this structural gap by turning the fixed
gain into pointwise feedback. It keeps the original posterior direction and
limits only how far the implied terminal mean may extrapolate, while retaining
as much nominal guidance as the constraint permits.

Let $\lambda\geq1$ denote the nominal CFG scale and $\Gamma\geq1$ the
permitted relative mean amplification. The two usual model evaluations give
\begin{align}
 m_t^{\mathrm c}(x)&=x+(1-t)\vc(x), m_t^{\mathrm u}(x)&=x+(1-t)\vu(x),\\
 \Delta_t(x)&=m_t^{\mathrm c}(x)-m_t^{\mathrm u}(x)
 =(1-t)(\vc(x)-\vu(x))\notag\\
 &=\frac{(1-t)^2}{t}\nabla_x\log p(c\mid X_t=x).
 \label{eq:method-means}
\end{align}
Unconstrained nominal CFG implies
$m_t^\lambda=m_t^{\mathrm c}+(\lambda-1)\Delta_t$. PMC-CFG instead solves
\begin{equation}
 \beta_t^*(x)=\max\left\{\beta:
 \begin{array}{l}
 0\leq\beta\leq\lambda-1,\\[-0.1em]
 \norm{m_t^{\mathrm c}+\beta\Delta_t}
 \leq\Gamma\norm{m_t^{\mathrm c}}
 \end{array}\right\}
 \label{eq:program-main}
\end{equation}
and uses
\begin{equation}
 v_t^{\mathrm{cap}}(x)=\vc(x)+\beta_t^*(x)(\vc(x)-\vu(x)).
 \label{eq:velocity-cap-main}
\end{equation}
The constraint is evaluated independently per sample. 

\subsection{Closed-form scale and guarantees}

Define
\begin{equation}
 a=\norm{\Delta_t}^2,\qquad
 b=\inner{m_t^{\mathrm c}}{\Delta_t},\qquad
 q=\norm{m_t^{\mathrm c}}^2.
 \label{eq:abc-main}
\end{equation}
For $a>0$, let
\begin{equation}
 \beta^{\mathrm{cap}}
 =\frac{-b+\sqrt{b^2+(\Gamma^2-1)aq}}{a}.
 \label{eq:root-main}
\end{equation}

\begin{theorem}[Maximal feasible guidance]
\label{thm:maximal-main}
For every $\Gamma,\lambda\geq1$, problem~\eqref{eq:program-main} is feasible.
Its largest solution is
\begin{equation}
 \beta_t^*=\min\{\lambda-1,\beta^{\mathrm{cap}}\}
 \quad (a>0).
 \label{eq:beta-main}
\end{equation}
If $a=0$, all scales produce the same velocity and we set
$\beta_t^*=\lambda-1$. In exact arithmetic,
\begin{equation}
 \norm{x+(1-t)v_t^{\mathrm{cap}}(x)}
 \leq\Gamma\norm{x+(1-t)\vc(x)}.
 \label{eq:guarantee-main}
\end{equation}
Nominal CFG is unchanged whenever it is feasible.
\end{theorem}

Theorem~\ref{thm:maximal-main} follows from the feasible interval of
 $a\beta^2+2b\beta+(1-\Gamma^2)q\leq0$. Its complete proof, including
 degenerate and floating-point cases, is in Appendix~\ref{app:pmc-proof}.

The construction completes fixed-scale CFG in a precise, limited sense. First,
it is \emph{direction preserving}: the update remains on the original CFG ray
$m_t^{\mathrm c}+\beta\Delta_t$. Second, it is \emph{always feasible}: the
pure-conditional choice $\beta=0$ satisfies the constraint. Third, it is a
\emph{minimal intervention}: among all feasible points on that ray it selects
the largest $\beta$, and hence exactly recovers nominal CFG whenever no cap is
needed. Finally, it is \emph{state adaptive}: the scale depends on the current
mean, gap, and their angle rather than on time alone. These properties do not
assert global correctness of the endpoint distribution; they guarantee that
each model evaluation retains the strongest feasible posterior guidance under
the stated mean-space constraint.

\subsection{Suppressing posterior-driven extrapolation}

The cap is strongest when the posterior-driven displacement points outward.
If $\theta$ is the angle between $m_t^{\mathrm c}$ and $\Delta_t$, then
\begin{equation}
 \beta^{\mathrm{cap}}
 =\frac{\norm{m_t^{\mathrm c}}}{\norm{\Delta_t}}
 \left[-\cos\theta+\sqrt{\cos^2\theta+\Gamma^2-1}\right].
 \label{eq:angle-main}
\end{equation}
Thus, an aligned gap that increases the implied-mean norm is suppressed most
strongly, whereas a gap that initially moves the mean inward permits more
guidance.

In the centered single-component GMM,
\begin{equation}
 m_0^{\mathrm{cap}}=\min\{\lambda,\Gamma\}\mu_0,
 \label{eq:initial-cap-main}
\end{equation}
so PMC-CFG directly caps the posterior-driven extrapolation
$m_0^\lambda=\lambda\mu_0$ identified in
 \eqref{eq:initial-displacement-main}. Appendix~\ref{app:pmc-proof} gives the
 derivations, coordinate properties, and stable pseudocode.

The cap is also self-releasing rather than uniformly conservative. In the GMM,
Proposition~\ref{prop:gmm-gap-app} proves that the posterior-mean gap is bounded
by a strictly decreasing envelope that vanishes as $t\uparrow1$; consequently,
the nominal scale is recovered near every nonzero endpoint. PMC-CFG therefore
implements a data-dependent coarse-to-fine schedule: it suppresses early
large-gap extrapolation and restores full guidance once the conditional and
unconditional terminal estimates become sufficiently consistent.

\section{Experiments}
\label{sec:experiments}
\begin{figure*}[t]
\centering
\begin{minipage}{0.105\textwidth}
\centering\tiny $\lambda=1$\\[-0.3em]$A=\{0\}$\\[-0.35em]
\includegraphics[width=\linewidth]{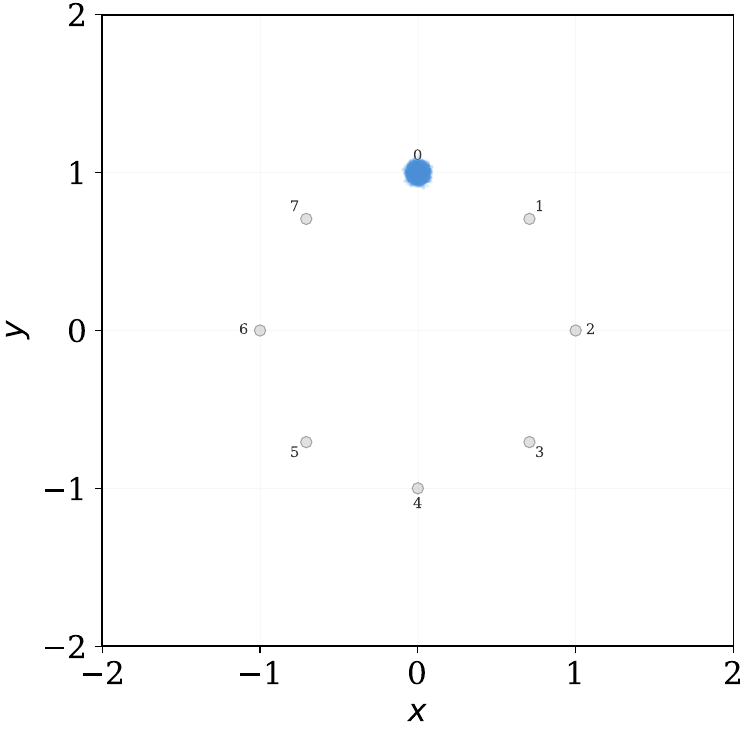}
\end{minipage}\hfill
\begin{minipage}{0.105\textwidth}
\centering\tiny Fixed $\lambda=3$\\[-0.3em]$A=\{0\}$\\[-0.35em]
\includegraphics[width=\linewidth]{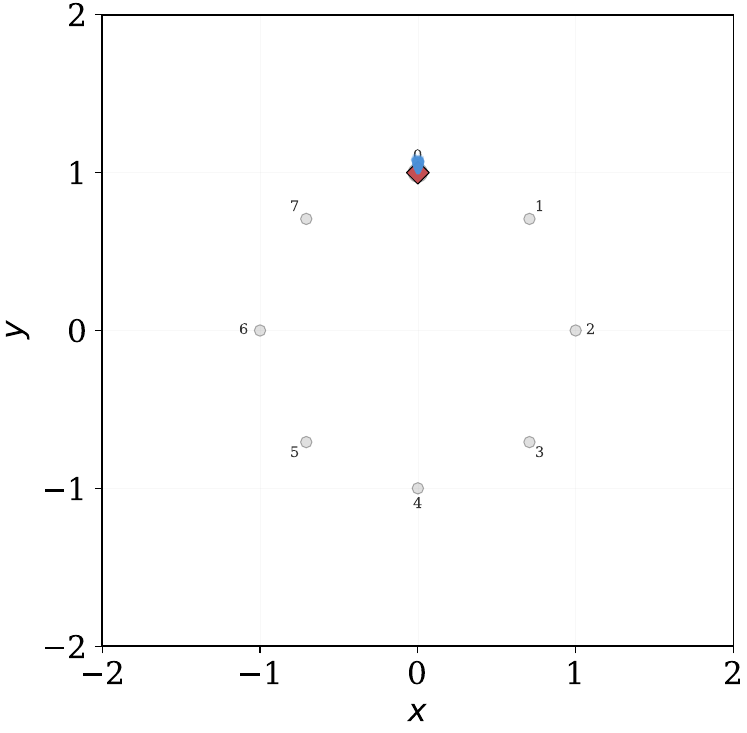}
\end{minipage}\hfill
\begin{minipage}{0.105\textwidth}
\centering\tiny PMC $\ A=\{0\}$\\[-0.35em]
\includegraphics[width=\linewidth]{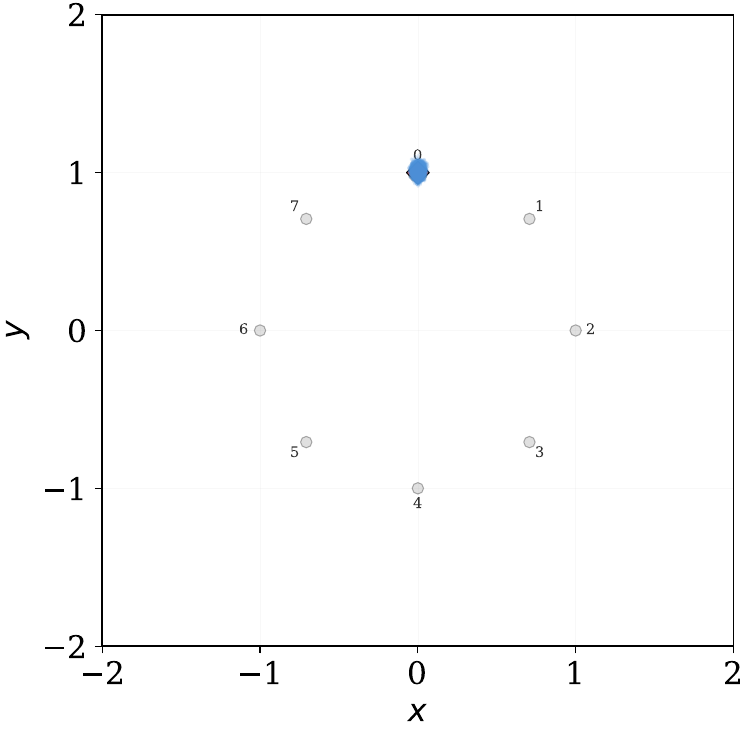}
\end{minipage}\hfill
\begin{minipage}{0.105\textwidth}
\centering\tiny $\lambda=1$\\[-0.3em]$A=\{0,1\}$\\[-0.35em]
\includegraphics[width=\linewidth]{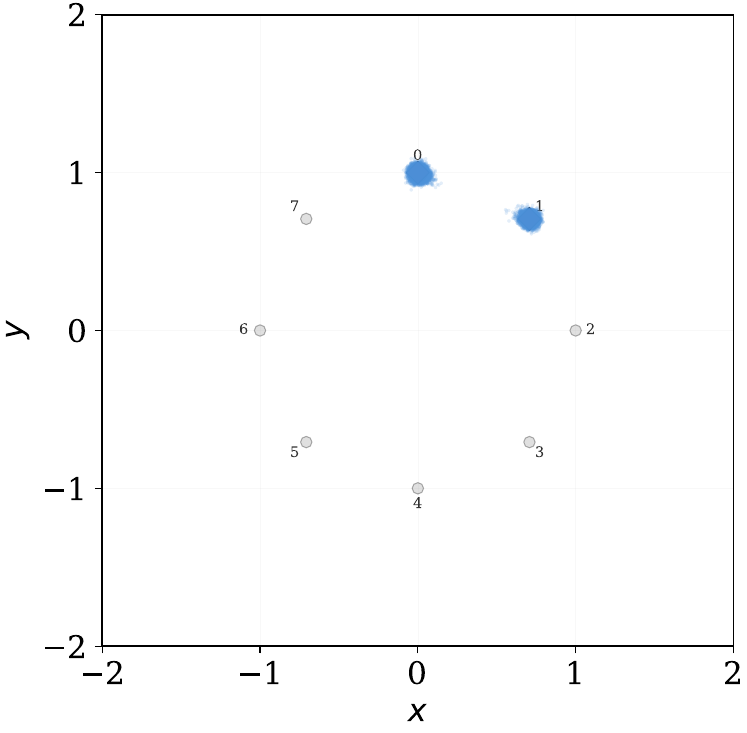}
\end{minipage}\hfill
\begin{minipage}{0.105\textwidth}
\centering\tiny Fixed $\lambda=3$\\[-0.3em]$A=\{0,1\}$\\[-0.35em]
\includegraphics[width=\linewidth]{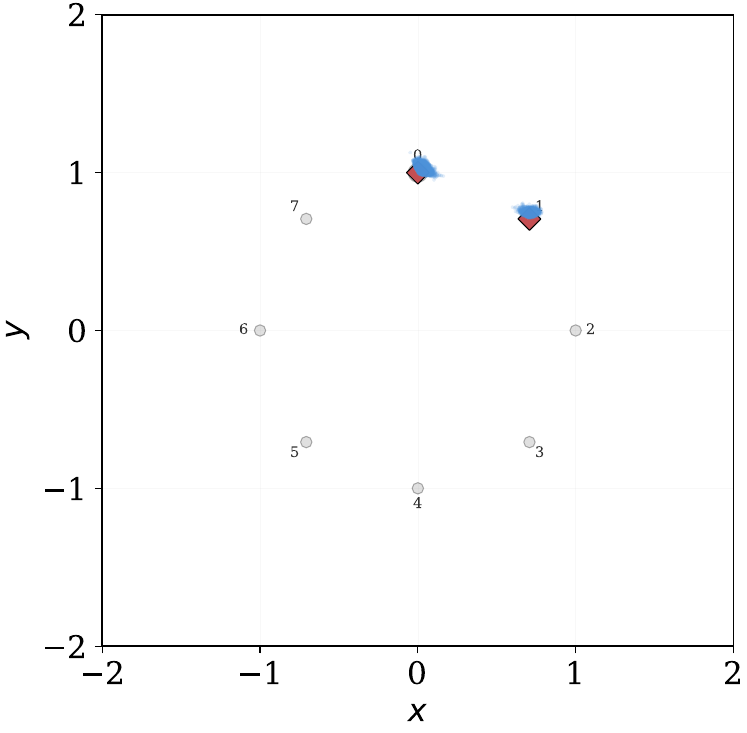}
\end{minipage}\hfill
\begin{minipage}{0.105\textwidth}
\centering\tiny PMC $\ A=\{0,1\}$\\[-0.35em]
\includegraphics[width=\linewidth]{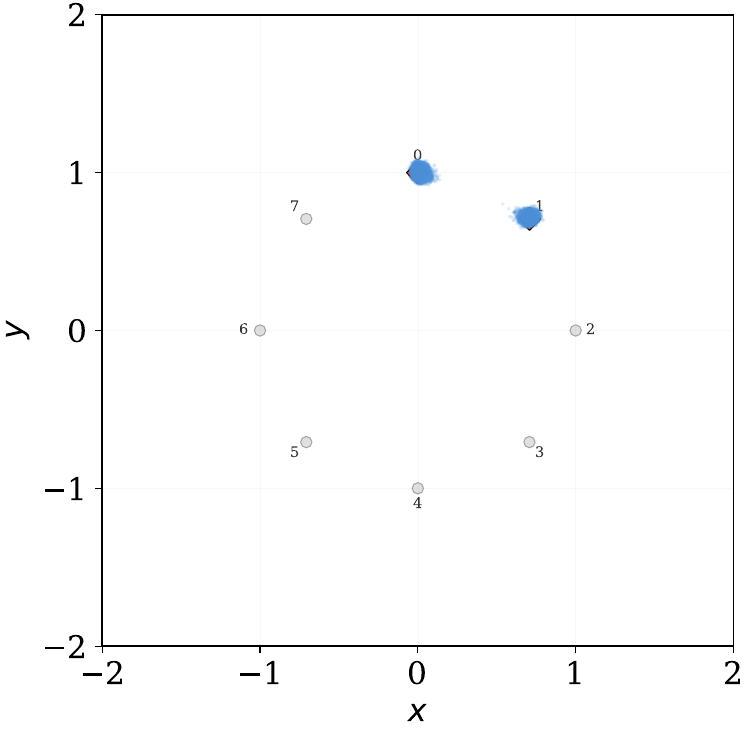}
\end{minipage}\hfill
\begin{minipage}{0.105\textwidth}
\centering\tiny $\lambda=1$\\[-0.3em]$A=\{0,1,2\}$\\[-0.35em]
\includegraphics[width=\linewidth]{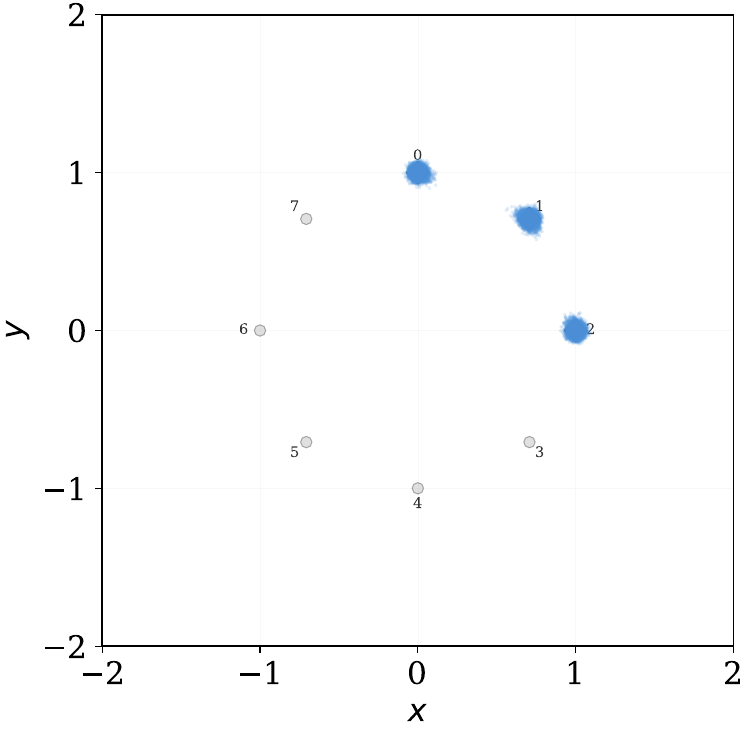}
\end{minipage}\hfill
\begin{minipage}{0.105\textwidth}
\centering\tiny Fixed $\lambda=3$\\[-0.3em]$A=\{0,1,2\}$\\[-0.35em]
\includegraphics[width=\linewidth]{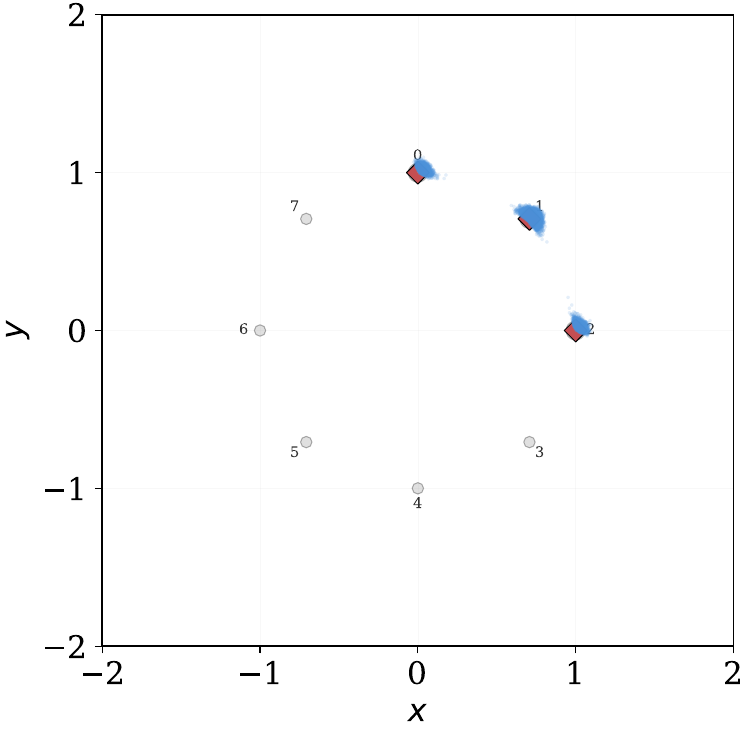}
\end{minipage}\hfill
\begin{minipage}{0.105\textwidth}
\centering\tiny PMC $\ A=\{0,1,2\}$\\[-0.35em]
\includegraphics[width=\linewidth]{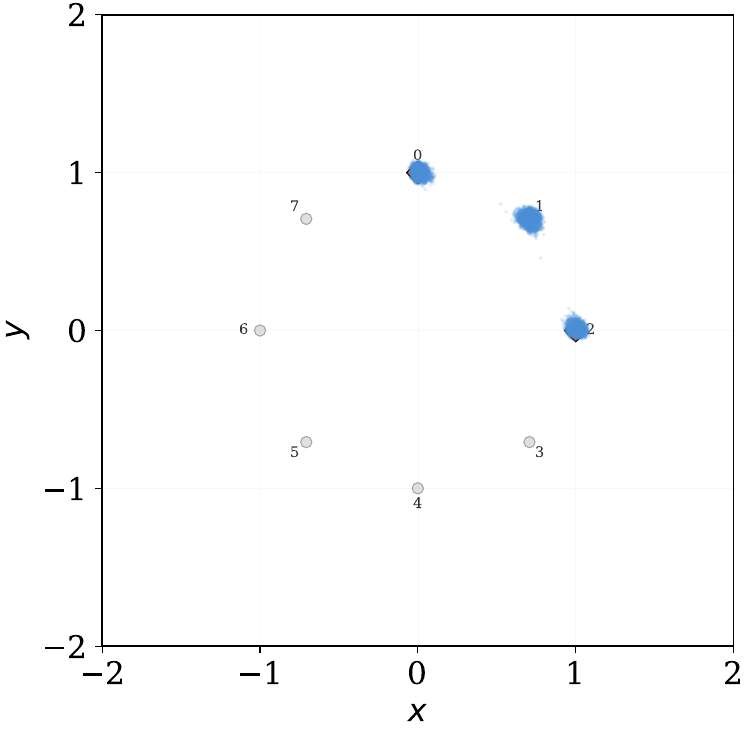}
\end{minipage}
\vspace{-0.45em}
\caption{Endpoint samples for the analytic mixture. Gray markers are component
means, red diamonds are conditioned means, and blue points are samples. PMC-CFG reduces the overshoot and
mode concentration caused by strong fixed guidance.}
\label{fig:gmm-samples}
\vspace{-0.6em}
\end{figure*}
\begin{table*}[t]
\caption{Single-run GMM endpoint statistics ($n=2\times10^5$ per row).}
\label{tab:gmm}
\centering\scriptsize
\setlength{\tabcolsep}{4.2pt}
\begin{tabular}{clcclllcc}
\toprule
$A$ & Guidance & $\lambda$ & $\Gamma$ & Theoretical occ.\ (\%) & Empirical occ.\ (\%)
& Per-label variance ratio & TV$\downarrow$ & Mean error$\downarrow$\\
\midrule
$\{0\}$ & Conditional & 1 & --- & 100.00 & 100.00 & 0.571 & 0.00000 & 0.00005\\
$\{0\}$ & Fixed CFG & 3 & --- & 100.00 & 100.00 & 0.135 & 0.00000 & 0.04091\\
$\{0\}$ & PMC-CFG & 3 & 1.10 & 100.00 & 100.00 & 0.348 & 0.00000 & 0.00831\\
\addlinespace
$\{0,1\}$ & Conditional & 1 & --- & 68.00 / 32.00 & 68.34 / 31.66
& 0.603 / 0.630 & 0.00339 & 0.00247\\
$\{0,1\}$ & Fixed CFG & 3 & --- & 68.00 / 32.00 & 83.27 / 16.73
& 0.315 / 0.358 & 0.15272 & 0.11781\\
$\{0,1\}$ & PMC-CFG & 3 & 1.10 & 68.00 / 32.00 & 71.85 / 28.15
& 0.498 / 0.536 & 0.03851 & 0.02800\\
\addlinespace
$\{0,1,2\}$ & Conditional & 1 & --- & 33.33 / 33.33 / 33.33 & 32.82 / 34.30 / 32.88
& 0.622 / 0.669 / 0.623 & 0.00967 & 0.00270\\
$\{0,1,2\}$ & Fixed CFG & 3 & --- & 33.33 / 33.33 / 33.33 & 15.74 / 68.63 / 15.64
& 0.368 / 0.526 / 0.368 & 0.35294 & 0.12883\\
$\{0,1,2\}$ & PMC-CFG & 3 & 1.10 & 33.33 / 33.33 / 33.33 & 29.26 / 41.38 / 29.37
& 0.522 / 0.647 / 0.522 & 0.08046 & 0.02980\\
\bottomrule
\end{tabular}
\end{table*}
\subsection{Experimental setup}

Within each benchmark, all methods use matched model weights, conditions,
initial noise, solver, step count, and decoder. Full implementation and metric
details are given in Appendix~\ref{app:experiments}.

\paragraph{GMM.}
We draw $2\times10^5$ samples from an eight-component circular mixture and
compare the conditional field, fixed CFG ($\lambda=3$), and PMC-CFG
($\lambda=3$, $\Gamma=1.1$) across singleton, adjacent-pair, and
adjacent-triple conditions.

\paragraph{ImageNet-256.}
We evaluate class-conditional SiT-XL/2 \cite{deng2009imagenet,ma2024sit} on 50K
samples at 20/50 NFE, testing $\lambda\in\{1.5,2.0,2.5\}$ and
$\Gamma\in\{1.05,1.10\}$. We compare fixed CFG, C2FG
\cite{gao2026c2fg}, APG \cite{sadat2025apg}, and PMC-CFG, and report FID, sFID, IS,
precision, and recall
\cite{heusel2017ttur,salimans2016improved,dhariwal2021diffusion,
kynkaanniemi2019precisionrecall}.

\paragraph{Stable Diffusion 3.5.}
We evaluate SD3.5 Medium \cite{esser2024sd3} at 10 NFE and
$\lambda\in\{3,5,7,9\}$ on 5,000 COCO 2017 validation prompts
\cite{lin2014coco}, sharing prompts and initial latents across methods. We
compare fixed CFG, C2FG \cite{gao2026c2fg}, APG \cite{sadat2025apg}, and
PMC-CFG, and report CLIP, ImageReward, FID,
Recall@$3$, and mean HSV saturation; the latter is a
diagnostic rather than a directional quality metric.

\subsection{Results}

\paragraph{GMM.}

Figure~\ref{fig:gmm-samples} and Table~\ref{tab:gmm} confirm the predicted
failure modes: strong fixed CFG
displaces and contracts a single mode, while in multimodal conditions it
concentrates mass on the heavier or geometrically central branch. PMC-CFG
consistently reduces these mean and occupancy errors and restores within-branch
spread; Appendix Figure~\ref{fig:gmm-distributions} visualizes the corresponding
marginal densities. These endpoint trends agree with the predicted posterior-mean
extrapolation, local contraction, and branch amplification; discretization and
run-level caveats are detailed in Appendix~\ref{app:experiments}.

\paragraph{ImageNet-256.}

\begin{table*}[t]
\caption{ImageNet-256 results on 50K samples. The best FID at each NFE is bold.}
\label{tab:imagenet}
\centering\tiny
\setlength{\tabcolsep}{2.2pt}
\begin{tabular}{lccrrrrrrrrrr}
\toprule
& & & \multicolumn{5}{c}{20 NFE} & \multicolumn{5}{c}{50 NFE}\\
\cmidrule(lr){4-8}\cmidrule(lr){9-13}
Method & $\lambda$ & $\Gamma$ & FID$\downarrow$ & sFID$\downarrow$
& IS$\uparrow$ & Prec.$\uparrow$ & Rec.$\uparrow$ & FID$\downarrow$ & sFID$\downarrow$
& IS$\uparrow$ & Prec.$\uparrow$ & Rec.$\uparrow$\\
\midrule
Fixed CFG & 1.50 & ---  & 3.58 & 5.71 & 239.49 & 0.79 & 0.57
& 2.39 & 4.80 & 258.11 & 0.80 & 0.59\\
APG       & 1.50 & ---  & 3.61 & 6.34 & 245.37 & 0.78 & 0.58
& 2.29 & 5.20 & 263.51 & 0.80 & 0.60\\
C2FG      & 1.50 & ---  & 3.09 & 5.45 & 260.62 & 0.80 & 0.55
& 2.21 & 4.64 & 278.73 & 0.81 & 0.58\\
PMC-CFG   & 1.50 & 1.05 & 3.21 & 5.79 & 253.32 & 0.79 & 0.58
& \textbf{2.10} & 4.75 & 274.26 & 0.81 & 0.59\\
PMC-CFG   & 1.50 & 1.10 & 3.03 & 5.51 & 261.95 & 0.80 & 0.57
& 2.12 & 4.60 & 283.15 & 0.82 & 0.58\\
\addlinespace
Fixed CFG & 2.00 & ---  & 3.58 & 4.74 & 340.73 & 0.86 & 0.50
& 3.62 & 4.55 & 357.64 & 0.87 & 0.51\\
APG       & 2.00 & ---  & 3.49 & 4.80 & 348.26 & 0.85 & 0.51
& 3.31 & 4.18 & 365.52 & 0.86 & 0.53\\
C2FG      & 2.00 & ---  & 3.98 & 4.67 & 358.61 & 0.87 & 0.49
& 4.12 & 4.58 & 375.54 & 0.87 & 0.50\\
PMC-CFG   & 2.00 & 1.05 & \textbf{2.88} & 4.79 & 337.62 & 0.84 & 0.53
& 2.85 & 4.15 & 359.72 & 0.85 & 0.55\\
PMC-CFG   & 2.00 & 1.10 & 3.42 & 4.40 & 354.70 & 0.86 & 0.52
& 3.62 & 4.09 & 376.31 & 0.87 & 0.52\\
\addlinespace
Fixed CFG & 2.50 & ---  & 6.22 & 5.42 & 402.62 & 0.89 & 0.43
& 6.52 & 5.63 & 414.58 & 0.89 & 0.45\\
APG       & 2.50 & ---  & 6.02 & 4.33 & 412.72 & 0.89 & 0.44
& 6.17 & 4.06 & 428.46 & 0.89 & 0.46\\
C2FG      & 2.50 & ---  & 6.84 & 5.48 & 417.47 & 0.90 & 0.42
& 7.17 & 5.76 & 427.26 & 0.90 & 0.44\\
PMC-CFG   & 2.50 & 1.05 & 3.82 & 4.46 & 382.53 & 0.86 & 0.51
& 4.11 & 4.07 & 403.15 & 0.87 & 0.52\\
PMC-CFG   & 2.50 & 1.10 & 5.07 & 4.14 & 406.12 & 0.88 & 0.48
& 5.56 & 4.20 & 421.97 & 0.89 & 0.49\\
\bottomrule
\end{tabular}
\end{table*}
\begin{table*}[t]
\caption{SD3.5 Medium results on 5,000 COCO 2017 validation prompts at 10 NFE.
Bold indicates the best directional metric within each guidance scale.}
\label{tab:sd35}
\centering\small
\setlength{\tabcolsep}{5pt}
\begin{tabular}{llccccc}
\toprule
Scale $\lambda$ & Method & CLIP$\uparrow$ & ImageReward$\uparrow$ & FID$\downarrow$
& Sat. & Recall@$3\uparrow$\\
\midrule
$3$ & Fixed CFG & 0.3171 & 0.6476 & 26.94 & 0.301 & 0.4456\\
    & APG       & 0.3088 & 0.2580 & 25.72 & 0.197 & \textbf{0.5020}\\
    & C2FG      & \textbf{0.3174} & \textbf{0.6715} & 26.97 & 0.302 & 0.4394\\
    & PMC-CFG   & 0.3148 & 0.5918 & \textbf{25.35} & 0.223 & 0.4722\\
\addlinespace
$5$ & Fixed CFG & 0.3192 & 0.7101 & 29.20 & 0.407 & 0.4234\\
    & APG       & 0.3138 & 0.5728 & \textbf{24.68} & 0.208 & \textbf{0.4852}\\
    & C2FG      & \textbf{0.3193} & 0.7271 & 29.06 & 0.408 & 0.4290\\
    & PMC-CFG   & 0.3168 & \textbf{0.7538} & 25.65 & 0.227 & 0.4620\\
\addlinespace
$7$ & Fixed CFG & 0.3145 & 0.4328 & 31.40 & 0.492 & 0.4208\\
    & APG       & 0.3160 & 0.7194 & \textbf{25.21} & 0.220 & \textbf{0.4784}\\
    & C2FG      & 0.3148 & 0.4516 & 31.20 & 0.493 & 0.4202\\
    & PMC-CFG   & \textbf{0.3174} & \textbf{0.8033} & 25.67 & 0.229 & 0.4540\\
\addlinespace
$9$ & Fixed CFG & 0.3049 & $-0.0295$ & 40.83 & 0.555 & \textbf{0.4512}\\
    & APG       & 0.3172 & 0.8097 & 25.88 & 0.234 & 0.4472\\
    & C2FG      & 0.3056 & $-0.0108$ & 40.54 & 0.556 & 0.4470\\
    & PMC-CFG   & \textbf{0.3178} & \textbf{0.8262} & \textbf{25.61} & 0.230 & \textbf{0.4512}\\
\bottomrule
\end{tabular}
\end{table*}

Table~\ref{tab:imagenet} shows that PMC-CFG is most effective under strong
guidance. At $\lambda=1.5$, C2FG gives lower FID while APG gives higher recall.
At larger scales, APG modestly improves FID and recall over fixed CFG, whereas
PMC-CFG with $\Gamma=1.05$ yields substantially larger improvements in both
metrics at $\lambda\in\{2.0,2.5\}$ and both solver budgets, with only a small
precision trade-off.

The FID-optimal nominal scale also exhibits a solver-budget-dependent trend.
As NFE increases from 20 to 50, the best scale for APG and PMC-CFG with
$\Gamma=1.05$ shifts from $\lambda=2.0$ to $\lambda=1.5$; fixed CFG changes
from a tie between these scales to a clear preference for $\lambda=1.5$,
while C2FG and PMC-CFG with $\Gamma=1.10$ already favor $\lambda=1.5$ at both
budgets. Moreover, $\lambda=2.5$ is suboptimal for every method and budget.
Thus, moderate guidance can help some controlled methods under a coarse solver,
but additional solver steps consistently remove any empirical need for a
larger nominal scale in this evaluation.

A plausible explanation is a change in the dominant source of error. Increasing
NFE reduces numerical integration error for a weakly guided flow, but it does
not remove guidance-induced mean displacement and local contraction because
these are properties of the modified vector field itself. Once discretization
error is reduced, increasing the guidance scale therefore exposes more of this
structural bias rather than improving integration accuracy. This interpretation
is consistent with, but not causally established by, the present endpoint
metrics. Mechanism diagnostics are provided in
Appendix~\ref{app:imagenet-mechanism}. Consistently,
Figure~\ref{fig:qualitative-comparison}(a) shows that PMC-CFG preserves more of
the unguided layout diversity while recovering sharp boundaries and texture,
whereas fixed CFG more often homogenizes composition and increases saturation.

\paragraph{Stable Diffusion 3.5.}

Table~\ref{tab:sd35} shows that both APG and PMC-CFG avoid the FID and
saturation degradation of fixed CFG and C2FG at high guidance. APG obtains the
best FID and Recall@$3$ at $\lambda\in\{5,7\}$ and the best Recall@$3$ at
$\lambda=3$. PMC-CFG gives the best FID at $\lambda\in\{3,9\}$, outperforms
all baselines in CLIP and ImageReward at $\lambda\in\{7,9\}$, and ties fixed
CFG for the best Recall@$3$ at $\lambda=9$. Thus APG gives a stronger
distributional trade-off at intermediate scales, whereas PMC-CFG preserves
stronger text alignment as the nominal scale increases.
Figure~\ref{fig:qualitative-comparison}(b) mirrors these metrics: PMC-CFG keeps
the requested text legible while preserving more variation in pose,
background, and color than fixed CFG, which more often collapses toward a
centered, homogeneous composition.

\begin{figure}[t]
\centering
\begin{minipage}[t]{0.485\linewidth}
\centering
\setlength{\tabcolsep}{0.6pt}
\begin{tabular}{@{}ccc@{}}
\shortstack[c]{\tiny No guidance\\[-1pt]\tiny $\lambda=1$}
& \shortstack[c]{\tiny Fixed CFG\\[-1pt]\tiny $\lambda=2.5$}
& \shortstack[c]{\tiny PMC-CFG\\[-1pt]\tiny $\lambda=2.5$} \\
\includegraphics[width=0.30\linewidth]{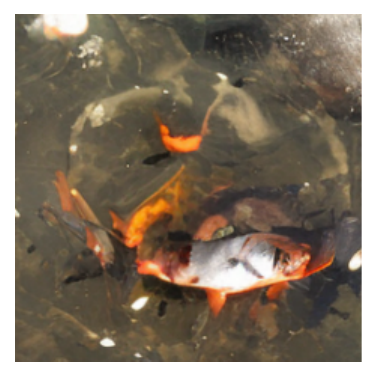}
& \includegraphics[width=0.30\linewidth]{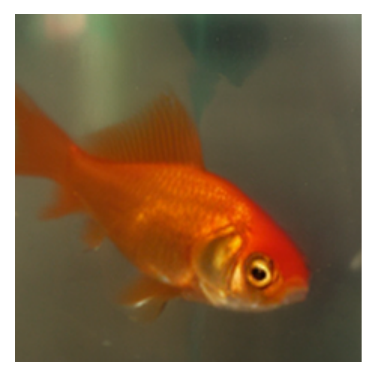}
& \includegraphics[width=0.30\linewidth]{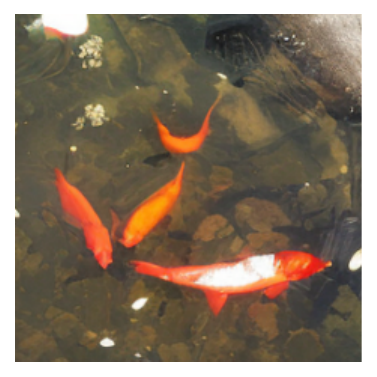} \\[-2pt]
\includegraphics[width=0.30\linewidth]{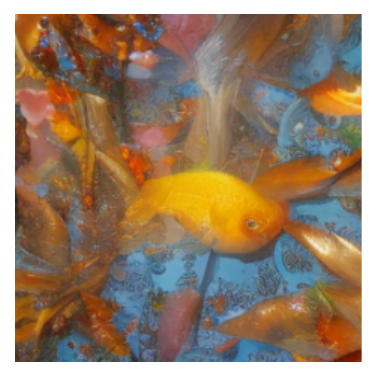}
& \includegraphics[width=0.30\linewidth]{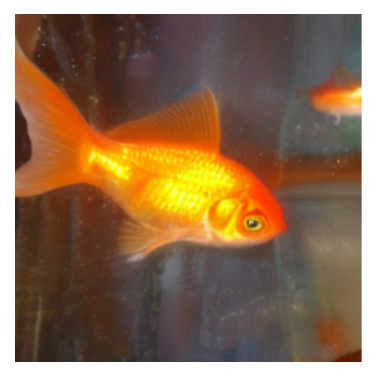}
& \includegraphics[width=0.30\linewidth]{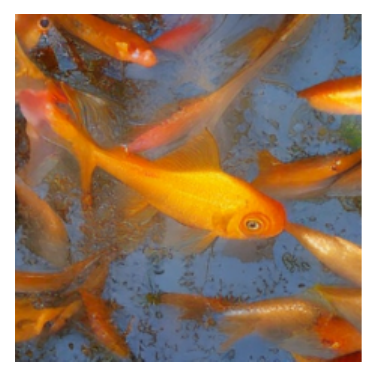} \\[-2pt]
\includegraphics[width=0.30\linewidth]{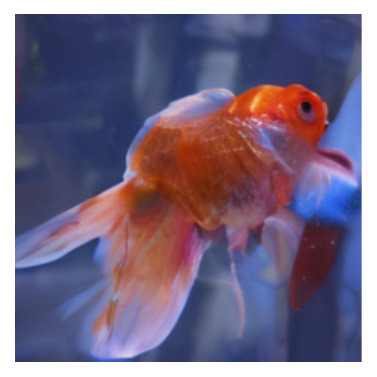}
& \includegraphics[width=0.30\linewidth]{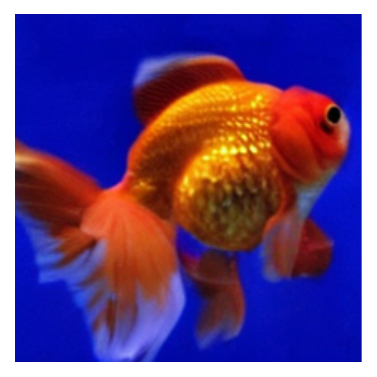}
& \includegraphics[width=0.30\linewidth]{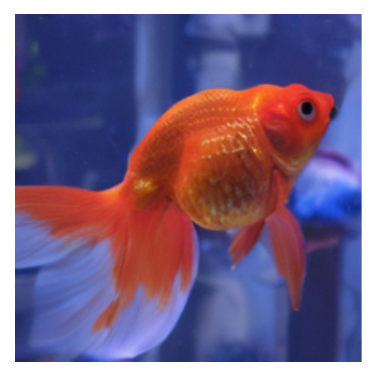} \\[-2pt]
\includegraphics[width=0.30\linewidth]{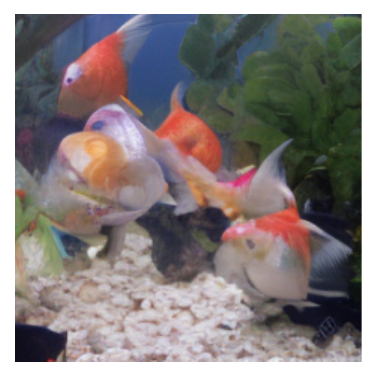}
& \includegraphics[width=0.30\linewidth]{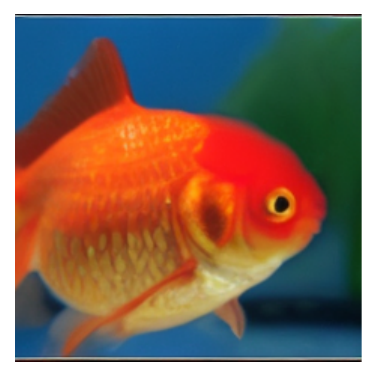}
& \includegraphics[width=0.30\linewidth]{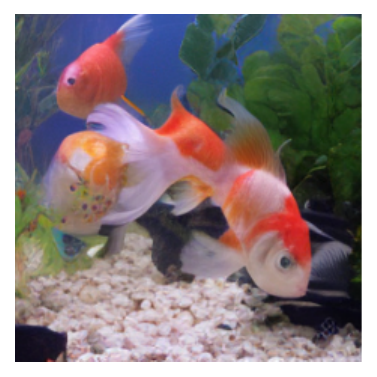}
\end{tabular}
\vspace{-0.35em}

\textbf{(a) ImageNet-256}
\end{minipage}\hfill
\begin{minipage}[t]{0.485\linewidth}
\centering
\setlength{\tabcolsep}{0.6pt}
\begin{tabular}{@{}ccc@{}}
\shortstack[c]{\tiny No guidance\\[-1pt]\tiny $\lambda=1$}
& \shortstack[c]{\tiny Fixed CFG\\[-1pt]\tiny $\lambda=5$}
& \shortstack[c]{\tiny PMC-CFG\\[-1pt]\tiny $\lambda=5$} \\
\includegraphics[width=0.29\linewidth]{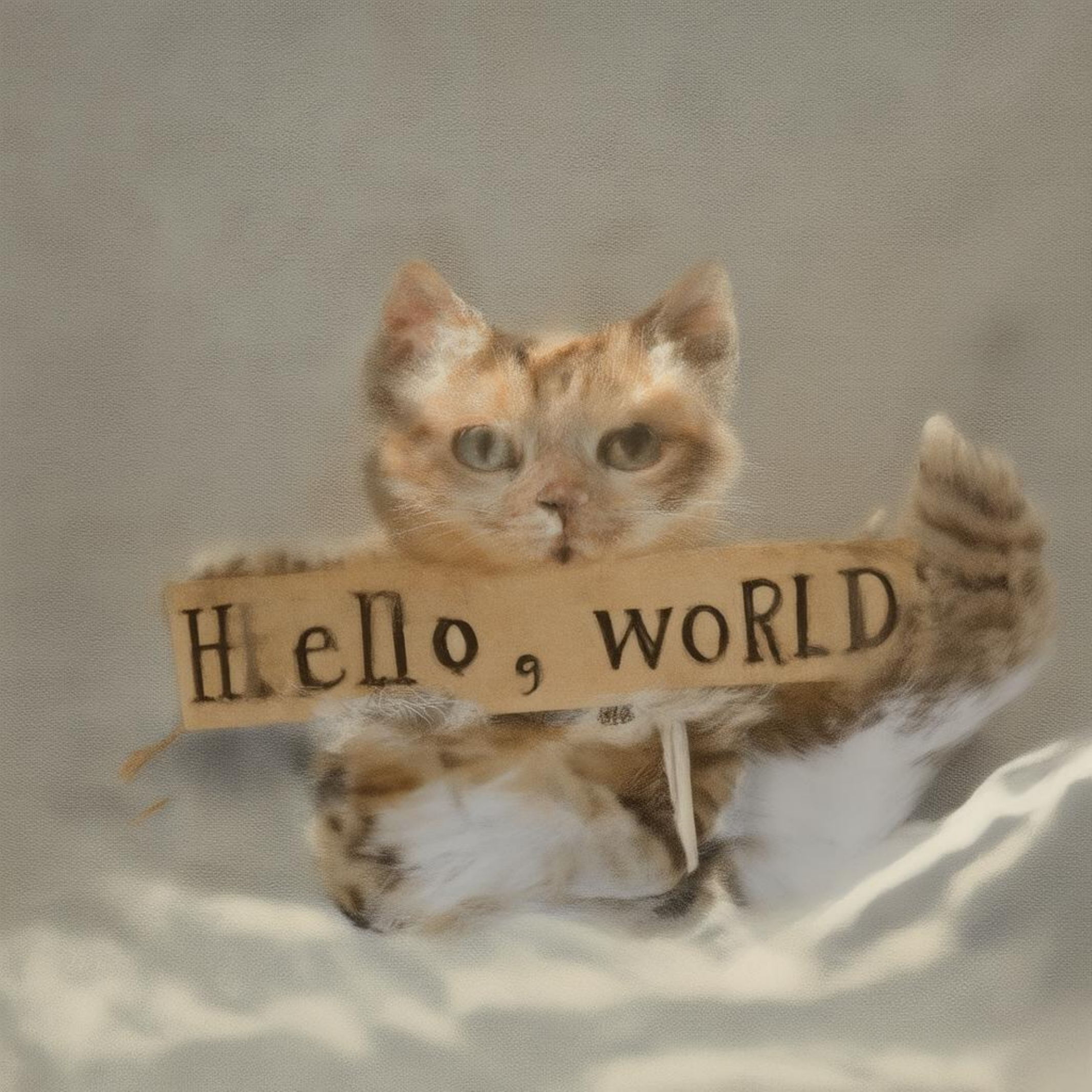}
& \includegraphics[width=0.29\linewidth]{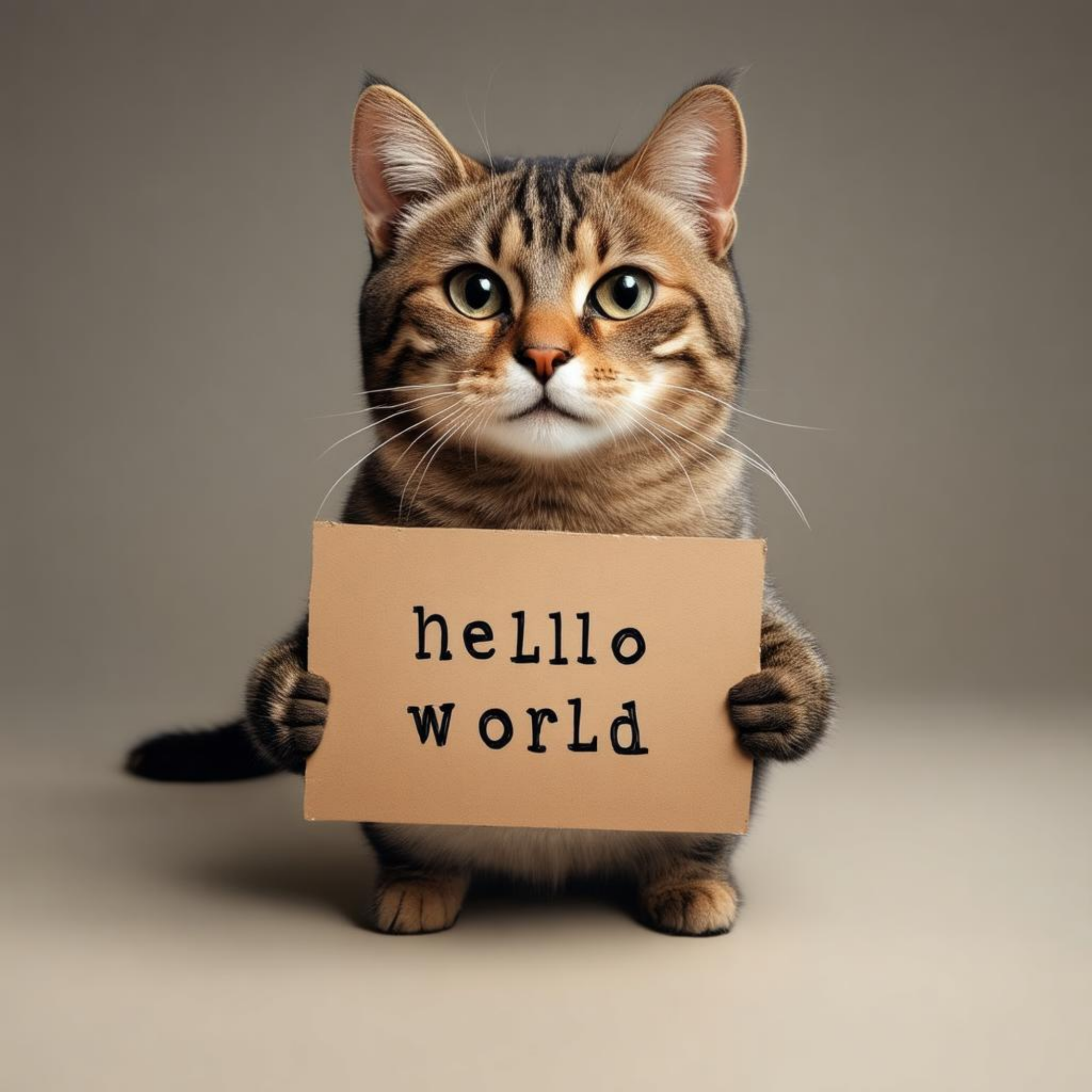}
& \includegraphics[width=0.29\linewidth]{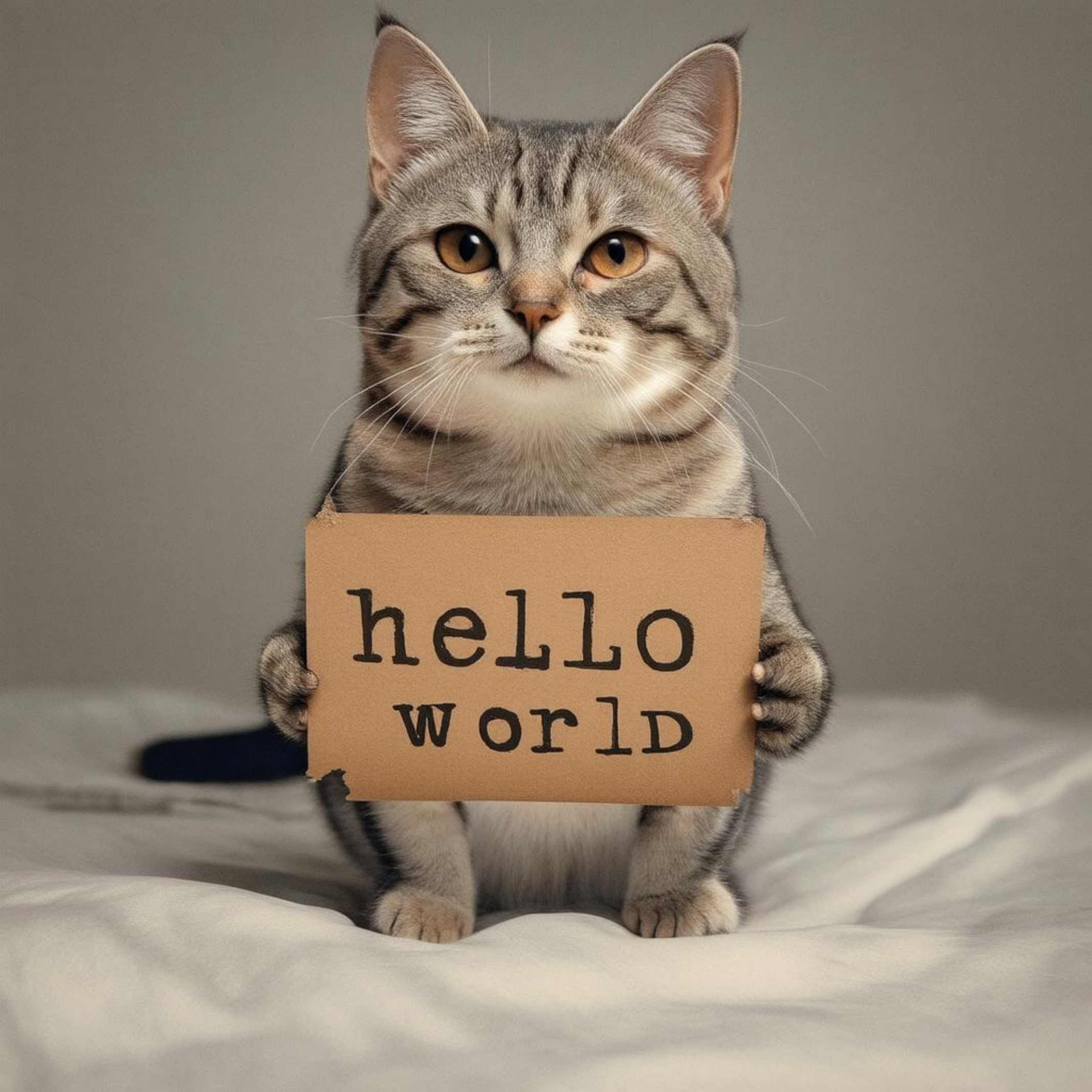} \\[-2pt]
\includegraphics[width=0.29\linewidth]{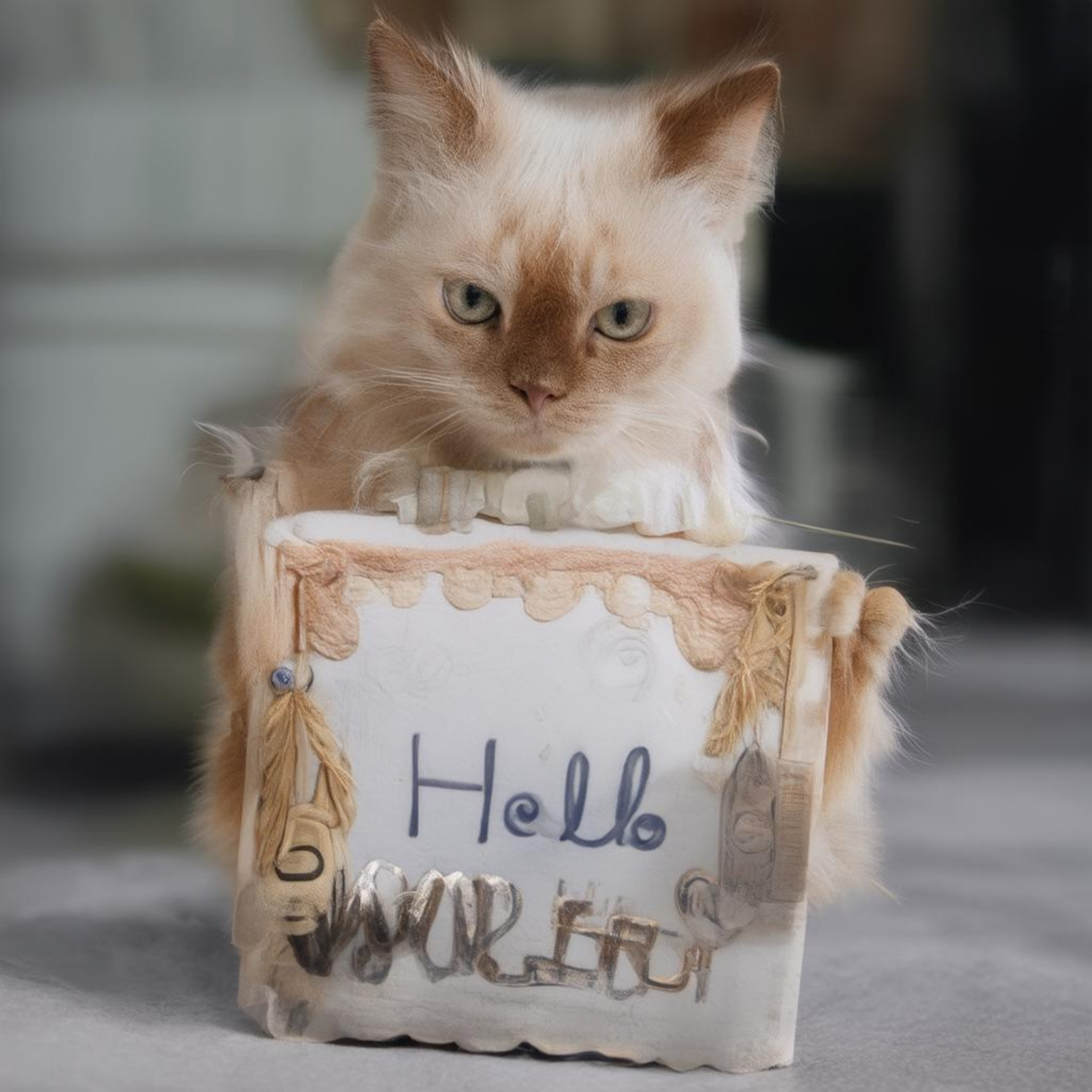}
& \includegraphics[width=0.29\linewidth]{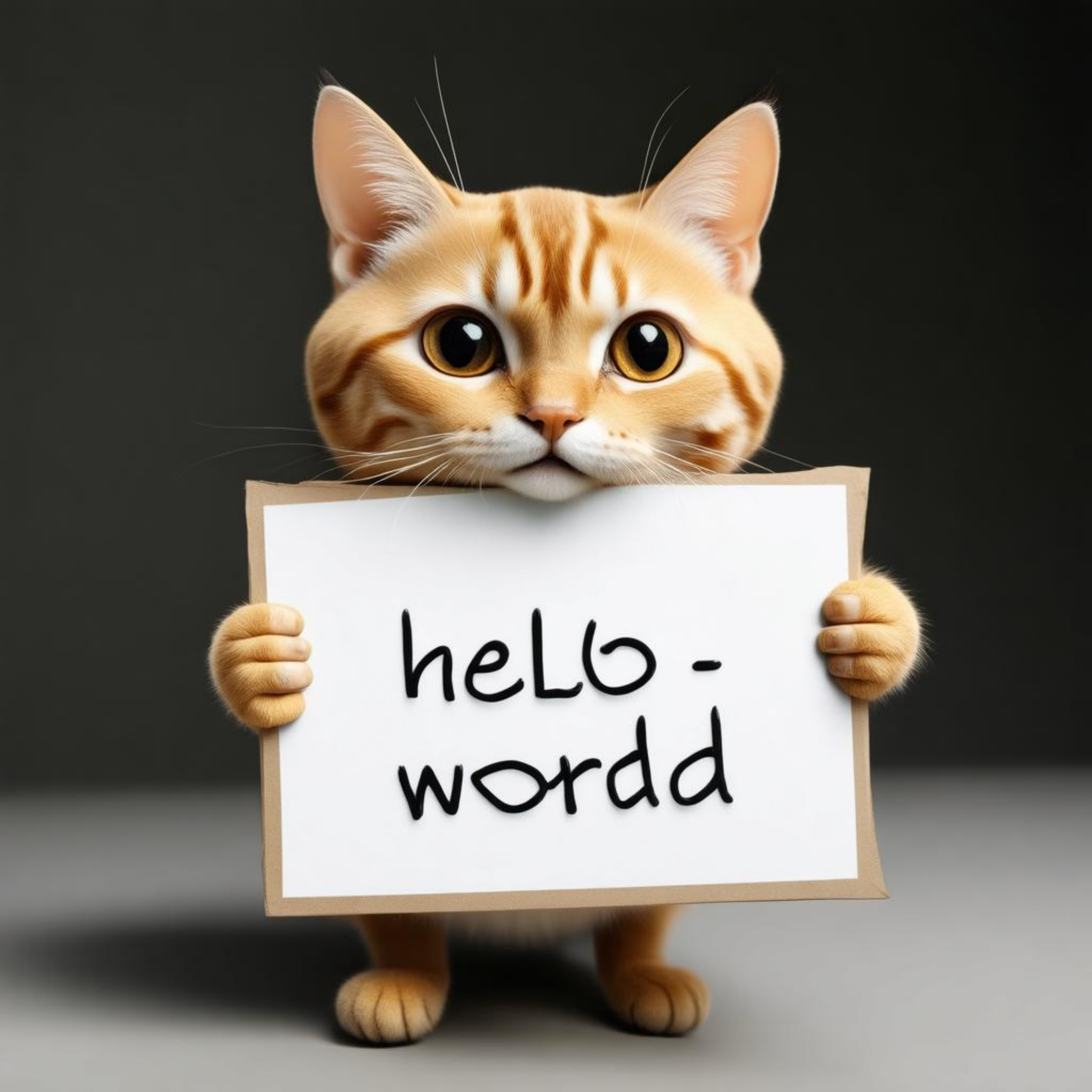}
& \includegraphics[width=0.29\linewidth]{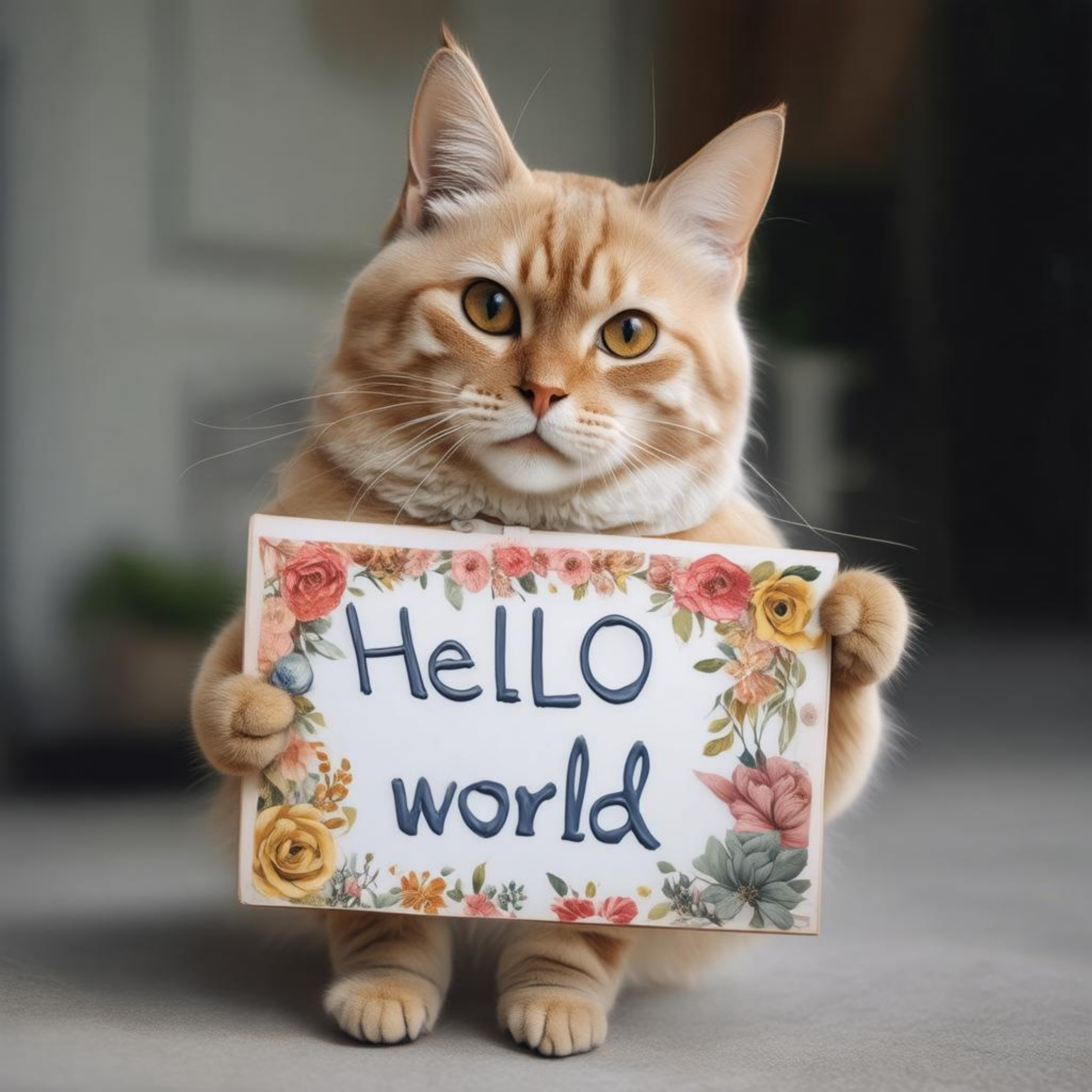} \\[-2pt]
\includegraphics[width=0.29\linewidth]{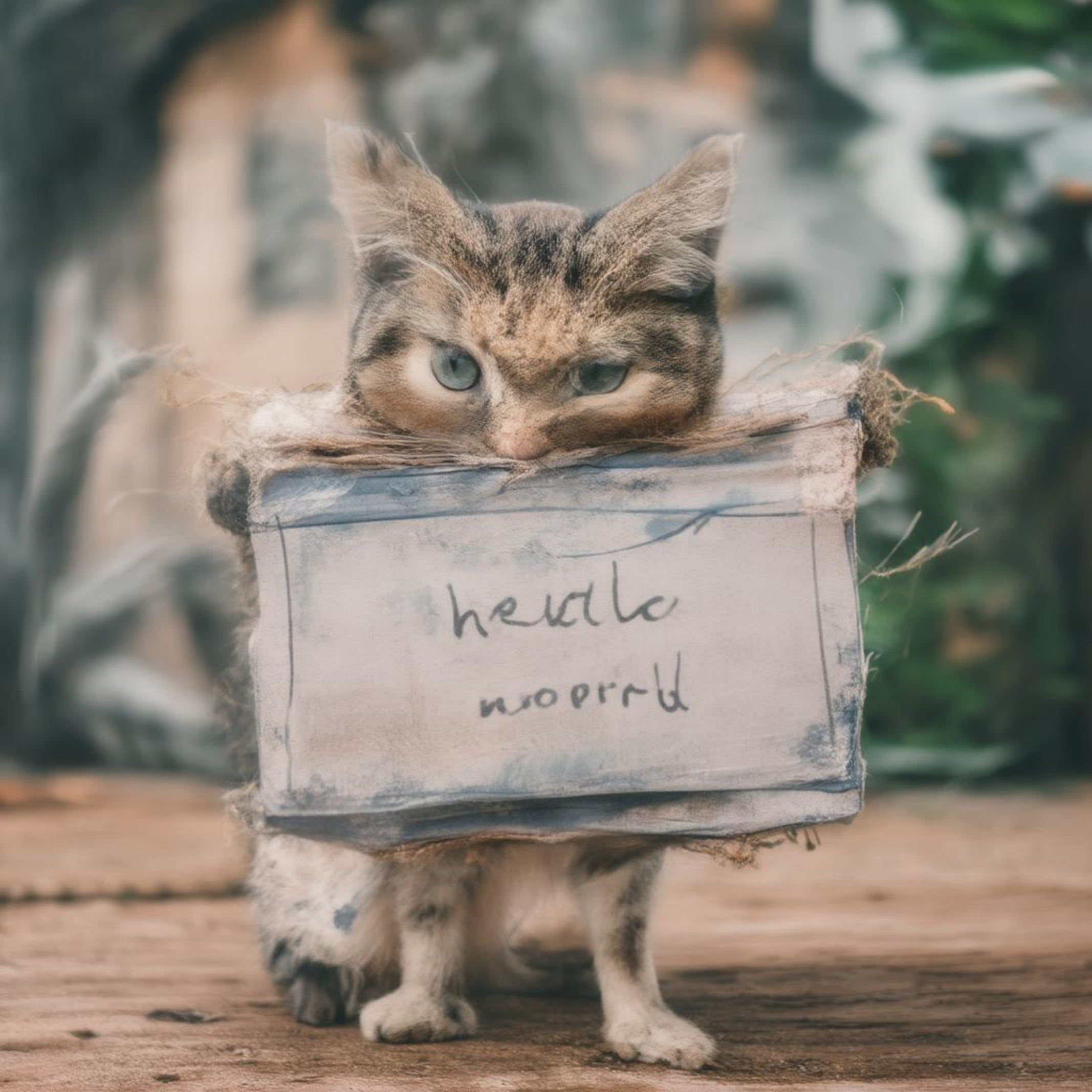}
& \includegraphics[width=0.29\linewidth]{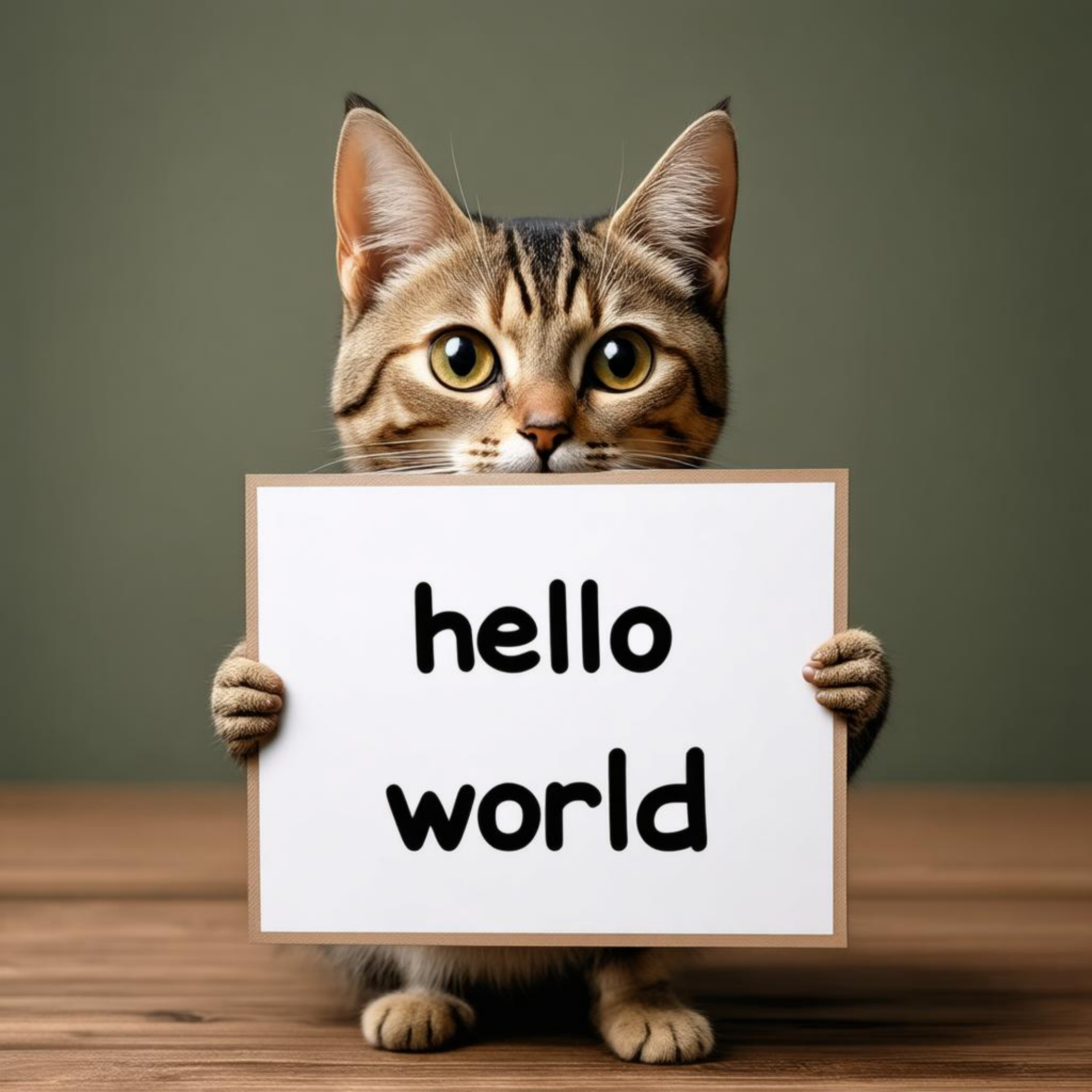}
& \includegraphics[width=0.29\linewidth]{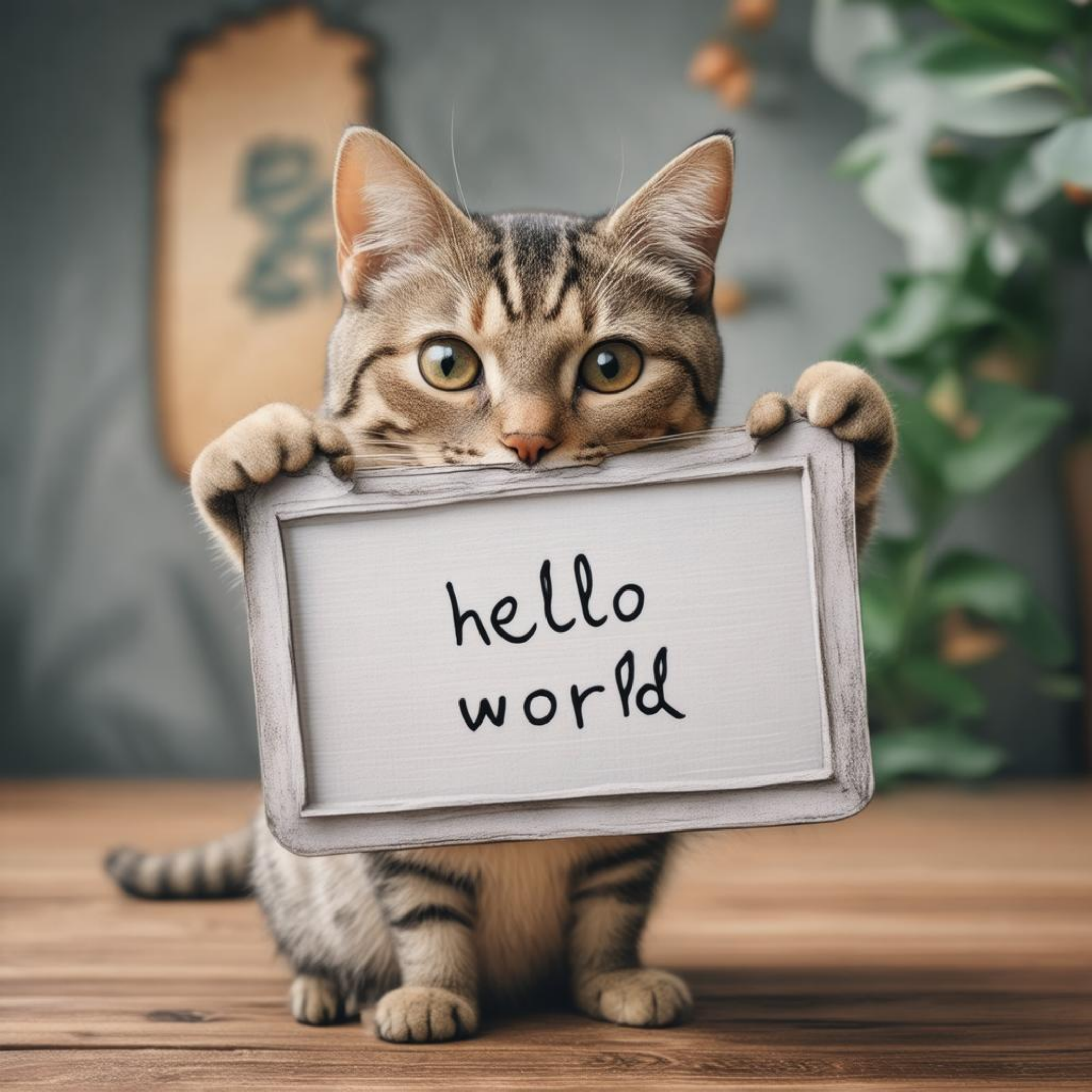} \\[-2pt]
\includegraphics[width=0.29\linewidth]{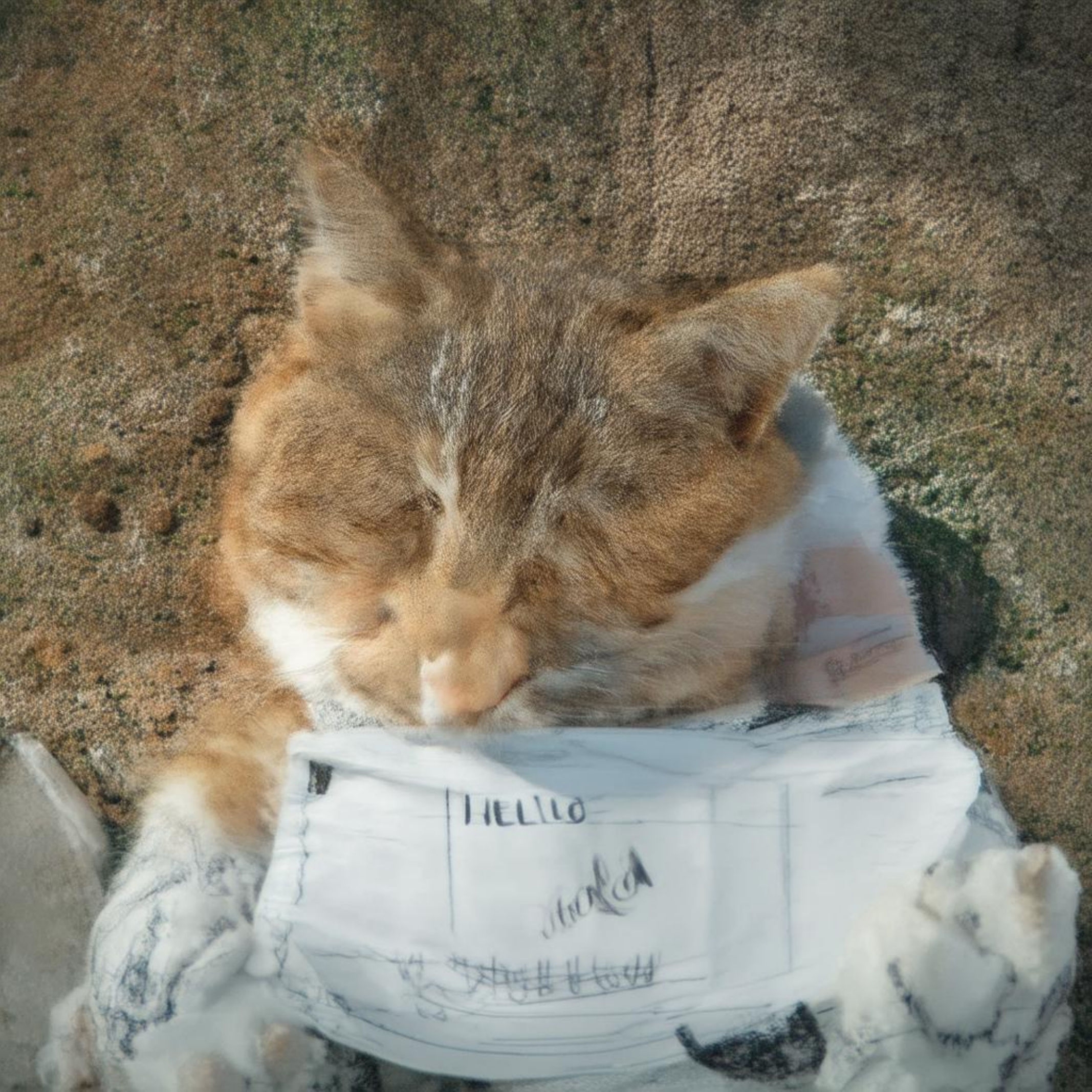}
& \includegraphics[width=0.29\linewidth]{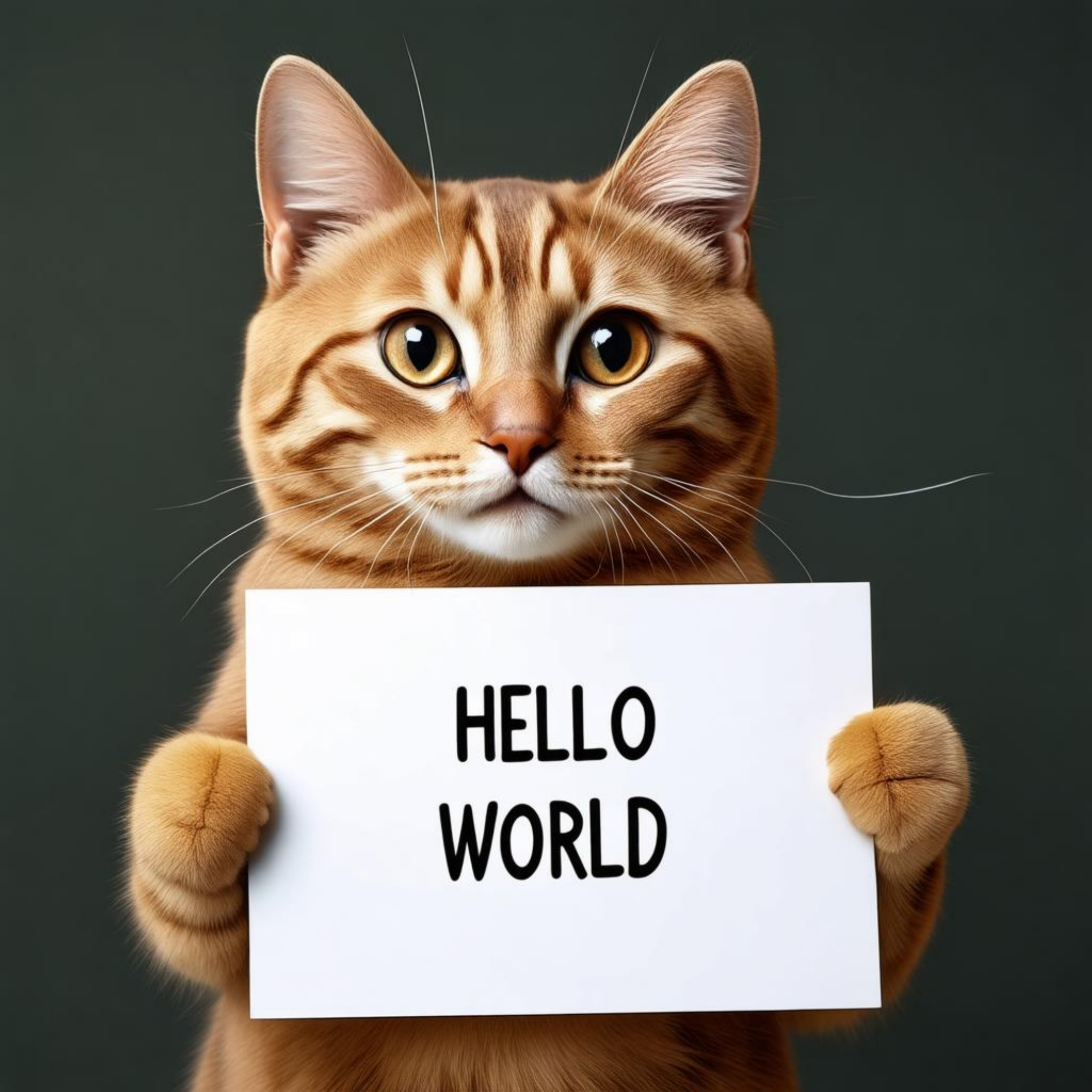}
& \includegraphics[width=0.29\linewidth]{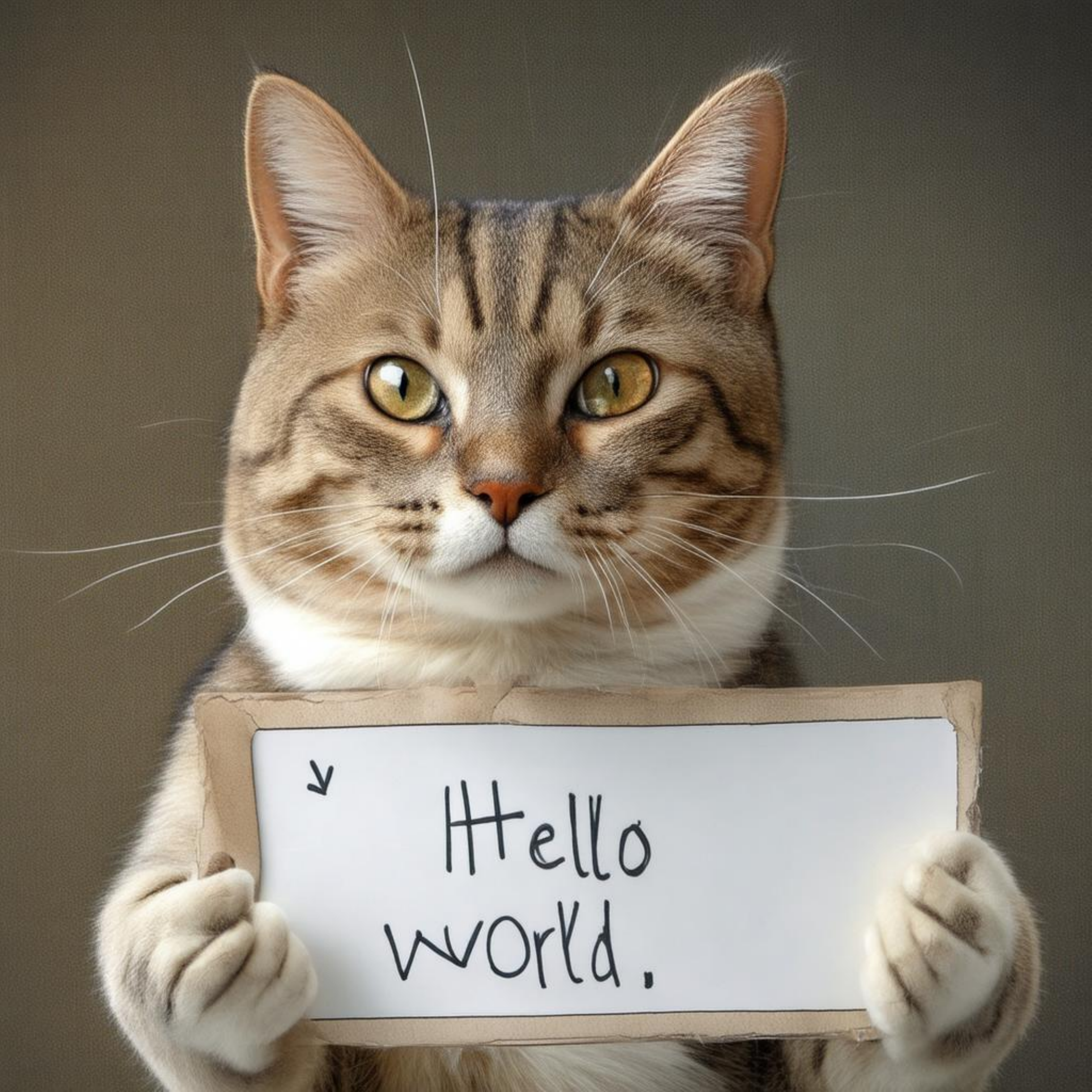} \\[-2pt]
\end{tabular}
\vspace{-0.35em}

\textbf{(b) Stable Diffusion 3.5}
\end{minipage}
\vspace{-0.35em}
\caption{Matched qualitative comparisons. Each row uses the same prompt or
class and seed. PMC-CFG better retains the unguided composition than fixed CFG
on both ImageNet-256 and SD3.5.}
\label{fig:qualitative-comparison}
\label{fig:imagenet-qualitative}
\label{fig:sd35-qualitative}
\vspace{-0.6em}
\end{figure}

\FloatBarrier
\section{Discussion and Conclusion}
\label{sec:conclusion}

We characterized linear Flow Matching as a non-autonomous gradient flow and
showed that CFG tilts its time-varying potential by the condition posterior.
The resulting posterior force explains both sides of CFG: it increases the
instantaneous conditional-alignment rate by a nonnegative squared-gradient term,
yet a fixed gain can accumulate mean displacement and local volume contraction,
causing overshoot and diversity loss. Motivated by this shared mechanism, we
introduced PMC-CFG, a training-free, per-sample feedback controller that
preserves the useful CFG direction but caps extrapolation of the model-implied
terminal posterior mean. It retains the largest feasible scale, leaves nominal
CFG unchanged when feasible, and requires no additional model evaluations.

Experiments on an analytic GMM, ImageNet-256, and SD3.5 support the predicted
mechanisms and show that PMC-CFG reduces overshoot and concentration, improves
the FID--recall trade-off, and prevents the quality and saturation degradation
 of strong fixed guidance while preserving competitive alignment. The guarantee
 is pointwise in the model-implied mean rather than a global bound on the final
 sample or a claim that the controlled flow exactly samples a prescribed tilted
 distribution; moreover, its norm depends on the latent coordinate system.
 Future work could
characterize image-distribution geometry---including anisotropy, local density,
curvature, and multimodal structure---and use these properties to design
distribution-aware caps and time-dependent guidance schedules with stronger
trajectory-level guarantees.

\clearpage
\bibliographystyle{unsrtnat}
\bibliography{pmc_cfg_conference}

\begin{thebibliography}{33}
\providecommand{\natexlab}[1]{#1}
\providecommand{\url}[1]{\texttt{#1}}
\expandafter\ifx\csname urlstyle\endcsname\relax
  \providecommand{\doi}[1]{doi: #1}\else
  \providecommand{\doi}{doi: \begingroup \urlstyle{rm}\Url}\fi

\bibitem[Ho et~al.(2020)Ho, Jain, and Abbeel]{ho2020ddpm}
Jonathan Ho, Ajay Jain, and Pieter Abbeel.
\newblock Denoising diffusion probabilistic models.
\newblock In \emph{Advances in Neural Information Processing Systems}, volume~33, pages 6840--6851, 2020.

\bibitem[Lipman et~al.(2023)Lipman, Chen, Ben-Hamu, Nickel, and Le]{lipman2023flow}
Yaron Lipman, Ricky T.~Q. Chen, Heli Ben-Hamu, Maximilian Nickel, and Matt Le.
\newblock Flow matching for generative modeling.
\newblock In \emph{International Conference on Learning Representations}, 2023.

\bibitem[Albergo and Vanden-Eijnden(2022)]{albergo2022building}
Michael~S Albergo and Eric Vanden-Eijnden.
\newblock Building normalizing flows with stochastic interpolants.
\newblock \emph{arXiv preprint arXiv:2209.15571}, 2022.

\bibitem[Liu et~al.(2023)Liu, Gong, and Liu]{liu2023flow}
Xingchao Liu, Chengyue Gong, and Qiang Liu.
\newblock Flow straight and fast: Learning to generate and transfer data with rectified flow.
\newblock In \emph{International Conference on Learning Representations}, 2023.

\bibitem[Dhariwal and Nichol(2021)]{dhariwal2021diffusion}
Prafulla Dhariwal and Alexander Nichol.
\newblock Diffusion models beat gans on image synthesis.
\newblock \emph{Advances in neural information processing systems}, 34:\penalty0 8780--8794, 2021.

\bibitem[Ho and Salimans(2022)]{ho2022classifierfree}
Jonathan Ho and Tim Salimans.
\newblock Classifier-free diffusion guidance.
\newblock \emph{arXiv preprint arXiv:2207.12598}, 2022.

\bibitem[Zheng et~al.(2023)Zheng, Le, Shaul, Lipman, Grover, and Chen]{zheng2023guidedflows}
Qinqing Zheng, Matt Le, Neta Shaul, Yaron Lipman, Aditya Grover, and Ricky T.~Q. Chen.
\newblock Guided flows for generative modeling and decision making.
\newblock \emph{arXiv preprint arXiv:2311.13443}, 2023.
\newblock \doi{10.48550/arXiv.2311.13443}.

\bibitem[Karras et~al.(2024)Karras, Aittala, Kynk{\"a}{\"a}nniemi, Lehtinen, Aila, and Laine]{karras2024badversion}
Tero Karras, Miika Aittala, Tero Kynk{\"a}{\"a}nniemi, Jaakko Lehtinen, Timo Aila, and Samuli Laine.
\newblock Guiding a diffusion model with a bad version of itself.
\newblock In \emph{Advances in Neural Information Processing Systems}, 2024.

\bibitem[Pavasovic et~al.(2026)Pavasovic, Verbeek, Biroli, and M{\'e}zard]{pavasovic2026overshoot}
Krunoslav~Lehman Pavasovic, Jakob Verbeek, Giulio Biroli, and Marc M{\'e}zard.
\newblock Overshoot and shrinkage in classifier-free guidance: From theory to practice.
\newblock In \emph{International Conference on Learning Representations}, 2026.

\bibitem[WANG et~al.(2024)WANG, Dufour, Andreou, CANI, Abrevaya, Picard, and Kalogeiton]{wang2024schedulers}
Xi~WANG, Nicolas Dufour, Nefeli Andreou, Marie-Paule CANI, Victoria~Fernandez Abrevaya, David Picard, and Vicky Kalogeiton.
\newblock Analysis of classifier-free guidance weight schedulers.
\newblock \emph{Transactions on Machine Learning Research}, 2024.
\newblock ISSN 2835-8856.
\newblock URL \url{https://openreview.net/forum?id=SUMtDJqicd}.

\bibitem[Gao et~al.(2026)Gao, Zheng, Zou, Yang, Liu, Fan, Zhang, Zhang, Chen, Jiang, Li, and Wang]{gao2026c2fg}
Jiayang Gao, Tianyi Zheng, Jiayang Zou, Fengxiang Yang, Shice Liu, Luyao Fan, Zheyu Zhang, Hao Zhang, Jinwei Chen, Peng-Tao Jiang, Bo~Li, and Jia Wang.
\newblock C{\textasciicircum}2fg: Control classifier-free guidance via score discrepancy analysis.
\newblock In \emph{Proceedings of the IEEE/CVF Conference on Computer Vision and Pattern Recognition (CVPR)}, pages 34398--34407, June 2026.

\bibitem[Jiang and Ma(2026)]{jiang2026analytic}
Enze Jiang and Zheng Ma.
\newblock Analytic distribution of classifier-free guidance for schedule design.
\newblock \emph{arXiv preprint arXiv:2607.19725}, 2026.

\bibitem[Kynk{\"a}{\"a}nniemi et~al.(2024)Kynk{\"a}{\"a}nniemi, Aittala, Karras, Laine, Aila, and Lehtinen]{kynkaanniemi2024interval}
Tuomas Kynk{\"a}{\"a}nniemi, Miika Aittala, Tero Karras, Samuli Laine, Timo Aila, and Jaakko Lehtinen.
\newblock Applying guidance in a limited interval improves sample and distribution quality in diffusion models.
\newblock In \emph{Advances in Neural Information Processing Systems}, volume~37, pages 122458--122483, 2024.
\newblock \doi{10.52202/079017-3892}.

\bibitem[Chung et~al.(2025)Chung, Kim, Park, Nam, and Ye]{chung2025cfgpp}
Hyungjin Chung, Jeongsol Kim, Geon~Yeong Park, Hyelin Nam, and Jong~Chul Ye.
\newblock {CFG++}: Manifold-constrained classifier free guidance for diffusion models.
\newblock In \emph{International Conference on Learning Representations}, 2025.

\bibitem[Xia et~al.(2025)Xia, Xue, Shen, Yi, Gong, and Liu]{xia2025rectified}
Mengfei Xia, Nan Xue, Yujun Shen, Ran Yi, Tieliang Gong, and Yong-Jin Liu.
\newblock Rectified diffusion guidance for conditional generation.
\newblock In \emph{2025 IEEE/CVF Conference on Computer Vision and Pattern Recognition (CVPR)}, pages 13371--13380. IEEE, 2025.

\bibitem[Sadat et~al.(2025)Sadat, Hilliges, and Weber]{sadat2025apg}
Seyedmorteza Sadat, Otmar Hilliges, and Romann~M. Weber.
\newblock Eliminating oversaturation and artifacts of high guidance scales in diffusion models.
\newblock In \emph{International Conference on Learning Representations}, 2025.
\newblock URL \url{https://openreview.net/forum?id=e2ONKX6qzJ}.

\bibitem[Wang et~al.(2026)Wang, Liu, Chi, Liu, Xue, and Duan]{wang2026cfgctrl}
Hanyang Wang, Yiyang Liu, Jiawei Chi, Fangfu Liu, Ran Xue, and Yueqi Duan.
\newblock Cfg-ctrl: Control-based classifier-free diffusion guidance.
\newblock In \emph{Proceedings of the IEEE/CVF Conference on Computer Vision and Pattern Recognition (CVPR)}, pages 11437--11447, June 2026.

\bibitem[Bradley and Nakkiran(2024)]{bradley2024classifier}
Arwen Bradley and Preetum Nakkiran.
\newblock Classifier-free guidance is a predictor-corrector.
\newblock \emph{arXiv preprint arXiv:2408.09000}, 2024.

\bibitem[Azangulov et~al.(2026)Azangulov, Potaptchik, Li, Aamari, Deligiannidis, and Rousseau]{azangulov2026adaptive}
Iskander Azangulov, Peter Potaptchik, Qinyu Li, Eddie Aamari, George Deligiannidis, and Judith Rousseau.
\newblock Adaptive diffusion guidance via stochastic optimal control.
\newblock In \emph{Proceedings of the 29th International Conference on Artificial Intelligence and Statistics}, volume 300 of \emph{Proceedings of Machine Learning Research}, pages 4087--4095, 2026.
\newblock URL \url{https://proceedings.mlr.press/v300/azangulov26a.html}.

\bibitem[Cai et~al.(2026)Cai, Liu, Su, and Wang]{cai2026manifold}
Jian-Feng Cai, Haixia Liu, Zhengyi Su, and Chao Wang.
\newblock Improving classifier-free guidance of flow matching via manifold projection.
\newblock \emph{arXiv preprint arXiv:2601.21892}, 2026.
\newblock \doi{10.48550/arXiv.2601.21892}.

\bibitem[Efron(2011)]{efron2011tweedie}
Bradley Efron.
\newblock Tweedie's formula and selection bias.
\newblock \emph{Journal of the American Statistical Association}, 106\penalty0 (496):\penalty0 1602--1614, 2011.
\newblock \doi{10.1198/jasa.2011.tm11181}.

\bibitem[Deng et~al.(2009)Deng, Dong, Socher, Li, Li, and Fei-Fei]{deng2009imagenet}
Jia Deng, Wei Dong, Richard Socher, Li-Jia Li, Kai Li, and Li~Fei-Fei.
\newblock Imagenet: A large-scale hierarchical image database.
\newblock In \emph{2009 IEEE conference on computer vision and pattern recognition}, pages 248--255. Ieee, 2009.

\bibitem[Ma et~al.(2024)Ma, Goldstein, Albergo, Boffi, Vanden-Eijnden, and Xie]{ma2024sit}
Nanye Ma, Mark Goldstein, Michael~S. Albergo, Nicholas~M. Boffi, Eric Vanden-Eijnden, and Saining Xie.
\newblock {SiT}: Exploring flow and diffusion-based generative models with scalable interpolant transformers.
\newblock In \emph{European Conference on Computer Vision}, pages 23--40, 2024.
\newblock \doi{10.1007/978-3-031-72980-5_2}.

\bibitem[Heusel et~al.(2017)Heusel, Ramsauer, Unterthiner, Nessler, and Hochreiter]{heusel2017ttur}
Martin Heusel, Hubert Ramsauer, Thomas Unterthiner, Bernhard Nessler, and Sepp Hochreiter.
\newblock {GANs} trained by a two time-scale update rule converge to a local {Nash} equilibrium.
\newblock In \emph{Advances in Neural Information Processing Systems}, volume~30, pages 6626--6637, 2017.

\bibitem[Salimans et~al.(2016)Salimans, Goodfellow, Zaremba, Cheung, Radford, and Chen]{salimans2016improved}
Tim Salimans, Ian Goodfellow, Wojciech Zaremba, Vicki Cheung, Alec Radford, and Xi~Chen.
\newblock Improved techniques for training {GANs}.
\newblock In \emph{Advances in Neural Information Processing Systems}, volume~29, 2016.

\bibitem[Kynk{\"a}{\"a}nniemi et~al.(2019)Kynk{\"a}{\"a}nniemi, Karras, Laine, Lehtinen, and Aila]{kynkaanniemi2019precisionrecall}
Tuomas Kynk{\"a}{\"a}nniemi, Tero Karras, Samuli Laine, Jaakko Lehtinen, and Timo Aila.
\newblock Improved precision and recall metric for assessing generative models.
\newblock In \emph{Advances in Neural Information Processing Systems}, volume~32, 2019.

\bibitem[Esser et~al.(2024)Esser, Kulal, Blattmann, Entezari, M{\"u}ller, Saini, Levi, Lorenz, Sauer, Boesel, et~al.]{esser2024sd3}
Patrick Esser, Sumith Kulal, Andreas Blattmann, Rahim Entezari, Jonas M{\"u}ller, Harry Saini, Yam Levi, Dominik Lorenz, Axel Sauer, Frederic Boesel, et~al.
\newblock Scaling rectified flow transformers for high-resolution image synthesis.
\newblock In \emph{Forty-first international conference on machine learning}, 2024.

\bibitem[Lin et~al.(2014)Lin, Maire, Belongie, Hays, Perona, Ramanan, Doll{\'a}r, and Zitnick]{lin2014coco}
Tsung-Yi Lin, Michael Maire, Serge Belongie, James Hays, Pietro Perona, Deva Ramanan, Piotr Doll{\'a}r, and C.~Lawrence Zitnick.
\newblock Microsoft {COCO}: Common objects in context.
\newblock In \emph{European Conference on Computer Vision}, pages 740--755, 2014.

\bibitem[{Stability AI}(2024)]{stabilityai2024sd35}
{Stability AI}.
\newblock Stable diffusion 3.5 medium model card.
\newblock \url{https://huggingface.co/stabilityai/stable-diffusion-3.5-medium}, 2024.

\bibitem[Radford et~al.(2021)Radford, Kim, Hallacy, Ramesh, Goh, Agarwal, Sastry, Askell, Mishkin, Clark, Krueger, and Sutskever]{radford2021clip}
Alec Radford, Jong~Wook Kim, Chris Hallacy, Aditya Ramesh, Gabriel Goh, Sandhini Agarwal, Girish Sastry, Amanda Askell, Pamela Mishkin, Jack Clark, Gretchen Krueger, and Ilya Sutskever.
\newblock Learning transferable visual models from natural language supervision.
\newblock In \emph{International Conference on Machine Learning}, pages 8748--8763. PmLR, 2021.

\bibitem[Xu et~al.(2023)Xu, Liu, Wu, Tong, Li, Ding, Tang, and Dong]{xu2023imagereward}
Jiazheng Xu, Xiao Liu, Yuchen Wu, Yuxuan Tong, Qinkai Li, Ming Ding, Jie Tang, and Yuxiao Dong.
\newblock {ImageReward}: Learning and evaluating human preferences for text-to-image generation.
\newblock In \emph{Advances in Neural Information Processing Systems}, volume~36, 2023.
\newblock \doi{10.52202/075280-0700}.

\bibitem[Detlefsen et~al.(2022)Detlefsen, Borovec, Schock, Jha, Koker, Di~Liello, Stancl, Quan, Grechkin, and Falcon]{detlefsen2022torchmetrics}
Nicki~Skafte Detlefsen, Jiri Borovec, Justus Schock, Ananya~Harsh Jha, Teddy Koker, Luca Di~Liello, Daniel Stancl, Changsheng Quan, Maxim Grechkin, and William Falcon.
\newblock {TorchMetrics}: Measuring reproducibility in {PyTorch}.
\newblock \emph{Journal of Open Source Software}, 7\penalty0 (70):\penalty0 4101, 2022.
\newblock \doi{10.21105/joss.04101}.

\bibitem[Parmar et~al.(2022)Parmar, Zhang, and Zhu]{parmar2022cleanfid}
Gaurav Parmar, Richard Zhang, and Jun-Yan Zhu.
\newblock On aliased resizing and surprising subtleties in {GAN} evaluation.
\newblock In \emph{IEEE/CVF Conference on Computer Vision and Pattern Recognition}, pages 11400--11410, 2022.
\newblock \doi{10.1109/CVPR52688.2022.01112}.

\end{thebibliography}

\clearpage
\appendix

\section{Linear Flow-Matching Identities}
\label{app:linear-identities}

This appendix derives the linear-path identities used in
Section~\ref{sec:preliminaries}, Proposition~\ref{prop:posterior-force}, and the
potential representations in Section~\ref{sec:analysis} without invoking a
denoising approximation.

\begin{lemma}[Velocity and Tweedie identities]
\label{lem:linear-full}
For the path~\eqref{eq:path-main}, assume $p_t$ is positive and
differentiation under the integral is valid. Then, for $t\in(0,1)$,
\begin{align}
 v_t(x)&=\frac{\E[X_1\mid X_t=x]-x}{1-t},
 \label{eq:app-velocity}\\
 \nabla_x\log p_t(x)
 &=-\frac{x-t\E[X_1\mid X_t=x]}{(1-t)^2},
 \label{eq:app-tweedie}\\
 v_t(x)&=\frac{x}{t}+\frac{1-t}{t}\nabla_x\log p_t(x).
 \label{eq:app-score-velocity}
\end{align}
\end{lemma}

\begin{proof}
Fix a pair $(X_0,X_1)$ and differentiate its interpolation:
\[
 \dot X_t=\frac{\dd}{\dd t}\{(1-t)X_0+tX_1\}=X_1-X_0.
\]
For $t<1$, the path equation gives
\[
 X_0=\frac{X_t-tX_1}{1-t}.
\]
Substituting this expression into the paired velocity and collecting the
$X_1$ terms yields
\begin{align*}
 X_1-X_0
 &=X_1-\frac{X_t-tX_1}{1-t}\\
 &=\frac{(1-t)X_1-X_t+tX_1}{1-t}\\
 &=\frac{X_1-X_t}{1-t}.
\end{align*}
The exact marginal velocity is the conditional mean of the paired velocity.
Therefore, at $X_t=x$,
\begin{align*}
 v_t(x)
 &=\E[\dot X_t\mid X_t=x]\\
 &=\E\left[\frac{X_1-X_t}{1-t}\,\middle|\,X_t=x\right]\\
 &=\frac{\E[X_1\mid X_t=x]-x}{1-t},
\end{align*}
which proves~\eqref{eq:app-velocity}.

Conditioning instead on $X_1=y$ gives the Gaussian kernel
\[
 p_t(x\mid y)=\frac{1}{[2\pi(1-t)^2]^{d/2}}
 \exp\left[-\frac{\norm{x-ty}^2}{2(1-t)^2}\right].
\]
The kernel score is
\[
 \nabla_x\log p_t(x\mid y)=-\frac{x-ty}{(1-t)^2},
\]
and hence
\[
 \nabla_xp_t(x\mid y)
 =-p_t(x\mid y)\frac{x-ty}{(1-t)^2}.
\]
Since $p_t(x)=\int p_1(y)p_t(x\mid y)\dd y$, differentiation under the integral
and Bayes' rule give
\begin{align*}
 \nabla_xp_t(x)
 &=\int p_1(y)p_t(x\mid y)
   \left[-\frac{x-ty}{(1-t)^2}\right]\dd y\\
 &=-\frac{1}{(1-t)^2}
   \left[xp_t(x)-t\int yp_1(y)p_t(x\mid y)\dd y\right]\\
 &=-\frac{p_t(x)}{(1-t)^2}
   \left[x-t\int y\frac{p_1(y)p_t(x\mid y)}{p_t(x)}\dd y\right]\\
 &=-\frac{p_t(x)}{(1-t)^2}
   \left(x-t\E[X_1\mid X_t=x]\right).
\end{align*}
Dividing by $p_t(x)>0$ proves~\eqref{eq:app-tweedie}. Writing
$m_t(x)=\E[X_1\mid X_t=x]$, that identity is equivalent to
\begin{align*}
 (1-t)^2\nabla_x\log p_t(x)&=-x+tm_t(x),\\
 m_t(x)&=\frac{x}{t}+\frac{(1-t)^2}{t}\nabla_x\log p_t(x).
\end{align*}
Substitution into~\eqref{eq:app-velocity} gives
\begin{align*}
 v_t(x)
 &=\frac{1}{1-t}\left[
 \frac{x}{t}+\frac{(1-t)^2}{t}\nabla_x\log p_t(x)-x\right]\\
 &=\frac{1}{1-t}\left[
 \frac{1-t}{t}x+\frac{(1-t)^2}{t}\nabla_x\log p_t(x)\right]\\
 &=\frac{x}{t}+\frac{1-t}{t}\nabla_x\log p_t(x),
\end{align*}
which is~\eqref{eq:app-score-velocity}.
\end{proof}

\begin{proof}[Proof of Proposition~\ref{prop:posterior-force}]
Equation~\eqref{eq:velocity-main} follows from
Lemma~\ref{lem:linear-full}. Apply
\eqref{eq:app-score-velocity} separately to
$p_t^{\mathrm c}(x)=p_t(x\mid c)$ and
$p_t^{\mathrm u}(x)=p_t(x)$. This gives
\begin{align*}
 \vc(x)&=\frac{x}{t}+\frac{1-t}{t}\nabla_x\log\pc(x),\\
 \vu(x)&=\frac{x}{t}+\frac{1-t}{t}\nabla_x\log\pu(x).
\end{align*}
Subtracting cancels the common $x/t$ term:
\begin{align*}
 \vc(x)-\vu(x)
 &=\frac{1-t}{t}
 \left[\nabla_x\log\pc(x)-\nabla_x\log\pu(x)\right]\\
 &=\frac{1-t}{t}\nabla_x\log\frac{\pc(x)}{\pu(x)}.
\end{align*}
Bayes' rule gives
\[
 p_t(x\mid c)=\frac{p(c\mid X_t=x)p_t(x)}{p(c)},
 \qquad
 \frac{\pc(x)}{\pu(x)}=\frac{p(c\mid X_t=x)}{p(c)}.
\]
Because $p(c)$ is independent of $x$,
\[
 \nabla_x\log\frac{\pc(x)}{\pu(x)}
 =\nabla_x\log p(c\mid X_t=x),
\]
which proves~\eqref{eq:force-main}.

Finally, the posterior-mean identity in the preliminaries follows directly
from the CFG convention
$v_t^\gamma=\vu+\gamma(\vc-\vu)$. Then
\begin{align*}
 m_t^\gamma
 &=x+(1-t)[\vu+\gamma(\vc-\vu)]\\
 &=x+(1-t)\vu+\gamma(1-t)(\vc-\vu)\\
 &=m_t^{\mathrm u}+\gamma(m_t^{\mathrm c}-m_t^{\mathrm u})\\
 &=(1-\gamma)m_t^{\mathrm u}+\gamma m_t^{\mathrm c}\\
 &=m_t^{\mathrm c}+(\gamma-1)(m_t^{\mathrm c}-m_t^{\mathrm u}),
\end{align*}
proving~\eqref{eq:mean-main}.
\end{proof}

\paragraph{Posterior growth along a guided trajectory.}
Let $x_t$ solve $\dot x_t=v_t^\gamma(x_t)$. The chain rule gives
\[
 \frac{\dd}{\dd t}h_t(x_t)
 =\partial_th_t(x_t)+\nabla_xh_t(x_t)^\top v_t^\gamma(x_t).
\]
Using $v_t^\gamma=v_t^{\mathrm c}+(\gamma-1)
\bigl(v_t^{\mathrm c}-v_t^{\mathrm u}\bigr)$ and
then~\eqref{eq:force-main},
\begin{align*}
 \frac{\dd}{\dd t}h_t(x_t)
 &=\partial_th_t(x_t)+\nabla_xh_t(x_t)^\top v_t^{\mathrm c}(x_t)\\
 &\quad+(\gamma-1)\nabla_xh_t(x_t)^\top
   (v_t^{\mathrm c}(x_t)-v_t^{\mathrm u}(x_t))\\
 &=\partial_th_t(x_t)+\nabla_xh_t(x_t)^\top v_t^{\mathrm c}(x_t)\\
 &\quad+(\gamma-1)\frac{1-t}{t}
   \nabla_xh_t(x_t)^\top\nabla_xh_t(x_t)\\
 &=\partial_th_t(x_t)+\nabla_xh_t(x_t)^\top v_t^{\mathrm c}(x_t)
 +(\gamma-1)\frac{1-t}{t}\|\nabla_xh_t(x_t)\|^2.
\end{align*}
This is~\eqref{eq:posterior-growth-main}. Only the last term has a fixed sign
for $\gamma\geq1$; the first two terms remain unrestricted.

\begin{proof}[Proof of Proposition~\ref{prop:potential-main}]
Define
\[
 \Phi_t(x)=\frac{\norm{x}^2}{2t}+\frac{1-t}{t}\log p_t(x).
\]
Spatial differentiation gives
\begin{align*}
 \nabla_x\Phi_t(x)
 &=\nabla_x\frac{\norm{x}^2}{2t}
   +\frac{1-t}{t}\nabla_x\log p_t(x)\\
 &=\frac{x}{t}+\frac{1-t}{t}\nabla_x\log p_t(x)\\
 &=v_t(x)
\end{align*}
by~\eqref{eq:app-score-velocity}. With $U_t=-\Phi_t$,
$\nabla_xU_t=-\nabla_x\Phi_t=-v_t$, and therefore
$\dot x_t=v_t(x_t)=-\nabla_xU_t(x_t)$. The chain rule along the trajectory
yields
\[
\begin{aligned}
 \frac{\dd}{\dd t}U_t(x_t)
 &=\partial_tU_t(x_t)
   +\nabla_xU_t(x_t)^\top\dot x_t\\
 &=\partial_tU_t(x_t)-\norm{\nabla_xU_t(x_t)}^2,
\end{aligned}
\]
which proves~\eqref{eq:energy-balance-main}.
\end{proof}

\paragraph{CFG potential and instantaneous power tilting.}
Linearity of the spatial gradient gives
\[
 v_t^\gamma=(1-\gamma)\nabla\Phi_t^{\mathrm u}
 +\gamma\nabla\Phi_t^{\mathrm c}
 =\nabla[(1-\gamma)\Phi_t^{\mathrm u}+\gamma\Phi_t^{\mathrm c}].
\]
Bayes' rule implies
$\log\pc-\log\pu=h_t-\log p(c)$. To make the substitution explicit,
\begin{align*}
 \Phi_t^\gamma
 &=(1-\gamma)\Phi_t^{\mathrm u}+\gamma\Phi_t^{\mathrm c}\\
 &=\Phi_t^{\mathrm c}
   +(\gamma-1)(\Phi_t^{\mathrm c}-\Phi_t^{\mathrm u})\\
 &=\Phi_t^{\mathrm c}
   +(\gamma-1)\frac{1-t}{t}
     (\log\pc-\log\pu)\\
 &=\Phi_t^{\mathrm c}
   +(\gamma-1)\frac{1-t}{t}h_t
   -(\gamma-1)\frac{1-t}{t}\log p(c).
\end{align*}
The last term depends on $t$ and $c$ but not on $x$, so it is absorbed into
$K_t$:
\[
 K_t=-(\gamma-1)\frac{1-t}{t}\log p(c).
\]
Thus
\[
 \Phi_t^\gamma
 =\Phi_t^{\mathrm c}+(\gamma-1)\frac{1-t}{t}h_t+K_t,
\]
proving~\eqref{eq:cfg-potential-main}. Equivalently, the spatial score is that
of
\begin{align*}
 \widetilde p_{t,\gamma}(x)
 &\propto\exp\{(1-\gamma)\log\pu(x)+\gamma\log\pc(x)\}\\
 &=\pu(x)^{1-\gamma}\pc(x)^\gamma\\
 &\propto\pc(x)p(c\mid X_t=x)^{\gamma-1}.
\end{align*}
This equality is pointwise in $t$. It does not assert that the density
transported by the non-autonomous ODE equals a fixed power tilt at the
endpoint.

\section{Posterior Curvature and Volume Transport}
\label{app:curvature}

\begin{proof}[Proof of Proposition~\ref{prop:curvature-main}]
Let $q_t(y\mid x)$ denote the posterior density of $X_1=y$ given $X_t=x$.
Bayes' rule writes it as
\[
 q_t(y\mid x)=
 \frac{p_1(y)(2\pi(1-t)^2)^{-d/2}
 \exp[-\|x-ty\|^2/(2(1-t)^2)]}{p_t(x)}.
\]
Taking the spatial log-gradient gives
\begin{align*}
 \nabla_x\log q_t(y\mid x)
 &=-\frac{x-ty}{(1-t)^2}-\nabla_x\log p_t(x)\\
 &=-\frac{x-ty}{(1-t)^2}
   +\frac{x-tm_t(x)}{(1-t)^2}\\
 &=\frac{t}{(1-t)^2}(y-m_t(x)),
\end{align*}
where the second line uses~\eqref{eq:app-tweedie}. Thus
\begin{equation}
 \nabla_x\log q_t(y\mid x)
 =\frac{t}{(1-t)^2}\bigl(y-m_t(x)\bigr).
 \label{eq:app-posterior-score}
\end{equation}
For coordinates $i,j$, differentiate the posterior expectation under the
integral sign and use
$\partial_{x_j}q_t=q_t\partial_{x_j}\log q_t$:
\begin{align*}
 \frac{\partial m_{t,i}}{\partial x_j}
 &=\int y_i\frac{\partial q_t(y\mid x)}{\partial x_j}\dd y\\
 &=\frac{t}{(1-t)^2}
 \left(\E[X_{1,i}X_{1,j}\mid X_t=x]
 -m_{t,j}\E[X_{1,i}\mid X_t=x]\right)\\
 &=\frac{t}{(1-t)^2}
 \left(\E[X_{1,i}X_{1,j}\mid X_t=x]-m_{t,i}m_{t,j}\right)\\
 &=\frac{t}{(1-t)^2}\Cov(X_1\mid X_t=x)_{ij}.
\end{align*}
This proves~\eqref{eq:mean-jac-main}. Since $t$ is fixed under the spatial
derivative, differentiating $v_t=(m_t-x)/(1-t)$ gives
\[
 \nabla_xv_t
 =\frac{t}{(1-t)^3}\Cov(X_1\mid X_t=x)
 -\frac{1}{1-t}I_d.
\]
Writing the conditional and unconditional posterior covariances as
$\Sigma_t^{\mathrm c}(x)$ and $\Sigma_t^{\mathrm u}(x)$, respectively,
\begin{align*}
 \nabla_x(v_t^{\mathrm c}-v_t^{\mathrm u})
 &=\left[\frac{t}{(1-t)^3}\Sigma_t^{\mathrm c}
 -\frac1{1-t}I_d\right]
 -\left[\frac{t}{(1-t)^3}\Sigma_t^{\mathrm u}
 -\frac1{1-t}I_d\right]\\
 &=\frac{t}{(1-t)^3}
 (\Sigma_t^{\mathrm c}-\Sigma_t^{\mathrm u}).
\end{align*}
Because $v_t^\gamma=v_t^{\mathrm c}+
(\gamma-1)(v_t^{\mathrm c}-v_t^{\mathrm u})$, multiplication by
$\gamma-1$ proves the covariance expression~\eqref{eq:extra-jac-main}.
Alternatively, differentiating~\eqref{eq:force-main} in space gives
\begin{align*}
 H_t^\gamma-H_t^{\mathrm c}
 &=(\gamma-1)\nabla_x(v_t^{\mathrm c}-v_t^{\mathrm u})\\
 &=(\gamma-1)\frac{1-t}{t}\nabla_x^2h_t,
\end{align*}
which is~\eqref{eq:hessian-shift-main} and equals the covariance expression
above. Taking traces also gives the
corresponding extra divergence explicitly:
\[
 \nabla\!\cdot(v_t^\gamma-v_t^{\mathrm c})
 =(\gamma-1)\frac{t}{(1-t)^3}
 \left(\tr\Sigma_t^{\mathrm c}-\tr\Sigma_t^{\mathrm u}\right).
\]

Using the flow-map notation above, differentiate
$\partial_tX^\gamma(t,s,x_s)=v_t^\gamma(X^\gamma(t,s,x_s))$ with respect to
$x_s$. The chain rule gives
\[
 \dot J_t=(\nabla_xv_t^\gamma)(x_t)J_t,
 \qquad x_t=X^\gamma(t,s,x_s).
\]
For an invertible $J_t$, Jacobi's determinant formula and cyclicity of trace
yield
\[
 \frac{\dd}{\dd t}\log|\det J_t|
 =\tr(J_t^{-1}\dot J_t)
 =\tr\!\left(J_t^{-1}(\nabla_xv_t^\gamma)J_t\right)
 =\tr\!\left((\nabla_xv_t^\gamma)J_tJ_t^{-1}\right)
 =\tr(\nabla_xv_t^\gamma)
 =\nabla\!\cdot v_t^\gamma,
\]
where the vector field is evaluated at $x_t$. Integrating from $s$ to $t$ and
using $\log|\det J_s|=0$ gives
\[
 \log|\det J_t|=\int_s^t
 \nabla\!\cdot v_\tau^\gamma(x_\tau)\,\dd\tau,
\]
which proves~\eqref{eq:liouville-main}. Likewise, integrating
$\dot x_\tau=v_\tau^\gamma(x_\tau)$ gives the accumulated displacement
\[
 x_t-x_s=\int_s^t v_\tau^\gamma(x_\tau)\,\dd\tau.
\]
\end{proof}

\paragraph{Exact GMM Jacobian.}
Differentiating the responsibility softmax component by component gives
\[
 \nabla_x\pi_k
 =\pi_k\left(\frac{t}{C_t}\mu_k
 -\sum_j\pi_j\frac{t}{C_t}\mu_j\right)
 =\frac{t}{C_t}\pi_k(\mu_k-\bar\mu_t).
\]
Consequently, with
$\Sigma_\pi=\sum_k\pi_k(\mu_k-\bar\mu_t)(\mu_k-\bar\mu_t)^\top$,
\begin{align*}
 \nabla_x\bar\mu_t
 &=\sum_k\mu_k(\nabla_x\pi_k)^\top\\
 &=\frac{t}{C_t}
 \left(\sum_k\pi_k\mu_k\mu_k^\top
 -\bar\mu_t\bar\mu_t^\top\right)
 =\frac{t}{C_t}\Sigma_\pi.
\end{align*}
Now differentiate both terms of~\eqref{eq:app-gmm-velocity}:
\begin{align*}
 \nabla_xv_t
 &=\frac{1-t}{C_t}\nabla_x\bar\mu_t
 +\frac{t\sigma^2-(1-t)}{C_t}I_2\\
 &=\frac{t(1-t)}{C_t^2}\Sigma_\pi
 +\frac{t\sigma^2-(1-t)}{C_t}I_2.
\end{align*}
This makes the covariance mechanism in Proposition~\ref{prop:curvature-main}
fully explicit for the circular mixture.

\section{Circular-GMM Derivations}
\label{app:gmm}

Under component $k$, write $X_1=\mu_k+\sigma\varepsilon$ with
$\varepsilon\sim\mathcal N(0,I_2)$ independent of
$X_0\sim\mathcal N(0,I_2)$. The linear path becomes
\[
 X_t=(1-t)X_0+t\mu_k+t\sigma\varepsilon.
\]
Its conditional mean and covariance are
\begin{align*}
 \E[X_t\mid k]&=t\mu_k,\\
 \Cov(X_t\mid k)
 &=t^2\sigma^2I_2+(1-t)^2I_2
 =C_tI_2,
\end{align*}
where
$C_t=(1-t)^2+t^2\sigma^2$. Hence
\[
 X_t\mid k\sim\mathcal N(t\mu_k,C_tI_2).
\]
Bayes' rule gives the component responsibility
\[
 \pi_k(x,t)=p(k\mid X_t=x)
 =\frac{w_kp_t(x\mid k)}{\sum_jw_jp_t(x\mid j)}.
\]
The Gaussian normalizing factors are identical across components and cancel,
which yields
\begin{equation}
 \pi_k(x,t)=
 \frac{w_k\exp[-\norm{x-t\mu_k}^2/(2C_t)]}
 {\sum_jw_j\exp[-\norm{x-t\mu_j}^2/(2C_t)]}.
 \label{eq:app-responsibility}
\end{equation}
To obtain the softmax form, expand
\begin{align*}
 -\frac{\norm{x-t\mu_k}^2}{2C_t}
 &=-\frac{\norm{x}^2}{2C_t}
   +\frac{t}{C_t}x^\top\mu_k
   -\frac{t^2\norm{\mu_k}^2}{2C_t}.
\end{align*}
All centers satisfy $\norm{\mu_k}=r$, so the first and third terms do not
depend on $k$ and cancel between numerator and denominator. Therefore
\[
 \pi_k(x,t)=\softmax_k\left(
 \log w_k+\frac{t}{C_t}x^\top\mu_k\right).
\]

\subsection{Posterior means and exact velocity}
\label{app:gmm-gap-evolution}

Define
\[
 \alpha_t=\frac{(1-t)^2}{C_t},\qquad
 \rho_t=\frac{t\sigma^2}{C_t},\qquad 0\leq t<1.
\]
Within component $k$, the jointly Gaussian pair $(X_1,X_t)$ has moments
\begin{align*}
 \E[X_1\mid k]&=\mu_k,
 &\E[X_t\mid k]&=t\mu_k,\\
 \Cov(X_1\mid k)&=\sigma^2I_2,
 &\Cov(X_t\mid k)&=C_tI_2,\\
 \Cov(X_1,X_t\mid k)
 &=\Cov(X_1,(1-t)X_0+tX_1\mid k)
 =t\sigma^2I_2.
\end{align*}
The Gaussian conditioning formula therefore gives
\begin{align}
 m_{t,k}(x)
 &=\E[X_1\mid X_t=x,k]\notag\\
 &=\mu_k+t\sigma^2I_2(C_tI_2)^{-1}(x-t\mu_k)\notag\\
 &=\mu_k+\frac{t\sigma^2}{C_t}(x-t\mu_k)\notag\\
 &=\left(1-\frac{t^2\sigma^2}{C_t}\right)\mu_k
   +\frac{t\sigma^2}{C_t}x\notag\\
 &=\frac{C_t-t^2\sigma^2}{C_t}\mu_k
   +\frac{t\sigma^2}{C_t}x\notag\\
 &=\frac{(1-t)^2}{C_t}\mu_k+
   \frac{t\sigma^2}{C_t}x.
 \label{eq:app-component-mean}
\end{align}
For a nonempty component subset $A$, define
\begin{align}
 \pi_{k\mid A}(x,t)
 &=\frac{\pi_k(x,t)}{\sum_{j\in A}\pi_j(x,t)},\qquad k\in A,\notag\\
 \bar\mu_t(x)&=\sum_k\pi_k(x,t)\mu_k,\notag\\
 \bar\mu_A(x)&=\sum_{k\in A}\pi_{k\mid A}(x,t)\mu_k.
 \label{eq:app-conditional-responsibility}
\end{align}
Averaging~\eqref{eq:app-component-mean} over all components or over $A$ gives
\begin{align}
 m_t^{\mathrm u}(x)
 &=\frac{(1-t)^2}{C_t}\bar\mu_t(x)
   +\frac{t\sigma^2}{C_t}x,
 \label{eq:app-gmm-mean}\\
 m_t^{\mathrm c}(x)
 &=\frac{(1-t)^2}{C_t}\bar\mu_A(x)
   +\frac{t\sigma^2}{C_t}x.
 \label{eq:app-gmm-conditional-mean}
\end{align}
The shared state term cancels, yielding
\begin{equation}
 \Delta_t(x)=m_t^{\mathrm c}(x)-m_t^{\mathrm u}(x)
 =\alpha_t
 \bigl(\bar\mu_A(x)-\bar\mu_t(x)\bigr).
 \label{eq:app-gmm-gap}
\end{equation}

\begin{proposition}[GMM gap decay and cap release]
\label{prop:gmm-gap-app}
For every $x$ and $0\leq t<1$,
\begin{equation}
 \norm{\Delta_t(x)}
 \leq 2r\alpha_t
 =\frac{2r}{1+\sigma^2[t/(1-t)]^2}
 \xrightarrow[t\uparrow1]{}0,
 \label{eq:gmm-gap-envelope-app}
\end{equation}
and the envelope $2r\alpha_t$ decreases strictly on $(0,1)$. If
$x_t\to x_1$, then $m_t^{\mathrm u}(x_t),m_t^{\mathrm c}(x_t)\to x_1$.
The nominal PMC-CFG scale is feasible whenever
\begin{equation}
 2r(\lambda-1)\alpha_t
 \leq(\Gamma-1)\norm{m_t^{\mathrm c}(x)}.
 \label{eq:gmm-cap-release-app}
\end{equation}
Consequently, PMC-CFG recovers $\gamma_t^{\mathrm{eff}}=\lambda$ sufficiently
near every nonzero GMM endpoint.
\end{proposition}

\begin{proof}
Both weighted centers lie in the convex hull of $\{\mu_k\}$, so
$\norm{\bar\mu_t}\leq r$ and $\norm{\bar\mu_A}\leq r$. The triangle inequality
therefore gives
\[
 \norm{\Delta_t(x)}
 \leq\alpha_t\bigl(\norm{\bar\mu_A(x)}+\norm{\bar\mu_t(x)}\bigr)
 \leq2r\alpha_t.
\]
For $0\leq t<1$,
\[
 \frac{(1-t)^2}{C_t}
 =\left[1+\sigma^2\left(\frac{t}{1-t}\right)^2\right]^{-1},
\qquad
 \frac{\dd\alpha_t}{\dd t}
 =-\frac{2\sigma^2t(1-t)}{C_t^2}<0
 \quad(0<t<1),
\]
so $\alpha_t$ decreases from $1$ to $0$. Since
$\rho_t=t\sigma^2/C_t\to1$, equations~\eqref{eq:app-gmm-mean}
and~\eqref{eq:app-gmm-conditional-mean} also give
$m_t^{\mathrm u}(x_t),m_t^{\mathrm c}(x_t)\to x_1$ whenever $x_t\to x_1$.
Finally, under~\eqref{eq:gmm-cap-release-app},
\begin{align*}
 \norm{m_t^{\mathrm c}+(\lambda-1)\Delta_t}
 &\leq\norm{m_t^{\mathrm c}}+(\lambda-1)\norm{\Delta_t}\\
 &\leq\Gamma\norm{m_t^{\mathrm c}},
\end{align*}
so the nominal scale is feasible. Along a trajectory converging to
$x_1\ne0$, $\alpha_t\to0$ and $m_t^{\mathrm c}\to x_1$; therefore
\eqref{eq:gmm-cap-release-app} holds for all sufficiently large $t$.
\end{proof}

Writing \(\bar\mu\) for the weighted center of either field,
using~\eqref{eq:app-velocity} gives the corresponding exact velocity
\begin{align}
 v_t(x)
 &=\frac{m_t(x)-x}{1-t}\notag\\
 &=\frac{1-t}{C_t}\bar\mu(x)
   +\frac{t\sigma^2-C_t}{(1-t)C_t}x\notag\\
 &=\frac{1-t}{C_t}\bar\mu(x)
   +\frac{t\sigma^2-(1-t)}{C_t}x,
 \label{eq:app-gmm-velocity}
\end{align}
where the last equality uses
\begin{align*}
 t\sigma^2-C_t
 &=t\sigma^2-(1-t)^2-t^2\sigma^2\\
 &=(1-t)[t\sigma^2-(1-t)].
\end{align*}
For the singleton condition $A=\{0\}$, $\bar\mu_A=\mu_0$. At $t=0$,
$X_t=X_0$ is independent of $X_1$, and a centered mixture therefore satisfies
\begin{align*}
 m_0^{\mathrm u}(x)&=\E[X_1]=\sum_kw_k\mu_k=0,\\
 m_0^{\mathrm c}(x)&=\E[X_1\mid c=0]=\mu_0.
\end{align*}
Substitution into~\eqref{eq:mean-main} gives
\begin{align*}
 m_0^\gamma
 &=\mu_0+(\gamma-1)(\mu_0-0)
 =\gamma\mu_0,
\end{align*}
which proves~\eqref{eq:initial-displacement-main}.

\subsection{Force and curvature}

Using the responsibilities from~\eqref{eq:app-conditional-responsibility}, define
\[
 \Sigma_A=\sum_{k\in A}\pi_{k\mid A}
 (\mu_k-\bar\mu_A)(\mu_k-\bar\mu_A)^\top.
\]
We first differentiate the unconditional responsibilities. Writing
$\eta_k=\log w_k+(t/C_t)x^\top\mu_k$ and
$\pi_k=e^{\eta_k}/\sum_je^{\eta_j}$, the quotient rule gives
\begin{align*}
 \nabla_x\pi_k
 &=\pi_k\left(\nabla_x\eta_k-
       \sum_j\pi_j\nabla_x\eta_j\right)\\
 &=\frac{t}{C_t}\pi_k
       \left(\mu_k-\sum_j\pi_j\mu_j\right)\\
 &=\frac{t}{C_t}\pi_k(\mu_k-\bar\mu_t).
\end{align*}
Consequently,
\begin{align*}
 \nabla_x\bar\mu_t
 &=\sum_k\mu_k(\nabla_x\pi_k)^\top\\
 &=\frac{t}{C_t}\sum_k\pi_k
      \mu_k(\mu_k-\bar\mu_t)^\top\\
 &=\frac{t}{C_t}\left(
      \sum_k\pi_k\mu_k\mu_k^\top-
      \bar\mu_t\bar\mu_t^\top\right)
 =\frac{t}{C_t}\Sigma_\pi.
\end{align*}
The same calculation after renormalizing over $A$ gives
$\nabla_x\bar\mu_A=(t/C_t)\Sigma_A$.

For completeness, the condition posterior is a ratio of two mixture sums:
\[
 h_t^A(x)=\log p(A\mid x)
 =\log\sum_{k\in A}w_ke^{-\|x-t\mu_k\|^2/(2C_t)}
 -\log\sum_jw_je^{-\|x-t\mu_j\|^2/(2C_t)}.
\]
Differentiating the first log-sum gives
$-(x-t\bar\mu_A)/C_t$, and differentiating the second gives
$-(x-t\bar\mu_t)/C_t$. Their common $-x/C_t$ terms cancel. Therefore
\begin{align}
 \nabla_x h_t^A(x)
 &=-\frac{x-t\bar\mu_A}{C_t}
   +\frac{x-t\bar\mu_t}{C_t}
 =\frac{t}{C_t}(\bar\mu_A-\bar\mu_t),
 \label{eq:app-gmm-h-gradient}\\
 \nabla_x^2h_t^A(x)
 &=\frac{t}{C_t}
   (\nabla_x\bar\mu_A-\nabla_x\bar\mu_t)
 =\frac{t^2}{C_t^2}(\Sigma_A-\Sigma_\pi).
 \label{eq:app-gmm-h-hessian}
\end{align}
By~\eqref{eq:cfg-potential-main}, the extra CFG force is
\begin{align*}
 v_t^\gamma-v_t^{\mathrm c}
 &=(\gamma-1)\frac{1-t}{t}\nabla_xh_t^A
 =(\gamma-1)\frac{1-t}{C_t}
   (\bar\mu_A-\bar\mu_t),
\end{align*}
and one more derivative gives
\[
 \nabla_x(v_t^\gamma-v_t^{\mathrm c})
 =(\gamma-1)\frac{t(1-t)}{C_t^2}
 (\Sigma_A-\Sigma_\pi).
\]
This proves the general GMM force and curvature corrections.

For the eight equally weighted centers, $\mu_k=ru_k$ with angles
$\theta_k=2\pi k/8$. At $x=0$, all logits are equal, so $\pi_k=1/8$ and
\begin{align*}
 \bar\mu_t(0)&=\frac r8\sum_{k=0}^7
   (\cos\theta_k,\sin\theta_k)^\top=0,\\
 \Sigma_\pi(0)&=\frac{r^2}{8}\sum_{k=0}^7u_ku_k^\top
 =\frac{r^2}{8}
 \begin{pmatrix}
  \sum_k\cos^2\theta_k&\sum_k\sin\theta_k\cos\theta_k\\
  \sum_k\sin\theta_k\cos\theta_k&\sum_k\sin^2\theta_k
 \end{pmatrix}
 =\frac{r^2}{2}I_2,
\end{align*}
where the three trigonometric sums are $4,0,4$. For the singleton condition
$A=\{0\}$, $\bar\mu_A=\mu_0$ and $\Sigma_A=0$. Direct substitution yields
\begin{align*}
 (v_t^\gamma-v_t^{\mathrm c})(0)
 &=(\gamma-1)\frac{1-t}{C_t}\mu_0,\\
 \nabla_x(v_t^\gamma-v_t^{\mathrm c})(0)
 &=(\gamma-1)\frac{t(1-t)}{C_t^2}
   \left(0-\frac{r^2}{2}I_2\right),
\end{align*}
which proves~\eqref{eq:gmm-center-force-main} and gives
\begin{equation}
 \nabla_x(v_t^\gamma-v_t^{\mathrm c})(0)
 =-(\gamma-1)\frac{t(1-t)r^2}{2C_t^2}I_2.
 \label{eq:gmm-center-curvature-main}
\end{equation}
For the pure conditional field, substituting $\bar\mu_A=\mu_0$ into
\eqref{eq:app-gmm-velocity} gives
\[
 v_t^{\mathrm c}(x)=\frac{1-t}{C_t}\mu_0+\alpha_tx,
 \qquad
 \alpha_t=\frac{t\sigma^2-(1-t)}{C_t}
 =\frac{t(1+\sigma^2)-1}{C_t}.
\]
The extra force has value
$(\gamma-1)(1-t)\mu_0/C_t$ and first-order Jacobian
from~\eqref{eq:gmm-center-curvature-main}. Hence its Taylor expansion at the
origin is
\[
 v_t^\gamma(x)=v_t^\gamma(0)+\nabla_xv_t^\gamma(0)x+o(\|x\|),
\]
with $v_t^\gamma(0)=\gamma(1-t)\mu_0/C_t$. Dropping the remainder gives
\begin{equation}
 v_t^\gamma(x)\approx
 \gamma\frac{1-t}{C_t}\mu_0+
 \left[\alpha_t-(\gamma-1)
 \frac{t(1-t)r^2}{2C_t^2}\right]x.
 \label{eq:gmm-linearization-main}
\end{equation}
At $t=0$, the bracketed coefficient equals $-1$, while the constant term equals
$\gamma\mu_0$. Thus the frozen linear field is exactly
$\gamma\mu_0-x$ and its zero is $x=\gamma\mu_0$. Continuity gives the stated
small-$t$ limit.

For $A=\{0,2\}$, define the orthonormal directions
$u_+=(u_0+u_2)/\sqrt2$ and $u_-=(u_0-u_2)/\sqrt2$. At the origin the two
conditional responsibilities equal $1/2$, so
\begin{align*}
 \bar\mu_A&=\frac12(ru_0+ru_2)=\frac{r}{\sqrt2}u_+,\\
 \mu_0-\bar\mu_A&=\frac{r}{\sqrt2}u_-,\qquad
 \mu_2-\bar\mu_A=-\frac{r}{\sqrt2}u_-,\\
 \Sigma_A&=\frac12\frac{r^2}{2}u_-u_-^\top
           +\frac12\frac{r^2}{2}u_-u_-^\top
 =\frac{r^2}{2}u_-u_-^\top.
\end{align*}
Therefore
\[
 \Sigma_Au_+=0,\qquad
 \Sigma_Au_-=\frac{r^2}{2}u_-,
\]
whereas the unconditional covariance at the origin is
$\Sigma_\pi=(r^2/2)I_2$ and hence
$\Sigma_\pi u_\pm=(r^2/2)u_\pm$. Multiplying the Hessian difference
\eqref{eq:app-gmm-h-hessian} by $(\gamma-1)(1-t)/t$ gives
\begin{align}
 u_+^\top(H_t^\gamma-H_t^{\mathrm c})u_+
 &=-(\gamma-1)\frac{t(1-t)r^2}{2C_t^2},\\
 u_-^\top(H_t^\gamma-H_t^{\mathrm c})u_-&=0.
 \label{eq:gmm-splitting-main}
\end{align}
Finally, differentiating the pure conditional velocity gives
\[
 \nabla_xv_t^{\mathrm c}
 =\alpha_tI_2+\frac{t(1-t)}{C_t^2}\Sigma_A.
\]
Along $u_-$, the covariance term equals
$t(1-t)r^2/(2C_t^2)$. The CFG correction along this direction is zero, so the
complete splitting eigenvalue is
\[
 \lambda_-(t)=\alpha_t+\frac{t(1-t)r^2}{2C_t^2}.
\]
It is independent of $\gamma$ at the symmetric center.

\subsection{Branch reweighting for adjacent conditions}
\label{app:gmm-branch-reweighting}

We first prove the continuous equal-weight middle-branch effect. Rotate the
coordinate system so that the middle direction $u_1$ is the positive horizontal
axis. The three directions in $A=\{0,1,2\}$ then have angles
$-\delta,0,\delta$, where $\delta=\pi/4$. On a fixed shell, write
$x_{\rho,\phi}=\rho(\cos\phi,\sin\phi)^\top$ in these rotated coordinates and
set
\[
 \kappa=\frac{tr\rho}{C_t},\qquad
 a=\kappa\cos\phi,\qquad b=\kappa\sin\phi,
 \qquad q=\frac1{\sqrt2}.
\]
For $0\leq\phi\leq\delta$, we have $a\geq b\geq0$. Common Gaussian factors
and equal mixture weights cancel from the condition posterior. Its odds between
the conditioned set $A$ and its complement $B$ are therefore
\begin{equation}
 \frac{P_t(\rho,\phi)}{1-P_t(\rho,\phi)}
 =\frac{N(a,b)}{M(a,b)},
 \label{eq:app-middle-odds}
\end{equation}
where
\begin{align*}
 N(a,b)&=e^a+2e^{qa}\cosh(qb),\\
 M(a,b)&=e^{-a}+2e^{-qa}\cosh(qb)+2\cosh b.
\end{align*}
Reflection across the middle ray preserves both $A$ and the full eight-point
mixture, immediately giving
$P_t(\rho,\phi)=P_t(\rho,-\phi)$.

It remains to establish strict angular decrease. Let
$L(a,b)=\log N(a,b)-\log M(a,b)$. Its partial derivative with respect to $a$
is strictly positive: every direction in $A$ has horizontal coordinate at
least $q$, whereas every direction in $B$ has horizontal coordinate at most
zero. Equivalently, $\partial_aL$ is the difference between the exponentially
weighted mean horizontal coordinates in $A$ and $B$.

For $b>0$, the corresponding weighted vertical mean in $B$ is strictly larger
than that in $A$, so $\partial_bL<0$. To verify the sign directly, cross
multiplication reduces that mean difference to the positive quantity
\begin{align*}
 &e^a\sinh b
 +2e^{qa}\bigl[\cosh(qb)\sinh b
                -q\sinh(qb)\cosh b\bigr]\\
 &\qquad+2q\sinh(qb)\sinh((1-q)a).
\end{align*}
The bracket is positive because
\[
 \tanh b-q\tanh(qb)>0
 \qquad (b>0,\;0<q<1).
\]
Since $\partial_\phi a=-b$ and $\partial_\phi b=a$, we obtain
\begin{equation}
 \partial_\phi L
 =-b\,\partial_aL+a\,\partial_bL<0
 \qquad(0<\phi\leq\delta).
 \label{eq:app-middle-monotonicity}
\end{equation}
The logistic map from odds to probability is strictly increasing, so the same
sign holds for $\partial_\phi P_t(\rho,\phi)$. This proves
\eqref{eq:gmm-middle-posterior-main}.

For $x_{\rho,\phi}$, the angular derivative satisfies
$e_\phi^\top\nabla_x=\rho^{-1}\partial_\phi$. Combining this identity with
the posterior-force formula~\eqref{eq:force-main} and
\eqref{eq:app-middle-monotonicity} proves
\eqref{eq:gmm-middle-force-main}. This force is evaluated at a particular
time--state pair; its effect on terminal branch occupancy is determined only
after integrating the non-autonomous ODE.

We next prove Proposition~\ref{prop:gmm-weight-amplification-main}. Let
$A=\{0,1\}$, assume $w_0>w_1>0$, and impose reflection symmetry on the
complement,
\[
 w_2=w_7,\qquad w_3=w_6,\qquad w_4=w_5.
\]
At $x_{0,t}$ and $x_{1,t}$, canceling the same common radial factor gives the
unnormalized conditional likelihoods
\begin{align}
 N_0(a)&=w_0e^a+w_1e^{a/\sqrt2},\notag\\
 N_1(a)&=w_0e^{a/\sqrt2}+w_1e^a,
 \label{eq:app-unequal-N}\\
 N_0(a)-N_1(a)
 &=(w_0-w_1)(e^a-e^{a/\sqrt2})>0.
 \label{eq:app-unequal-gap}
\end{align}
To verify the complement contribution, reflect across the angular bisector
between $u_0$ and $u_1$. This reflection exchanges the component pairs
$2\leftrightarrow7$, $3\leftrightarrow6$, and $4\leftrightarrow5$, and maps
$x_{0,t}$ to $x_{1,t}$. The assumed equality within every pair makes the sum
of complement likelihoods, denoted $E(a)$, equal at the two centers. Hence
\[
 p_k^A(t):=p(A\mid x_{k,t})
 =\frac{N_k(a)}{N_k(a)+E(a)},\qquad
 p_0^A(t)>p_1^A(t).
\]
The last inequality follows because $z\mapsto z/(z+E)$ is strictly increasing
for $E>0$. For $0<t<1$ and $\gamma>1$, the difference between the CFG
potential corrections at the two centers is
\[
 (\gamma-1)\frac{1-t}{t}
 \log\frac{p_0^A(t)}{p_1^A(t)}>0.
\]
Thus the time-varying posterior potential is biased toward the already heavier
branch at these centers. This pointwise potential ordering does not determine
the transported density or terminal branch occupancy. Without complement
symmetry, the two center contributions must be evaluated separately, and
$w_0>w_1$ alone no longer guarantees the posterior ordering.
This proves Proposition~\ref{prop:gmm-weight-amplification-main}.

\subsection{Proof of the confidence--density theorem}

\begin{proof}[Proof of Theorem~\ref{thm:far-field-main}]
On the ray $x=su_0$, define
\[
 \vartheta_k=\theta_k-\theta_0,\qquad
 \kappa=\frac{str}{C_t}.
\]
Because $\mu_k=ru_k$ and $u_0^\top u_k=\cos\vartheta_k$, the squared distance
expands as
\[
 \|su_0-t\mu_k\|^2
 =s^2+t^2r^2-2str\cos\vartheta_k.
\]
Thus every likelihood contains the common factor
$\exp[-(s^2+t^2r^2)/(2C_t)]$, while the component-dependent factor is
$e^{\kappa\cos\vartheta_k}$. Canceling the common factor in
\eqref{eq:app-responsibility} gives
\begin{equation}
 P_0(s,t)=\frac{w_0e^\kappa}
 {\sum_kw_ke^{\kappa\cos\vartheta_k}},
 \qquad
 \pi_k(s,t)=\frac{w_ke^{\kappa\cos\vartheta_k}}
 {\sum_jw_je^{\kappa\cos\vartheta_j}}.
 \label{eq:app-ray-posterior}
\end{equation}
Taking logs gives
\[
 \log P_0=\log w_0+\kappa-
 \log\sum_kw_ke^{\kappa\cos\vartheta_k}.
\]
Since $\partial_s\kappa=tr/C_t$, differentiation gives
\begin{align}
 \partial_s\log P_0
 &=\frac{tr}{C_t}\left(
 1-\frac{\sum_kw_ke^{\kappa\cos\vartheta_k}\cos\vartheta_k}
          {\sum_jw_je^{\kappa\cos\vartheta_j}}\right)\notag\\
 &=\frac{tr}{C_t}\left(1-\sum_k\pi_k\cos\vartheta_k\right)\geq0,
 \label{eq:posterior-ray-main}\\
 \partial_s^2\log P_0
 &=-\left(\frac{tr}{C_t}\right)^2
 \left[\sum_k\pi_k\cos^2\vartheta_k
 -\left(\sum_k\pi_k\cos\vartheta_k\right)^2\right]\notag\\
 &=-\left(\frac{tr}{C_t}\right)^2
 \Var_\pi(\cos\vartheta_k)\leq0.
 \label{eq:posterior-concavity-main}
\end{align}
The first sign follows from $\cos\vartheta_k\leq1$. For the second derivative,
the intermediate responsibility derivative is
\[
 \partial_s\pi_k=\frac{tr}{C_t}\pi_k
 \left(\cos\vartheta_k-\sum_j\pi_j\cos\vartheta_j\right),
\]
and summing $\cos\vartheta_k\,\partial_s\pi_k$ produces the variance above.

For the asymptotic posterior, divide numerator and denominator of
\eqref{eq:app-ray-posterior} by $w_0e^\kappa$:
\[
 P_0(s,t)=
 \left[1+\sum_{j\ne0}\frac{w_j}{w_0}
 e^{-\kappa(1-\cos\vartheta_j)}\right]^{-1}.
\]
Uniqueness of $u_0$ means
$1-\cos\vartheta_j>0$ for $j\ne0$, hence $P_0(s,t)\to1$.

The marginal density on the ray is
\[
 p_t(su_0)=\frac{e^{-(s^2+t^2r^2)/(2C_t)}}{2\pi C_t}
 \sum_kw_ke^{\kappa\cos\vartheta_k}.
\]
Taking the logarithm and differentiating the Gaussian factor and the log-sum
separately gives
\[
 \partial_s\log p_t(su_0)
 =-\frac{s}{C_t}
  +\frac{tr}{C_t}\sum_k\pi_k\cos\vartheta_k
 \leq-\frac{s-tr}{C_t}.
\]
The inequality holds because the responsibility-weighted average of the
cosines is at most one.
This is strictly negative for $s>tr$, while the posterior derivative is
nonnegative.

Finally, factor the $k=0$ term from the mixture sum:
\begin{align*}
 p_t(su_0)
 &=\frac{w_0}{2\pi C_t}
 \exp\left[-\frac{(s-tr)^2}{2C_t}\right]
 \left[1+\sum_{j\ne0}\frac{w_j}{w_0}
 e^{-\kappa(1-\cos\vartheta_j)}\right].
\end{align*}
The bracket tends to one by the same strict angular gaps used for the posterior.
Therefore
\[
 p_t(su_0)\sim\frac{w_0}{2\pi C_t}
 \exp\left[-\frac{(s-tr)^2}{2C_t}\right]\to0.
\]
All claims follow.
\end{proof}

\section{PMC-CFG Proofs and Implementation}
\label{app:pmc-proof}

\subsection{Feasibility, maximality, and degenerate cases}

\begin{proof}[Proof of Theorem~\ref{thm:maximal-main}]
Squaring the nonnegative norms in~\eqref{eq:program-main} and expanding gives
\begin{align*}
 \norm{m_t^{\mathrm c}+\beta\Delta_t}^2
 -\Gamma^2\norm{m_t^{\mathrm c}}^2
 &=\left(\norm{m_t^{\mathrm c}}^2
 +2\beta\langle m_t^{\mathrm c},\Delta_t\rangle
 +\beta^2\norm{\Delta_t}^2\right)
 -\Gamma^2\norm{m_t^{\mathrm c}}^2\\
 &=\norm{\Delta_t}^2\beta^2
 +2\langle m_t^{\mathrm c},\Delta_t\rangle\beta
 +(1-\Gamma^2)\norm{m_t^{\mathrm c}}^2\\
 &=a\beta^2+2b\beta+(1-\Gamma^2)q.
\end{align*}
Thus the norm constraint is equivalent to
$f(\beta):=a\beta^2+2b\beta+(1-\Gamma^2)q\leq0$.
Since $\Gamma\geq1$ and $q\geq0$, the constant term is nonpositive;
therefore $\beta=0$ is feasible and the program is never empty.

Suppose $a>0$. Applying the quadratic formula to
$a\beta^2+2b\beta+(1-\Gamma^2)q=0$ gives the discriminant
\begin{align*}
 (2b)^2-4a(1-\Gamma^2)q
 =4[b^2+(\Gamma^2-1)aq]\geq0
\end{align*}
and roots
\[
 \beta_\pm
 =\frac{-2b\pm2\sqrt{b^2+(\Gamma^2-1)aq}}{2a}
 =\frac{-b\pm\sqrt{b^2+(\Gamma^2-1)aq}}{a}.
\]
Because the square root is at least $|b|$, its two numerators satisfy
$-b-\sqrt{\cdot}\leq0$ and $-b+\sqrt{\cdot}\geq0$; hence
$\beta_-\leq0\leq\beta_+$. Equivalently, their product is
$(1-\Gamma^2)q/a\leq0$. Since $a>0$, the quadratic opens upward and its
nonpositive set is exactly $[\beta_-,\beta_+]$. Intersecting it with the nominal interval
$[0,\lambda-1]$ gives
\[
 [0,\lambda-1]\cap[\beta_-,\beta_+]
 =[0,\min\{\lambda-1,\beta_+\}].
\]
The largest feasible element is therefore~\eqref{eq:beta-main}, with
$\beta_+=\beta^{\mathrm{cap}}$ from~\eqref{eq:root-main}.

If $a=0$, then $a=\|\Delta_t\|^2$ implies $\Delta_t=0$ and consequently
$b=\langle m_t^{\mathrm c},\Delta_t\rangle=0$. Every candidate mean equals
$m_t^{\mathrm c}$, and the constraint reduces to
$\|m_t^{\mathrm c}\|\leq\Gamma\|m_t^{\mathrm c}\|$, which holds because
$\Gamma\geq1$. The output velocity is also independent of $\beta$, so choosing
$\lambda-1$ preserves the nominal convention.

In both cases the selected $\beta_t^*$ belongs to the feasible set. Substituting
it into the original, unsquared constraint yields
\[
 \|m_t^{\mathrm c}+\beta_t^*\Delta_t\|
 \leq\Gamma\|m_t^{\mathrm c}\|,
\]
which is~\eqref{eq:guarantee-main}. If $\lambda-1$ is feasible, it is already
the largest point in $[0,\lambda-1]$ and is therefore selected; nominal CFG is
then unchanged.
\end{proof}

\subsection{Directional form, invariance, and the GMM corollary}

\begin{proposition}[Directional cap]
\label{prop:directional-app}
If $m_t^{\mathrm c}\ne0$ and $\Delta_t\ne0$, and $\theta$ is their angle, then
\eqref{eq:angle-main} holds. For aligned, orthogonal, and opposing vectors,
respectively,
\begin{align*}
 \cos\theta=1:\quad&
 \beta^{\mathrm{cap}}
 =(\Gamma-1)\frac{\norm{m_t^{\mathrm c}}}{\norm{\Delta_t}},\\
 \cos\theta=0:\quad&
 \beta^{\mathrm{cap}}
 =\sqrt{\Gamma^2-1}
 \frac{\norm{m_t^{\mathrm c}}}{\norm{\Delta_t}},\\
 \cos\theta=-1:\quad&
 \beta^{\mathrm{cap}}
 =(\Gamma+1)\frac{\norm{m_t^{\mathrm c}}}{\norm{\Delta_t}}.
\end{align*}
\end{proposition}

\begin{proof}
Substitute
$a=\|\Delta_t\|^2$, $q=\|m_t^{\mathrm c}\|^2$, and
$b=\|m_t^{\mathrm c}\|\|\Delta_t\|\cos\theta$ into~\eqref{eq:root-main}.
The numerator becomes
\begin{align*}
 &-\|m_t^{\mathrm c}\|\|\Delta_t\|\cos\theta\\
 &\quad+
 \sqrt{\|m_t^{\mathrm c}\|^2\|\Delta_t\|^2\cos^2\theta
 +(\Gamma^2-1)\|\Delta_t\|^2\|m_t^{\mathrm c}\|^2}\\
 &=\|m_t^{\mathrm c}\|\|\Delta_t\|
 \left[-\cos\theta+
 \sqrt{\cos^2\theta+\Gamma^2-1}\right].
\end{align*}
Dividing by $a=\|\Delta_t\|^2$ proves~\eqref{eq:angle-main}. For
$\cos\theta=1$, the bracket is $-1+\sqrt{\Gamma^2}=\Gamma-1$; for
$\cos\theta=0$, it is $\sqrt{\Gamma^2-1}$; and for
$\cos\theta=-1$, it is $1+\sqrt{\Gamma^2}=\Gamma+1$, because
$\Gamma\geq1$. These are the three displayed cases.
\end{proof}

\begin{proposition}[Coordinate behavior]
\label{prop:coordinate-app}
The origin-centered constraint is invariant under a common positive scalar
rescaling and an orthogonal transformation. It is not generally invariant
under translation or anisotropic scaling. If an anchor $r_t$ is transformed
together with the means, the constraint
\begin{equation}
 \norm{m_t^{\mathrm{cap}}-r_t}
 \leq\Gamma\norm{m_t^{\mathrm c}-r_t}
 \label{eq:anchored-app}
\end{equation}
is translation covariant.
\end{proposition}

\begin{proof}
Let $T(x)=sQx$ with $s>0$ and $Q^\top Q=I$. The transformed candidate and
baseline means are
\[
 \widetilde m^{\mathrm{cap}}=sQ(m^{\mathrm c}+\beta\Delta),
 \qquad \widetilde m^{\mathrm c}=sQm^{\mathrm c}.
\]
Therefore
\begin{align*}
 \|\widetilde m^{\mathrm{cap}}\|^2
 &=s^2(m^{\mathrm c}+\beta\Delta)^\top Q^\top Q
   (m^{\mathrm c}+\beta\Delta)
 =s^2\|m^{\mathrm c}+\beta\Delta\|^2,\\
 \Gamma^2\|\widetilde m^{\mathrm c}\|^2
 &=s^2\Gamma^2\|m^{\mathrm c}\|^2.
\end{align*}
The common positive factor $s^2$ cancels, so the feasible set and its largest
$\beta$ are unchanged.

Under translation by $h$, the constraint would compare
$\|m^{\mathrm c}+\beta\Delta+h\|$ with
$\Gamma\|m^{\mathrm c}+h\|$. The two added cross terms depend differently on
$\beta$, so this is not generally equivalent to the original constraint. Under
an anisotropic linear map $A$, a squared norm becomes
$\|Ay\|^2=y^\top A^\top Ay$; it is a common scalar multiple of the Euclidean
norm for all $y$ only when $A^\top A=s^2I$.

For the anchored form, translate
$m^{\mathrm{cap}},m^{\mathrm c},r$ by the same $h$. Then
\begin{align*}
 (m^{\mathrm{cap}}+h)-(r+h)&=m^{\mathrm{cap}}-r,\\
 (m^{\mathrm c}+h)-(r+h)&=m^{\mathrm c}-r.
\end{align*}
Both sides of~\eqref{eq:anchored-app} are therefore unchanged, proving
translation covariance.
\end{proof}

\begin{corollary}[Exact initial GMM clipping]
\label{cor:gmm-cap-app}
For the centered circular GMM conditioned on component $0$,
\eqref{eq:initial-cap-main} holds.
\end{corollary}

\begin{proof}
Appendix~\ref{app:gmm} established
$m_0^{\mathrm c}=\mu_0$ and $m_0^{\mathrm u}=0$, hence
$\Delta_0=m_0^{\mathrm c}-m_0^{\mathrm u}=\mu_0$. Every candidate mean is
\[
 m_0^{\mathrm c}+\beta\Delta_0=(1+\beta)\mu_0.
\]
Because $\beta\geq0$, $1+\beta>0$, and because $\|\mu_0\|=r>0$, the
relative-ball constraint becomes
\begin{align*}
 \|(1+\beta)\mu_0\|&\leq\Gamma\|\mu_0\|\\
 (1+\beta)r&\leq\Gamma r\\
 \beta&\leq\Gamma-1.
\end{align*}
The nominal interval independently gives $\beta\leq\lambda-1$. Thus the
feasible interval is
$[0,\min\{\lambda-1,\Gamma-1\}]$. Maximizing $\beta$ yields
$1+\beta_0^*=\min\{\lambda,\Gamma\}$, proving the claim.
\end{proof}

\subsection{Feedback Jacobian and well-posedness}

Let $d_t=\vc-\vu$. In a region where the active set is fixed and
$\beta_t^*$ is differentiable, the product rule gives
\begin{equation}
 \nabla_xv_t^{\mathrm{cap}}
 =\nabla_x\vc+\beta_t^*\nabla_xd_t
 +d_t(\nabla_x\beta_t^*)^\top.
 \label{eq:feedback-main}
\end{equation}
Indeed, for output coordinate $i$ and input coordinate $j$,
\begin{align*}
 \partial_{x_j}(v_t^{\mathrm{cap}})_i
 &=\partial_{x_j}(v_t^{\mathrm c})_i
 +\partial_{x_j}(\beta_t^*(d_t)_i)\\
 &=\partial_{x_j}(v_t^{\mathrm c})_i
 +\beta_t^*\partial_{x_j}(d_t)_i
 +(d_t)_i\partial_{x_j}\beta_t^*,
\end{align*}
which is exactly~\eqref{eq:feedback-main}. Its trace is
Taking traces also yields
\[
 \nabla\!\cdot v_t^{\mathrm{cap}}
 =\nabla\!\cdot\vc+\beta_t^*\nabla\!\cdot d_t
 +(\nabla_x\beta_t^*)^\top d_t.
\]
The conditional field and $d_t$ are spatial gradients under fixed CFG, so
their Jacobians are symmetric. Writing $g_t=\nabla_x\beta_t^*$, the only
possibly nonsymmetric term is $d_tg_t^\top$. Since its transpose is
$g_td_t^\top$, the antisymmetric part of the controlled Jacobian is exactly
\begin{equation}
 \frac12\left[d_t(\nabla_x\beta_t^*)^\top
 -(\nabla_x\beta_t^*)d_t^\top\right],
 \label{eq:antisymmetric-main}
\end{equation}
It vanishes when $d_tg_t^\top=g_td_t^\top$. For nonzero vectors this equality
holds exactly when $d_t$ and $g_t$ are collinear; it also holds when either is
zero. This establishes the stated local integrability condition.

The derivative of the active cap can also be written explicitly without
differentiating the square-root formula. Define
\[
 F(x,\beta)=a(x)\beta^2+2b(x)\beta+(1-\Gamma^2)q(x).
\]
On the active boundary, $F(x,\beta_t^*(x))=0$. Spatial differentiation gives
\begin{align*}
 0=\nabla_xF
 &=\beta_t^{*2}\nabla_xa+2\beta_t^*\nabla_xb
 +(1-\Gamma^2)\nabla_xq
 +(2a\beta_t^*+2b)\nabla_x\beta_t^*.
\end{align*}
If $a\beta_t^*+b\ne0$, solving for the last term yields
\begin{equation}
 \nabla_x\beta_t^*
 =-\frac{\beta_t^{*2}\nabla_xa+2\beta_t^*\nabla_xb
 +(1-\Gamma^2)\nabla_xq}
 {2(a\beta_t^*+b)}.
 \label{eq:active-beta-gradient-app}
\end{equation}
The denominator is nonzero precisely away from a repeated root. In the inactive
region $\beta_t^*=\lambda-1$ is constant in $x$, so
$\nabla_x\beta_t^*=0$ and the feedback term disappears.

When the discriminant is positive and the denominator above is nonzero, the
active root is smooth by the implicit-function theorem. At the switching surface
$\beta^{\mathrm{cap}}=\lambda-1$, the hard minimum is continuous and piecewise
smooth. Under the usual local-Lipschitz assumptions on model fields, this
piecewise-smooth feedback field admits the standard Carath\'eodory
interpretation. A smooth approximation to the minimum may improve numerical
derivatives, but its feasibility must be checked separately; smoothness alone
does not preserve~\eqref{eq:guarantee-main}.

\subsection{Stable per-sample implementation}

We log the controlled cap ratio
\begin{equation}
 \chi_t=
 \frac{\norm{m_t^{\mathrm c}+\beta_t^*\Delta_t}}
 {\norm{m_t^{\mathrm c}}+\varepsilon_m}.
 \label{eq:chi-main}
\end{equation}
The positive root admits a cancellation-resistant form when $b\geq0$.
Let $D=b^2+(\Gamma^2-1)aq$. Multiplying numerator and denominator by the
conjugate gives
\begin{align*}
 \frac{-b+\sqrt D}{a}
 &=\frac{(-b+\sqrt D)(b+\sqrt D)}{a(b+\sqrt D)}\\
 &=\frac{D-b^2}{a(b+\sqrt D)}
 =\frac{(\Gamma^2-1)q}{b+\sqrt D}.
\end{align*}
For $b\geq0$, the latter avoids subtracting two nearly equal positive numbers.
For $b<0$, the original numerator $-b+\sqrt D$ is an addition and is already
stable. If $b+\sqrt D=0$, necessarily $b=0$ and $D=0$; the positive root is
then zero, so it is assigned directly rather than evaluated as $0/0$.
For each sample and ODE evaluation:
\begin{enumerate}
 \item Evaluate $\vc$ and $\vu$, then compute
 $m_t^{\mathrm c}=x+(1-t)\vc$ and
 $\Delta_t=(1-t)(\vc-\vu)$.
 \item Test the nominal candidate
 $m_t^{\mathrm c}+(\lambda-1)\Delta_t$ directly. If it satisfies the bound,
 set $\beta_t^*=\lambda-1$.
 \item Otherwise compute $a,b,q$ from~\eqref{eq:abc-main}. If
 $a\leq\varepsilon_a$, use the conservative feasible value $\beta_t^*=0$.
 \item For $a>\varepsilon_a$, form $D=b^2+(\Gamma^2-1)aq$ and clamp it below
 by zero only to remove negative roundoff. Use the rationalized root above when
 $b\geq0$ and $b+\sqrt D>0$, use~\eqref{eq:root-main} when $b<0$, and set the
 root to zero when $b+\sqrt D=0$. Set
 $\beta_t^*=\min\{\lambda-1,\beta^{\mathrm{cap}}\}$.
 \item Use~\eqref{eq:velocity-cap-main} and log~\eqref{eq:chi-main} together
 with $\mathbf1\{\beta_t^*<\lambda-1\}$.
\end{enumerate}
The direct nominal test avoids unstable division when $\|\Delta_t\|$ is small.
The fallback $\beta=0$ remains feasible for every $\Gamma\geq1$. Norms and
inner products must reduce over each sample's latent dimensions, never over the
batch.

\section{Detailed Experimental Protocol}
\label{app:experiments}

\subsection{Eight-point distribution}

The endpoint experiments use a two-dimensional mixture with $K=8$ component
means placed clockwise on a circle of radius $r=1$. Component $0$ is at angle
$\pi/2$, and component $k$ is at
$\theta_k=\pi/2-k\pi/4$. Every component has isotropic standard deviation
$\sigma_k=0.03$. For the singleton $A=\{0\}$ and adjacent triple
$A=\{0,1,2\}$, the unconditional mixture weights are
$w_k=1/8$ for all $k$. The adjacent-pair experiment $A=\{0,1\}$ instead uses
deliberately imbalanced unconditional mixture weights,
\[
(w_0,\ldots,w_7)
=(0.17,0.08,0.125,0.125,0.125,0.125,0.125,0.125),
\]
which induce conditional weights $68\%/32\%$ after renormalization over
$A$. This pair is therefore an imbalance stress test rather than an equal-weight
counterpart of the other two conditions.

Both conditional and unconditional fields are analytic GMM posterior-mean
velocities; no neural network, learned approximation, or checkpoint is used.
The conditional weights are obtained by renormalizing the unconditional
weights over $A$, and fixed CFG is evaluated as
$v_t^{\mathrm u}+\lambda(v_t^{\mathrm c}-v_t^{\mathrm u})$. Thus
$\lambda=1$ is exactly the analytic conditional field. The PMC-CFG runs use
the same fields with nominal $\lambda=3$ and the relative norm cap
$\Gamma=1.1$ from Section~\ref{sec:method}.

For each row of Table~\ref{tab:gmm}, the implementation independently draws
$2\times10^5$ samples from $X_0\sim\mathcal N(0,I_2)$ in one call and uses
PyTorch's default floating-point dtype (float32). It then applies 50 explicit
Euler updates on the uniform grid $t_i=i/50$. The field is evaluated at the 50
left endpoints $t_i$, $i=0,\ldots,49$, and the final state is at $t=1$; there
is no $1-\varepsilon$ endpoint or special last-step rule. Counting one
evaluation of the composite guided field as one NFE gives 50 NFEs, equivalently
50 conditional and 50 unconditional analytic subfield evaluations. The nine
runs use constant nominal $\lambda\in\{1,3\}$ and enable no time schedule,
cosine-similarity gate, or relative-distance schedule. Three runs enable only
the PMC norm cap with $\Gamma=1.1$. Marginal histograms use 101 bins.

Each sampler call draws new initial noise, so different conditions and guidance
scales are not paired sample by sample. Accordingly, Table~\ref{tab:gmm} reports a
single $n=200000$ estimate per setting rather than a mean over repeated runs.
For endpoint metrics, samples are assigned by MAP under the conditional GMM,
not by an unconditional nearest-center rule. TV is one half of the $\ell_1$
distance between empirical and target conditional occupancies; mean error is
$\|\widehat\mu-\mu_A\|_2$ before assignment; and each per-label variance ratio
is $\operatorname{tr}(\widehat\Sigma_k)/(2\sigma_k^2)$, with
$\widehat\Sigma_k$ the unbiased covariance of samples assigned to branch $k$.

The $\lambda=1,A=\{0\}$ variance ratio near $0.571$ is the expected 50-step
Euler bias: the exact discrete prediction is $R_{50}=0.5706506$, whereas the
exact continuous-time transport has endpoint ratio one.

\subsection{ImageNet-256 protocol}

We use the released class-conditional SiT-XL/2 checkpoint for
$256\times256$ ImageNet generation \cite{deng2009imagenet,ma2024sit}, together with its native
latent normalization.
If the implementation parameterizes time or noise in a variable other than the
linear-path $t$, we convert the model output to the velocity and remaining-time
factor used in
$m_t=x+(1-t)v_t$ before applying PMC-CFG. This conversion is verified on
synthetic inputs and documented explicitly; silently reusing a diffusion noise
prediction as a velocity would invalidate the cap.

The evaluation draws 50 samples for each of 1,000 ImageNet classes, forming
the standard 50K set. We use the official ADM/guided-diffusion Inception
implementation and its reference statistics for FID, sFID, IS, precision, and
recall \cite{dhariwal2021diffusion,heusel2017ttur,
salimans2016improved,kynkaanniemi2019precisionrecall}. These five metrics are
computed from the same 50K sample set for every method.

\subsection{Stable Diffusion 3.5 protocol}

We evaluate the official Stable Diffusion 3.5 Medium checkpoint
\cite{stabilityai2024sd35} through the public Diffusers pipeline. We generate
one $1024\times1024$ image for each of 5,000 prompts from the COCO 2017
validation split \cite{lin2014coco}, following their fixed insertion order in
the prompt file. Each run uses one persistent CUDA generator initialized with
seed 42, and all methods share the same scheduler, text encoders, VAE, prompt
order, and initial latent sequence. Generation uses batch size one; under CFG,
the conditional and unconditional inputs are concatenated into a transformer
batch of two. We use empty-text unconditional conditioning unless a baseline
requires a different negative prompt, in which case that difference is
reported.

We use the \texttt{FlowMatchEulerDiscreteScheduler} with 10 deterministic Euler
steps, \texttt{num\_train\_timesteps=1000}, and a fixed flow shift of 3.0.
Dynamic shifting, stochastic sampling, and Karras, exponential, beta sigma
schedules are disabled. The SD3.5 pipeline runs in bfloat16, while scheduler
timesteps and sigmas, PMC-CFG norms and inner products, and clean-latent
estimates are computed in FP32. We test nominal guidance scales
$\lambda\in\{3,5,7,9\}$. Fixed CFG applies the nominal scale at every step;
PMC-CFG uses $\Gamma=1.15$ and caps each sample's guided clean-latent
displacement without exceeding the nominal scale. C2FG uses its
method-specific coefficient 0.2 and the normalized sampling progress $p$ on
the shifted sigma grid, giving the schedule $w(p)=w\exp(0.2p)$. APG applies
its projected, rescaled, momentum-guidance update \cite{sadat2025apg} to the
same conditional and unconditional predictions; its evaluated sample sets use
the same nominal scales, initial latent sequence, and solver configuration.

CLIP is the mean raw cosine similarity between matched images and prompts using
OpenAI CLIP ViT-B/32 \cite{radford2021clip}; it is not multiplied by 100.
ImageReward is the mean ImageReward-v1.0 score \cite{xu2023imagereward} and may
be negative. CLIP and FID use evaluation batch size eight, whereas ImageReward
is evaluated one image at a time. For FID, generated and COCO
validation images are resized directly to $299\times299$, converted to
8-bit tensors, and evaluated with the default 2048-dimensional Inception
features in TorchMetrics \cite{detlefsen2022torchmetrics}; these values should
therefore not be mixed with Clean-FID results \cite{parmar2022cleanfid}. Mean
saturation is computed from the normalized 8-bit HSV
saturation channel and is reported only to diagnose over- or under-saturation.
Finally, diversity is improved Recall@$3$
\cite{kynkaanniemi2019precisionrecall}: using Clean-FID's clean Inception
features, a real sample is covered if it lies inside the third-nearest-neighbor
ball of at least one generated sample. FID and Recall@$3$ use the 5,000 COCO
validation images as the reference set; the remaining metrics use only each
generated image and its matched prompt.

\FloatBarrier
\section{Additional Diagnostics and Qualitative Results}
\label{app:additional-results}

\subsection{GMM marginal endpoint densities}
\label{app:gmm-marginals}

Figure~\ref{fig:gmm-distributions} complements the endpoint samples and
statistics in the main text by showing how guidance changes the endpoint
marginal densities. The conditional baseline ($\lambda=1$, top row) closely follows
the target shape, apart from the known finite-step contraction. Standard
fixed-scale CFG ($\lambda=3$, middle row) produces a clear distributional
shift: for the singleton condition, the marginal peak moves away from the
target and becomes substantially narrower; for the pair and triple conditions,
the generated mass is additionally reweighted toward the heavier or
geometrically central branch. Consequently, the generated histograms no longer
align with either the location, spread, or relative mode mass of the
conditional target.

\begin{figure*}[t]
\centering
\begin{minipage}{0.32\textwidth}
\centering\footnotesize $\lambda=1,\ A=\{0\}$\\[-0.25em]
\includegraphics[width=\linewidth]{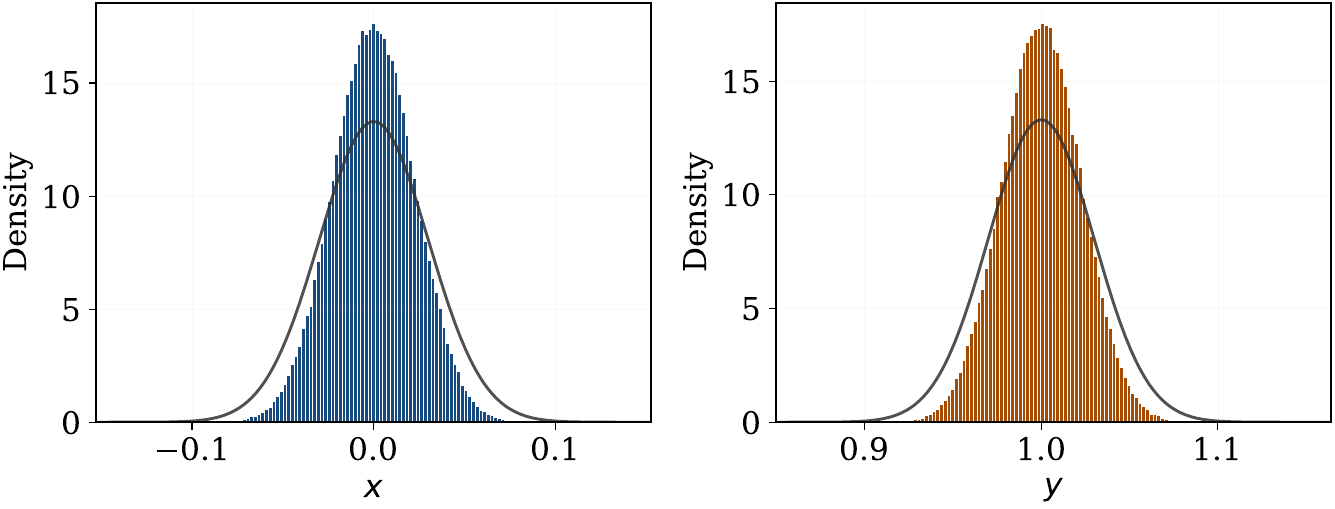}
\end{minipage}\hfill
\begin{minipage}{0.32\textwidth}
\centering\footnotesize $\lambda=1,\ A=\{0,1\}$\\[-0.25em]
\includegraphics[width=\linewidth]{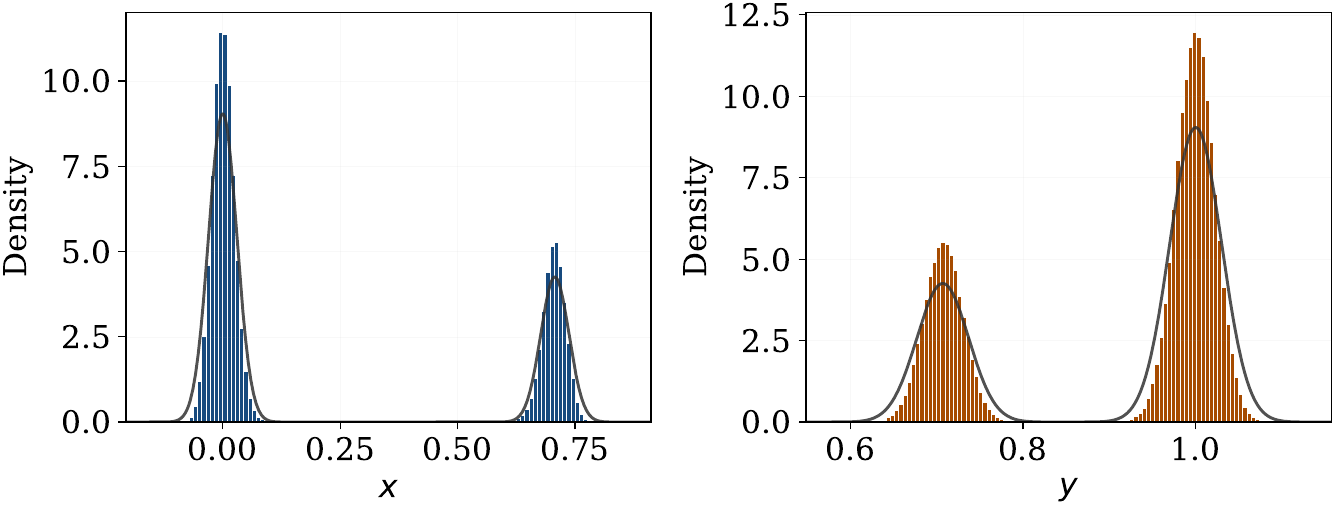}
\end{minipage}\hfill
\begin{minipage}{0.32\textwidth}
\centering\footnotesize $\lambda=1,\ A=\{0,1,2\}$\\[-0.25em]
\includegraphics[width=\linewidth]{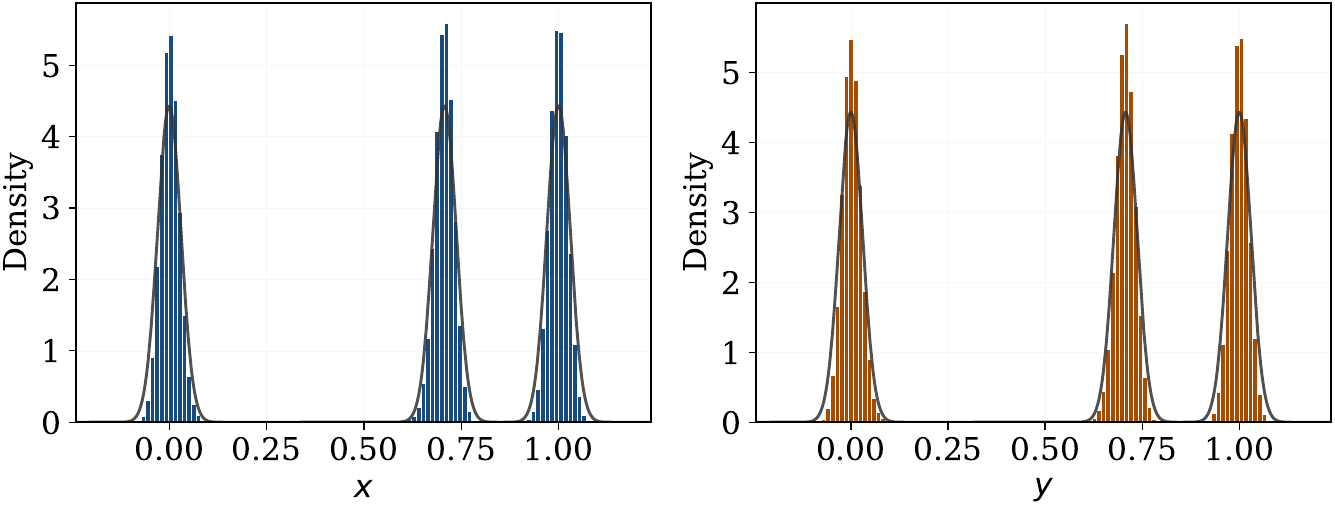}
\end{minipage}\\[0.35em]
\begin{minipage}{0.32\textwidth}
\centering\footnotesize $\lambda=3,\ A=\{0\}$\\[-0.25em]
\includegraphics[width=\linewidth]{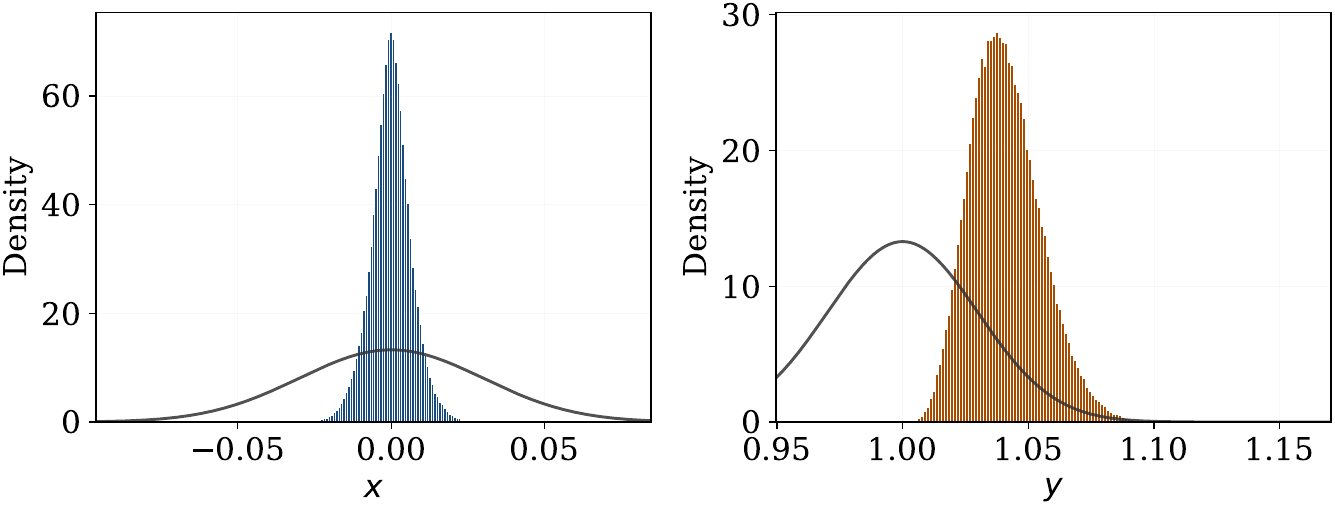}
\end{minipage}\hfill
\begin{minipage}{0.32\textwidth}
\centering\footnotesize $\lambda=3,\ A=\{0,1\}$\\[-0.25em]
\includegraphics[width=\linewidth]{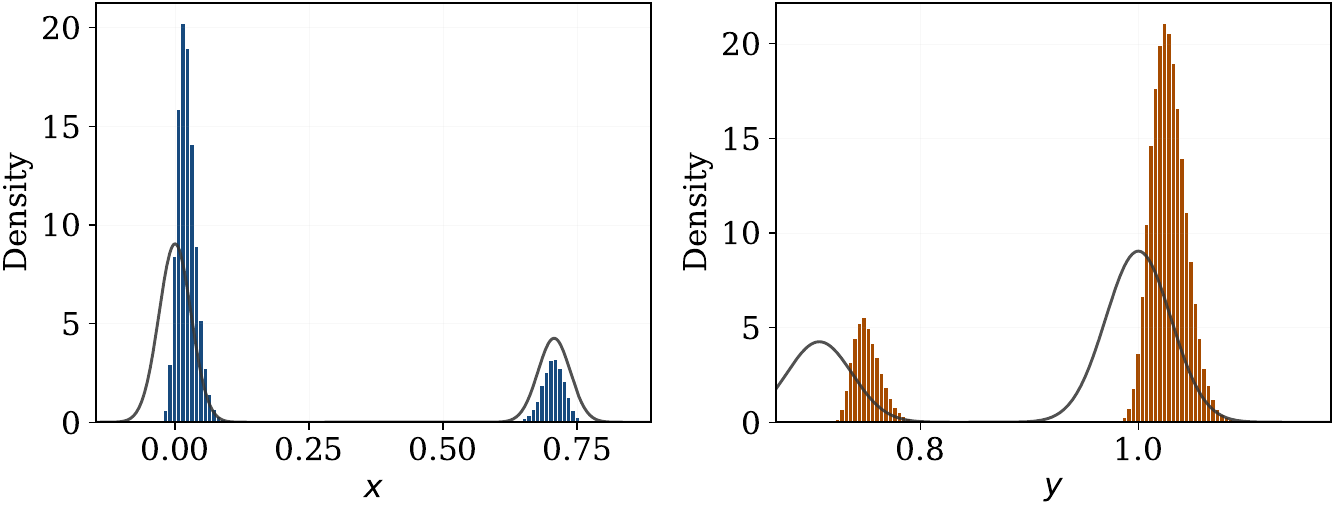}
\end{minipage}\hfill
\begin{minipage}{0.32\textwidth}
\centering\footnotesize $\lambda=3,\ A=\{0,1,2\}$\\[-0.25em]
\includegraphics[width=\linewidth]{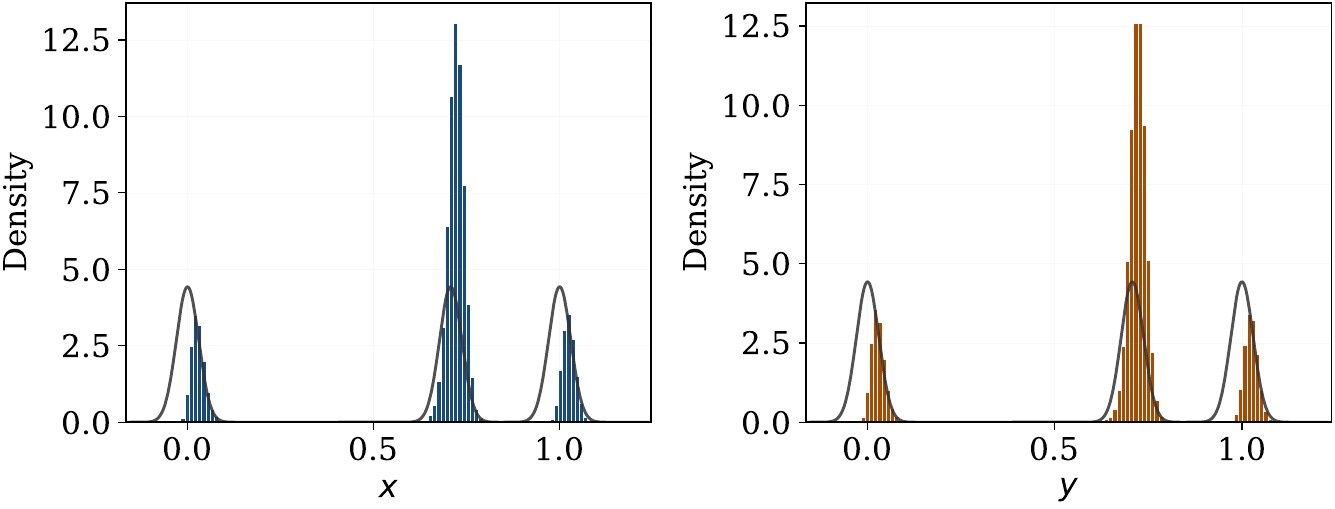}
\end{minipage}\\[0.35em]
\begin{minipage}{0.32\textwidth}
\centering\footnotesize PMC-CFG, $\lambda=3,\ \Gamma=1.1$\\[-0.2em]
$A=\{0\}$\\[-0.25em]
\includegraphics[width=\linewidth]{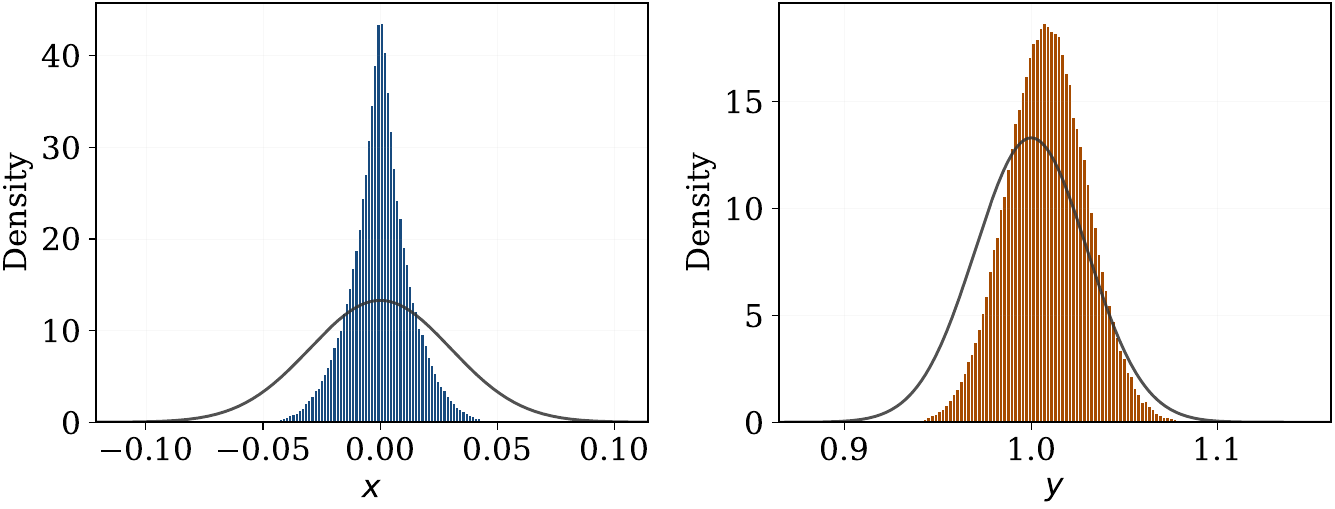}
\end{minipage}\hfill
\begin{minipage}{0.32\textwidth}
\centering\footnotesize PMC-CFG, $\lambda=3,\ \Gamma=1.1$\\[-0.2em]
$A=\{0,1\}$\\[-0.25em]
\includegraphics[width=\linewidth]{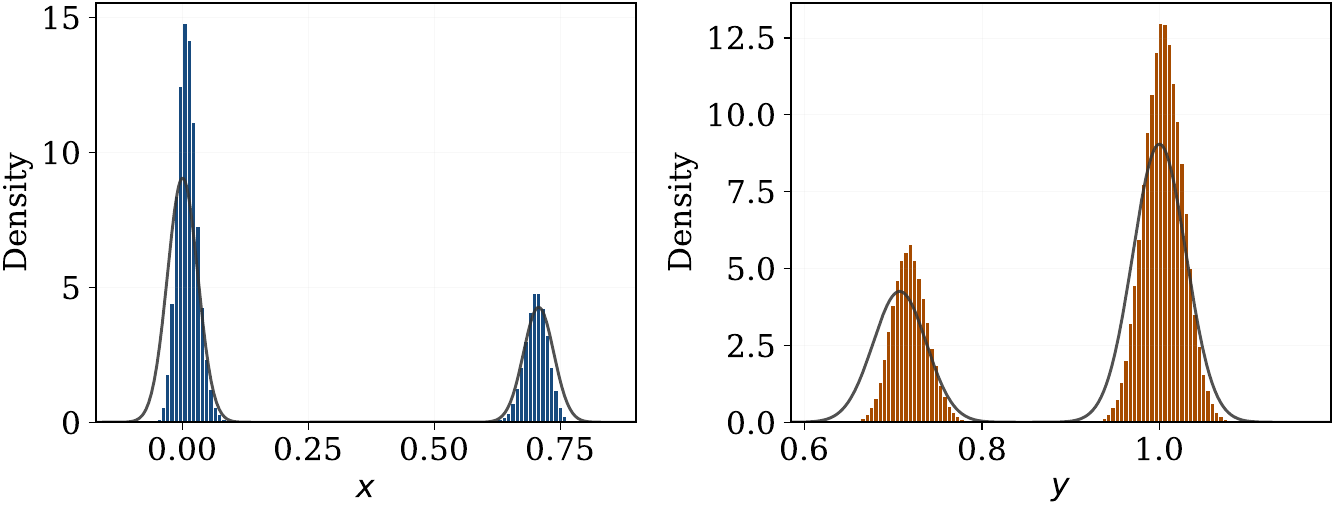}
\end{minipage}\hfill
\begin{minipage}{0.32\textwidth}
\centering\footnotesize PMC-CFG, $\lambda=3,\ \Gamma=1.1$\\[-0.2em]
$A=\{0,1,2\}$\\[-0.25em]
\includegraphics[width=\linewidth]{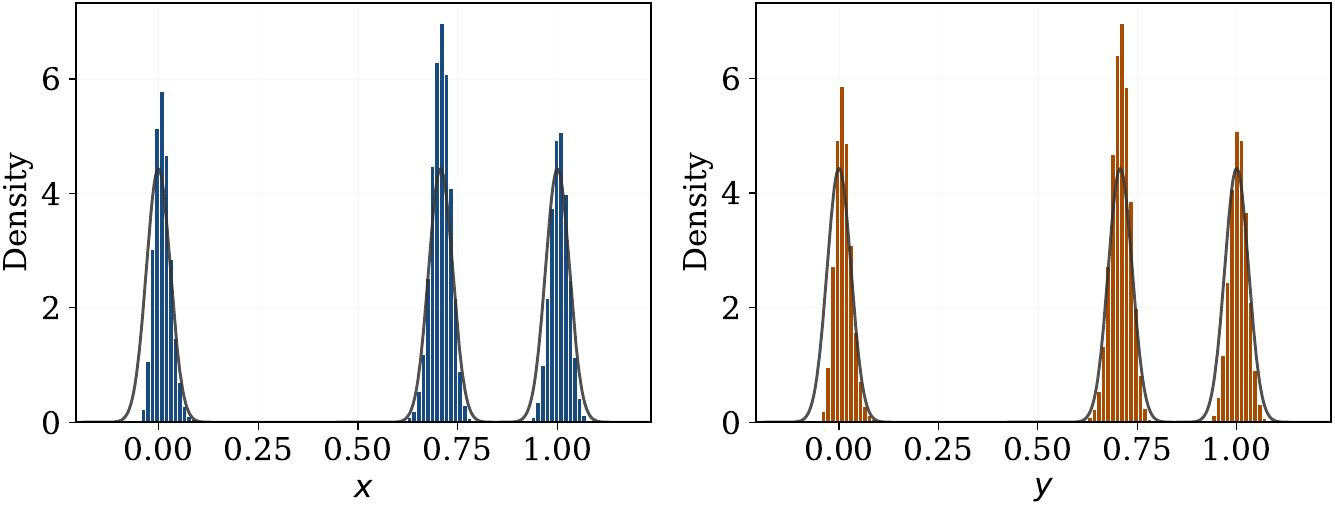}
\end{minipage}
\caption{Marginal endpoint densities for the same runs as
Figure~\ref{fig:gmm-samples}. Rows show the conditional baseline
($\lambda=1$), fixed CFG ($\lambda=3$), and PMC-CFG
($\lambda=3,\Gamma=1.1$). Histograms are generated marginals and dark curves
are the theoretical conditional target. Fixed CFG shifts and concentrates the
distribution, whereas PMC-CFG recovers substantially more of its location,
spread, and relative mode mass.}
\label{fig:gmm-distributions}
\end{figure*}

PMC-CFG uses the same nominal scale $\lambda=3$ but limits the early
posterior-mean extrapolation responsible for this displacement. The bottom row
therefore places the marginal peaks closer to the target, restores within-mode
spread, and reduces branch reweighting. Table~\ref{tab:gmm} quantifies the same
effect: relative to fixed CFG, PMC-CFG reduces mean error from $0.04091$ to
$0.00831$ for the singleton, from $0.11781$ to $0.02800$ for the pair, and from
$0.12883$ to $0.02980$ for the triple. For the multimodal cases, occupancy TV
also decreases from $0.15272$ to $0.03851$ and from $0.35294$ to $0.08046$,
respectively. Thus the cap mitigates both global displacement and
guidance-induced concentration rather than merely translating the generated
distribution.

\subsection{ImageNet-256 mechanism diagnostics}
\label{app:imagenet-mechanism}

Figure~\ref{fig:imagenet-dynamics} exhibits the posterior-mean evolution derived
for the GMM in Proposition~\ref{prop:gmm-gap-app}. Early in sampling,
$\norm{m_t^{\mathrm u}}$ is substantially smaller than
$\norm{m_t^{\mathrm c}}$ and $\norm{\Delta_t}$ is large, resembling the
centered-GMM initialization $m_0^{\mathrm u}=0$ and
$m_0^{\mathrm c}=\mu_0$. The unconditional prediction averages over many
possible classes, whereas conditioning supplies a directed endpoint estimate.
As $x_t$ becomes informative, the two predictions converge and
$\norm{\Delta_t}$ approaches zero, matching the GMM limit in
\eqref{eq:gmm-gap-envelope-app}.

This evolution of the posterior-mean gap explains the increasing effective
scale in the right panel. Writing $\gamma_t^{\mathrm{eff}}=1+\beta_t^*$, the
actual endpoint prediction used by PMC-CFG is
\begin{equation}
 m_t^{\mathrm{PMC}}
 =m_t^{\mathrm c}+(\gamma_t^{\mathrm{eff}}-1)\Delta_t,
 \qquad
 \gamma_t^{\mathrm{eff}}
 =1+\min\!\left\{\lambda-1,
 \frac{\norm{m_t^{\mathrm c}}}{\norm{\Delta_t}}
 \left[-\cos\theta_t+
 \sqrt{\cos^2\theta_t+\Gamma^2-1}\right]\right\},
 \label{eq:imagenet-effective-scale-app}
\end{equation}
for $\Delta_t\ne0$, where $\theta_t$ is the angle between
$m_t^{\mathrm c}$ and $\Delta_t$; if $\Delta_t=0$, we set
$\gamma_t^{\mathrm{eff}}=\lambda$. A large relative gap activates the cap and
keeps the early scale close to conditional guidance. As the gap contracts, the
admissible scale increases. Once
\begin{equation}
 \norm{m_t^{\mathrm c}+(\lambda-1)\Delta_t}
 \leq \Gamma\norm{m_t^{\mathrm c}},
 \label{eq:imagenet-nominal-feasible-app}
\end{equation}
nominal CFG is feasible, so the controller reaches and retains the upper limit
$\gamma_t^{\mathrm{eff}}=\lambda$.

This produces a coarse-to-fine guidance schedule. During the early,
pattern-forming phase, the smaller scale limits posterior-mean overshoot and
extra trajectory contraction, preserving variation while the global layout is
established. Later, $\Delta_t$ is small, so PMC-CFG can use the nominal scale
with a small absolute correction $(\lambda-1)\Delta_t$ to sharpen
class-specific boundaries and texture. This mechanism agrees with the improved
recall in Table~\ref{tab:imagenet} and the matched samples in
Figure~\ref{fig:imagenet-matched-extra}.

\begin{figure*}[t]
\centering
\begin{minipage}[t]{0.32\textwidth}
\centering
\includegraphics[width=\linewidth]{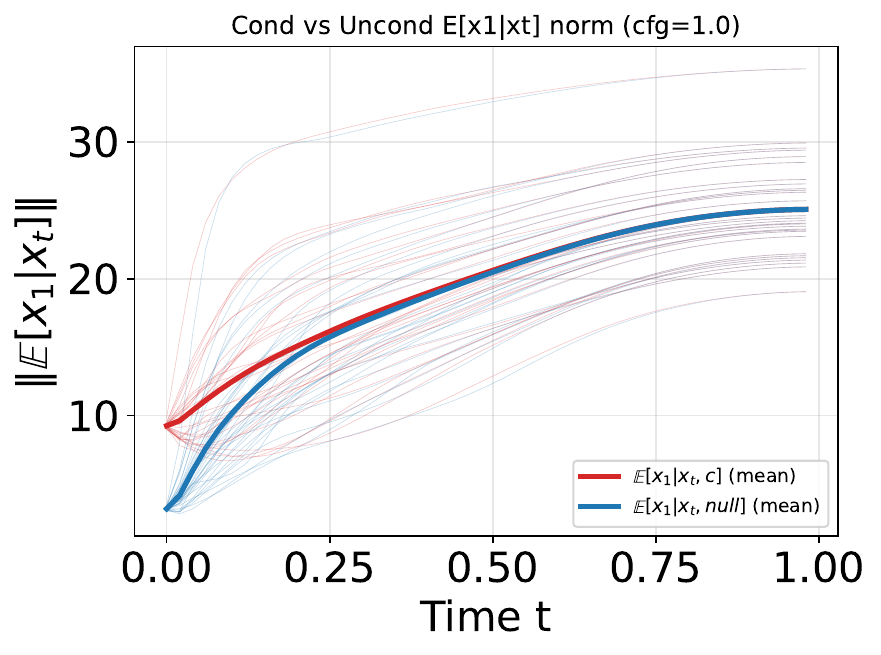}
\end{minipage}\hfill
\begin{minipage}[t]{0.32\textwidth}
\centering
\includegraphics[width=\linewidth]{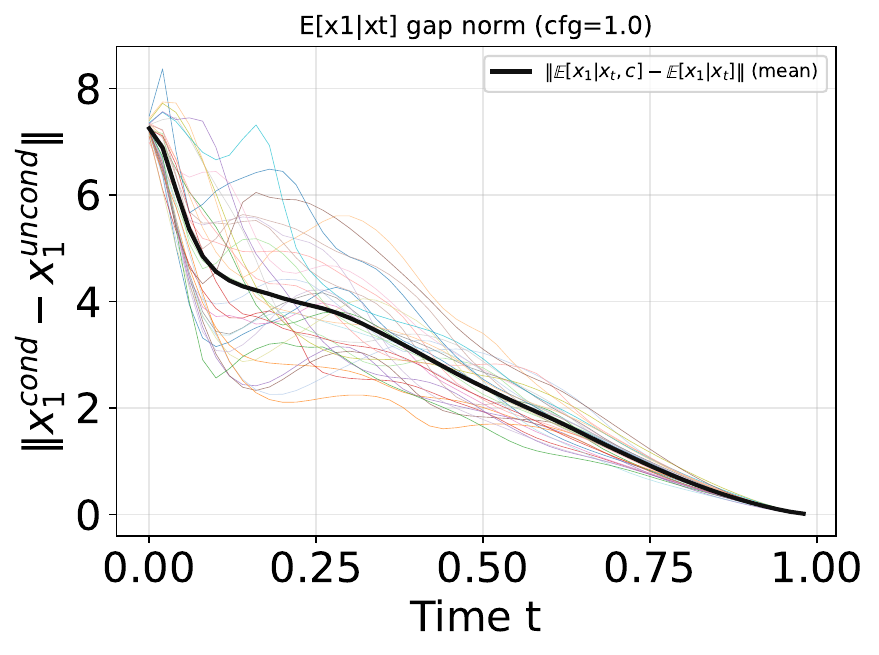}
\end{minipage}\hfill
\begin{minipage}[t]{0.32\textwidth}
\centering
\includegraphics[width=\linewidth]{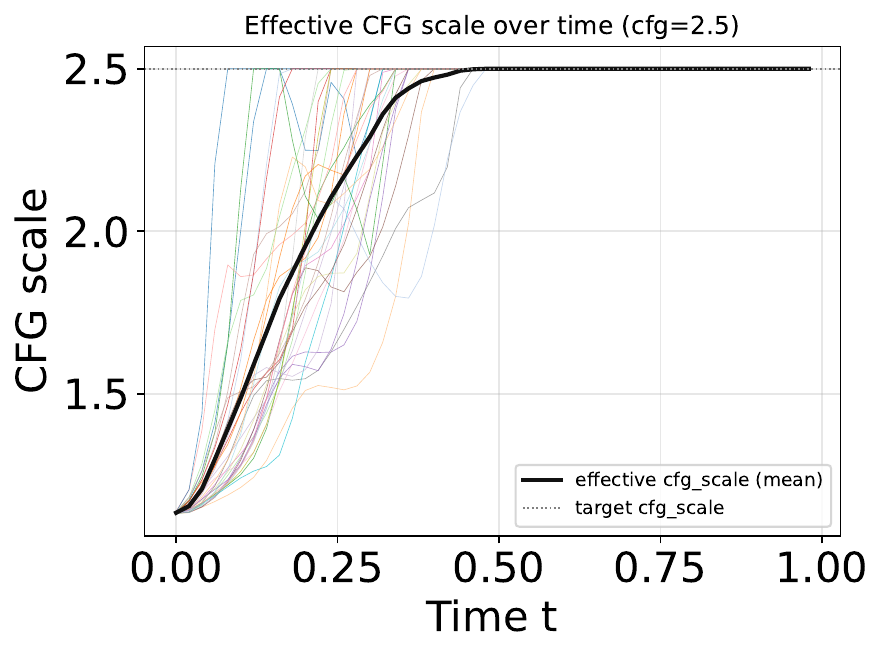}
\end{minipage}
\caption{ImageNet-256 mechanism diagnostics. From left to right: the conditional
and unconditional posterior-mean norms $\norm{m_t^{\mathrm c}}$ and
$\norm{m_t^{\mathrm u}}$ evaluated along the $\lambda=1$ trajectory; their gap
$\norm{\Delta_t}$ on the same trajectory; and the effective PMC-CFG scale
$\gamma_t^{\mathrm{eff}}=1+\beta_t^*$ for nominal scale $\lambda=2.5$. The
initially smaller unconditional norm and large gap resemble the centered-GMM
geometry; as the gap vanishes, the cap relaxes and the effective scale reaches
its nominal upper limit. Thin curves denote individual samples and thick curves
their averages.}
\label{fig:imagenet-dynamics}
\end{figure*}

\subsection{Additional qualitative results}
\label{app:additional-qualitative}

To complement the aggregate metrics, we provide matched qualitative results
that isolate how each guidance rule changes the composition and appearance of
individual samples. Within each comparison, no guidance, fixed CFG, and
PMC-CFG use the same class or text prompt and the same initial latent seed.
Figure~\ref{fig:imagenet-matched-extra} reports additional class-conditional
results on ImageNet-256, while Figure~\ref{fig:sd35-categories} covers four
text-prompt categories on SD3.5 Medium. The latter uses the following prompts:
\begin{itemize}
    \item \textbf{Dog:} \emph{A dog is shown sitting on rocks by the beach,
    photorealistic}.
    \item \textbf{Landscape:} \emph{A beautiful mountain and river landscape,
    photorealistic}.
    \item \textbf{Bedroom:} \emph{A bedroom with a large bed sitting under a
    painting, photorealistic}.
    \item \textbf{People:} \emph{a person on a small boat with another boat in
    the background, photorealistic}.
\end{itemize}

Across both ImageNet-256 and SD3.5, the matched samples show that PMC-CFG
largely preserves the objects, layout, and overall spatial pattern of the
unguided counterpart while producing clearer structures and finer details.
Compared with fixed CFG, PMC-CFG also avoids much of the excessive saturation,
improving image quality without replacing the composition established by the
unguided trajectory.

\begin{figure}[p]
\centering
\setlength{\tabcolsep}{0.5pt}
\begin{tabular}{@{}ccc@{\hspace{6pt}}ccc@{}}
\shortstack[c]{\tiny No guidance\\[-1pt]\tiny $\lambda=1$}
& \shortstack[c]{\tiny Fixed CFG\\[-1pt]\tiny $\lambda=2.5$}
& \shortstack[c]{\tiny PMC-CFG\\[-1pt]\tiny $\lambda=2.5$}
& \shortstack[c]{\tiny No guidance\\[-1pt]\tiny $\lambda=1$}
& \shortstack[c]{\tiny Fixed CFG\\[-1pt]\tiny $\lambda=2.5$}
& \shortstack[c]{\tiny PMC-CFG\\[-1pt]\tiny $\lambda=2.5$} \\
\includegraphics[width=0.155\linewidth]{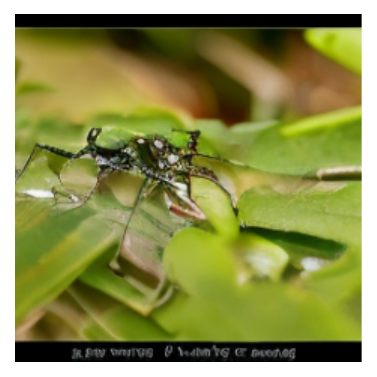}
& \includegraphics[width=0.155\linewidth]{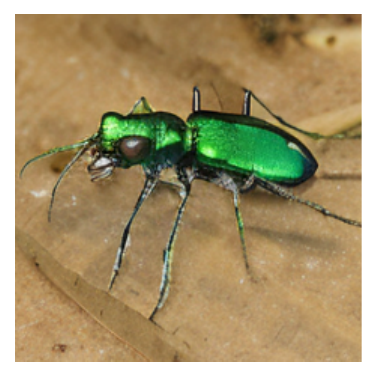}
& \includegraphics[width=0.155\linewidth]{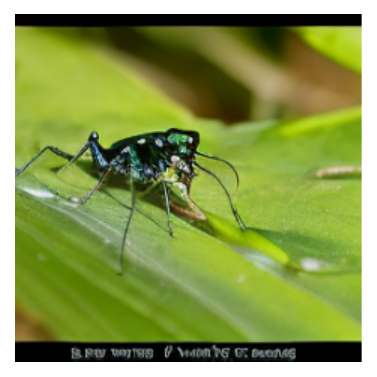}
& \includegraphics[width=0.155\linewidth]{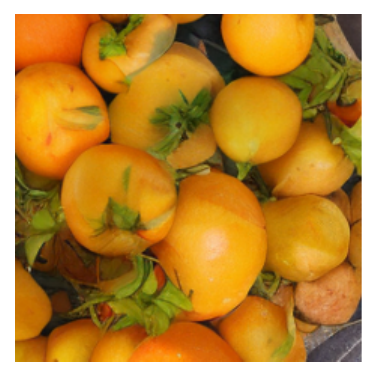}
& \includegraphics[width=0.155\linewidth]{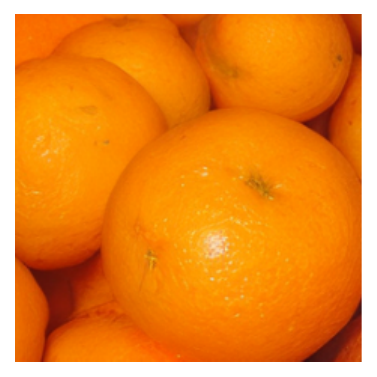}
& \includegraphics[width=0.155\linewidth]{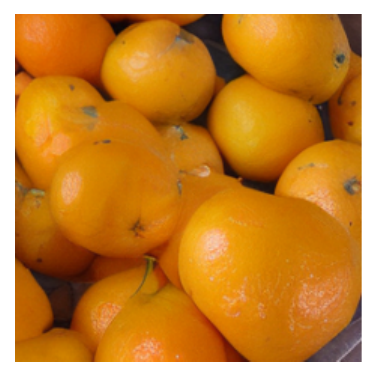} \\[-2pt]
\includegraphics[width=0.155\linewidth]{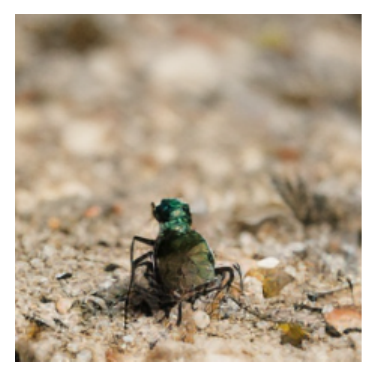}
& \includegraphics[width=0.155\linewidth]{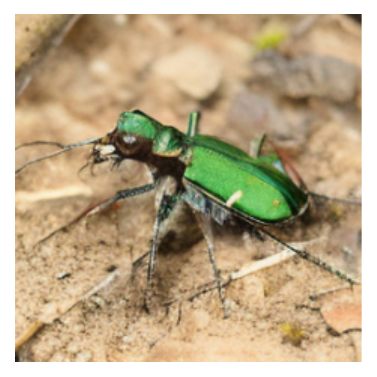}
& \includegraphics[width=0.155\linewidth]{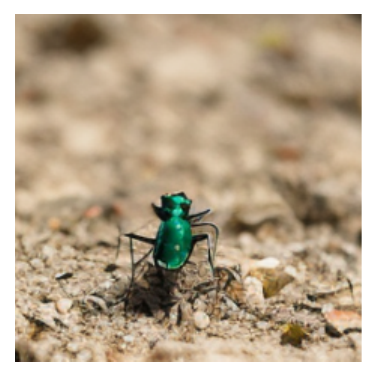}
& \includegraphics[width=0.155\linewidth]{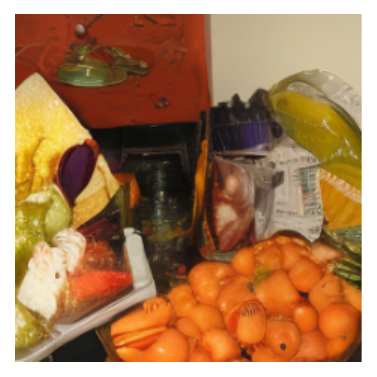}
& \includegraphics[width=0.155\linewidth]{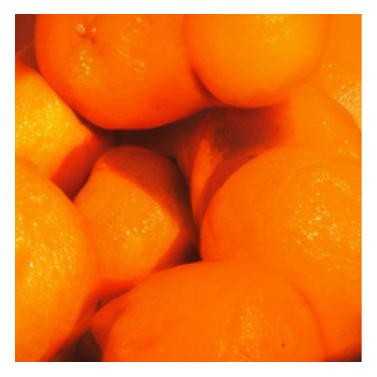}
& \includegraphics[width=0.155\linewidth]{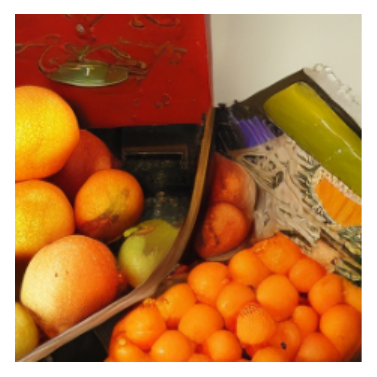} \\[-2pt]
\includegraphics[width=0.155\linewidth]{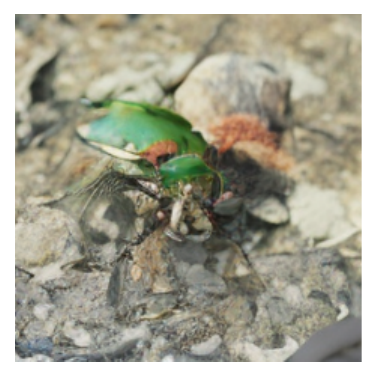}
& \includegraphics[width=0.155\linewidth]{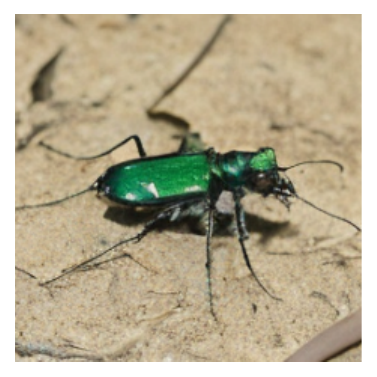}
& \includegraphics[width=0.155\linewidth]{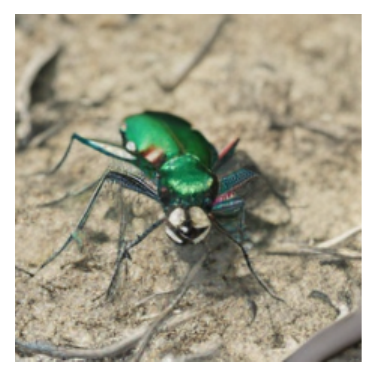}
& \includegraphics[width=0.155\linewidth]{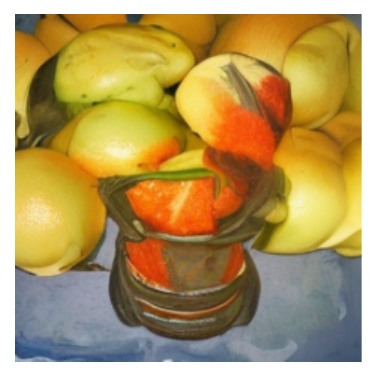}
& \includegraphics[width=0.155\linewidth]{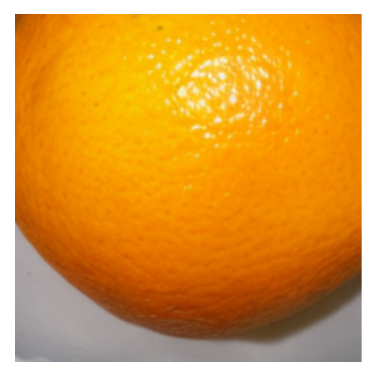}
& \includegraphics[width=0.155\linewidth]{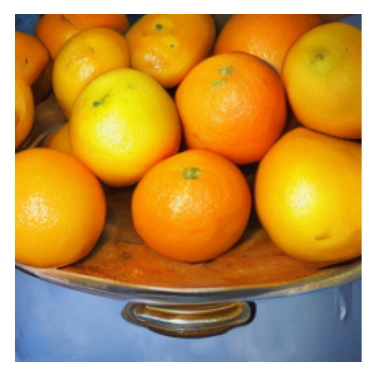} \\[-2pt]
\includegraphics[width=0.155\linewidth]{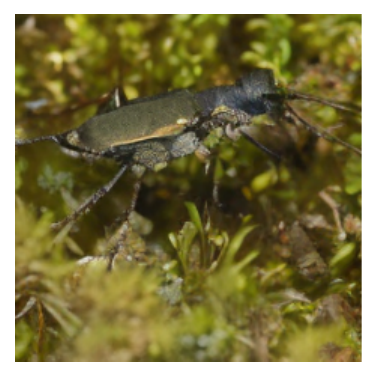}
& \includegraphics[width=0.155\linewidth]{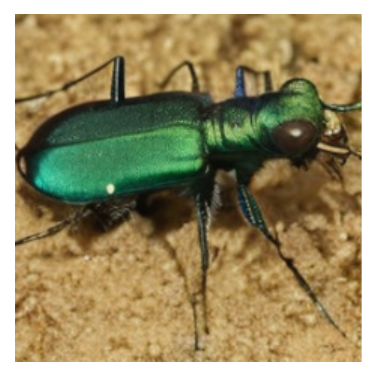}
& \includegraphics[width=0.155\linewidth]{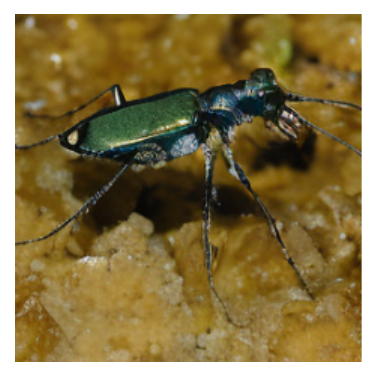}
& \includegraphics[width=0.155\linewidth]{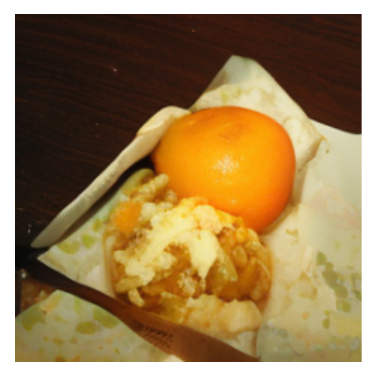}
& \includegraphics[width=0.155\linewidth]{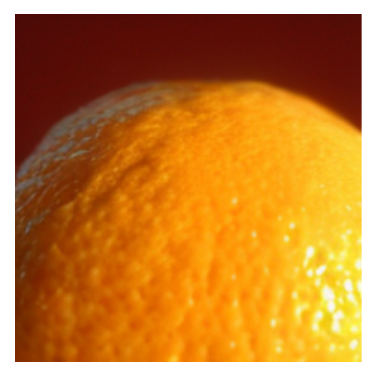}
& \includegraphics[width=0.155\linewidth]{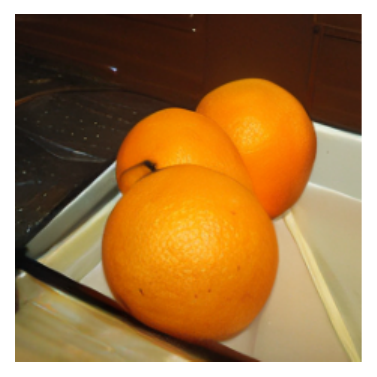} \\[-2pt]
\includegraphics[width=0.155\linewidth]{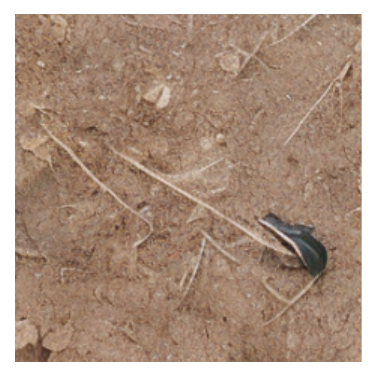}
& \includegraphics[width=0.155\linewidth]{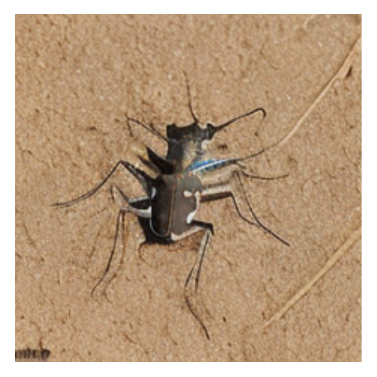}
& \includegraphics[width=0.155\linewidth]{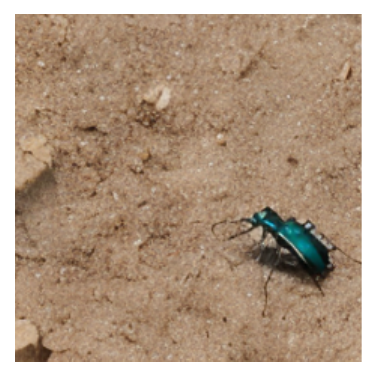}
& \includegraphics[width=0.155\linewidth]{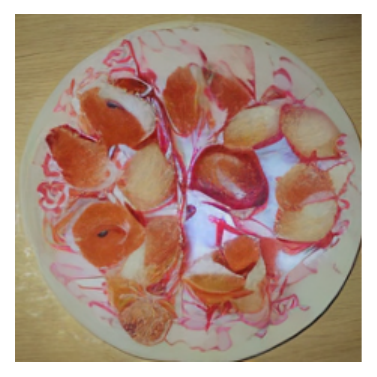}
& \includegraphics[width=0.155\linewidth]{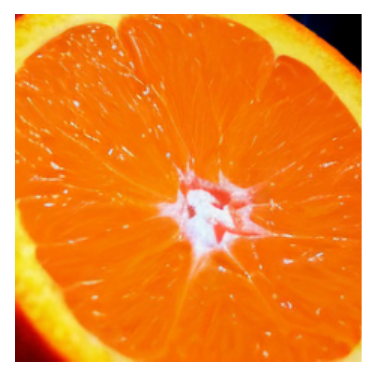}
& \includegraphics[width=0.155\linewidth]{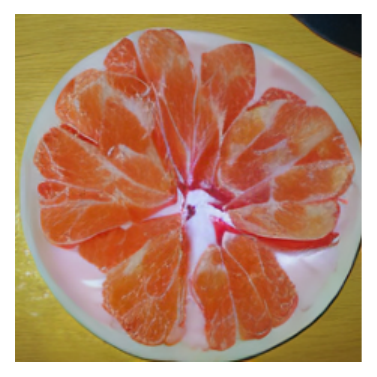} \\[-2pt]
\includegraphics[width=0.155\linewidth]{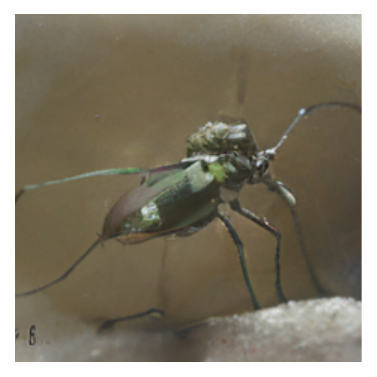}
& \includegraphics[width=0.155\linewidth]{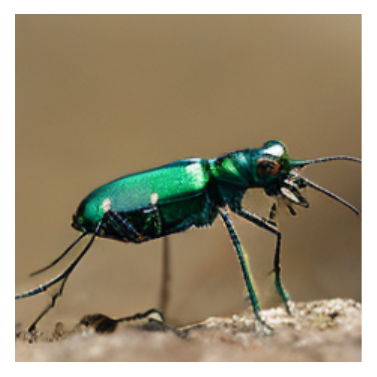}
& \includegraphics[width=0.155\linewidth]{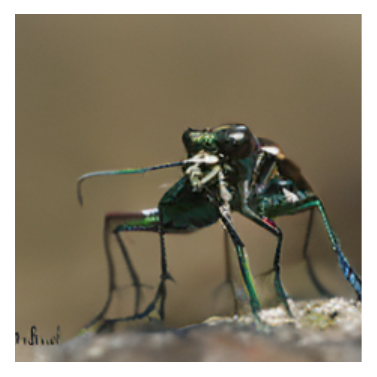}
& \includegraphics[width=0.155\linewidth]{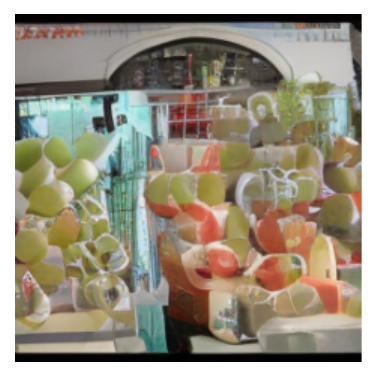}
& \includegraphics[width=0.155\linewidth]{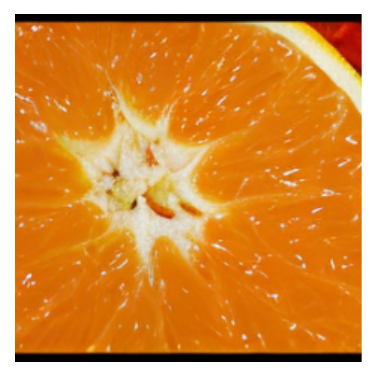}
& \includegraphics[width=0.155\linewidth]{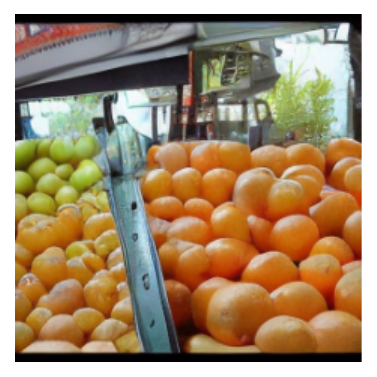} \\[-2pt]
\includegraphics[width=0.155\linewidth]{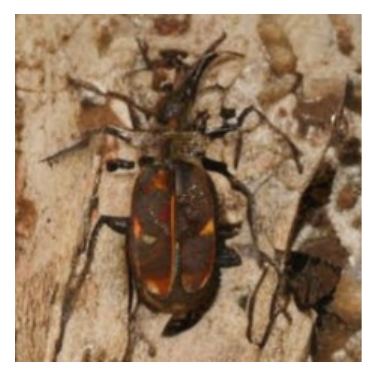}
& \includegraphics[width=0.155\linewidth]{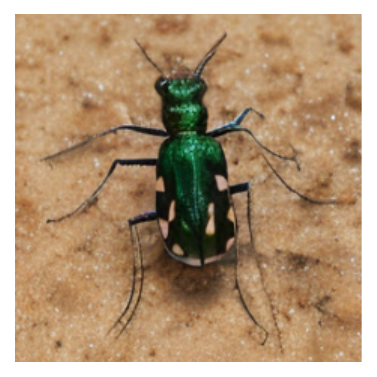}
& \includegraphics[width=0.155\linewidth]{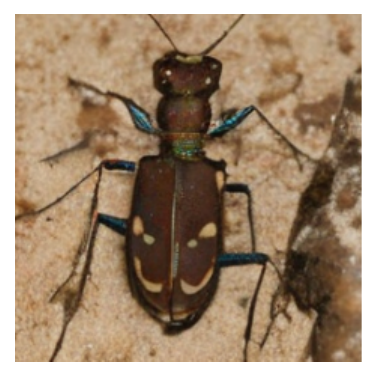}
& \includegraphics[width=0.155\linewidth]{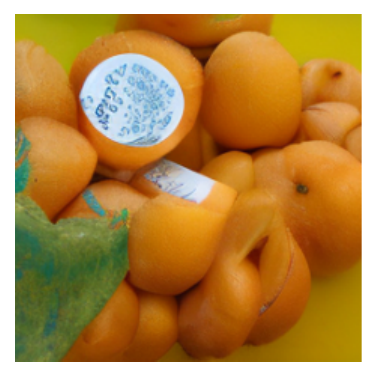}
& \includegraphics[width=0.155\linewidth]{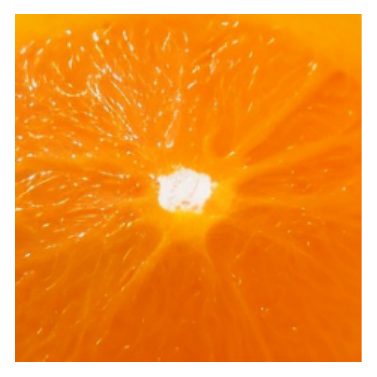}
& \includegraphics[width=0.155\linewidth]{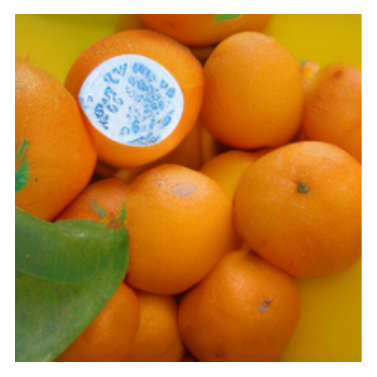} \\[-2pt]
\includegraphics[width=0.155\linewidth]{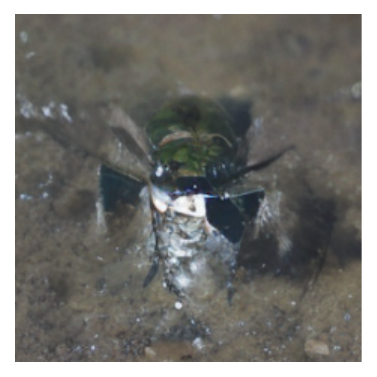}
& \includegraphics[width=0.155\linewidth]{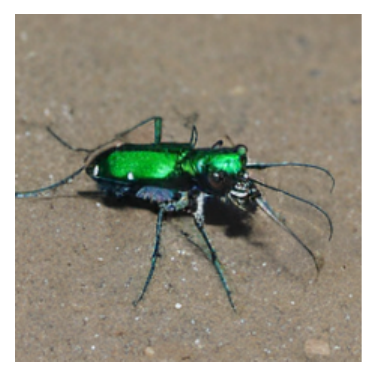}
& \includegraphics[width=0.155\linewidth]{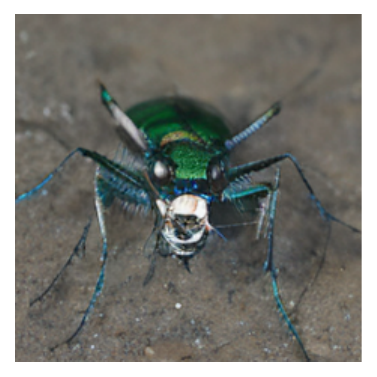}
& \includegraphics[width=0.155\linewidth]{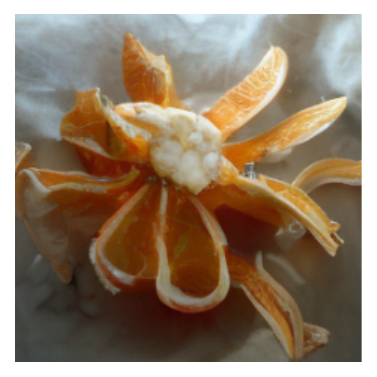}
& \includegraphics[width=0.155\linewidth]{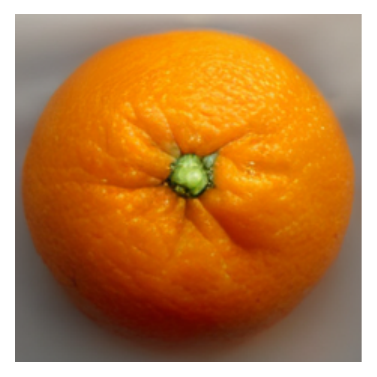}
& \includegraphics[width=0.155\linewidth]{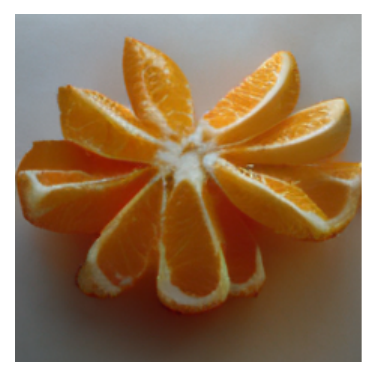}
\end{tabular}
\caption{Additional matched qualitative comparisons on ImageNet-256. Each
three-column group compares no guidance, fixed CFG, and PMC-CFG; corresponding
images within each group use the same class and initial latent seed. Fixed CFG
and PMC-CFG both use nominal guidance scale $\lambda=2.5$.}
\label{fig:imagenet-matched-extra}
\end{figure}

\begin{figure}[p]
\centering
\setlength{\tabcolsep}{0.5pt}
\resizebox{0.86\linewidth}{!}{%
\begin{tabular}{@{}ccc@{\hspace{6pt}}ccc@{}}
\multicolumn{3}{c}{\small Dog} & \multicolumn{3}{c}{\small Landscape} \\
\shortstack[c]{\tiny No guidance\\[-1pt]\tiny $\lambda=1$}
& \shortstack[c]{\tiny Fixed CFG\\[-1pt]\tiny $\lambda=5$}
& \shortstack[c]{\tiny PMC-CFG\\[-1pt]\tiny $\lambda=5$}
& \shortstack[c]{\tiny No guidance\\[-1pt]\tiny $\lambda=1$}
& \shortstack[c]{\tiny Fixed CFG\\[-1pt]\tiny $\lambda=5$}
& \shortstack[c]{\tiny PMC-CFG\\[-1pt]\tiny $\lambda=5$} \\
\includegraphics[width=0.155\linewidth]{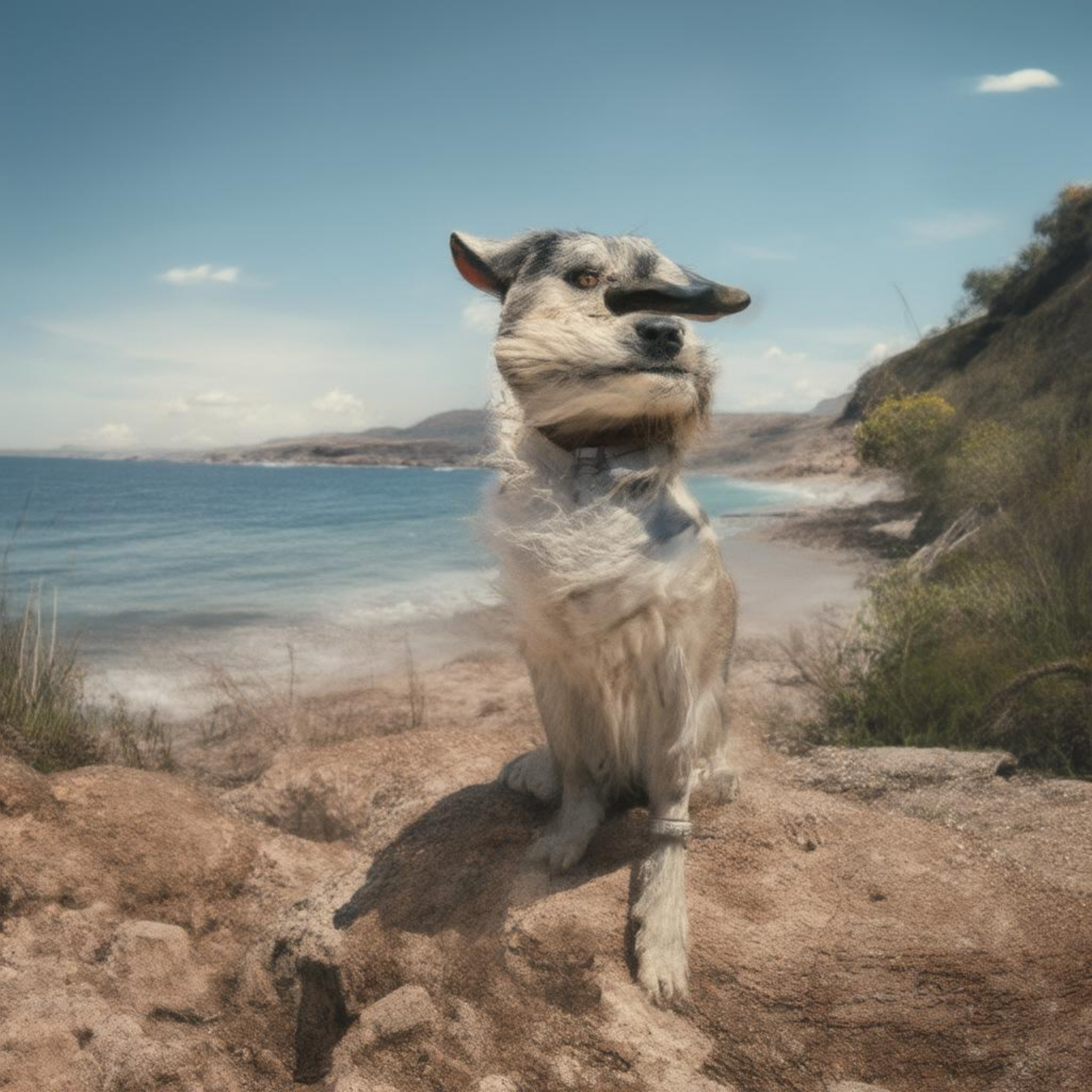}
& \includegraphics[width=0.155\linewidth]{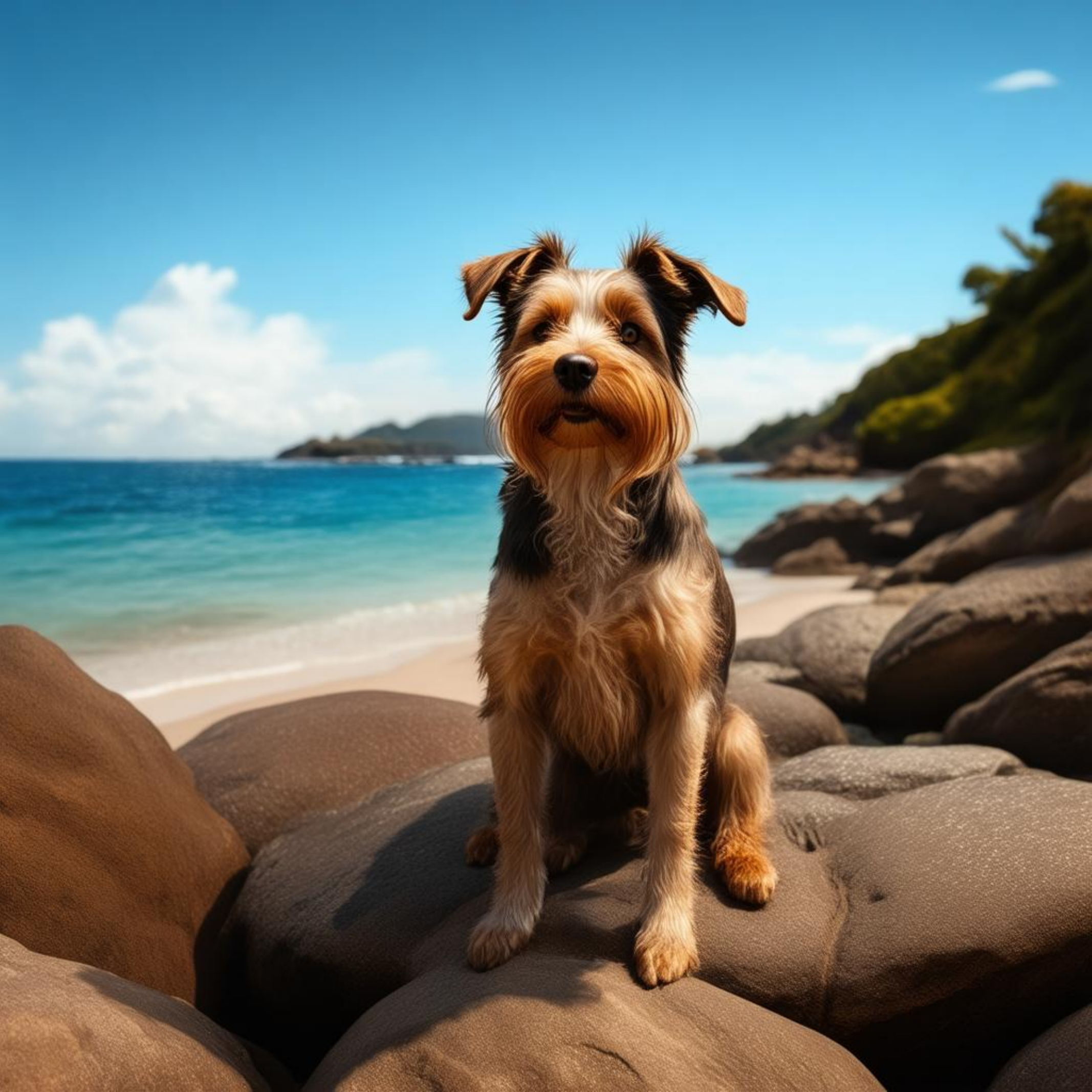}
& \includegraphics[width=0.155\linewidth]{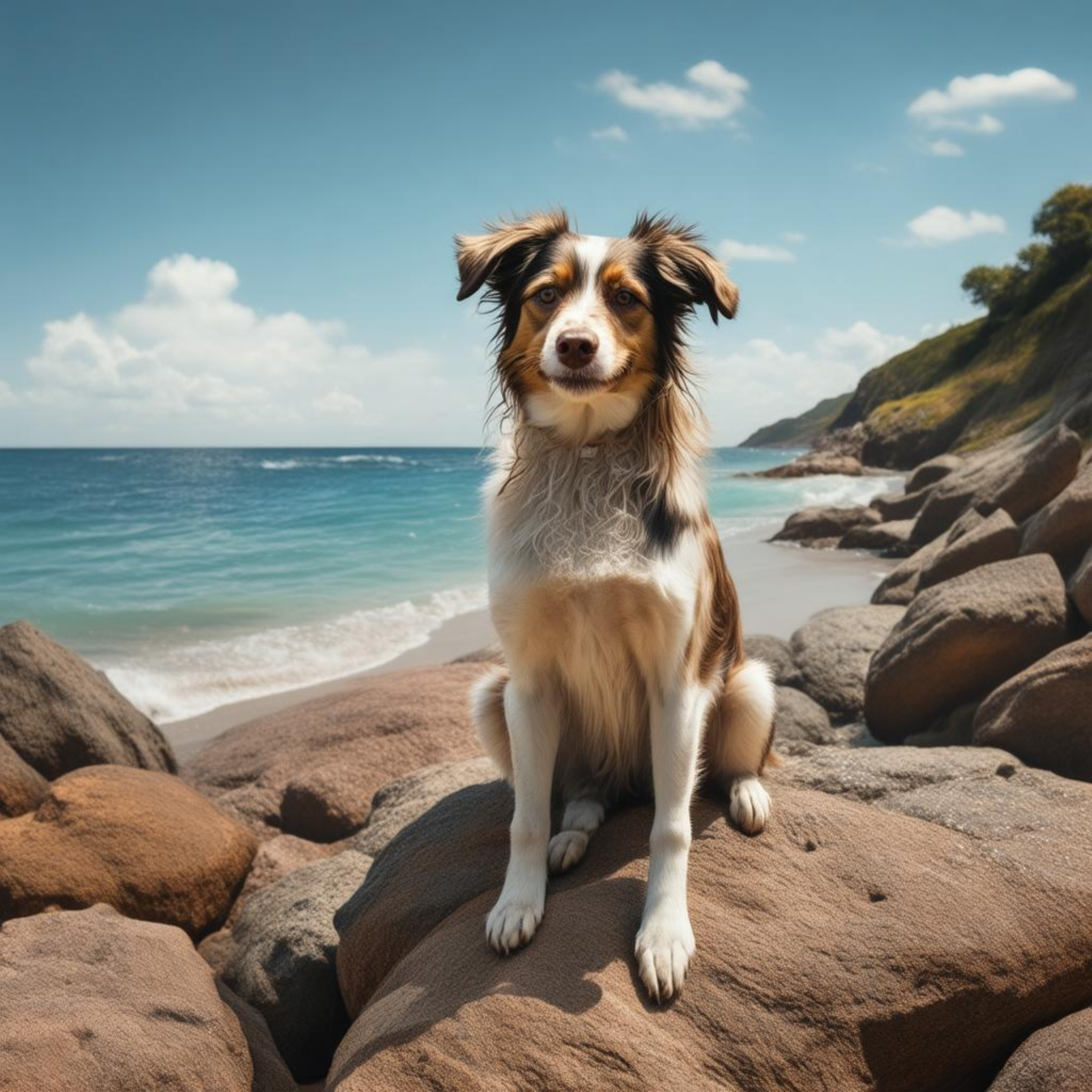}
& \includegraphics[width=0.155\linewidth]{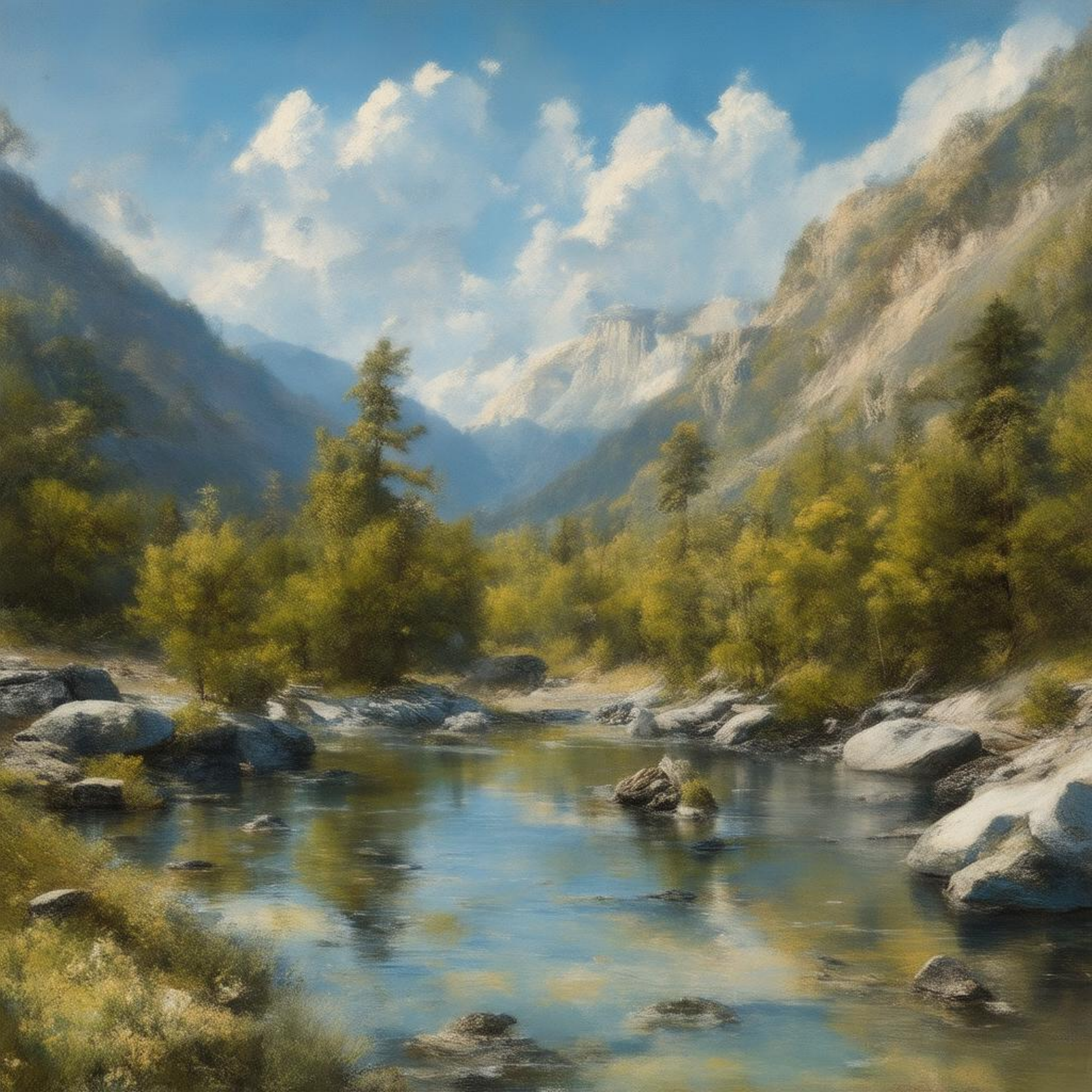}
& \includegraphics[width=0.155\linewidth]{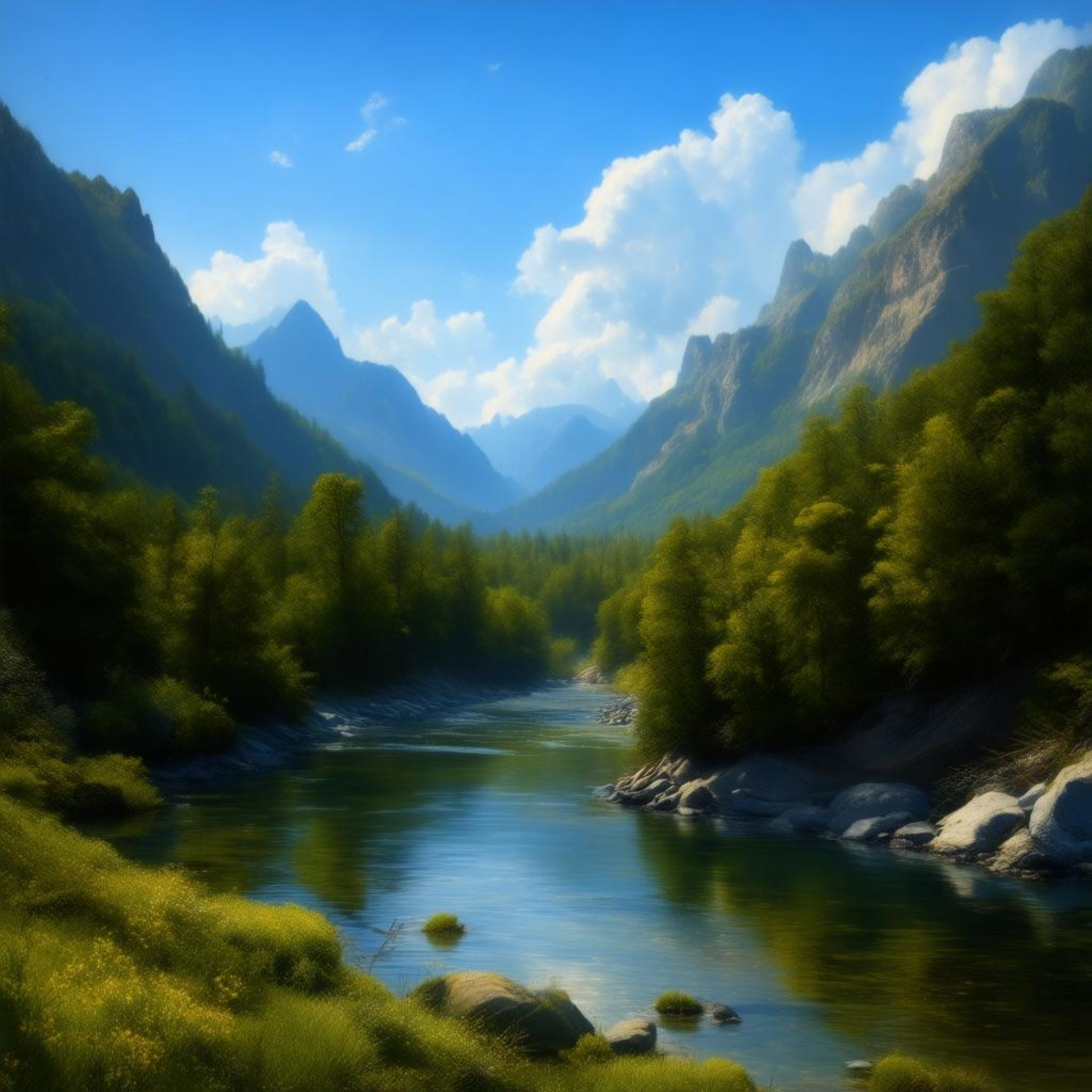}
& \includegraphics[width=0.155\linewidth]{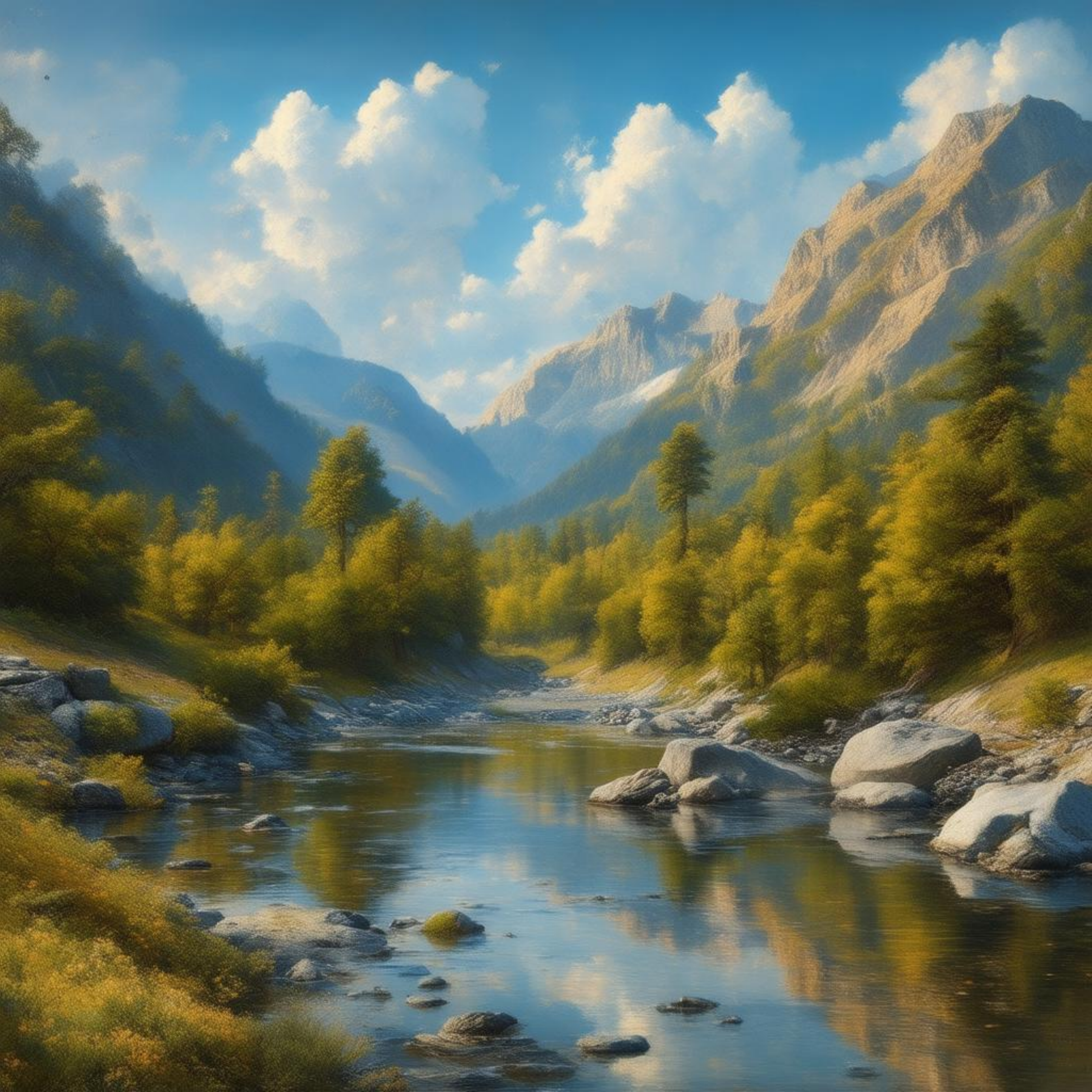} \\[-2pt]
\includegraphics[width=0.155\linewidth]{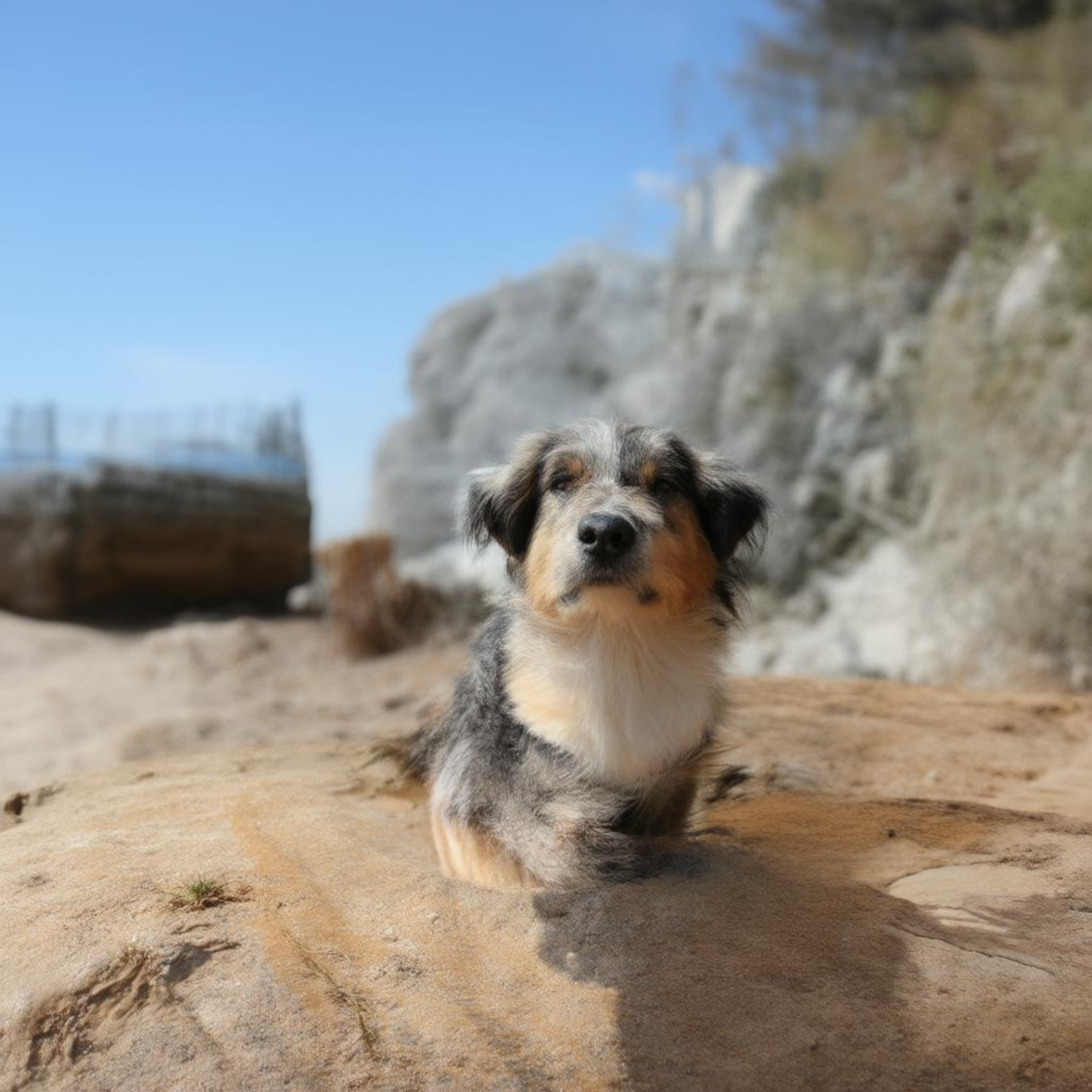}
& \includegraphics[width=0.155\linewidth]{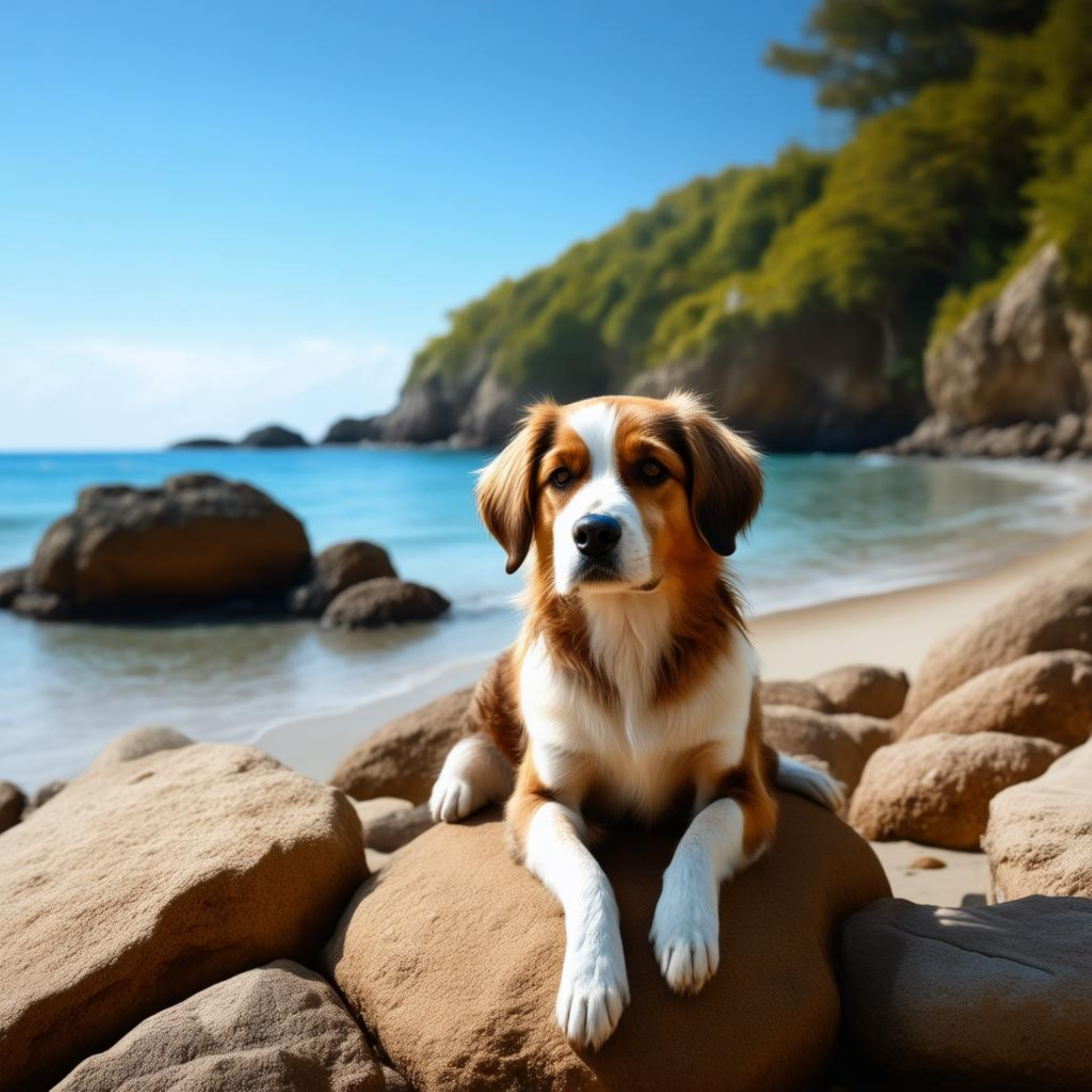}
& \includegraphics[width=0.155\linewidth]{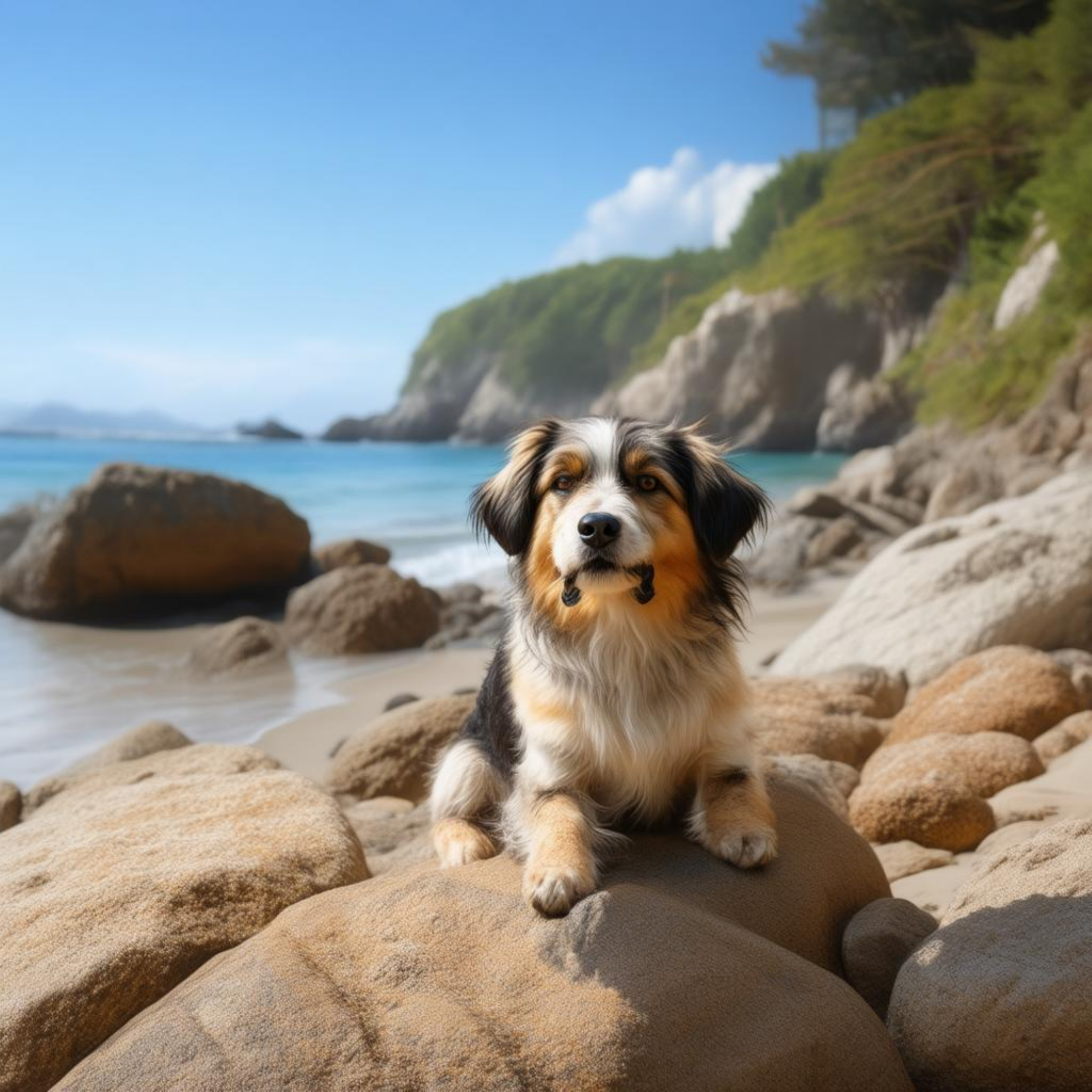}
& \includegraphics[width=0.155\linewidth]{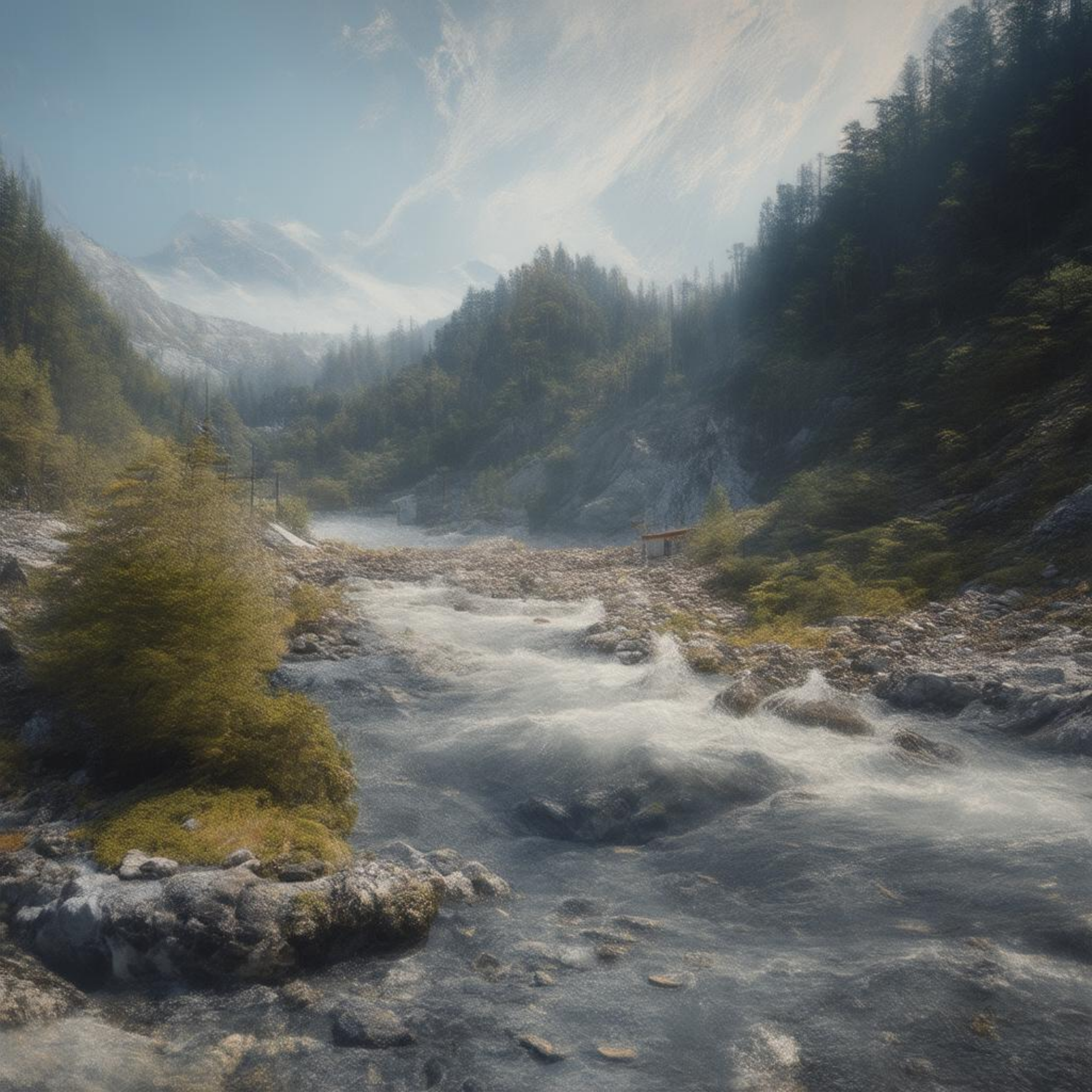}
& \includegraphics[width=0.155\linewidth]{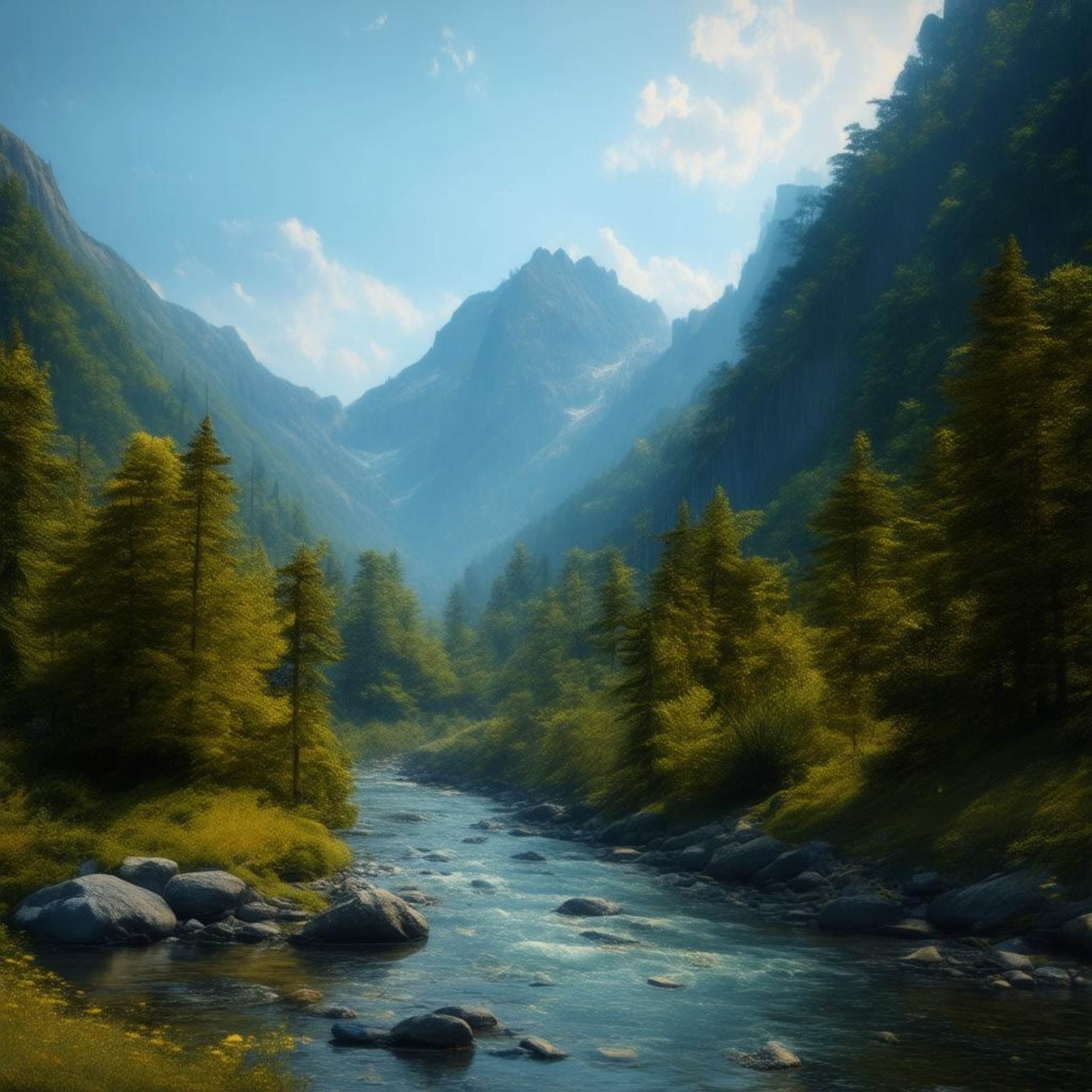}
& \includegraphics[width=0.155\linewidth]{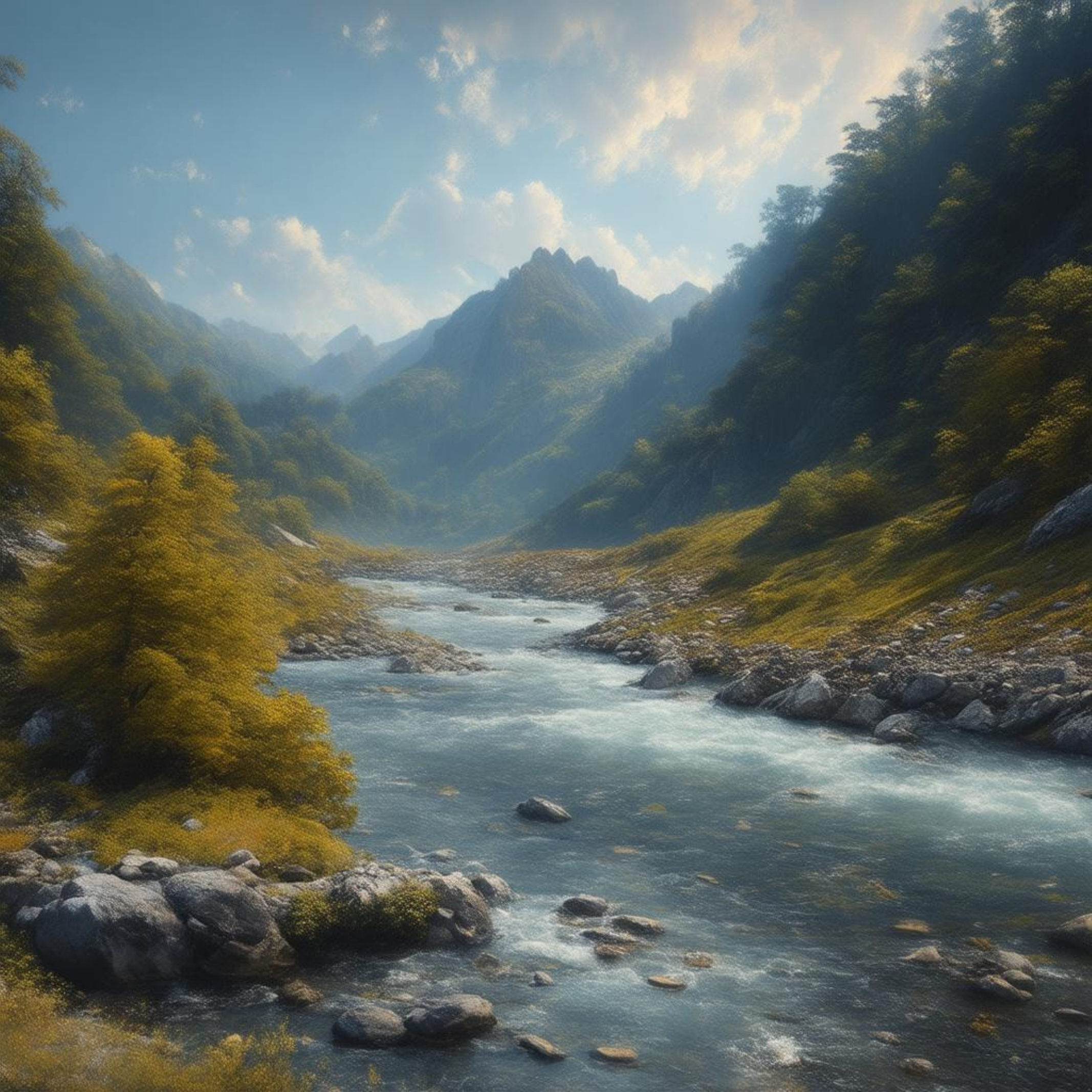} \\[-2pt]
\includegraphics[width=0.155\linewidth]{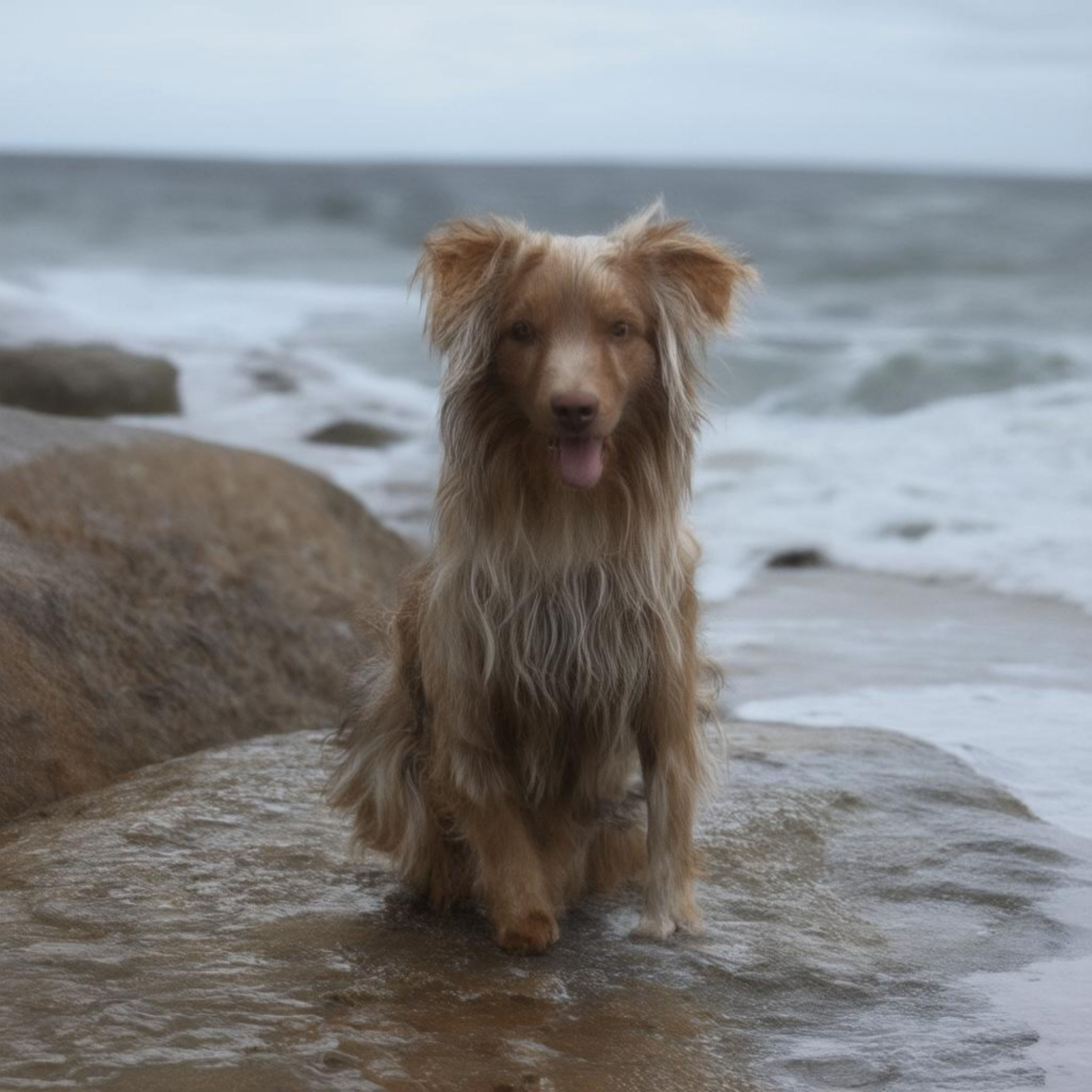}
& \includegraphics[width=0.155\linewidth]{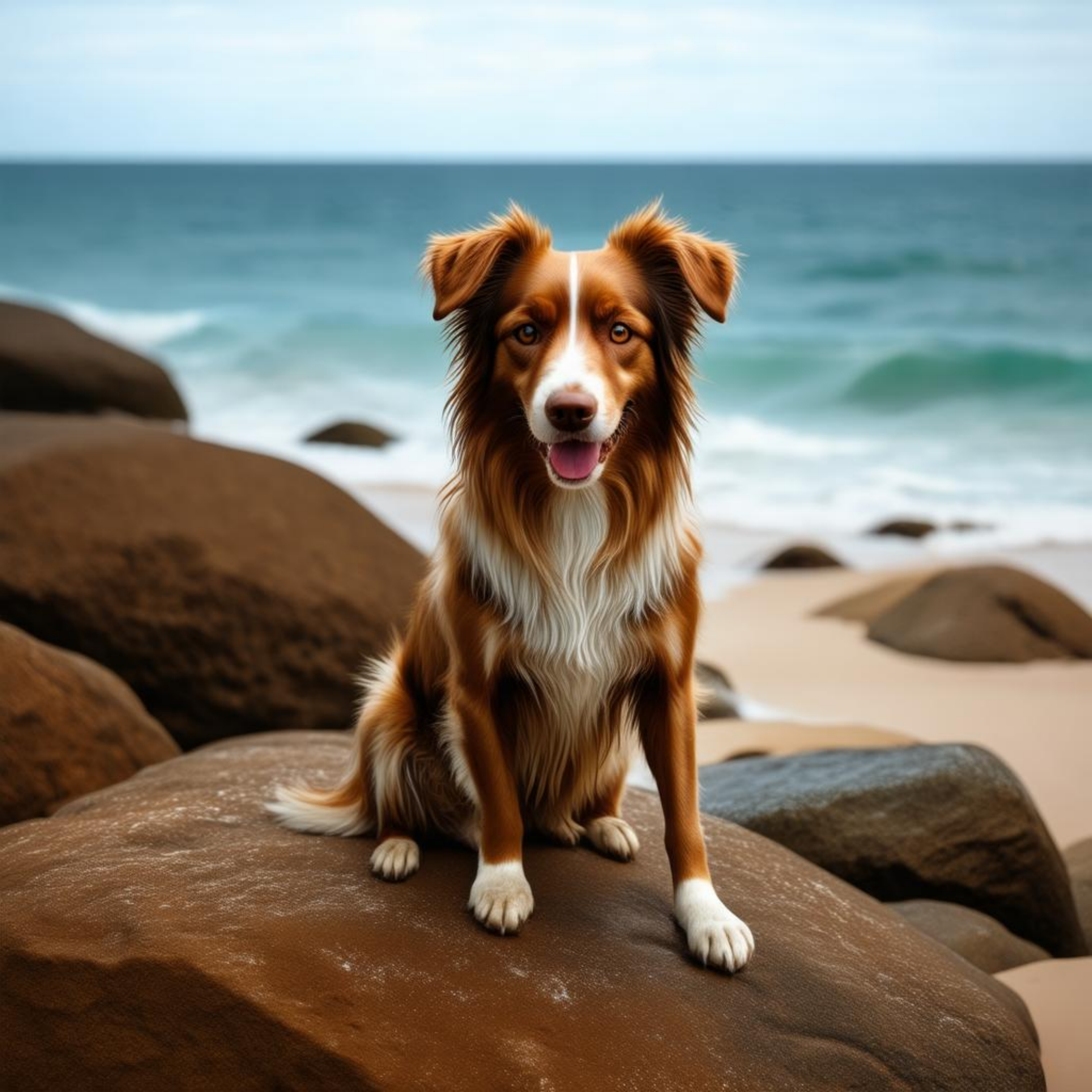}
& \includegraphics[width=0.155\linewidth]{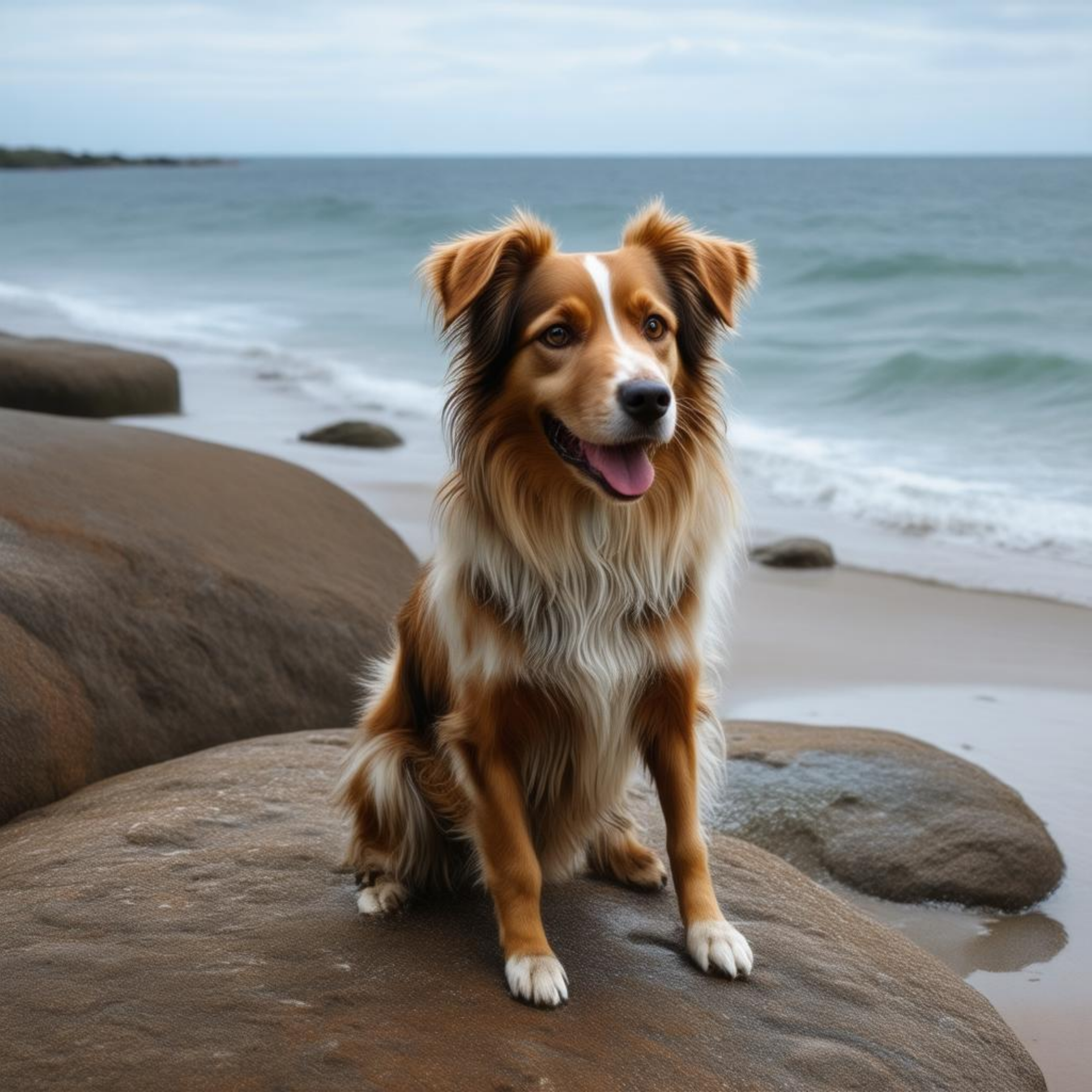}
& \includegraphics[width=0.155\linewidth]{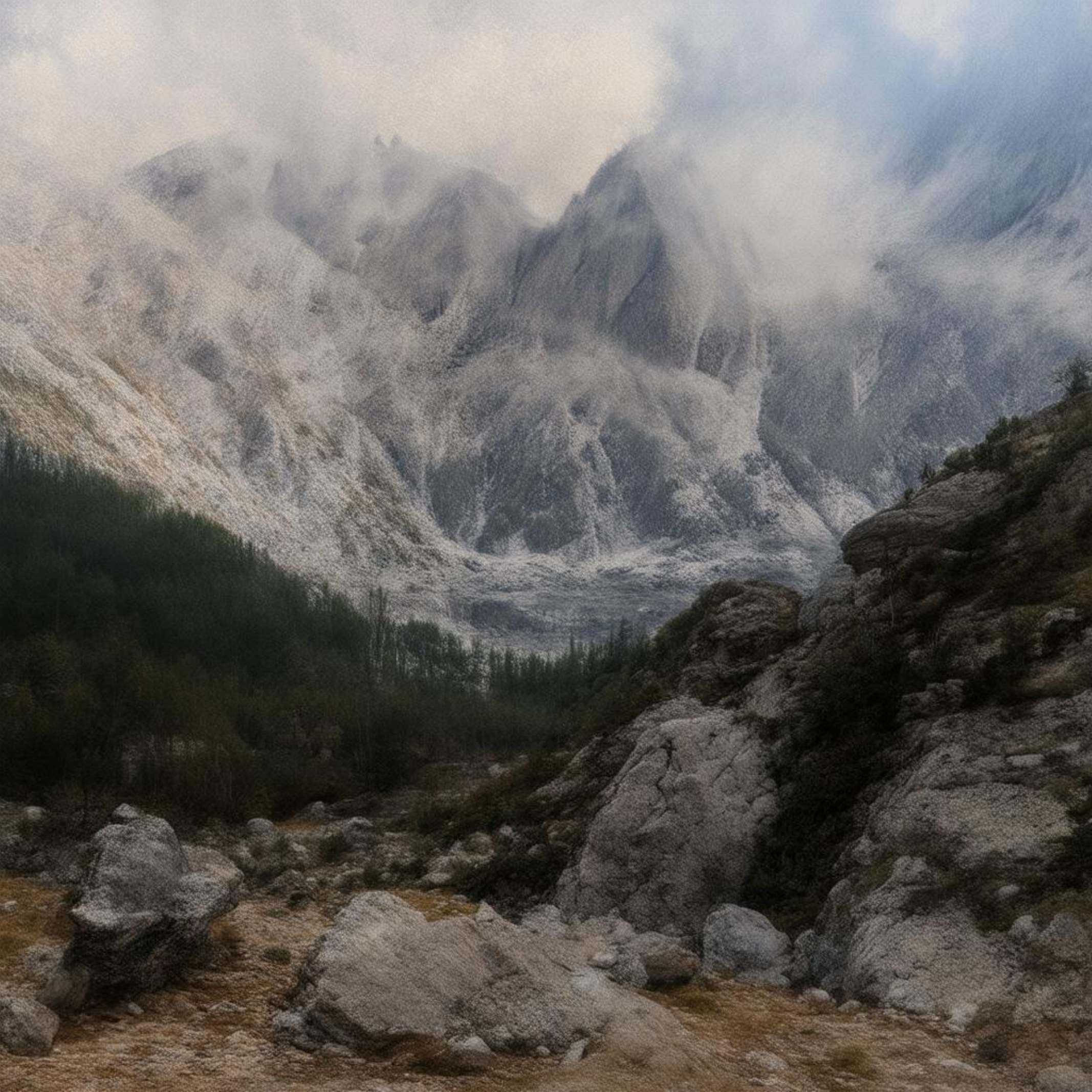}
& \includegraphics[width=0.155\linewidth]{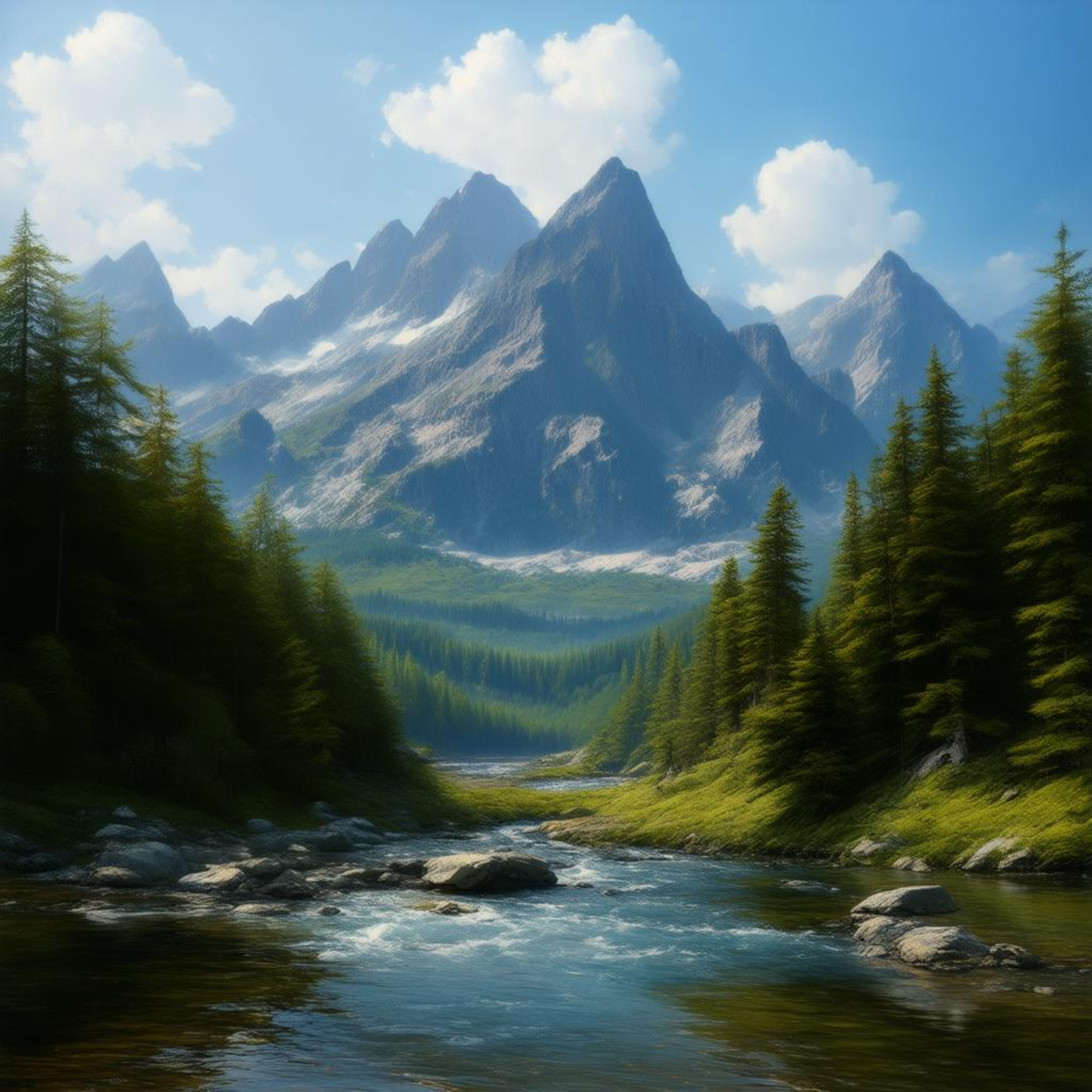}
& \includegraphics[width=0.155\linewidth]{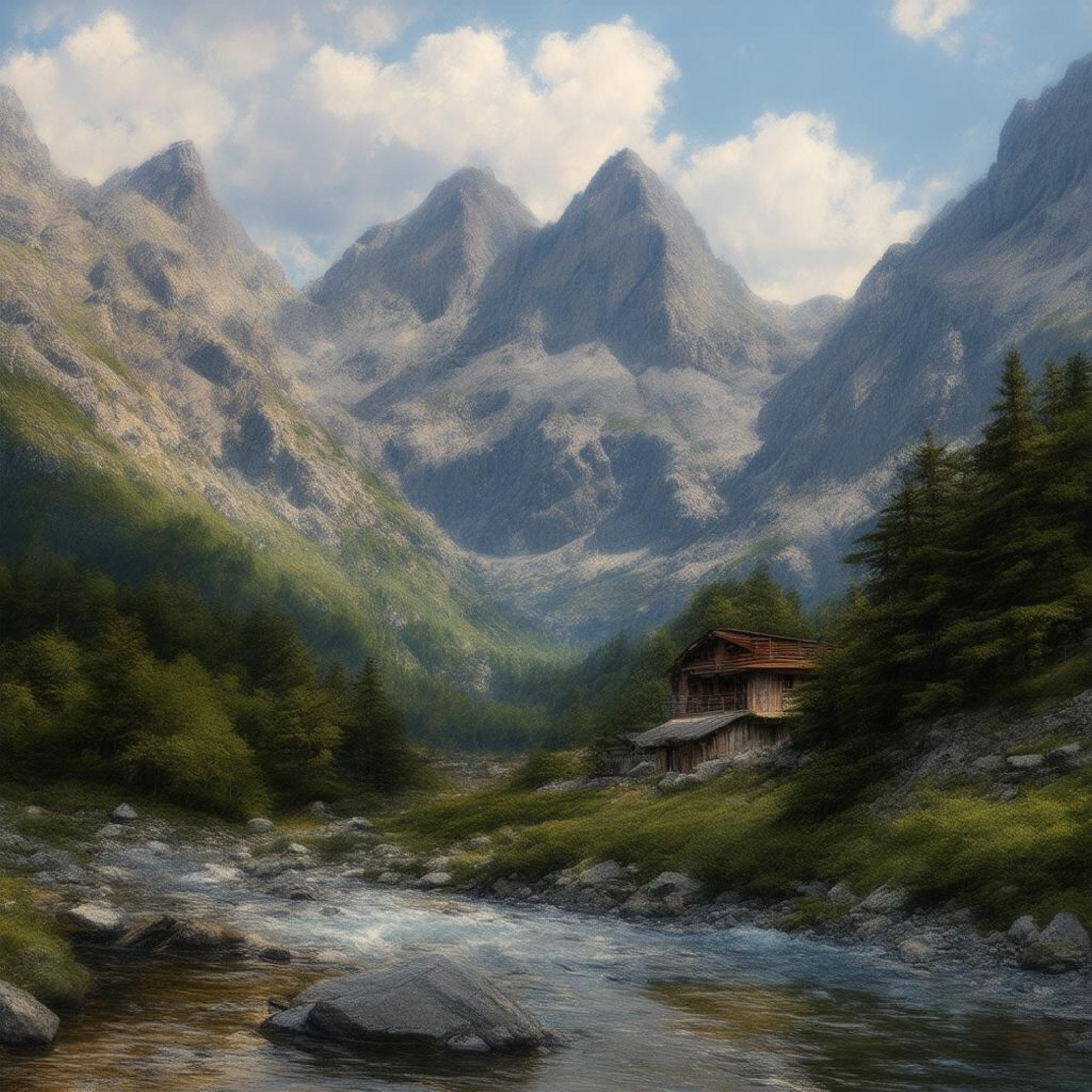} \\[-2pt]
\includegraphics[width=0.155\linewidth]{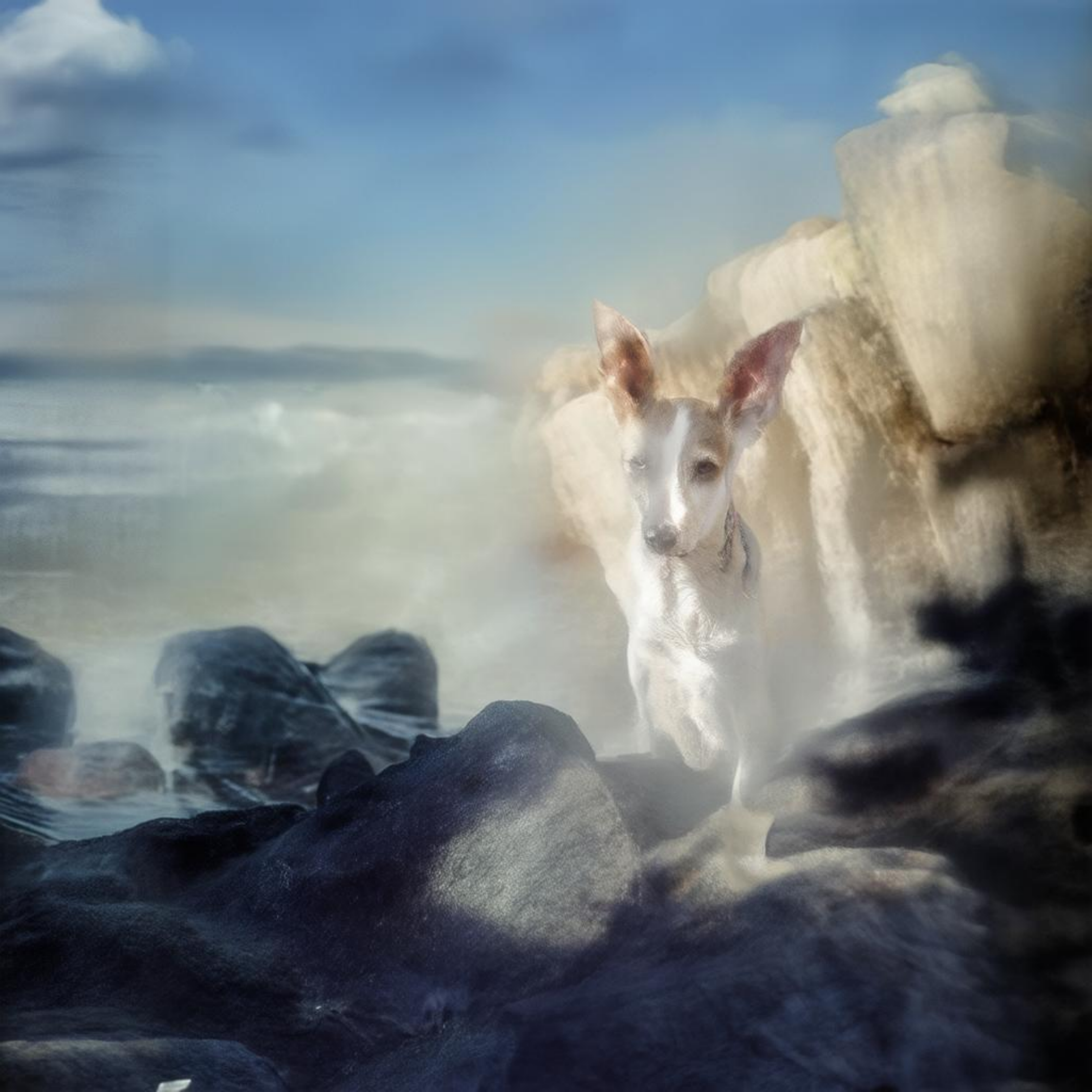}
& \includegraphics[width=0.155\linewidth]{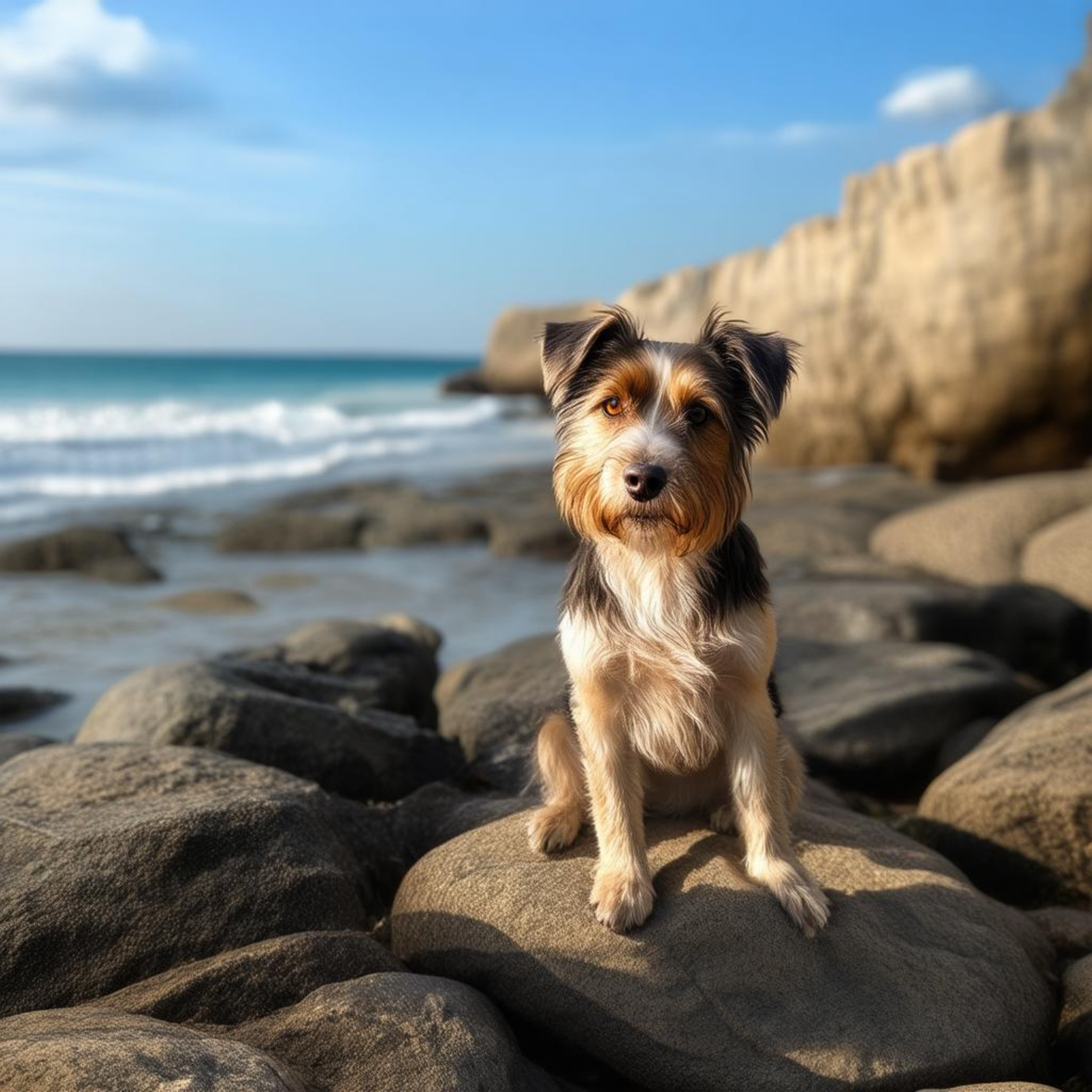}
& \includegraphics[width=0.155\linewidth]{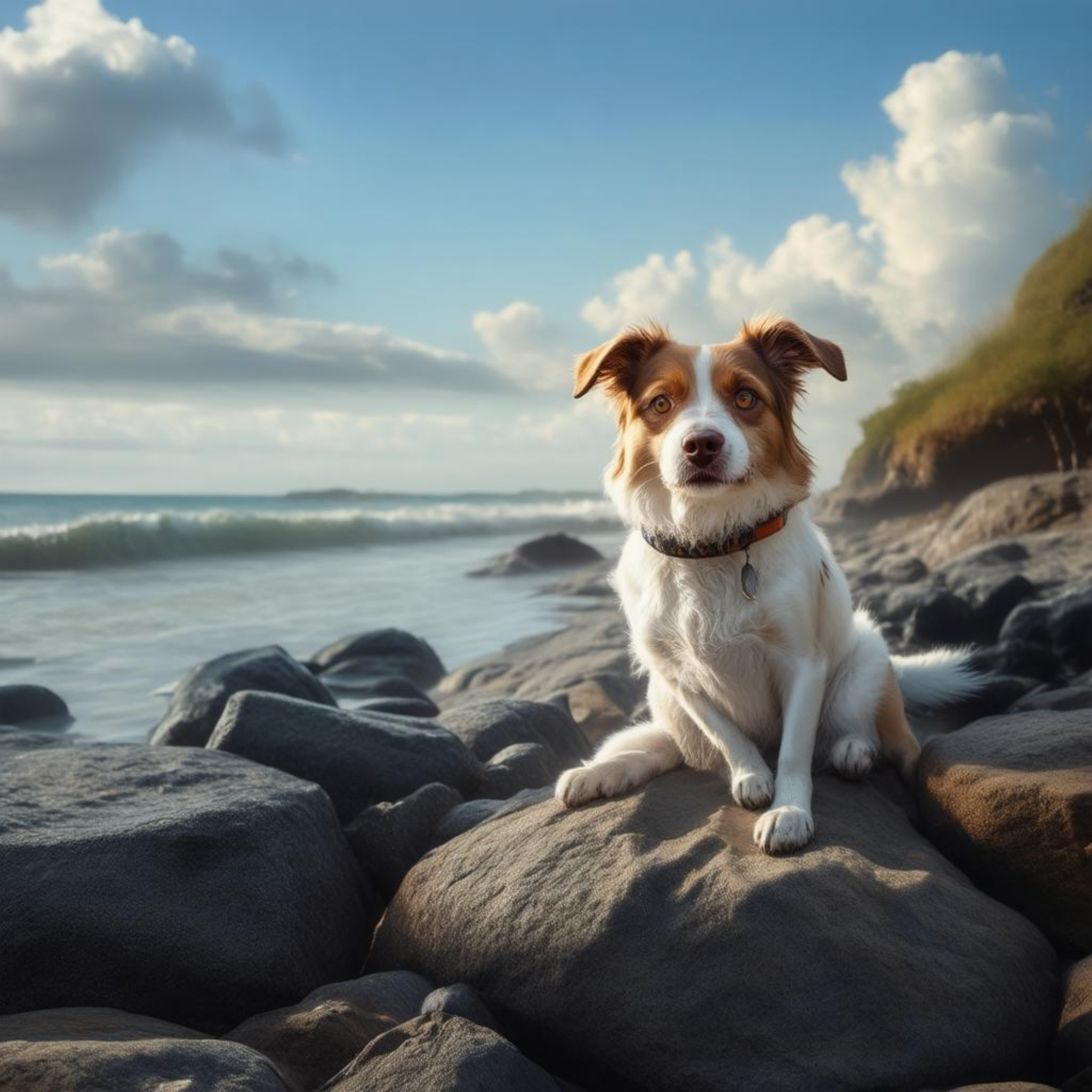}
& \includegraphics[width=0.155\linewidth]{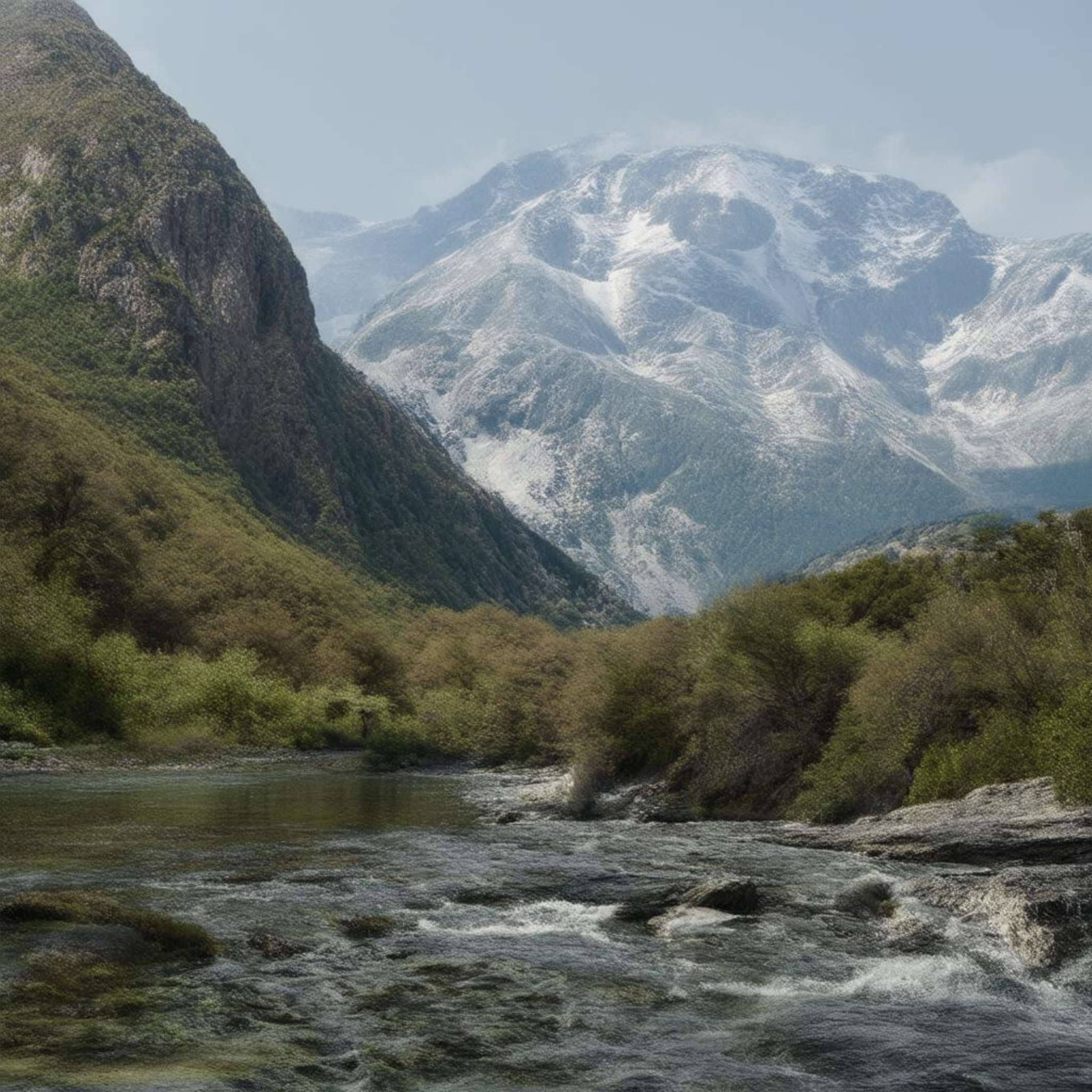}
& \includegraphics[width=0.155\linewidth]{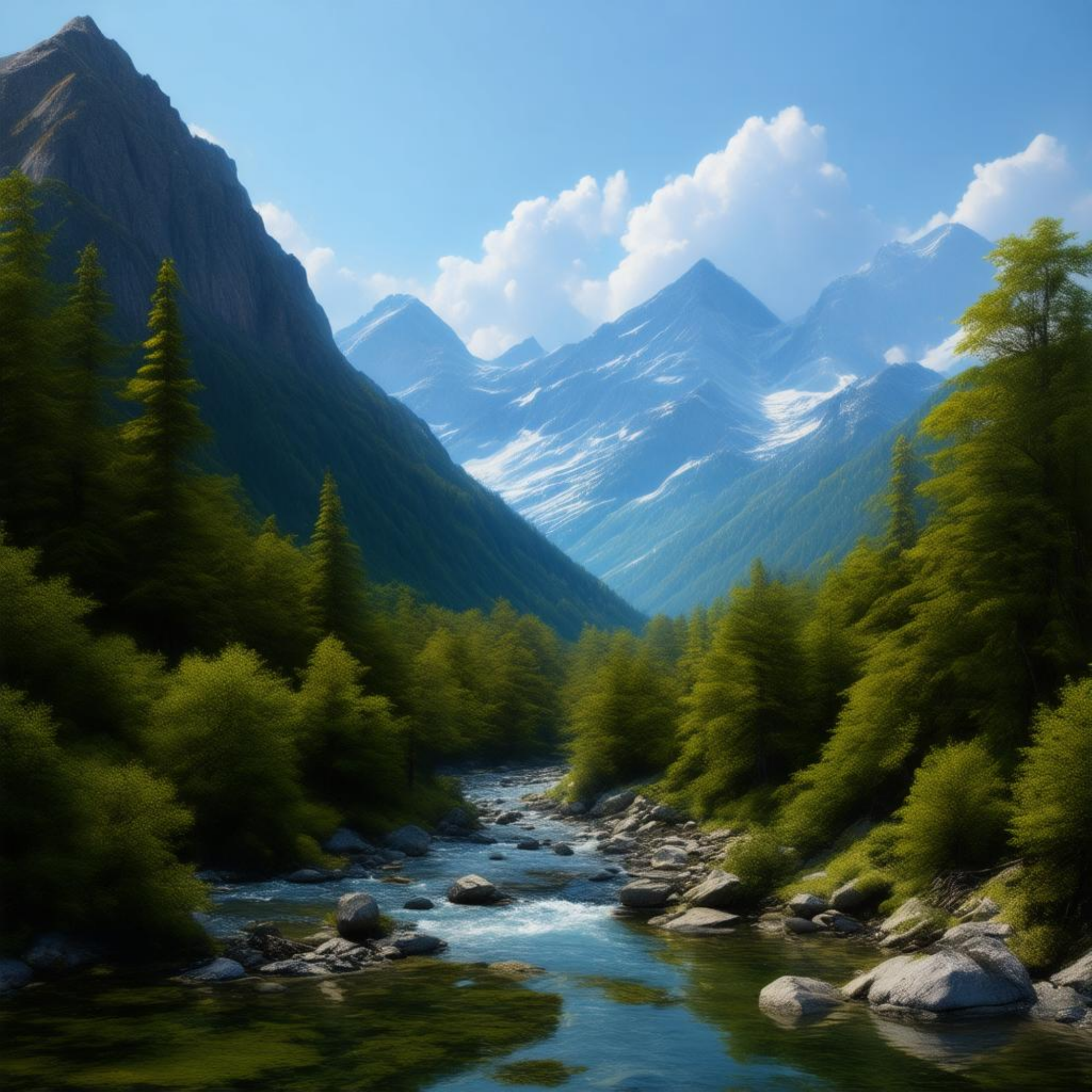}
& \includegraphics[width=0.155\linewidth]{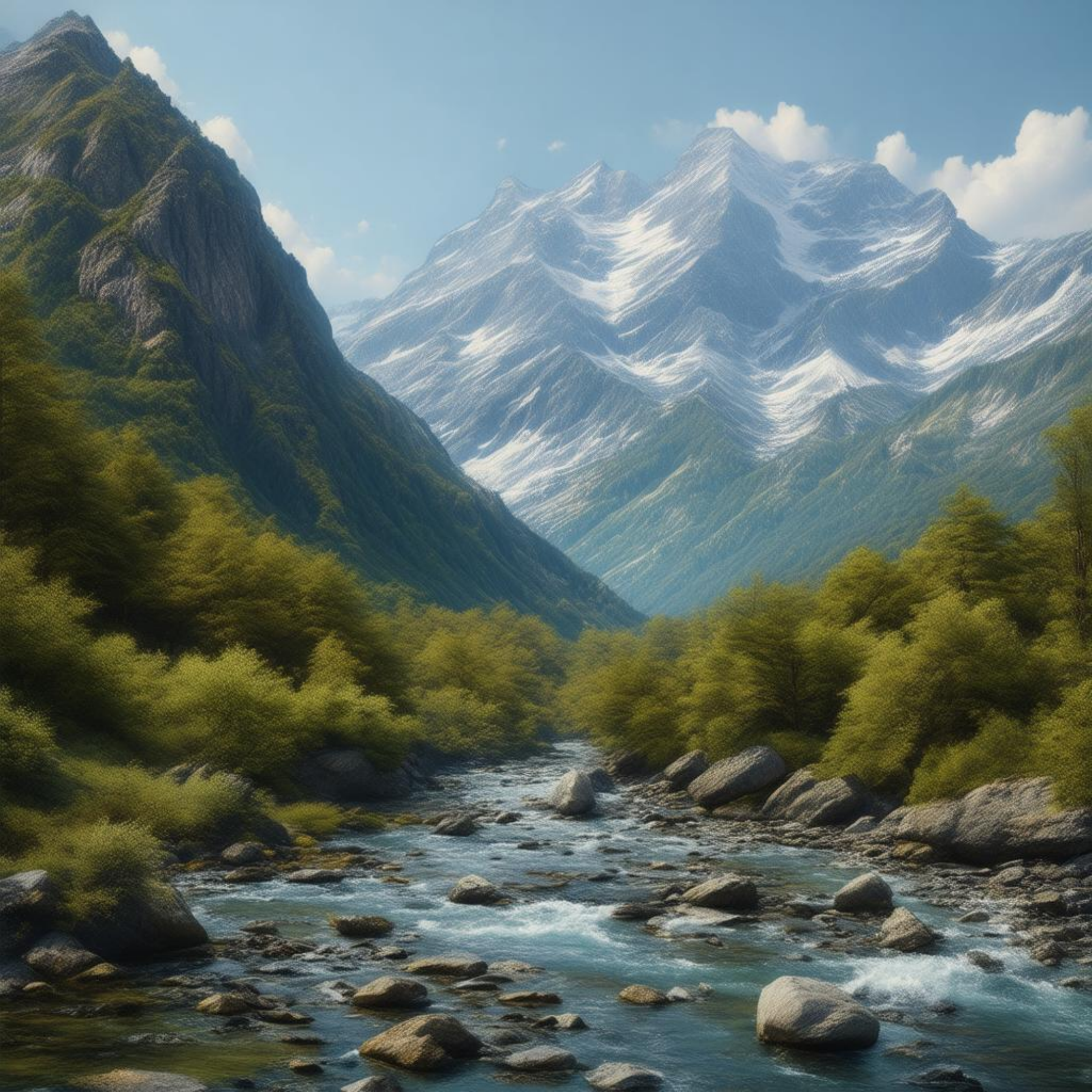} \\
\noalign{\vskip 6pt}
\multicolumn{3}{c}{\small Bedroom} & \multicolumn{3}{c}{\small People} \\
\shortstack[c]{\tiny No guidance\\[-1pt]\tiny $\lambda=1$}
& \shortstack[c]{\tiny Fixed CFG\\[-1pt]\tiny $\lambda=5$}
& \shortstack[c]{\tiny PMC-CFG\\[-1pt]\tiny $\lambda=5$}
& \shortstack[c]{\tiny No guidance\\[-1pt]\tiny $\lambda=1$}
& \shortstack[c]{\tiny Fixed CFG\\[-1pt]\tiny $\lambda=5$}
& \shortstack[c]{\tiny PMC-CFG\\[-1pt]\tiny $\lambda=5$} \\
\includegraphics[width=0.155\linewidth]{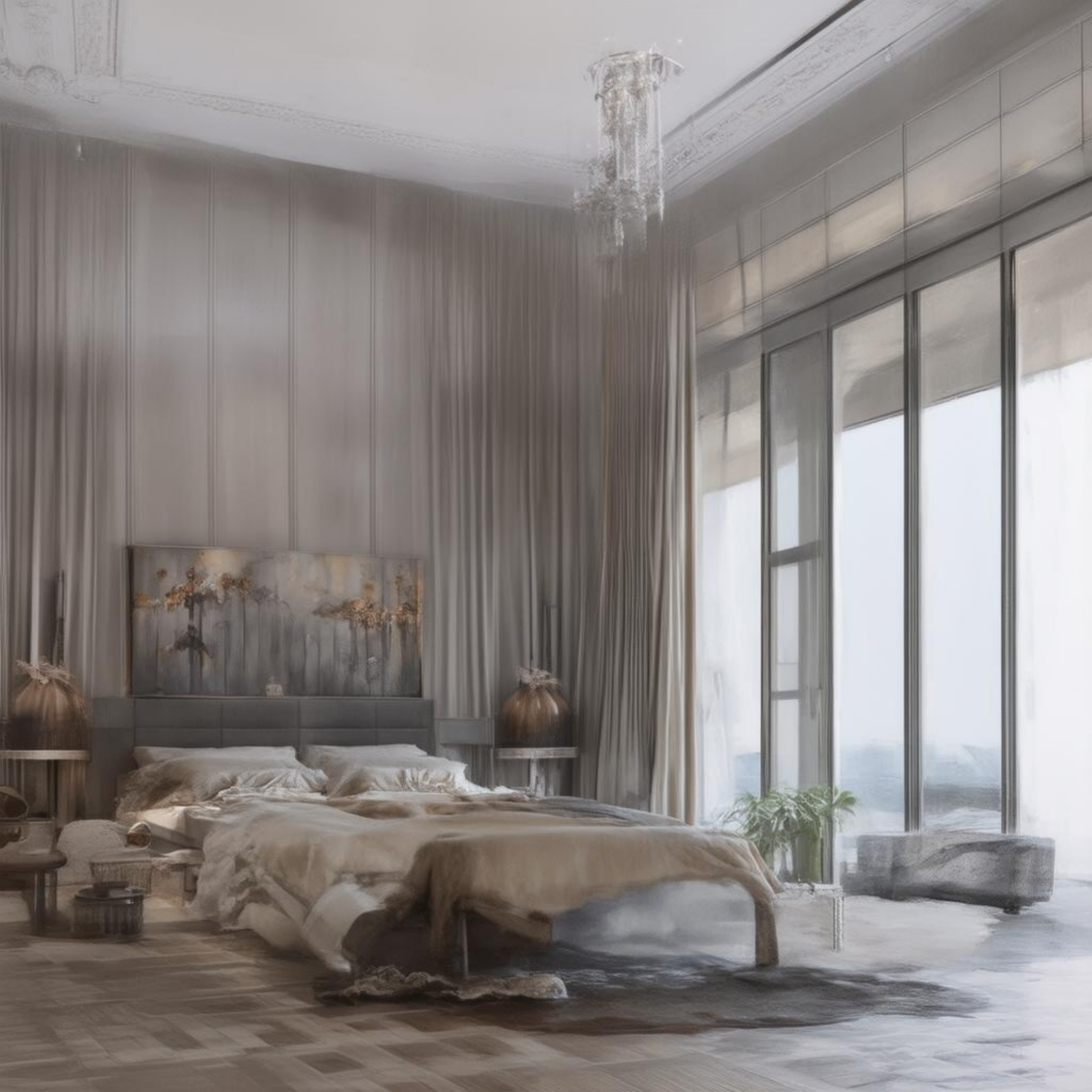}
& \includegraphics[width=0.155\linewidth]{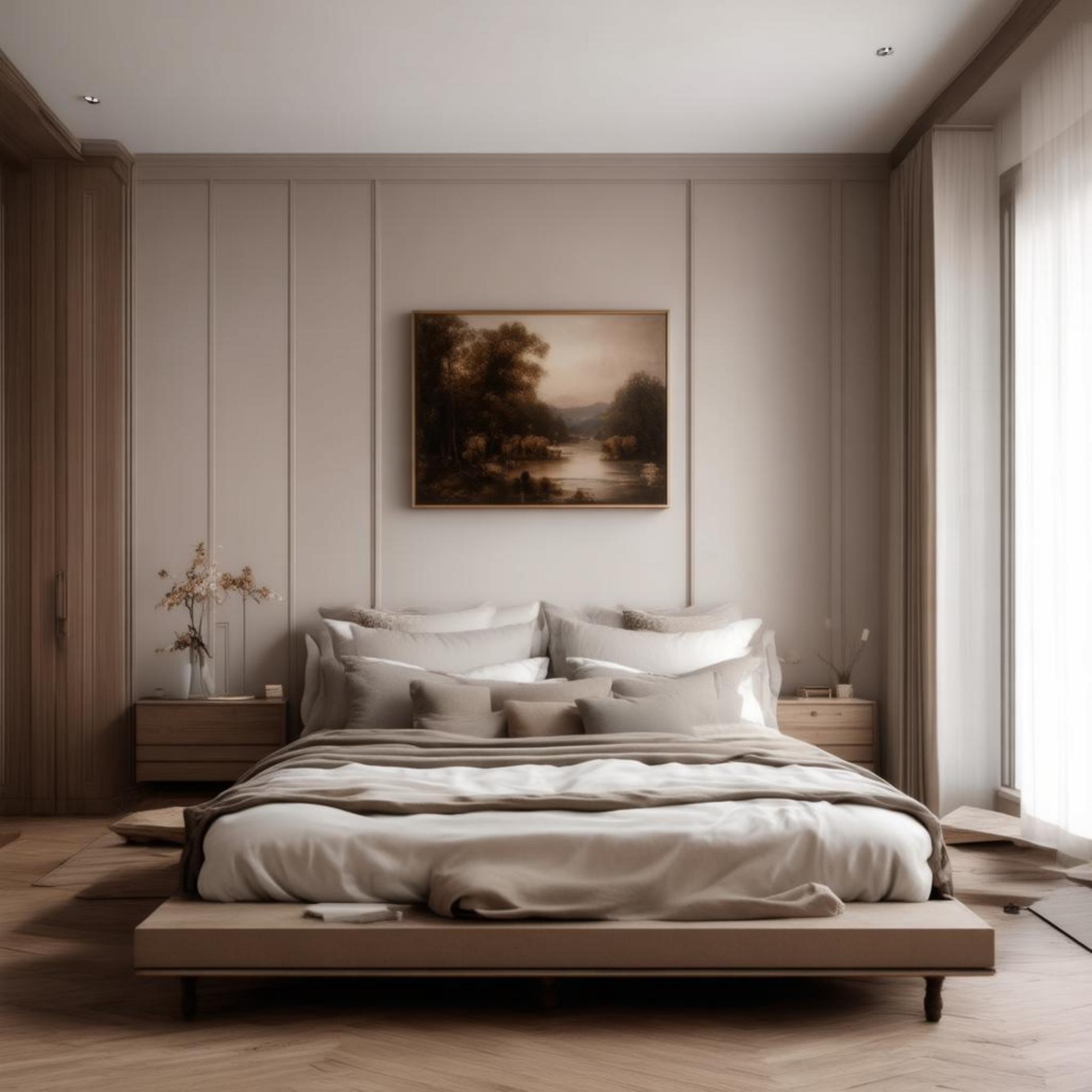}
& \includegraphics[width=0.155\linewidth]{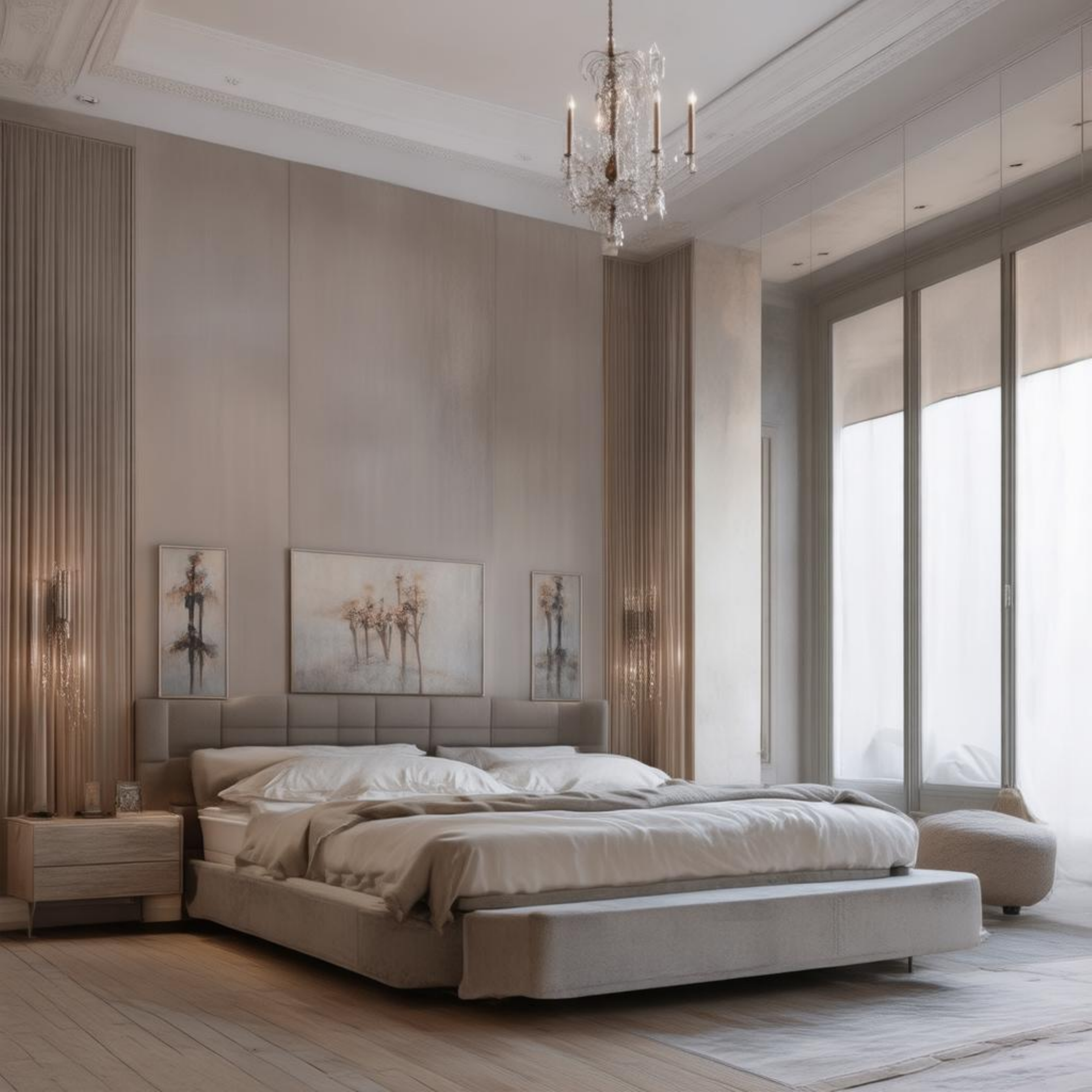}
& \includegraphics[width=0.155\linewidth]{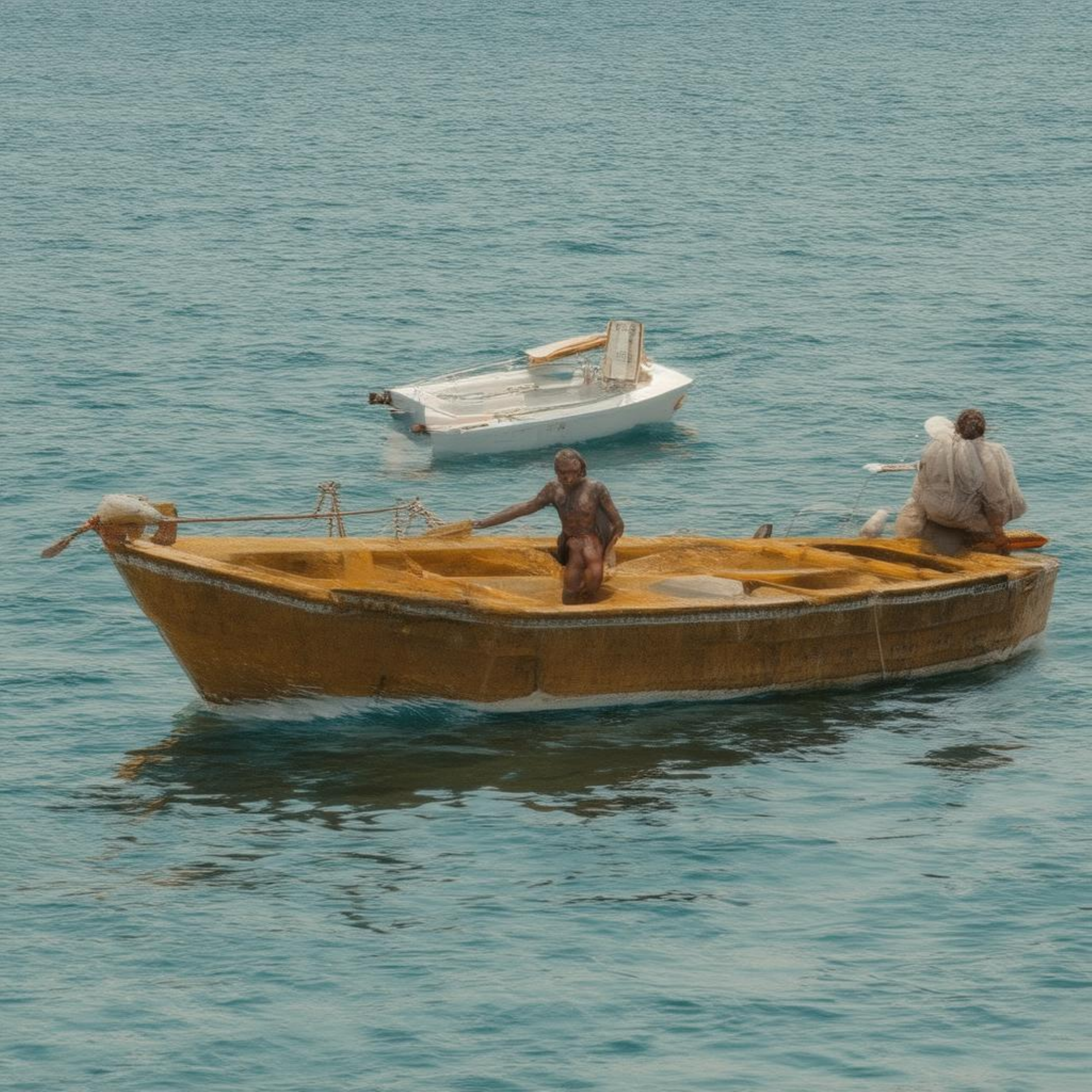}
& \includegraphics[width=0.155\linewidth]{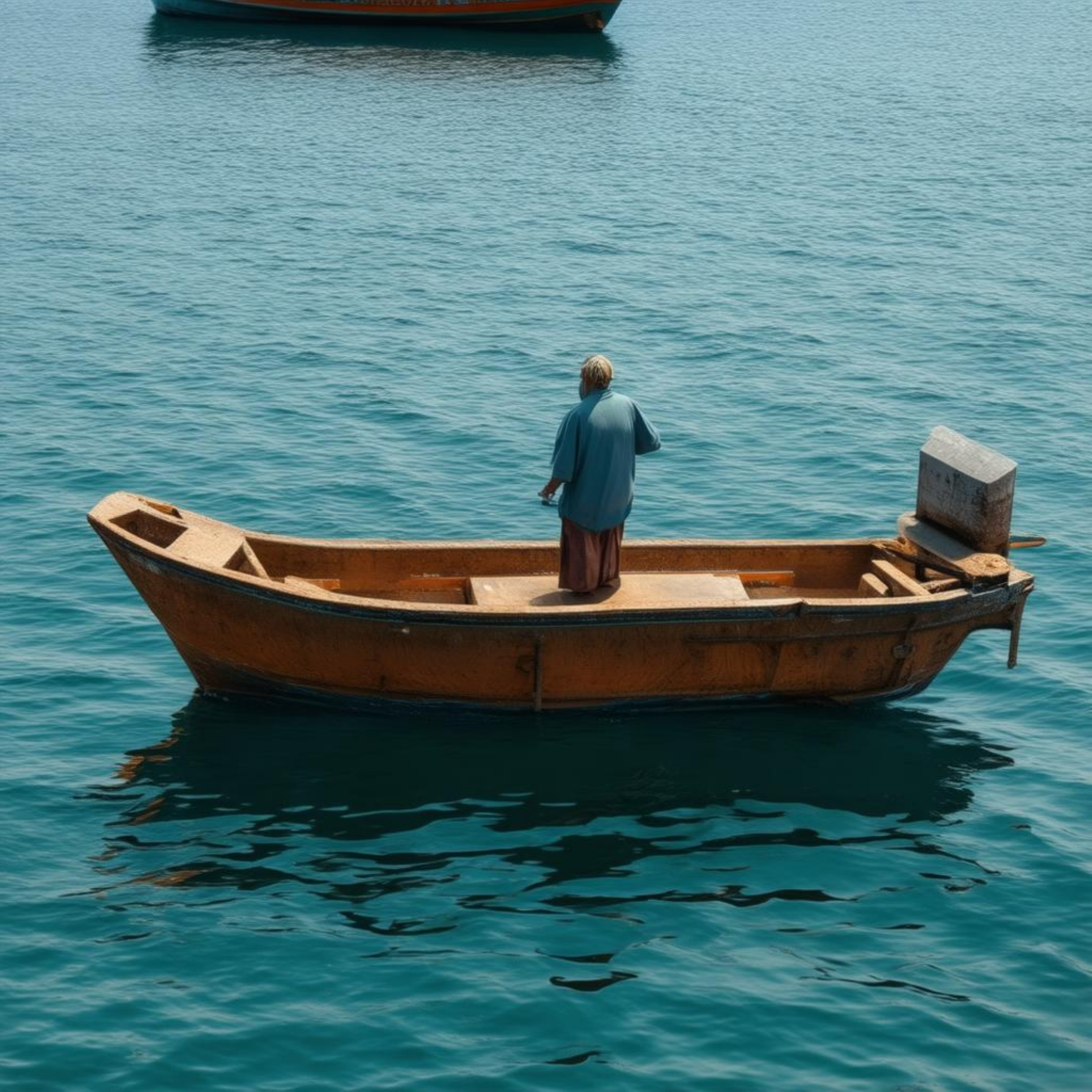}
& \includegraphics[width=0.155\linewidth]{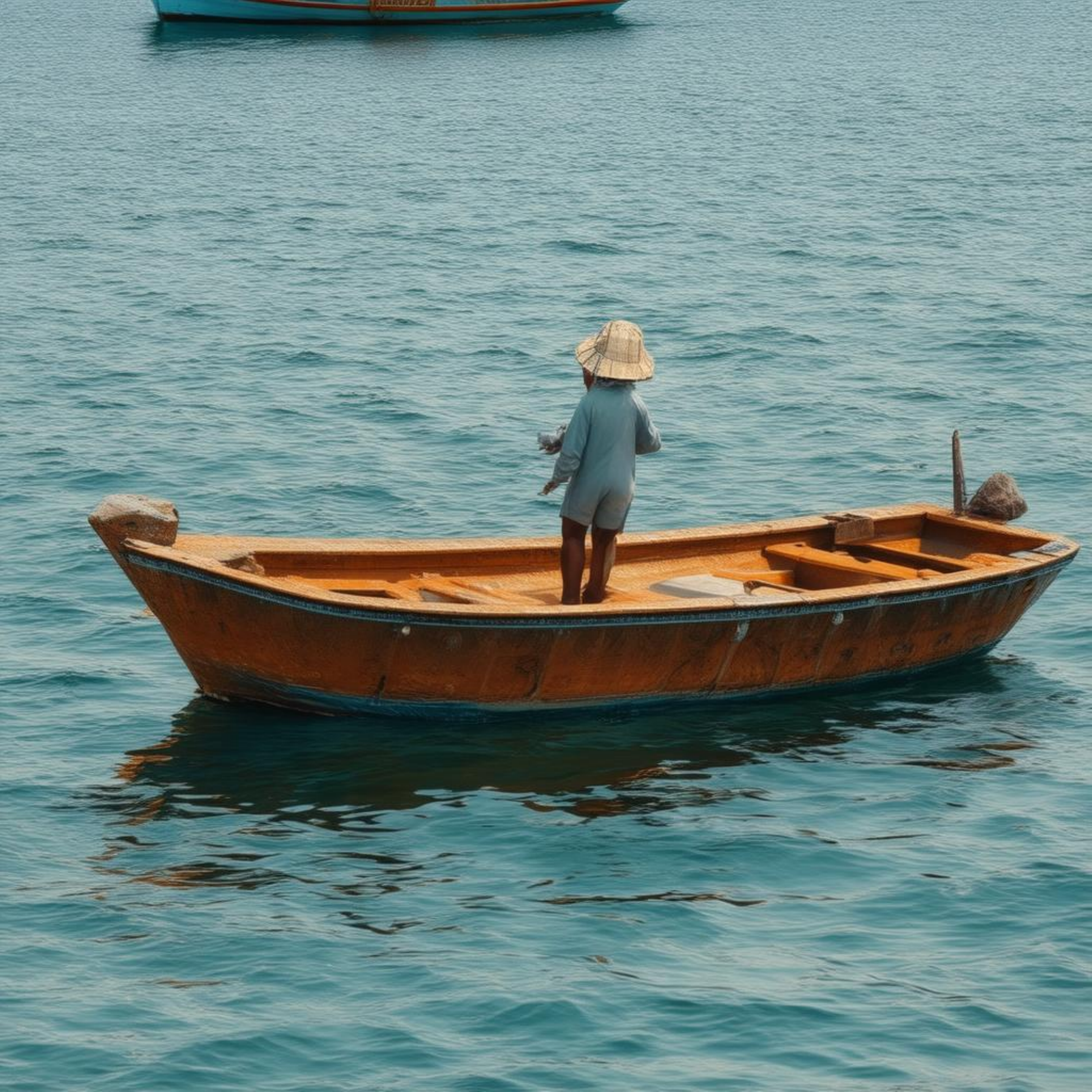} \\[-2pt]
\includegraphics[width=0.155\linewidth]{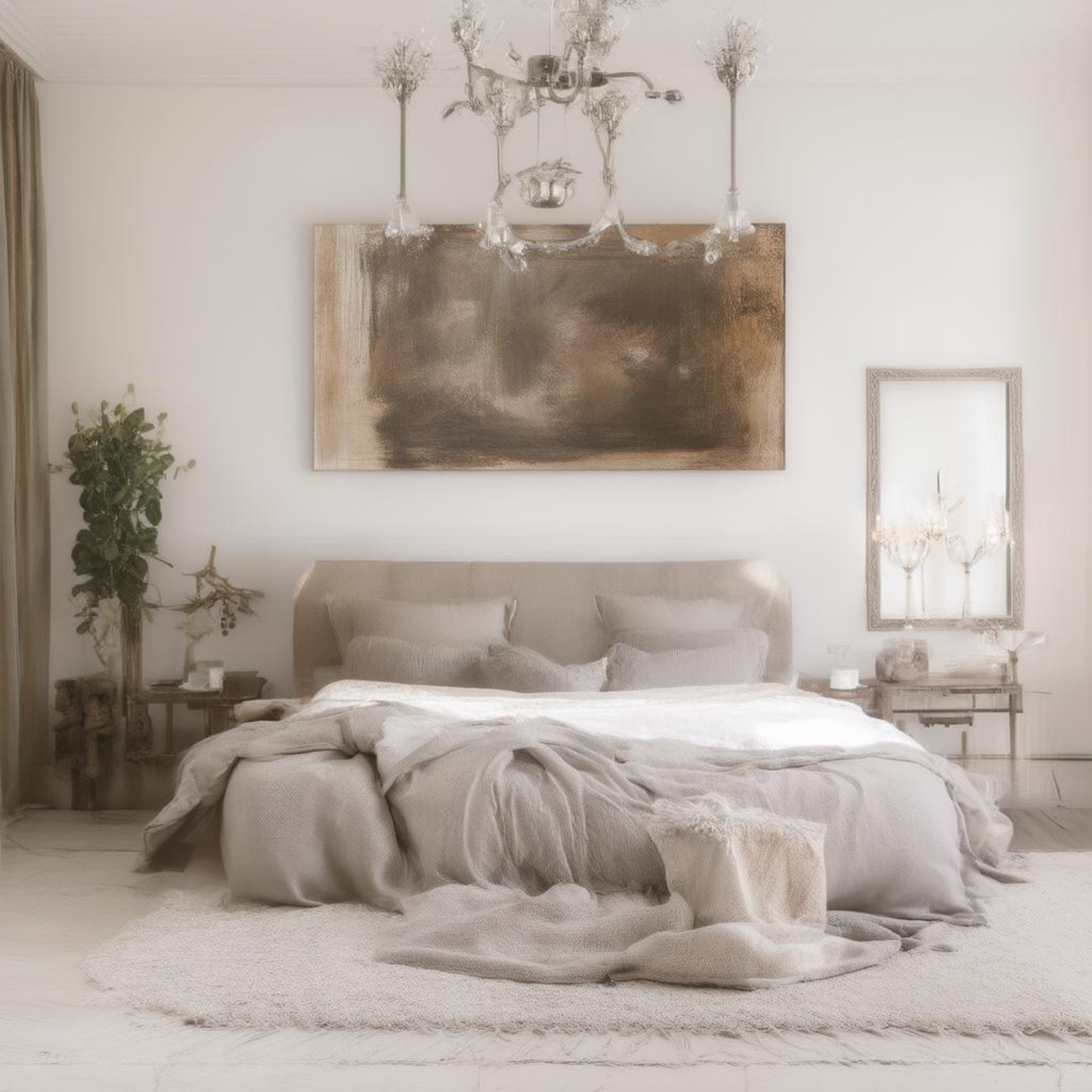}
& \includegraphics[width=0.155\linewidth]{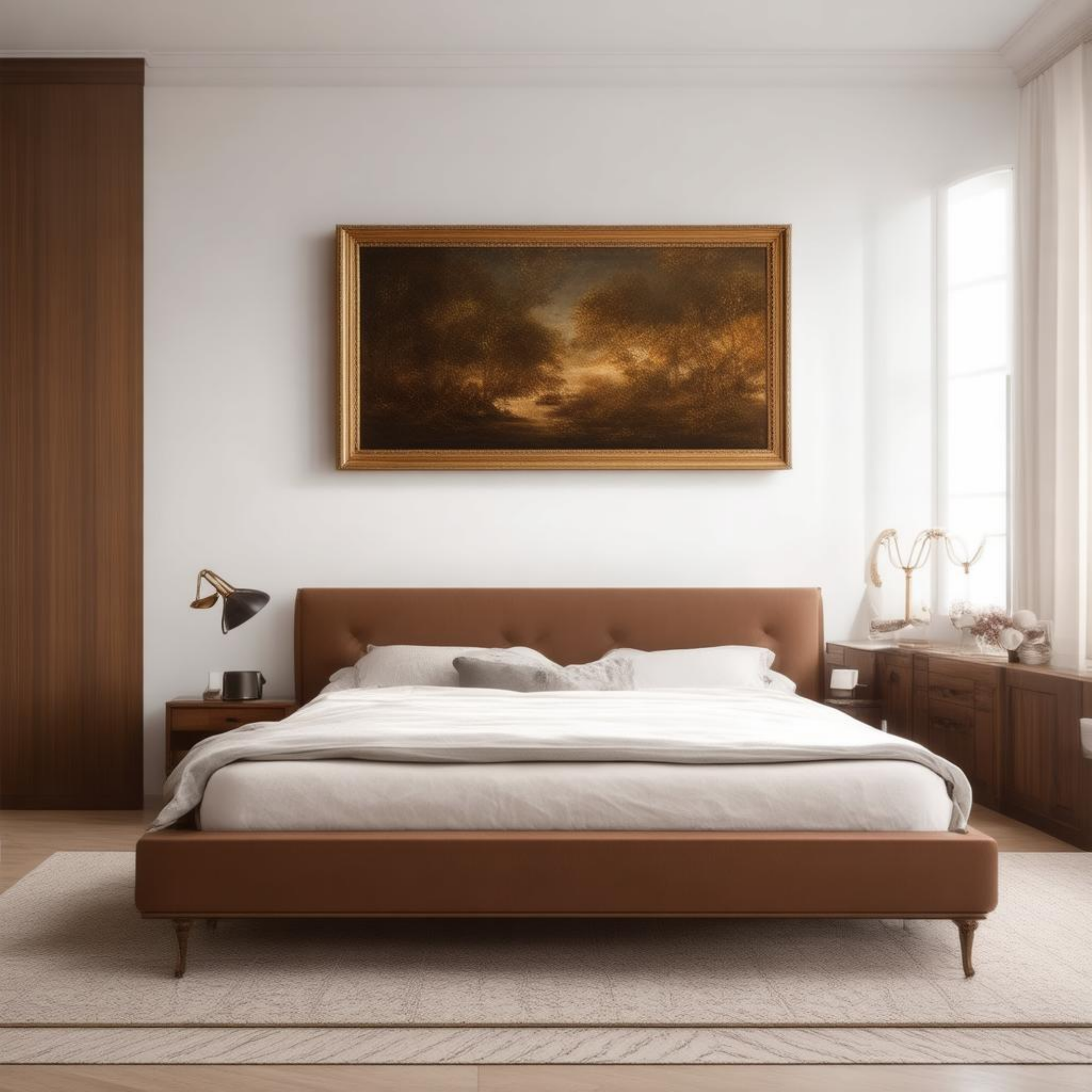}
& \includegraphics[width=0.155\linewidth]{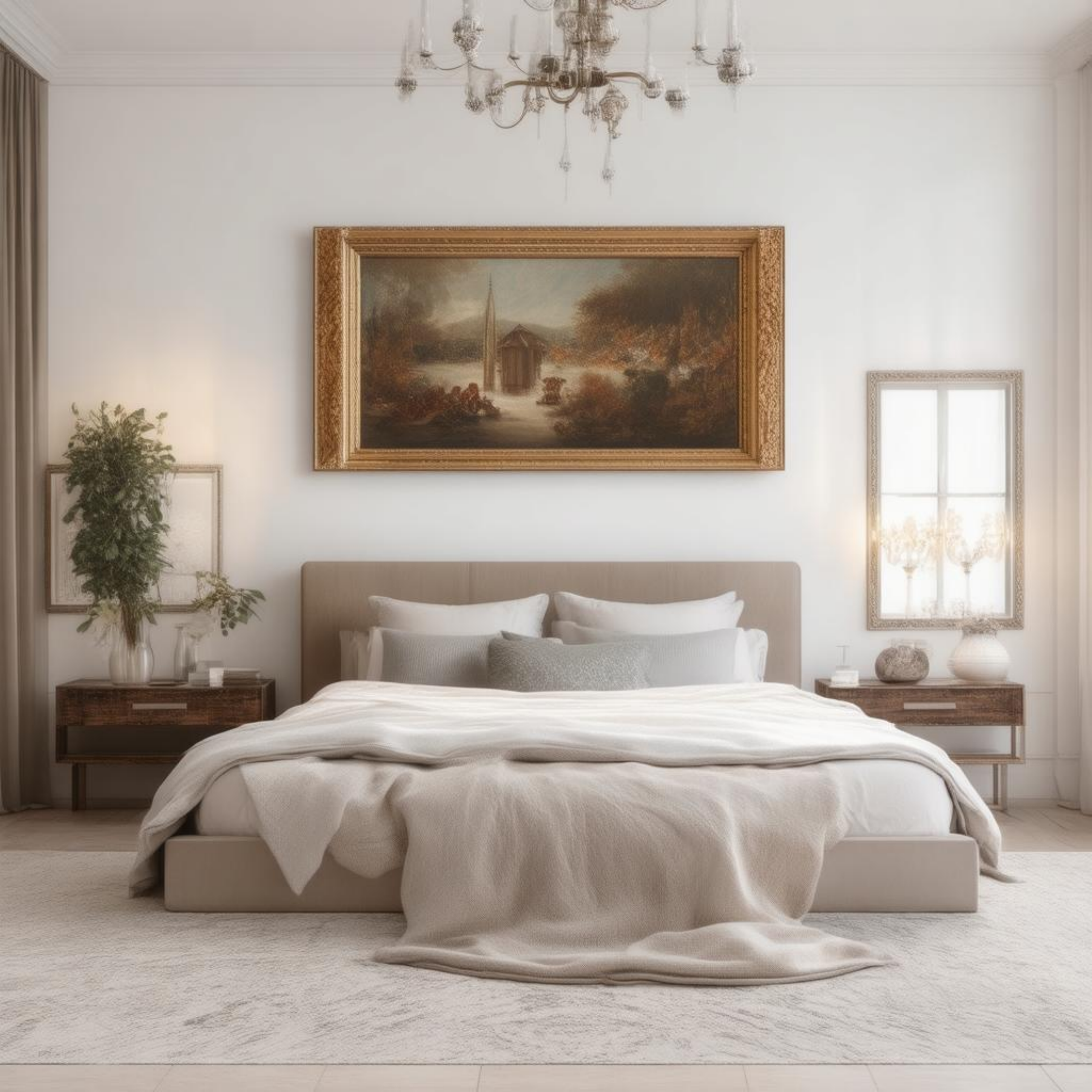}
& \includegraphics[width=0.155\linewidth]{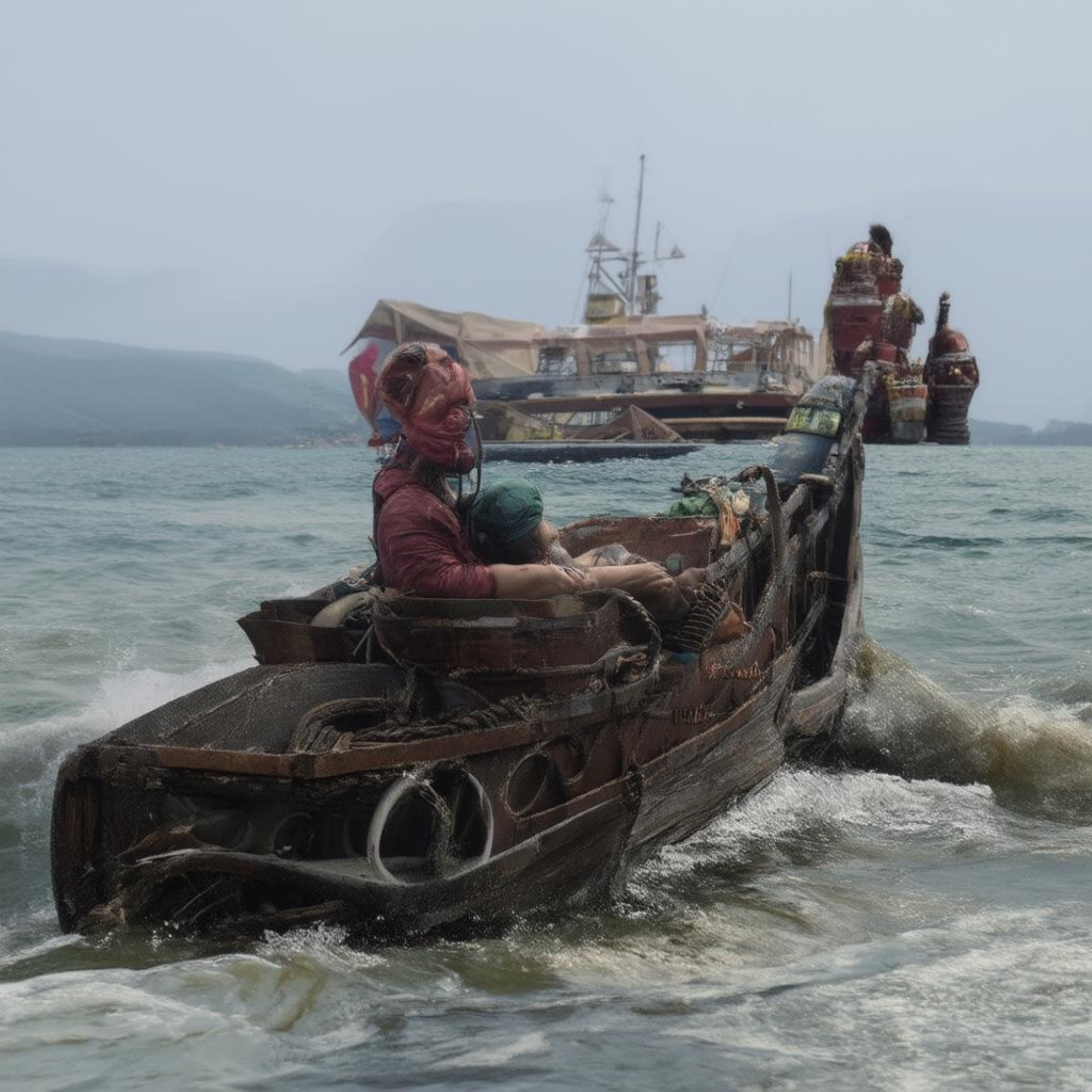}
& \includegraphics[width=0.155\linewidth]{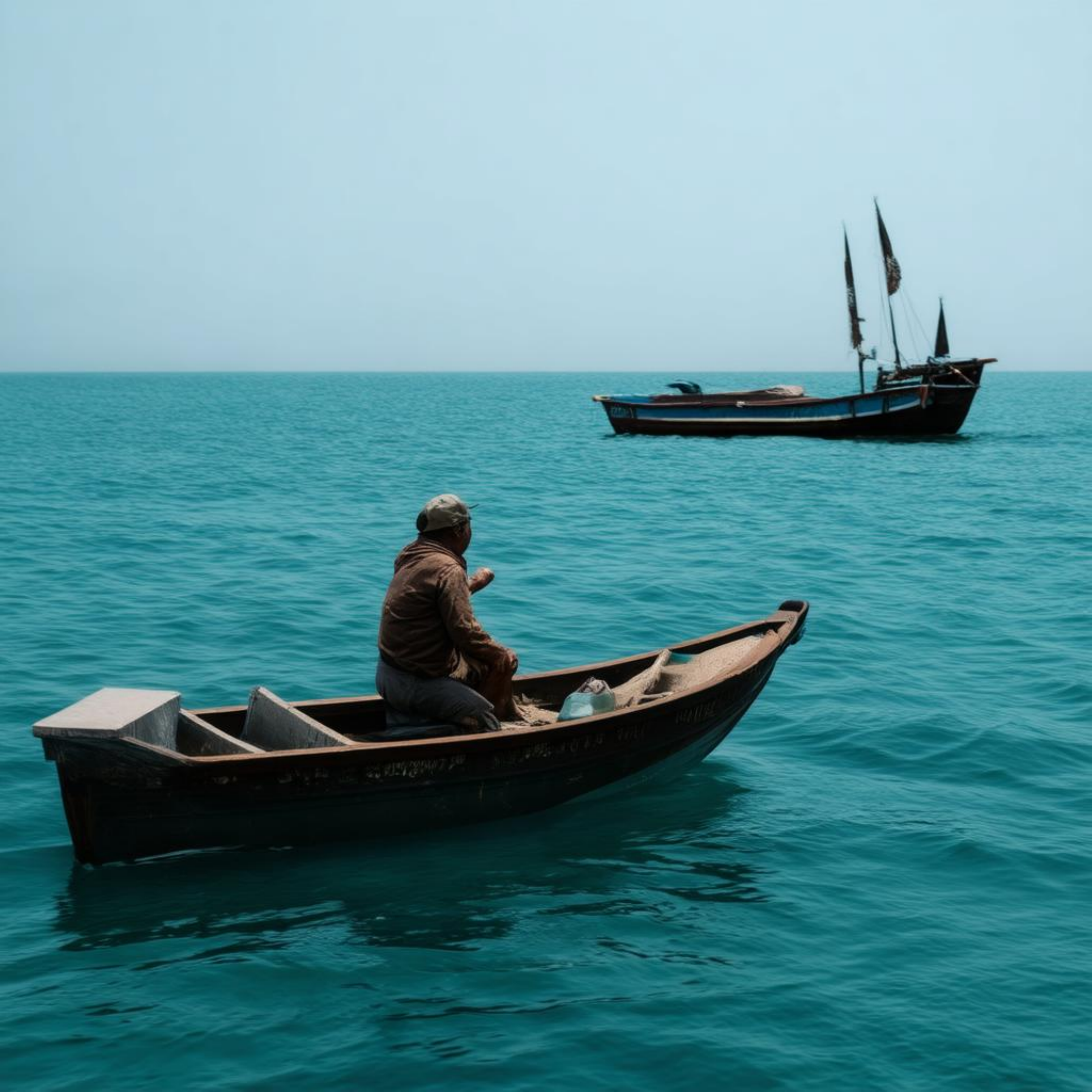}
& \includegraphics[width=0.155\linewidth]{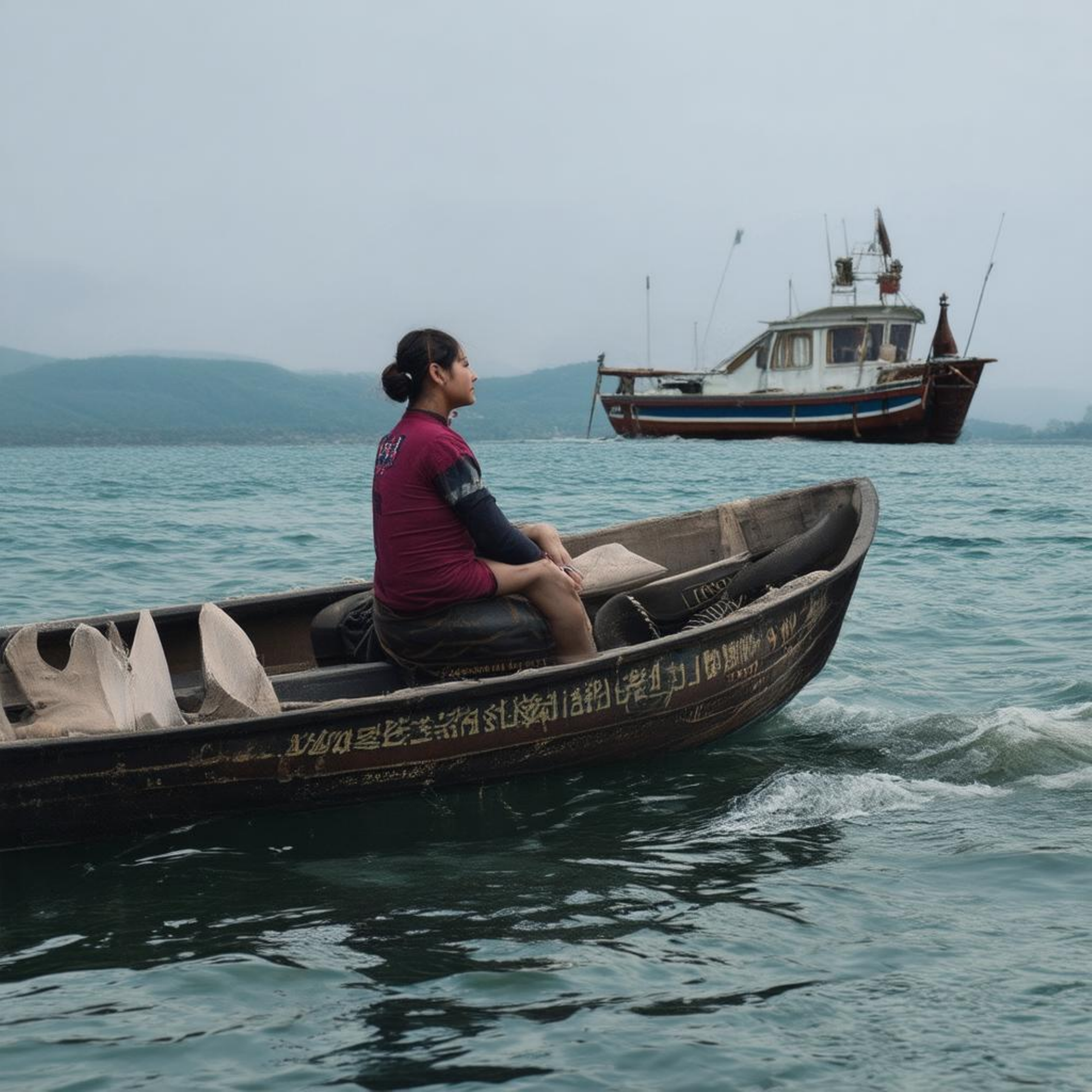} \\[-2pt]
\includegraphics[width=0.155\linewidth]{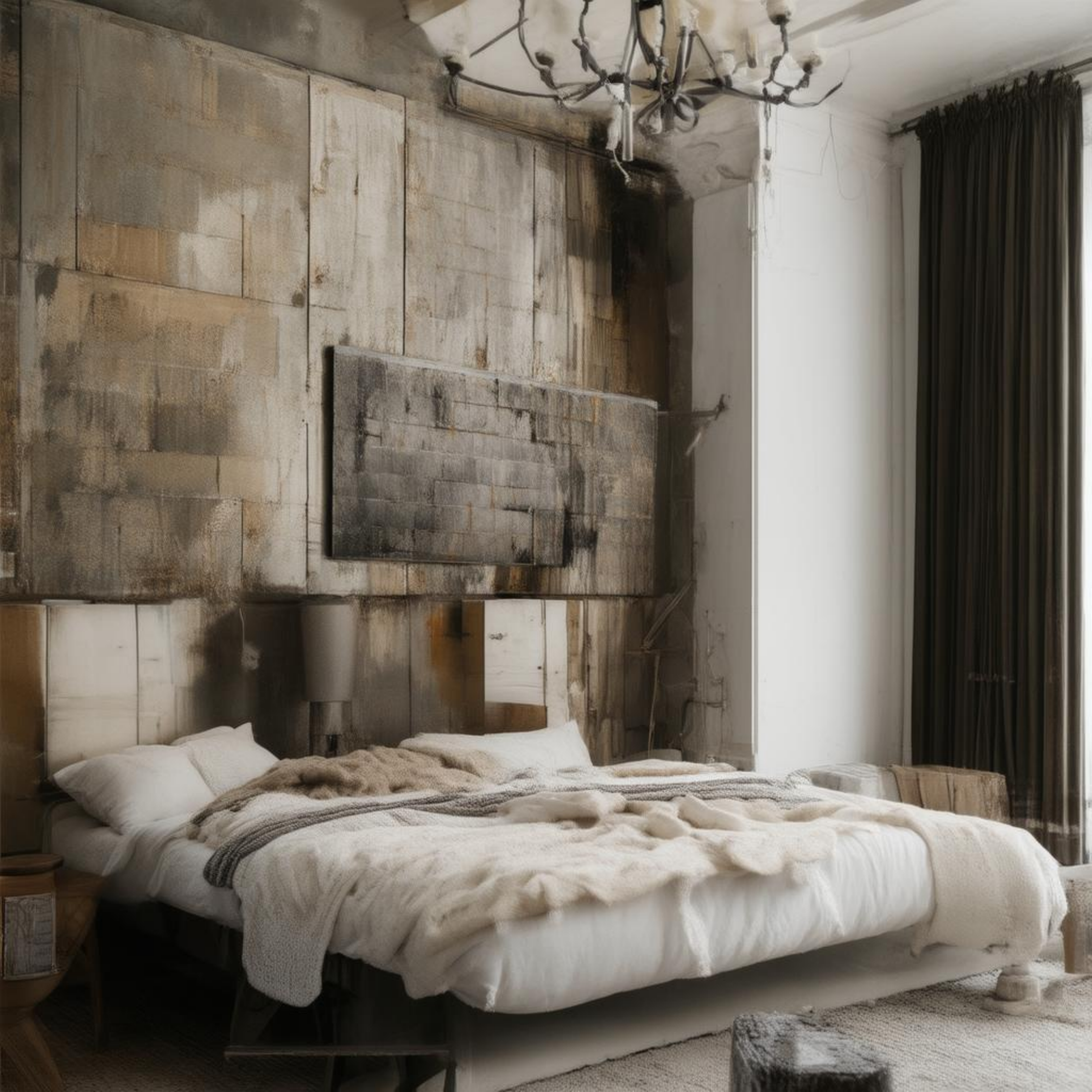}
& \includegraphics[width=0.155\linewidth]{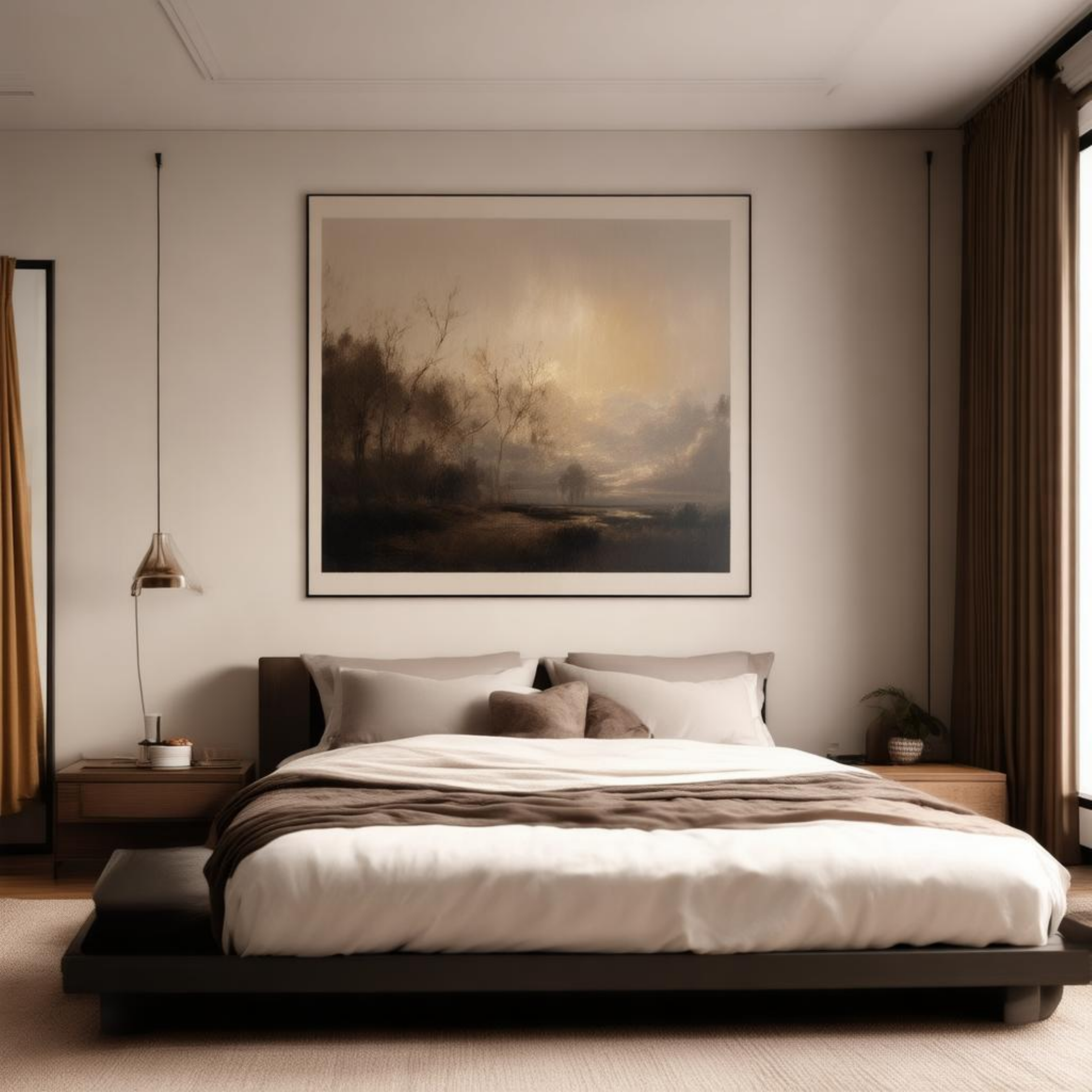}
& \includegraphics[width=0.155\linewidth]{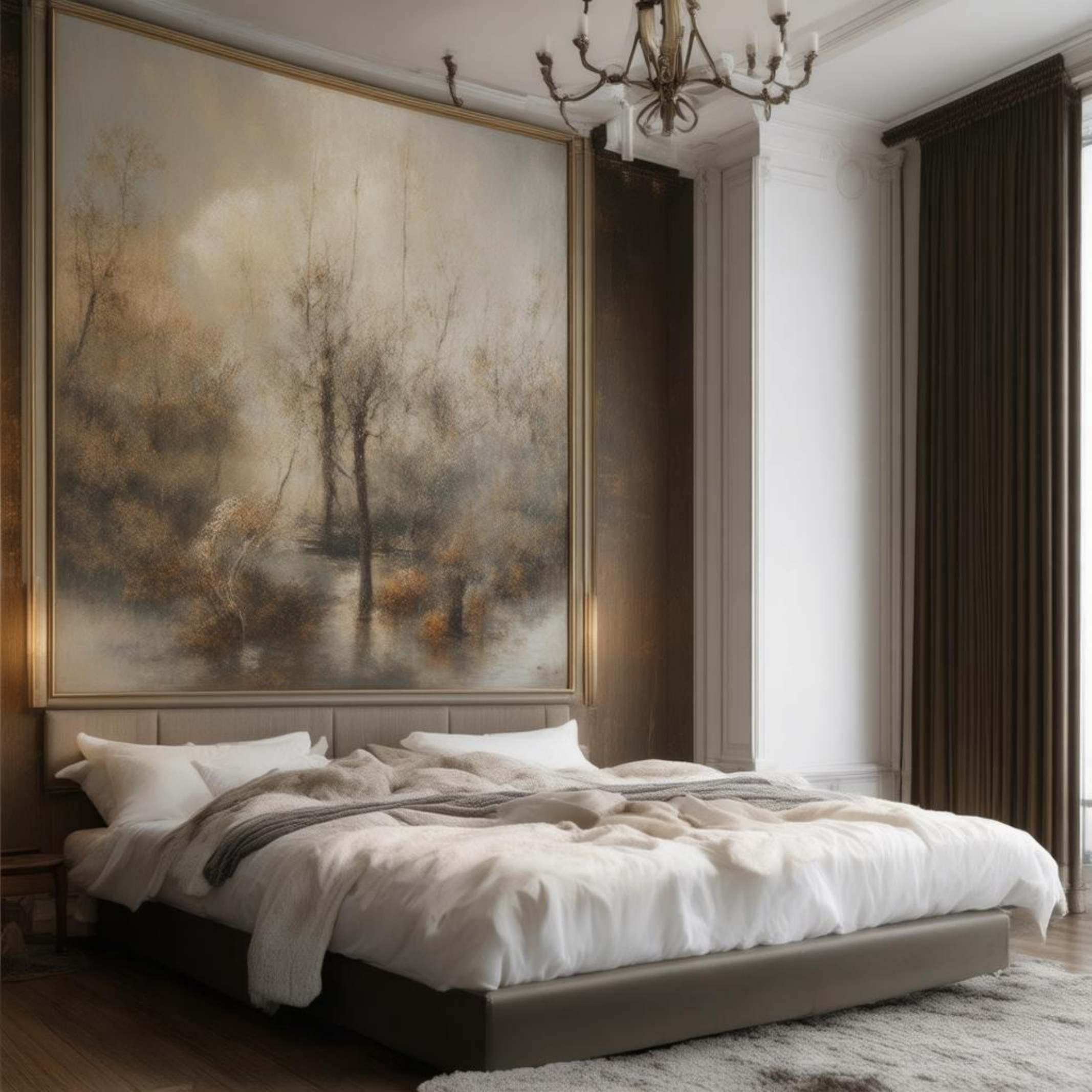}
& \includegraphics[width=0.155\linewidth]{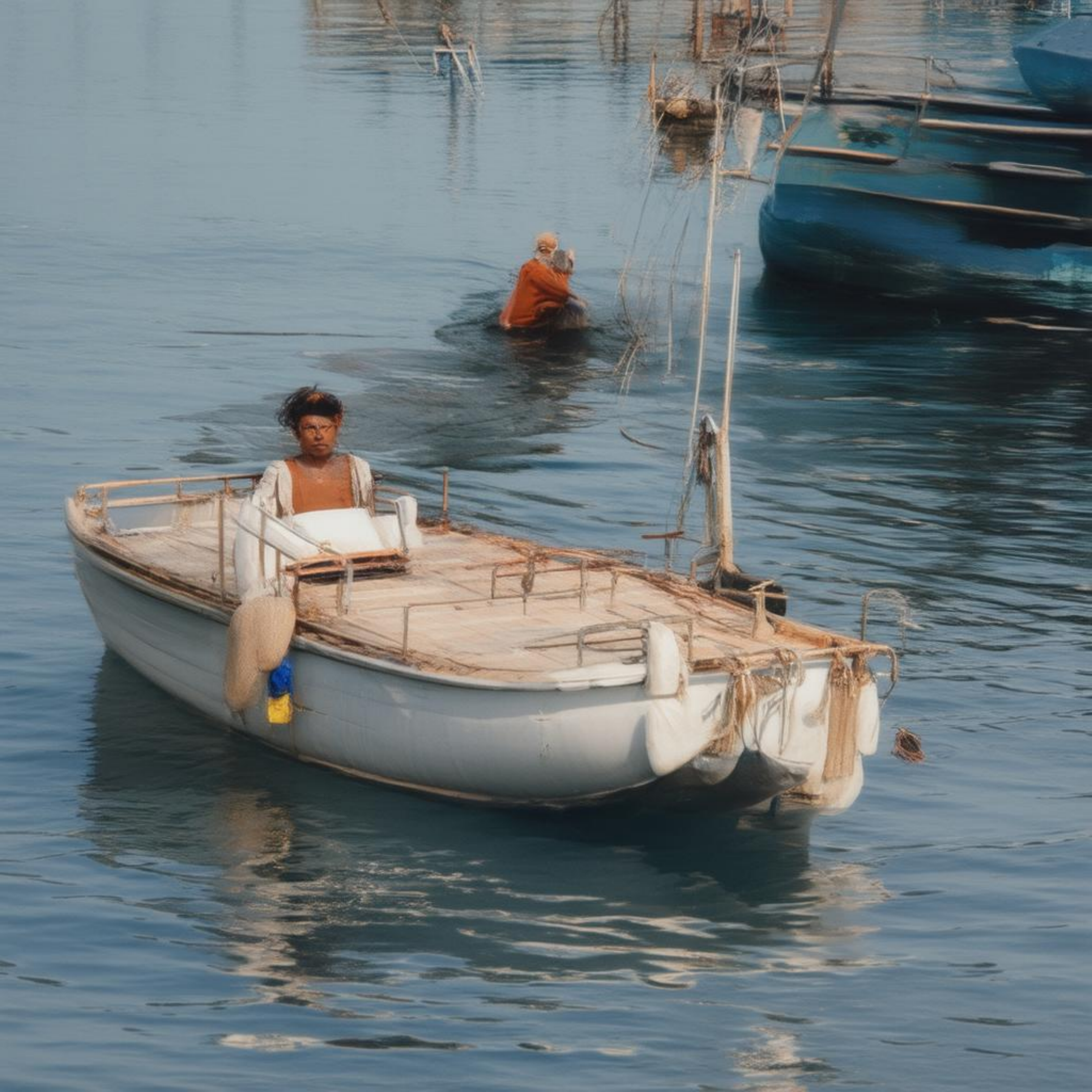}
& \includegraphics[width=0.155\linewidth]{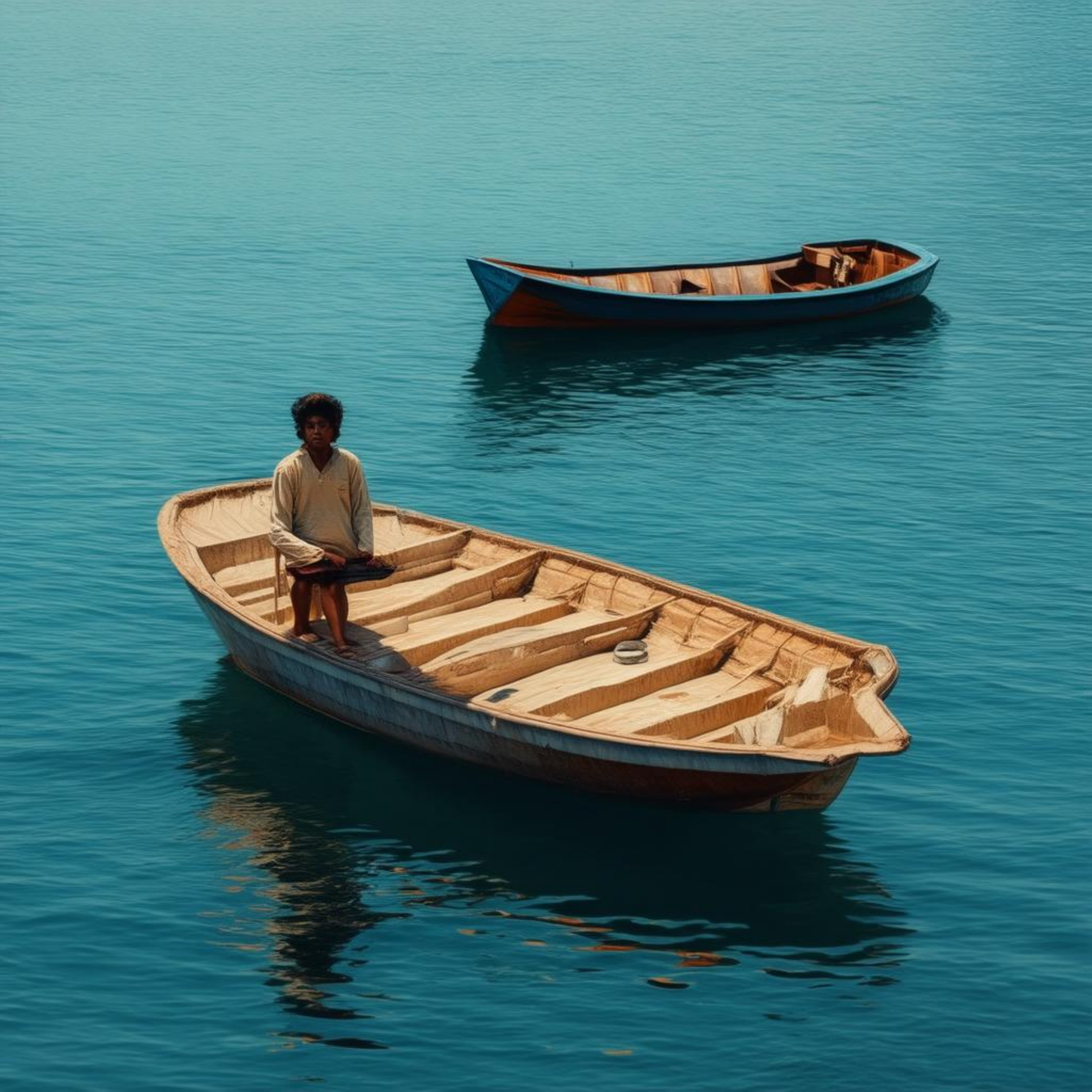}
& \includegraphics[width=0.155\linewidth]{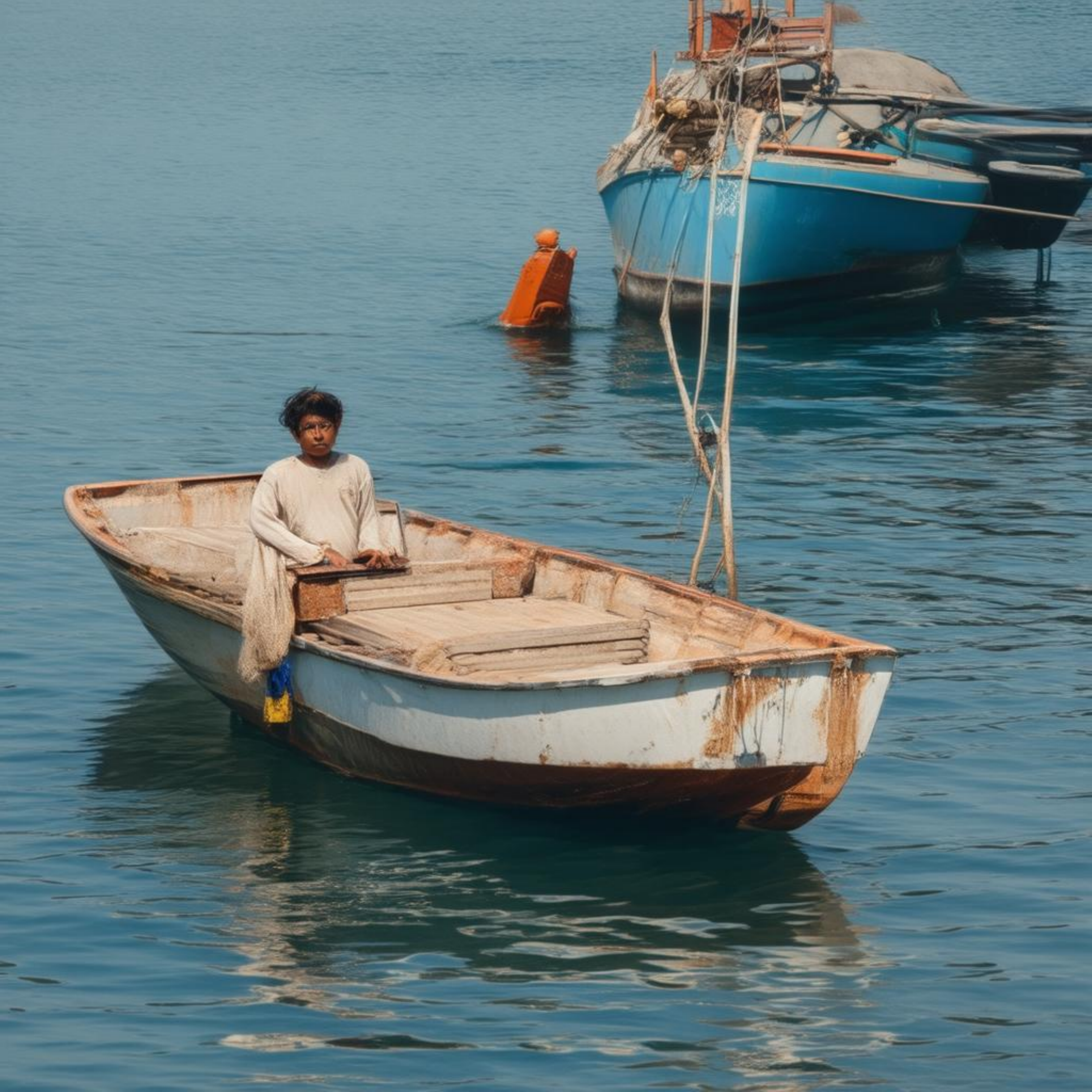} \\[-2pt]
\includegraphics[width=0.155\linewidth]{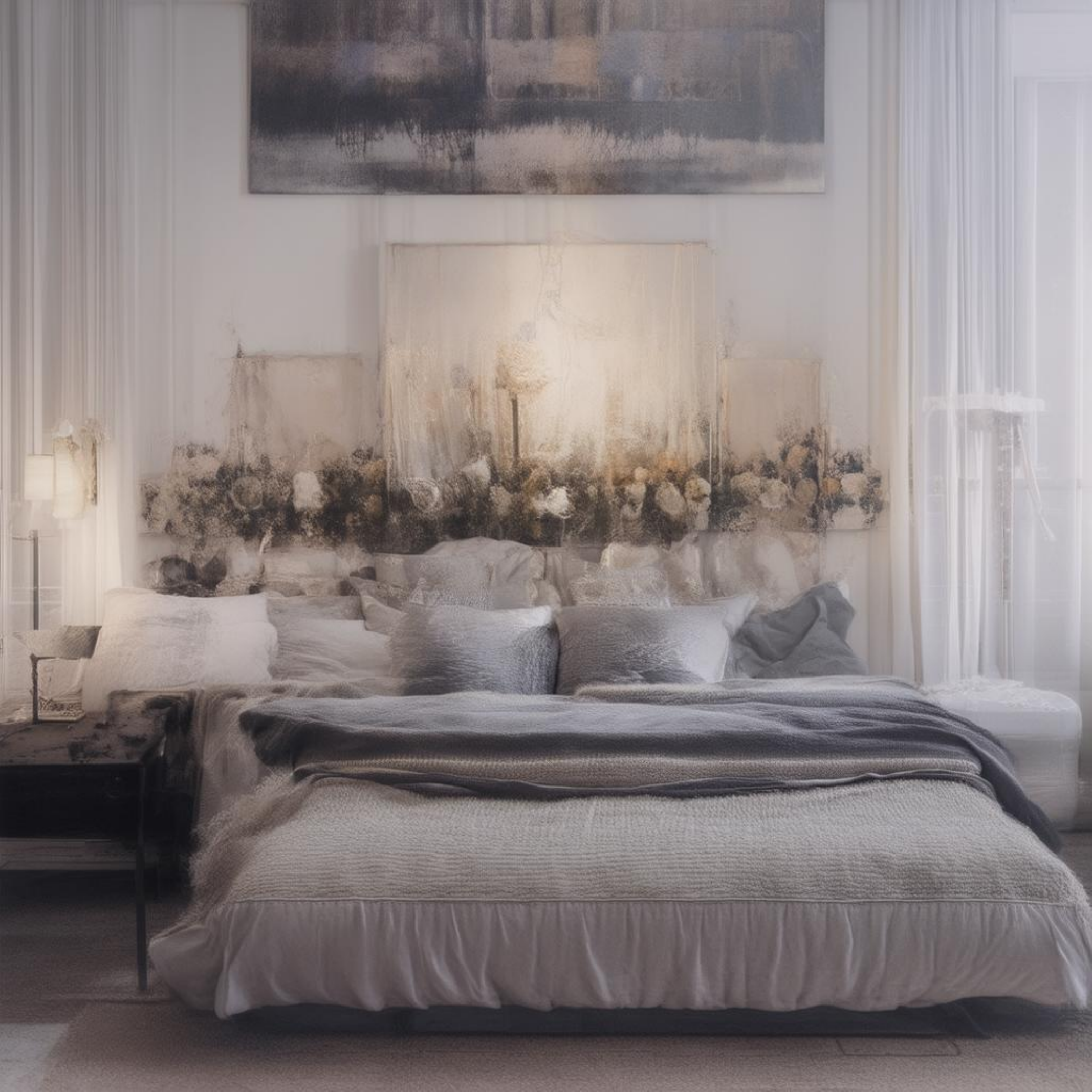}
& \includegraphics[width=0.155\linewidth]{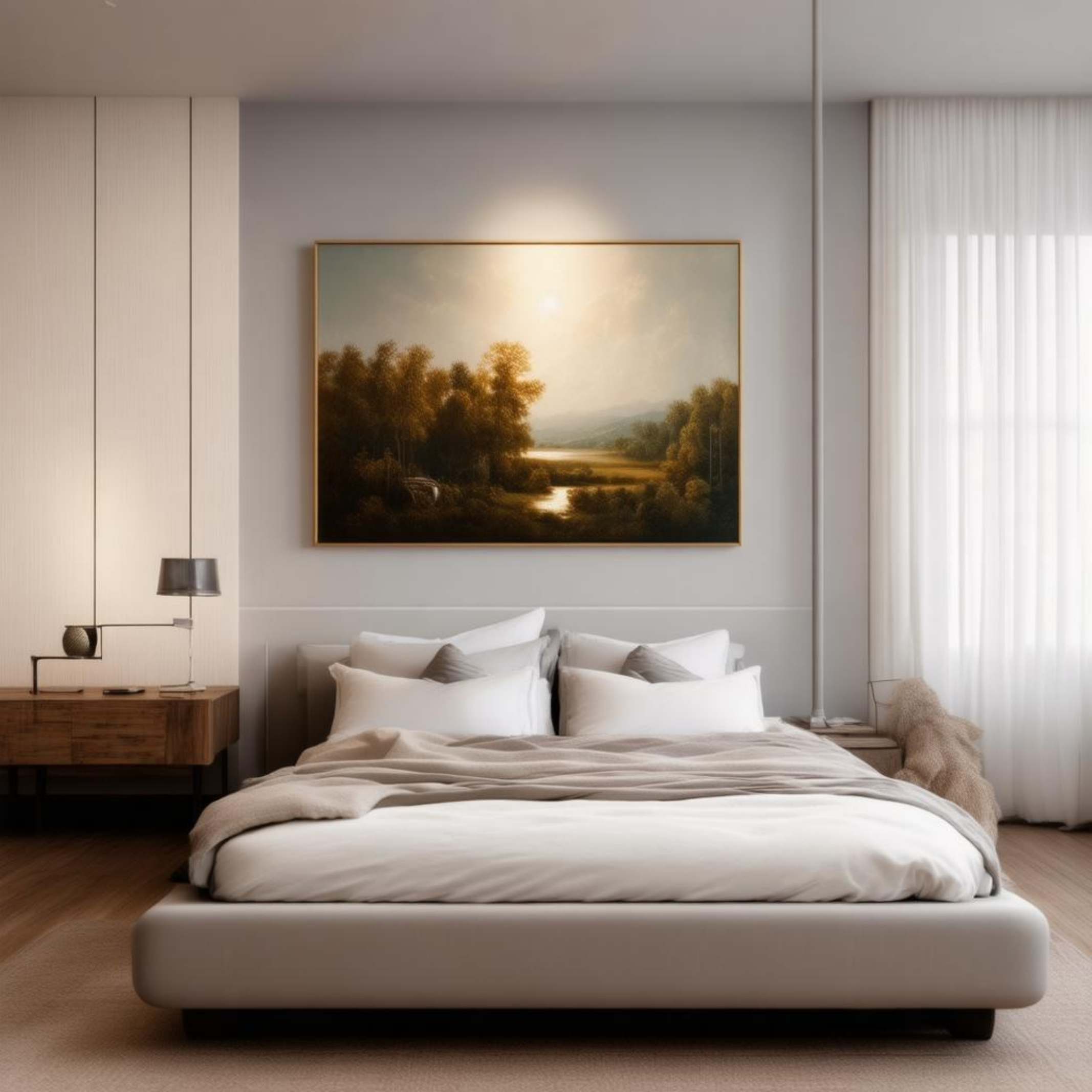}
& \includegraphics[width=0.155\linewidth]{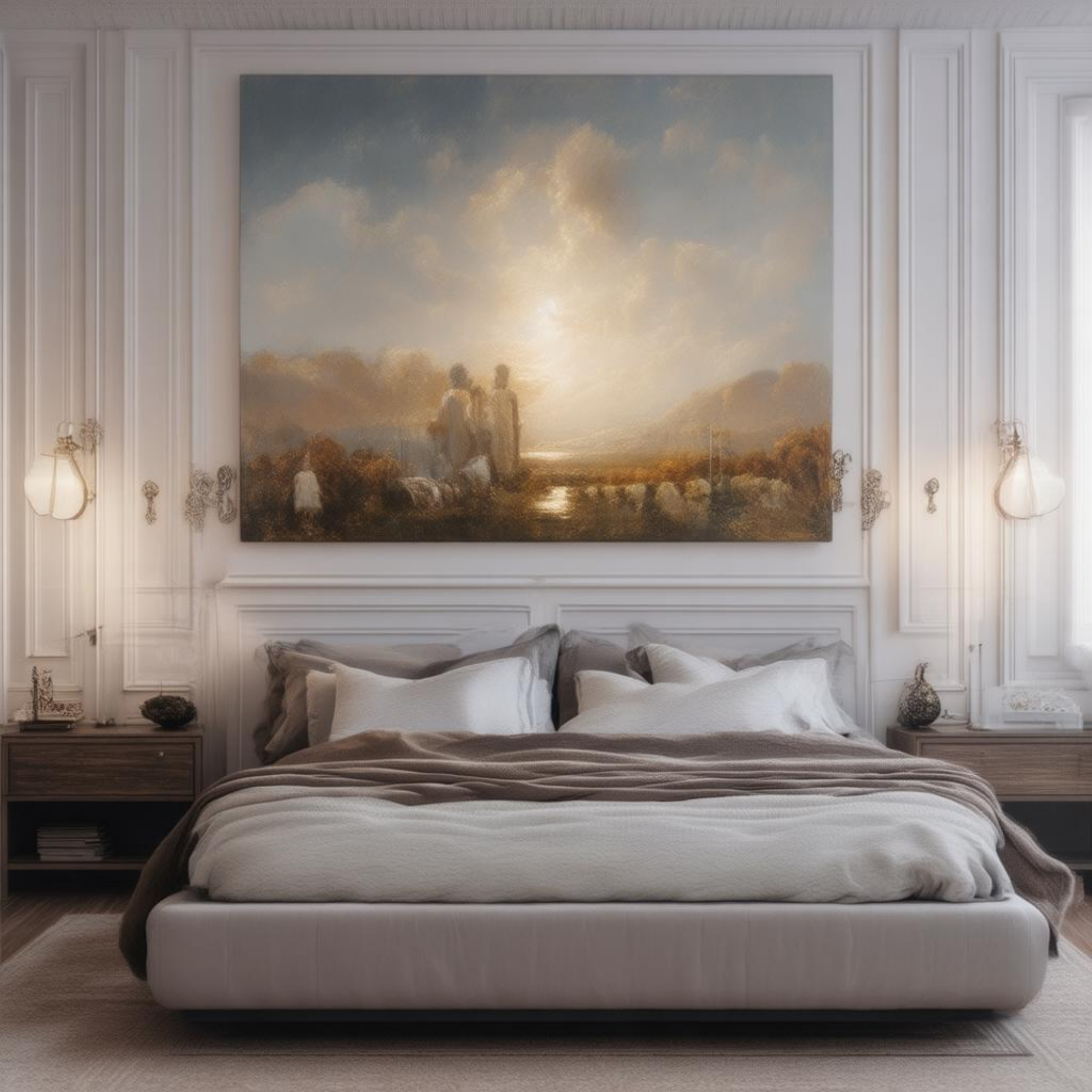}
& \includegraphics[width=0.155\linewidth]{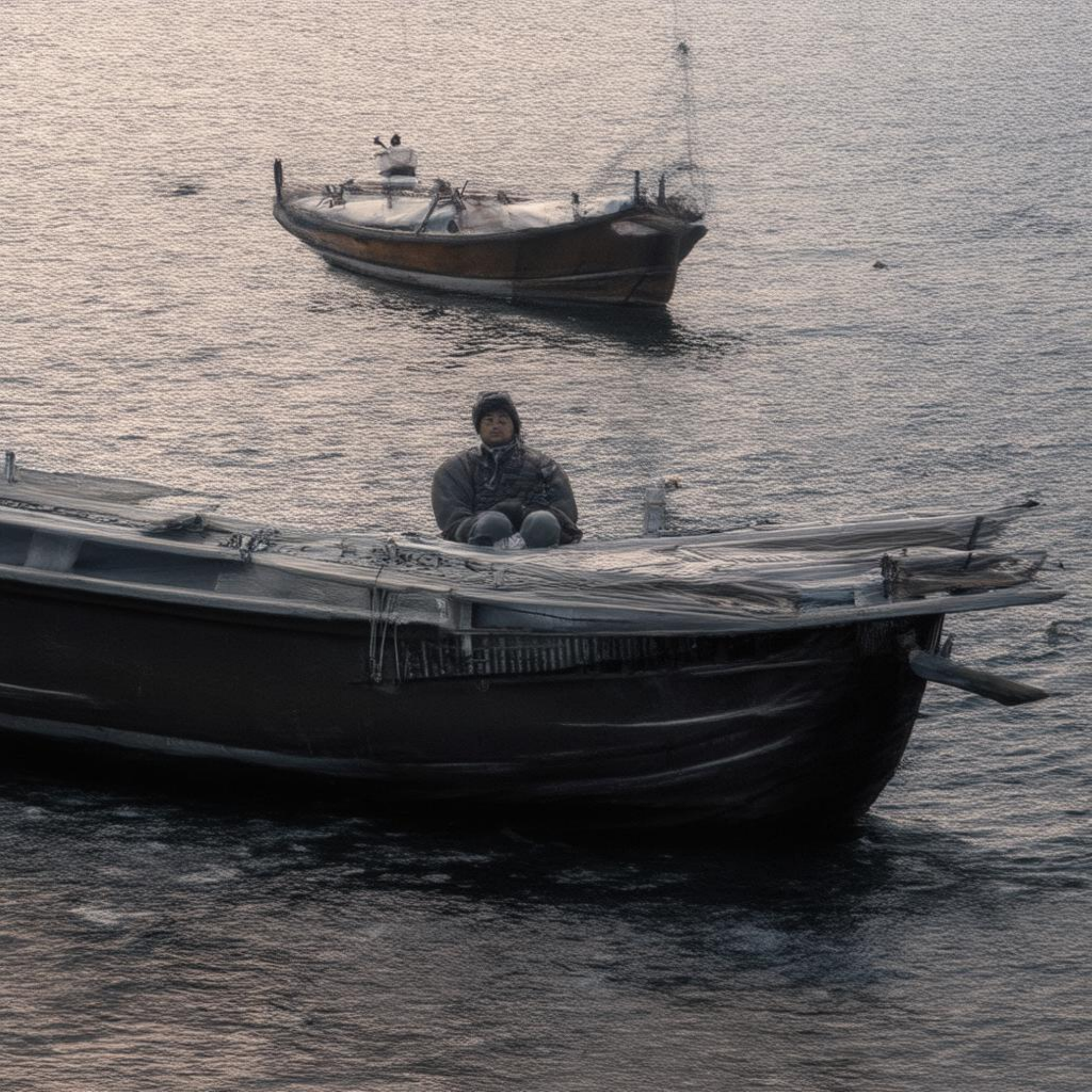}
& \includegraphics[width=0.155\linewidth]{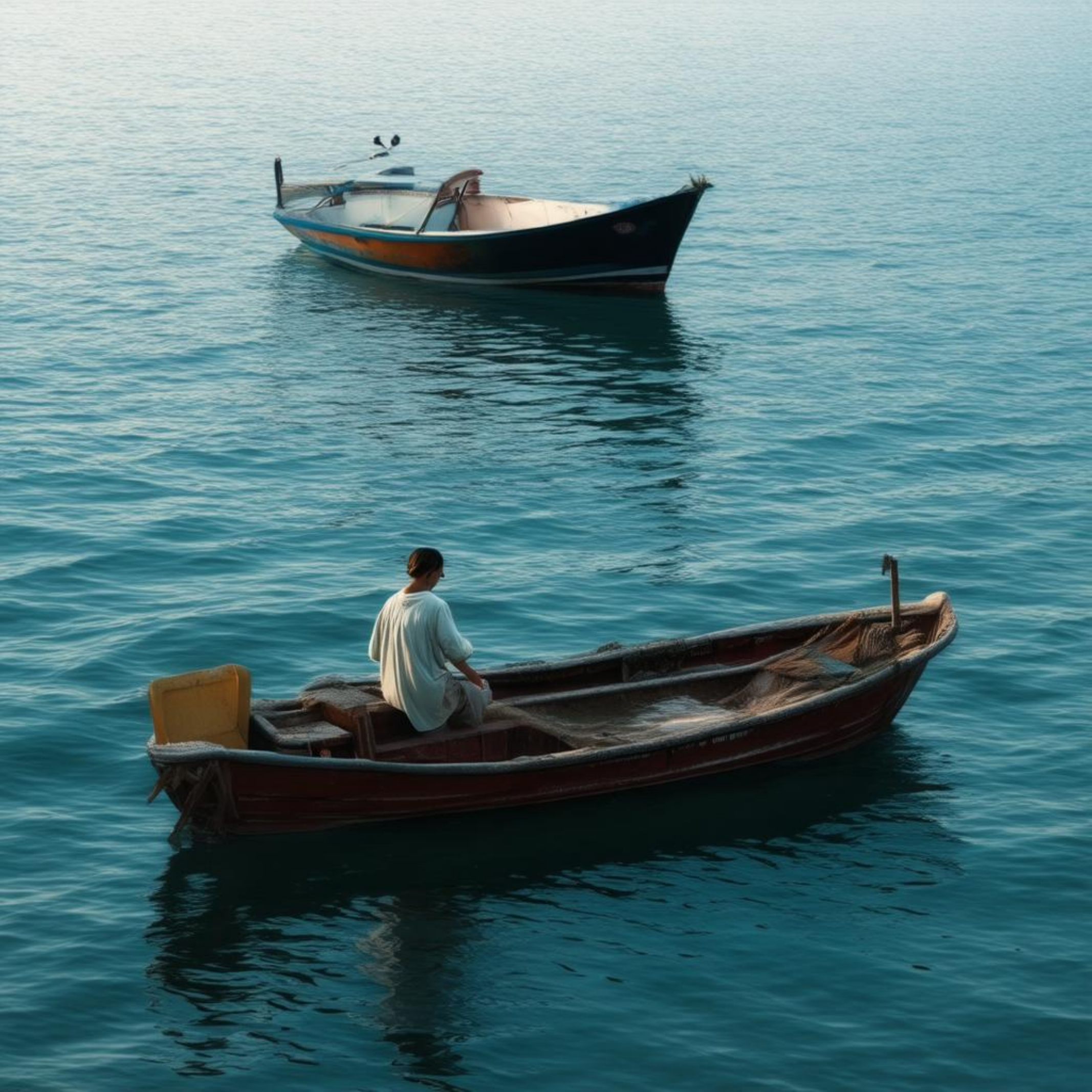}
& \includegraphics[width=0.155\linewidth]{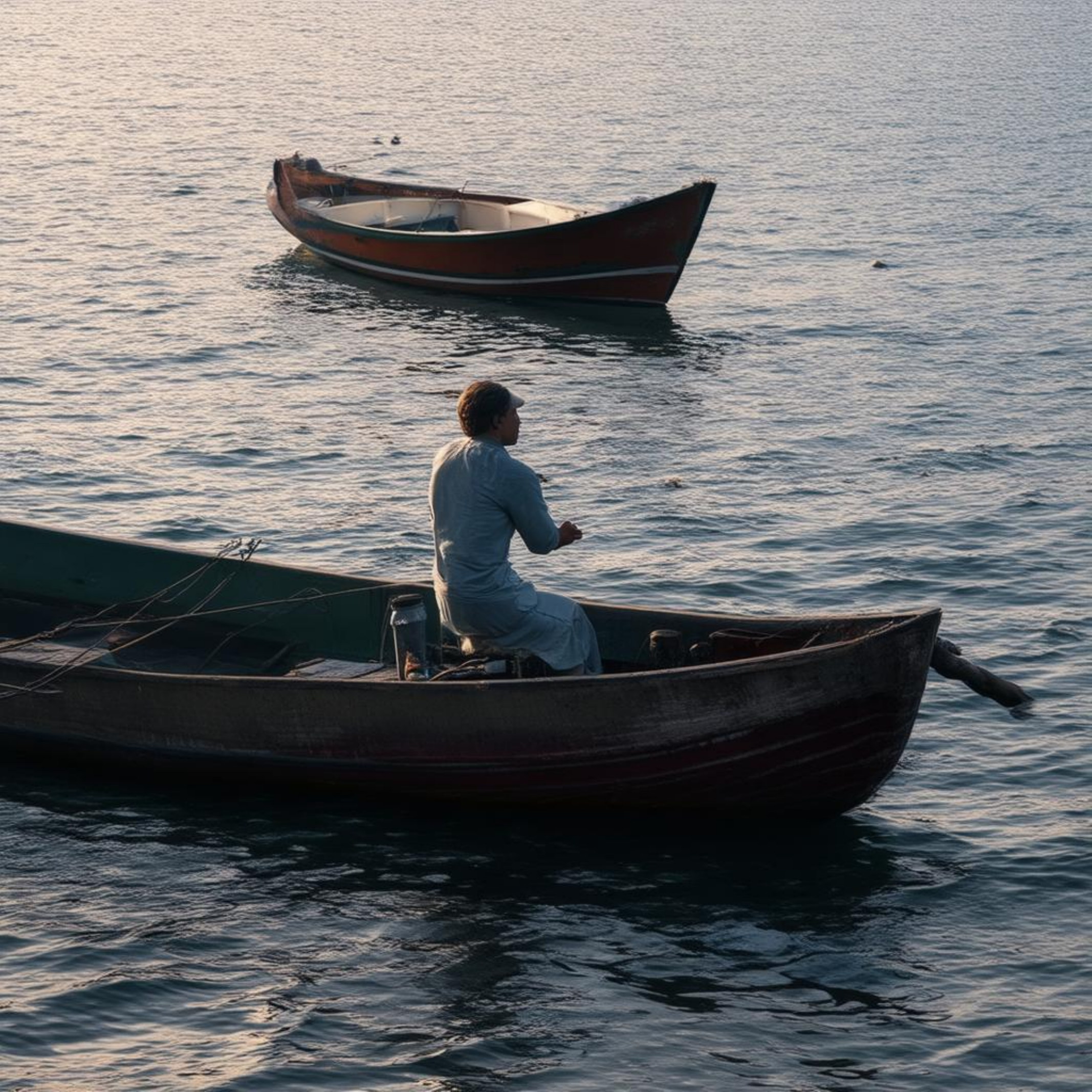}
\end{tabular}
}
\caption{Additional matched qualitative comparisons on SD3.5 Medium across
dog, landscape, bedroom, and people prompts. Each four-row category compares
no guidance, fixed CFG, and PMC-CFG using matched prompts and initial latent
seeds. Fixed CFG and PMC-CFG both use nominal guidance scale $\lambda=5$.}
\label{fig:sd35-categories}
\end{figure}

\FloatBarrier

\end{document}